\documentclass{article}

\usepackage[final]{ewrl_2026}

\PassOptionsToPackage{sort&compress,numbers}{natbib}

\usepackage[utf8]{inputenc} % allow utf-8 input
\usepackage[T1]{fontenc}    % use 8-bit T1 fonts
\usepackage{hyperref}       % hyperlinks
\usepackage{url}            % simple URL typesetting
\usepackage{booktabs}       % professional-quality tables
\usepackage{amsfonts}       % blackboard math symbols
\usepackage{nicefrac}       % compact symbols for 1/2, etc.
\usepackage{microtype}      % microtypography
\usepackage{xcolor}         % colors
\usepackage{titletoc}
\usepackage{amsfonts}
\usepackage{amsmath}
\usepackage{amssymb}
\usepackage{amsthm}
\usepackage{times}

\newtheorem{proposition}{Proposition}

\newtheorem{lemma}{Lemma}
\newtheorem{theorem}{Theorem}
\newtheorem{assumption}{Assumption}

\theoremstyle{definition}

\theoremstyle{plain}
\theoremstyle{remark}
\newtheorem{remark}{Remark}

\definecolor{linkblue}{RGB}{0, 70, 160}

\usepackage{balance}        % for balancing columns on the final page
\usepackage[utf8]{inputenc} % allow utf-8 input
\usepackage[T1]{fontenc}    % use 8-bit T1 fonts
\usepackage{hyperref}       % hyperlinks
\usepackage{url}            % simple URL typesetting
\usepackage{booktabs}       % professional-quality tables
\usepackage{amsfonts}       % blackboard math symbols
\usepackage{nicefrac}       % compact symbols for 1/2, etc.
\usepackage{tikz}
\usetikzlibrary{positioning, arrows.meta, calc, backgrounds, fit, spy}
\usepackage{bbm}
\usepackage{amssymb}
\usepackage{mathtools}
\usepackage{amsthm}
\usepackage{changepage}
\usepackage[most]{tcolorbox}
\usepackage{caption}

\usepackage[bottom]{footmisc}

\usepackage{graphics, graphicx}
\usepackage{xparse, xspace}
\usepackage[english]{babel}
\usepackage{upgreek}

\usepackage{tabularx}
\usepackage{enumerate, enumitem}
\usepackage{wrapfig}
\usepackage[subrefformat=parens, labelformat=parens]{subcaption}

\usepackage{setspace}
\usepackage[linesnumbered,ruled,noend]{algorithm2e}
\SetKwBlock{Cycle}{}{}
\SetKwInOut{Input}{input} 
\SetKw{KwOr}{or}
\SetAlCapHSkip{0em}

\usepackage{xfrac}
\usepackage{diagbox}
\usepackage{multirow}
\usepackage{colortbl}
\usepackage{makecell}
\usepackage{xcolor}
\usepackage{ifthen}

\definecolor{snsblue}{rgb}{0.12156862745098039, 0.4666666666666667, 0.7058823529411765}
\definecolor{snsorange}{rgb}{1.0, 0.4980392156862745, 0.054901960784313725}
\definecolor{snsgreen}{rgb}{0.17254901960784313, 0.6274509803921569, 0.17254901960784313}
\definecolor{snsred}{rgb}{0.8392156862745098, 0.15294117647058825, 0.1568627450980392}
\definecolor{snspurple}{rgb}{0.5803921568627451, 0.403921568627451, 0.7411764705882353}
\definecolor{snsbrown}{rgb}{0.5490196078431373, 0.33725490196078434, 0.29411764705882354}
\definecolor{lightblack}{rgb}{0.4, 0.4, 0.4}
\definecolor{snspink}{rgb}{0.8, 0.47058823529411764, 0.7372549019607844}

\usepackage{xifthen}
\usepackage{amsmath}
\usepackage{amsfonts}

\makeatletter
\protected\def\xvcenter{%
  \hbox\bgroup$\everyvbox{\everyvbox{}\aftergroup\m@th\aftergroup$\aftergroup\egroup}%
  \vcenter
}
\DeclareRobustCommand{\midscript}[1]{
  \mathchoice{\mid@script\scriptstyle{#1}}
    {\mid@script\scriptstyle{#1}}
    {\mid@script\scriptscriptstyle{#1}}
    {\mid@script\scriptscriptstyle{#1}}
}
\newcommand{\mid@script}[2]{
  \vcenter{\hbox{$\m@th#1#2$}}
}

\makeatletter

\usepackage{color}
\usepackage{pifont}
\makeatletter
\newcommand{\colorlabel}[1]{%
	\global\tag@true%
	\nonumber%
	\refstepcounter{equation}%
	\gdef\df@tag{\maketag@@@{{\color{#1}(\theequation)}}\def\@currentlabel{\theequation}}}
\makeatother

\newcommand{\mytilde}[0]{\mathds{\raise.17ex\hbox{$\scriptstyle\sim$}}} % tilde used for negation
\newcommand{\spacedmid}{\mspace{2mu} | \mspace{2mu}} % | with a bit of spacing
\newcommand{\evmid}{\mspace{3mu}\ifnum\currentgrouptype=16 \middle\fi|\mspace{3mu}} % adaptive-size | with some spacing, use in E[... | ...]

\DeclareMathOperator{\EV}{\mathbb{E}}

\DeclareMathOperator*{\argmax}{arg\,max}
\DeclareMathOperator*{\argmin}{arg\,min}

\providecommand{\probmodel}{\mathcal{P}}

\providecommand{\statespace}{\mathcal S}
\providecommand{\actionspace}{\mathcal A}

\newcommand{\goalspace}{\mathcal{G}}

\newcommand{\ind}{\mathbf{1}}
\newcommand{\Vtheta}{V^{\raisebox{-1pt}{$\scriptstyle\theta$}}}
\newcommand{\Qtheta}{Q^{\raisebox{-1pt}{$\scriptstyle\theta$}}}

\usepackage{relsize}
\newcommand{\stt}[1]{{\relscale{0.92}\texttt{#1}}}
\newcommand{\sttt}[1]{{\relscale{0.84}\texttt{#1}}}

\title{SUN: Reaching for Novelty in Reinforcement Learning}

\author{%
  Wenyan Yang\\
  Aalto University\\
  \sttt{wenyan.yang@aalto.fi}
  \And
  Arsenii Mustafin\\
  Aalto University\\
  \sttt{arsenii.mustafin@aalto.fi}
  \And
  Dominik Baumann\\
  Aalto University\\
  \sttt{dominik.baumann@aalto.fi}
  \And
  Joni Pajarinen\\
  Aalto University\\
  \sttt{joni.pajarinen@aalto.fi}
  \And
  Simone Parisi\\
  Tampere University\\
  \sttt{simone.parisi@tuni.fi}
}

\begin{document}
\maketitle

\begin{abstract}
Exploration in reinforcement learning (RL) remains a fundamental challenge.
% : classical strategies are sample-inefficient or rely on non-stationary intrinsic rewards that may destabilize learning.
Recent goal-conditioned RL strategies (which select goals to encourage broader state coverage) have shown promising results, but none scores a goal by novelty and reachability jointly: the two signals are traded off by hand, applied in sequence, or one is neglected outright.
In this paper, we introduce a reachability-aware goal-selection framework that explicitly integrates these two aspects, and that can be seamlessly incorporated into any off-policy RL algorithm.
To this aim, we propose \textit{SUccessor-to-Novelty (SUN)}, an indicator derived from successor value functions to identify goals that are both novel and reachable. 
% We discuss its properties and establish three theoretical results: an equivalence to count-based exploration bonuses, a short-horizon hitting-probability bound, and a guarantee that unreachable goals are rejected.
We prove that SUN is the value of a count-scaled goal-reaching reward, bounds short-horizon hitting probabilities, and rejects unreachable goals.
We further present an adaptive goal-selection strategy that leverages these properties, along with an accurate yet lightweight pseudocount to avoid the overhead of classic methods.
% Finally, we introduce new benchmarking environments (with unreachable states and irreversible transitions) that directly test reachability-aware exploration. SUN consistently outperforms state-of-the-art methods on these and standard benchmarks.
We back up all our claims with thorough benchmarks: SUN consistently outperforms state-of-the-art methods in standard and novel environments with unreachable or hard-to-reach states, irreversible transitions, obstacles, mazes, and unbounded spaces. Code at \url{https://github.com/sparisi/rl-anvil}

\end{abstract}

\section{Introduction}
\label{sec:intro}

Exploration is fundamental to reinforcement learning (RL): without effective exploration, agents collect uninformative data and fail to learn.
%, especially in sparse-reward or reward-free settings.
Classical dithering schemes, such as $\varepsilon$-greedy and entropy regularization, ignore environment structure and are sample-inefficient. Provably efficient algorithms \citep{auer2002finite,strehl2008analysis,jaksch2010near} offer strong guarantees but do not scale to large state spaces. Intrinsic motivation methods \citep{pathak2017curiosity,burda2019exploration,parisi2021interesting} require careful tuning and are non-stationary by construction: as the agent explores, the intrinsic reward shifts beneath the policy trained on it, destabilizing learning \citep{burda2019exploration}. 
\\[2pt]
A more recent family casts exploration as \emph{goal-conditioned RL} (GCRL) \citep{liu2022goal,colas2022autotelic}, where the agent follows a goal-conditioned policy trained on a stationary goal-reaching objective. Different \emph{goal-selection} mechanisms lead to different exploration strategies, but most of the existing work captures only half the picture. Density-based methods such as MEGA \citep{pitis2020maximumentropy}, Skew-Fit \citep{pong2020skewfit}, GoalGAN \citep{florensa2018automatic}, and Hindsight Goal Generation \citep{ren2019exploration} score goals by novelty, committing to rare goals that may be unreachable. Conversely, methods based on distances or success probabilities \citep{schaul2015universal,hartikainen2020dynamical} optimize reachability alone, neglecting rare but achievable goals. Neither extreme captures the right intuition: \emph{a useful exploration goal is one that is novel \textbf{and} reachable}. Figure \ref{fig:two_room_intro} summarizes this problem.
% \\[2pt]
In this paper, we address this gap with the following contributions. 
% \\[1pt]
\clearpage

\textbf{(1)} We present the \textit{{SU}ccessor-to-{N}ovelty {(SUN)}} indicator to balance \textit{novelty} and \emph{reachability}. SUN is compatible with any off-policy RL algorithm; in this paper, we instantiate it with DQN \citep{mnih2015human} and TD3 \citep{fujimoto2018addressing}.
%, and SAC \citep{haarnoja2018soft}.
\\[1pt]
\textbf{(2)} We present an accurate yet lightweight pseudocount that avoids the overhead of density-based methods, enabling efficient exploration with $\mathcal{O}(1)$ query cost.
\\[1pt]
\textbf{(3)} We show that SUN is the value of a goal-reaching reward scaled by the goal's inverse count, gives a closed-form lower bound on the short-horizon hitting probability, and suppresses unreachable goals.
% --- properties that no pure-novelty or pure-reachability score satisfies simultaneously.
\\[1pt]
\textbf{(4)} We introduce new benchmarking environments with unreachable states and irreversible transitions that directly stress-test reachability-aware exploration, and show that SUN consistently outperforms state-of-the-art methods on these and standard benchmarks.
% \vspace*{-4pt}
% \begin{center}
% \textit{Together, these contributions establish SUN as a principled and practical solution to the long-standing tension between novelty and reachability in exploration.}
% \end{center}
% \vspace*{-4pt}
\begin{tcolorbox}[colback=gray!10, colframe=gray!40, boxrule=0.5pt, arc=2pt, left=6pt, right=6pt, top=3pt, bottom=3pt]
Together, these contributions establish SUN as a principled and practical solution to the long-standing tension between novelty and reachability in exploration.
\end{tcolorbox}

SUN exploration fits in the field of reward-free exploration and goal-conditioned RL, and is especially close to the work of \citet{tarbouriech2022adaptive} (AdaGoal) and \citet{diaz-bone2025discover} (DISCOVER) in its use of successor value functions to drive exploration. Both also balance novelty and reachability, but estimate novelty through critic-ensemble disagreement, which is computationally expensive and, as we show, leaves both methods poorly calibrated between the two signals. SUN sidesteps these issues with a lightweight pseudocount and a novel goal-selection strategy (Section~\ref{sec:method}). Across all our benchmarks, it consistently and substantially outperforms both methods.

% \begin{figure}[t]
% \begin{minipage}{\columnwidth}
% \begin{wrapfigure}{l}{0.481\columnwidth}
%   \vspace{-\baselineskip}
%   \vspace*{-3pt}
%   \centering
%   \includegraphics[width=0.48\columnwidth]{fig/2room_intro}
%   \vspace{-\baselineskip}
%   \vspace*{-12pt}
% \end{wrapfigure}
% \refstepcounter{figure}\label{fig:2room_intro}%
% \noindent\textbf{Figure \thefigure:} \textbf{Reachability or novelty are not enough.} At every episode, the agent spawns in one of two isolated rooms. After exploring for some time, the second room has been rarely visited due to its lower spawning rate. Heatmaps show the score assigned to each tile by different goal-selection scores when the agent is in the top-left corner (red boxes mark the selected goal). \emph{Novelty alone} (e.g., visit counts inverse) picks tiles in the second room, which the agent cannot reach. \emph{Reachability alone} (e.g., distance) picks the agent's current tile, leading to no exploration. Only \emph{novelty and reachability combined} selects the least-visited tile \emph{within reach}. While simple, this example shows the importance of considering both reachability and novelty in exploration, and raises the central question of this paper: how to encode, learn, and combine reachability and novelty? Our SUN indicator provides principled answers.
% \end{minipage}
% \end{figure}

\begin{figure}[t]
\begin{minipage}{\columnwidth}
\begin{wrapfigure}{l}{0.52\columnwidth}
  \vspace{-\baselineskip}
  \vspace*{-3pt}
  \centering
  \includegraphics[width=1.02\linewidth]{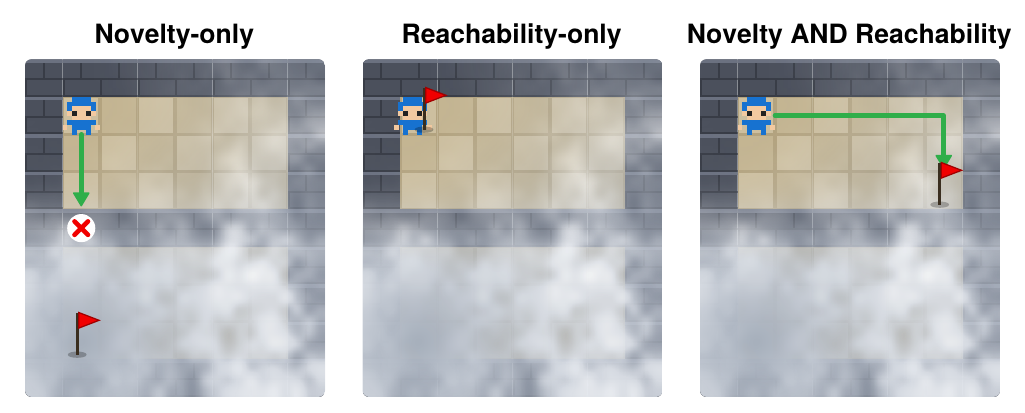}
  \vspace{-\baselineskip}
  \vspace*{-12pt}
\end{wrapfigure}
\refstepcounter{figure}\label{fig:two_room_intro}%
\noindent\textbf{Figure \thefigure:} {\textbf{Reachability or novelty are not enough.} At every episode, the agent spawns in one of two isolated rooms. After exploring for some time, the second room has been rarely visited due to its lower spawning rate. Red flags denote selected goals, and the fog marks how rarely a tile has been visited. \textit{Novelty alone} (e.g., preferring the least-explored tiles) selects a tile in the second room, which the agent cannot reach (it stops at the red cross). \textit{Reachability alone} (e.g., distance) picks the agent's current position, leading to no exploration. Only \textit{novelty and reachability combined} selects the least-visited tile \textit{within reach}. While simple, this example shows the importance of considering both reachability and novelty in exploration, and raises the central question of this paper: how to encode, learn, and combine reachability and novelty? Our SUN indicator provides principled answers.}
\end{minipage}
\end{figure}

\section{Problem Setting}
\label{sec:prelim}

\textbf{Optimal exploration.} A \textit{reward-free} 
% (or task-agnostic) 
Markov Decision Process (MDP) is defined by the tuple $\langle \statespace, \actionspace, \probmodel, p_0 \rangle$, where $\statespace$ is the state space, $\actionspace$ is the action space, $\probmodel(s' \spacedmid s, a)$ is the transition function, and $p_0$ is the initial state distribution. The objective is to explore the state space ``optimally'' without any task-specific reward. Two main lines of work formalize this notion of optimality differently.% --- \textit{and SUN bridges the two}.
\\[2pt]
The first line targets the state-visitation distribution: the goal is to learn a policy whose induced distribution maximizes a desired criterion, typically the entropy~\citep{hazan2019provably,lee2019efficient,mutti2021task,zhang2021renyi,jain2023maximum,adamczyk2026eve}. A maximum-entropy state-visitation distribution corresponds to uniform coverage of the state space, and provably efficient algorithms exist for this objective in tabular MDPs.
\\[2pt]
The second line frames exploration as goal-conditioned RL (GCRL) and the agent learns goal-conditioned policies $\pi(a \spacedmid s, g)$~\citep{lim2012autonomous,tarbouriech2020improved,tarbouriech2022adaptive}. The goal $g \in \goalspace$ may be a subset of the state, of the joint state-action, or of a learned representation thereof.\footnote{The goal space is environment-dependent. For example, discrete actions may highlight relevant dynamics (e.g., ``pick'' or ``push'' may terminate the episode) and exploration should \textit{explicitly} consider them. On the contrary, the state alone may be sufficient if it carries all information (e.g., agent pose in control tasks).} A goal $g$ is said to be \emph{reachable} from a reference state $s_0$ if there exists a policy $\pi$ that reaches $g$ from $s_0$ in bounded expected time. The objective is then to learn policies that can visit every goal reachable given different reference states. This formulation directly captures the intuition that exploration should focus on states the agent can actually reach, but does not specify a target visitation distribution.
\\[2pt]
Both objectives are principled but solving them exactly requires machinery --- e.g., Frank-Wolfe schemes for max-entropy~\citep{hazan2019provably}, PAC-style algorithms for reachable coverage~\citep{tarbouriech2020improved,tarbouriech2022adaptive} --- that does not scale to deep RL. Practical methods therefore approximate these objectives with greedy or local heuristics: thanks to careful goal-selection mechanisms, following $\pi(a \spacedmid s, g)$ induces a state-visitation distribution with broad and uniform state coverage. Our work follows this pragmatic line: we design a goal-selection rule that, at each step, prefers goals that are both underrepresented in the agent's current visitation distribution \emph{and} reachable.% under the current policy. 

\textbf{Successor Value Functions.}
% \label{subsec:w-functions}
In the GCRL literature, reachability is often encoded with the \emph{successor value function (SVF)}~\citep{dayan1993improving, kulkarni2016deep,blier2021learning,eysenbach2022contrastive,zheng2024contrastive}, which generalizes the value function and represents the cumulative $\gamma$-discounted occurrence of a goal $g$ under a policy $\pi$, i.e.,
\begin{equation}
    \setlength{\abovedisplayskip}{3pt}
    \setlength{\belowdisplayskip}{3pt}
    V^\pi(s_t, g) = \mathbb{E}\left[{\footnotesize\sum}_{k=t}^{\infty}\gamma^{k - t}\ind{\scriptstyle{\{s_k = g\}}} \evmid \pi, \probmodel, s_t\right], \label{eq:w-function}
\end{equation}
where $\gamma \in [0, 1)$ and $\ind\scriptstyle{\{s_k = g\}}$ is the reward function returning 1 if $s_k = g$ and 0 otherwise. 
The state-action analogue $Q^\pi(s_t, a_t, g)$ is defined likewise, and both admit a Bellman recursion as in standard value functions, with $\smash{V^\pi(s_t, g) = \EV_{a \sim \pi}[Q^\pi(s_t, a, g)]}$.
% The goal-conditioned policy is then trained to maximize the SVF as in classic RL. 
Similarly to classic value functions, SVFs are often approximated with parameterized functions $\Vtheta(s_t, g)$, and training them is a well-studied problem. In this paper, we rely on Hindsight Experience Replay (HER)~\citep{andrychowicz2017hindsight}.
\\[2pt]
In GCRL, once the goal is given (e.g., a desired robot pose or an environment coordinate), greedily following the SVF leads the agent to it, as the value increases the fewer steps are needed to reach the goal.\footnote{This holds because goals are \textit{terminal}: reaching $g$ ends the episode, so Eq.~\eqref{eq:w-function} reduces to $V^\pi(s_t, g) = \EV_\pi[\gamma^{\tau_g}] \in [0, 1]$, where $\tau_g$ is the hitting time of $g$. The SVF is thus a discounted reachability score.}
% rather than an occupancy.}
% \\[2pt]
This same mechanism extends naturally to reward-free exploration: \textit{select a goal appropriately} --- unlike in GCRL, the goal is not given by the task --- and then follow the SVF to reach it. The goal-selection is what determines whether exploration is optimal: a well-designed mechanism would guarantee coverage and uniformity over the goal space $\goalspace \subseteq \statespace \!\times\! \actionspace$. 
\begin{tcolorbox}[colback=gray!10, colframe=gray!40, boxrule=0.5pt, arc=2pt, left=6pt, right=6pt, top=4pt, bottom=4pt]
With these tools in hand, our method must address three concrete subproblems. First, how to combine the reachability and novelty signals into a single indicator. Second, how to design an effective goal-selection strategy given the above indicator. Third, how to compute a novelty signal that is both accurate and cheap enough to query at scale.
\end{tcolorbox}

% The reward in Eq.~\eqref{eq:w-function} is sometimes replaced by alternatives that target the same quantity through different reward shapings. For example, returning $-1$ until $s_i$ is reached and $0$ thereafter recovers a negative-distance interpretation~\citep{schaul2015universal,andrychowicz2017hindsight}. 
% % Continuous variants based on similarity kernels have also been used~\citep{pong2018temporal}. 
% The condition $s_k = s_i$ can also be relaxed to apply over a feature map rather than raw states, yielding the {successor features} of~\citet{barreto2017successor}, which generalize the SVFs to settings where exact state matching is impractical.

\section{Exploration via SUN}
\label{sec:method}

We present our answer to the three subproblems above: \textbf{SUN} (\textbf{SU}ccessor-to-\textbf{N}ovelty), a goal-selection indicator that combines an SVF-based reachability signal with a novelty signal in a single score.
\begin{tcolorbox}[colback=gray!10, colframe=gray!40, boxrule=0.5pt, arc=2pt, left=6pt, right=6pt, top=0pt, bottom=2pt]
\begin{equation}
\label{eq:sun-score}
\mathrm{SUN}(g \spacedmid s) \triangleq V^\pi(s, g) \cdot \nu(g),
\qquad
g_t = \argmax_{g \in \goalspace} \; \mathrm{SUN}(g \spacedmid s_t),
\end{equation}
\end{tcolorbox}
where $V^\pi(s_t, g)$ is the SVF estimating reachability of $g$ from $s_t$ under the goal-conditioned policy, and $\nu(g)$ is a novelty signal. Computing the $\argmax$ in Eq.~\eqref{eq:sun-score} is not feasible in continuous or large goal spaces, so we restrict it to a finite candidate set $\mathcal{C}_t \subset \goalspace$ sampled from a replay buffer. %~\citep{mnih2015human}. 
This choice pairs naturally with off-policy algorithms, which already maintain a buffer for training.
\\[2pt]
The rest of this section is organized as follows. Section~\ref{subsec:theory} presents properties that justify the indicator; Section~\ref{subsec:adaptive_selection} describes a novel goal-selection strategy that leverages these properties; Section~\ref{subsec:pseudo} introduces our novelty estimator; and Section~\ref{subsec:summary} summarizes SUN and its relation to prior work.

\subsection{The Indicator and Its Properties}
\label{subsec:theory}
We informally describe three properties of SUN that justify Eq.~\eqref{eq:sun-score} and that we will invoke in subsequent sections; proofs are in Appendix~\ref{app:theory}.
\begin{itemize}[leftmargin=15pt, itemsep=1pt, topsep=0pt]
\item[(a)] \textit{Count-scaled value.} If $\nu(g) = 1/n_g$ where $n_g$ is the goal visit count, SUN equals the value function of a reward inversely proportional to $n_g$, i.e., $\EV[\sum_{k=t} \gamma^{k-t} \ind{\scriptstyle{\{s_k = g\}}} / n_g]$. Thus SUN is not simply a product of two signals, but the value of a count-scaled goal-reaching objective.
\item[(b)] \textit{Reachability guarantee.} The reachability factor $V^\pi(s, g)$ controls hitting time: a high SVF value implies a high probability of reaching $g$ within a short horizon, $\Pr_{\pi_g}[\tau_g \leq n] \geq V^{\pi_g}(s, g) - \gamma^{n+1}$, where $\tau_g$ is the hitting time. The horizon scales as $\mathcal{O}(\log(2/V) / (1 - \gamma))$.
\item[(c)] \textit{Unreachable goals are suppressed.} If no policy in the agent's class can reach $g$, then $V^\pi(s, g) = 0$ for all such policies, and $\mathrm{SUN}(g \spacedmid s) = 0$ regardless of $\nu(g)$. This is the formal counterpart of Figure~\ref{fig:two_room_intro}: novelty alone selects unreachable goals, while SUN scores them zero.
\end{itemize}

\textbf{Why not an additive indicator?}
Common exploration strategies combine reachability and novelty \textit{additively}~\citep{diaz-bone2025discover}. Indeed, SUN's indicator could equally be defined additively as $V^\pi(s, g) + \nu(g)$.
However, additive formulations are sensitive to the relative scale of the two terms and typically require a tuning coefficient to balance them (see Appendix~\ref{app:pac}), especially if the SVF is approximated as in $\Vtheta$.
The multiplicative form removes the need to \textit{calibrate} the two terms against each other: they share a common ``zero'' (an unreachable or already-saturated goal scores zero on either factor and is rejected regardless of the other) and a common, known scale (both are non-negative and bounded by one). The trade-off is not thereby eliminated --- in log-space it is set by $\kappa = -\log\gamma$ (Appendix~\ref{subsec:thm-logspace}) --- but it is fixed and inherited from the discount used for value learning, rather than being a free coefficient that must be re-tuned whenever the scale of the novelty signal changes (see Section \ref{subsec:summary}).
In Section~\ref{subsec:ablation_results} we compare SUN against an additive UCB-style variant and find that the multiplicative indicator matches or exceeds it in aggregate and is markedly more robust across environments (Appendix \ref{app:rliable}).

\subsection{When Should The Agent Select A Goal? Adaptive Goal-Selection Strategy}
\label{subsec:adaptive_selection}
If $\Vtheta$ were exact, the goal with the highest SUN score would be identified correctly and reached in the minimum number of steps. \textit{Episodic goal-selection} --- selecting the goal at the start of an episode and keeping it until reached --- would then be enough, since the selected goal cannot be overtaken mid-episode (Appendix~\ref{subsec:thm-perstep}).
% If $\Vtheta$ were exact, acting greedily with respect to it would be optimal --- the best goal would be selected and reached in finite time (Appendix~\ref{subsec:thm-perstep}). Thus, \textit{episodic goal-selection} --- selecting the goal at the start of an episode and keeping it fixed until reached --- would be optimal. 
However, $\Vtheta$ is learned and approximate, and the agent may commit to unreachable goals, potentially not exploring at all. Or, in stochastic environments, the agent may suddenly find itself in states where the previously-selected goal is no longer reachable. 
% Episodic goal-selection therefore performs poorly in practice.
% \\[1pt]
The opposite strategy, \textit{per-step goal-selection}, compares the current goal against a fresh candidate set at every step to find a potentially better one. This can prevent commitment to unreachable goals, e.g., after a wrong action or a noisy transition. However, this strategy can be too unstable: as $\Vtheta$ is being learned, goal values shift quickly and the agent may pick different goals at every timestep, acting near-randomly.
\\[2pt]
For these reasons, we propose a novel \textit{adaptive} strategy, inspired by the theoretical properties of the SVF. Under deterministic dynamics, the true value at the current state should be monotonically non-decreasing along the trajectory toward the selected goal: as the agent moves closer, $V^\pi(s_t, g)$ grows. Thus, a drop in $\Vtheta(s_t, g)$ signals that the goal is either unreachable from the current state, or that the \textit{approximate} SVF was inaccurate at the time of selection --- in either case, the goal is no longer a reliable target. Concretely, at each step $t$ we compare the current value against the value at selection time $t_{\mathrm{sel}}$: if $\Vtheta(s_t, g) < \Vtheta(s_{t_{\mathrm{sel}}}, g)$, the current goal is discarded and a new one is selected from a freshly sampled candidate set; otherwise, the current goal is kept.\footnote{Under exact $\Vtheta$ and deterministic dynamics, episodic and adaptive selection coincide (Appendix~\ref{subsec:thm-perstep}).}
% \footnote{Under exact $\Vtheta$ and deterministic dynamics, episodic and adaptive selection coincide, and per-step selection coincides with them as well whenever its candidate set is the whole goal space (Appendix~\ref{subsec:thm-perstep}).}

\begin{tcolorbox}[colback=gray!10, colframe=gray!40, boxrule=0.5pt, arc=2pt, left=6pt, right=6pt, top=2pt, bottom=2pt]
This adaptive strategy preserves the stability of episodic commitment, and inherits per-step reselection's ability to escape bad commitments --- but only if the SVF changes frequently (because it is still being learned) or if it signals that something has gone wrong (e.g., due to environment noise). Section \ref{subsec:ablation_results} empirically validates our strategy.
\end{tcolorbox}

\subsection{Novelty Via Lightweight Pseudocounts}
\label{subsec:pseudo}
SUN combines two signals: reachability and novelty. Reachability is handled by learning SVFs with HER~\citep{andrychowicz2017hindsight} (see Appendix \ref{app:notes}). 
% We now describe our novelty estimator, which is the main contribution of this subsection.
% \\[2pt]
Novelty $\nu(g)$ in Eq.~\eqref{eq:sun-score} can be instantiated in many ways~\citep{pathak2017curiosity,burda2019exploration}. A principled choice is visit counts or density estimates~\citep{bellemare2016unifying,tang2017exploration}, so that rarely-visited goals receive a high novelty score: $\smash{\nu(g) = 1/n_g}$, where $\smash{n_g}$ is the number of times $g$ has been visited. In continuous spaces $\smash{n_g}$ cannot be tracked exactly and must be approximated by a pseudocount. Since we query $\nu(g)$ against many candidate goals $\mathcal{C}_t$ multiple times per episode, the pseudocount must be lightweight to compute --- standard approaches such as kernel density estimation (KDE) or neural density models do not satisfy this requirement. We instead propose a pseudocount that amortizes its cost into buffer insertion: counts are precomputed and stored alongside each buffer entry, making queries cheap.
\\[2pt]
For each buffer entry at index $i$, we store a count $\smash{n_i}$ of entries within a neighborhood of radius $\rho$ of $s_i$ in standardized feature space. 
% 
% Features are standardized by the interquartile range (floored so that nearly constant features are not amplified) estimated from the buffer and fixed once it settles. This matters because an insertion only updates the counts of its own neighbors, so a later change of scale would leave stored counts uncorrected (Appendix~\ref{app:pseudocount}).
% Figure~\ref{fig:pseudocount} illustrates the procedure; Appendix~\ref{app:pseudocount} gives the standardization details.
% When a new sample is inserted, we compute its count (number of neighbors plus one for itself) and increment the neighbors' counts (Figure~\ref{fig:pseudocount}).
% 
When a new sample is inserted, we compute its count (number of neighbors plus one for itself) and increment the neighbors' counts. Figure~\ref{fig:pseudocount} illustrates the procedure; Appendix~\ref{app:pseudocount} gives the standardization details.
% \\[2pt]
This is a fixed-radius nearest-neighbor density estimator with the cost moved from query time to insertion time. Its benefits are the following.
\begin{itemize}[leftmargin=*, noitemsep, topsep=0pt]
\item The radius $\rho$ is a single hyperparameter: standardization makes one scalar meaningful across features, which would otherwise each need their own radius.
% \item The radius $\rho$ is a single hyperparameter, applied in standardized space. Standardization makes a single scalar radius meaningful across features: without it, a separate radius would be needed for each feature dimension to account for differences in native scale.
\item By incrementing the count of every neighbor of the new sample, the stored counts are maintained across the entire buffer without ever recomputing them from scratch. At query time, the novelty of a candidate goal $g$ is read directly from the replay buffer: $\nu(g) = 1/n_g$. This has cost $\mathcal{O}(1)$.
\item The insertion cost is $\mathcal{O}(N \!\cdot\! M)$ in the buffer size $N$ and goal dimensionality $M$. Classic KDE costs $\mathcal{O}(N \!\cdot\! M \!\cdot\! B)$ per step for $B$ candidate goals, and neural density models can be even more expensive.
\end{itemize}

\begin{figure}[!t]
\vspace*{-9pt}
\centering
\begin{minipage}[c]{0.35\textwidth}
\centering
\resizebox{\linewidth}{!}{%
\begin{tikzpicture}[
  % baseline=(current bounding box.north),
  every node/.style={inner sep=0pt},
  buf/.style={circle, draw=black, fill=gray!15, line width=0.5pt, minimum size=0.50cm, font=\scriptsize},
  bufN/.style={circle, draw=black, fill=blue!25, line width=0.7pt, minimum size=0.50cm, font=\scriptsize},
  bufNew/.style={circle, draw=black, fill=orange!60, line width=0.9pt, minimum size=0.60cm, font=\scriptsize\bfseries},
  radius/.style={draw=orange!80!black, line width=0.7pt, dashed},
  panellabel/.style={font=\small\bfseries, anchor=south},
]

% =========================================================
% LEFT PANEL: buffer state BEFORE insertion
% =========================================================
\begin{scope}[xshift=0cm]

  \node[buf]  at (0.4, 0.3) {3};
  \node[buf]  at (1.0, 0.8) {2};
  \node[buf]  at (1.4, 0.0) {4};
  \node[buf]  at (2.0, 0.6) {2};
  \node[buf]  at (0.6, 1.6) {1};
  \node[buf]  at (1.8, 1.7) {2};
  \node[buf]  at (0.2, 1.0) {1};
\end{scope}

% =========================================================
% RIGHT PANEL: after inserting new sample
% =========================================================
\begin{scope}[xshift=3.5cm]

  \node[buf]  at (0.4, 0.3) {3};
  \node[bufN] at (1.0, 0.8) {3};
  \node[buf]  at (1.4, 0.0) {4};
  \node[bufN] at (2.0, 0.6) {3};
  \node[buf]  at (0.6, 1.6) {1};
  \node[buf]  at (1.8, 1.7) {2};
  \node[buf]  at (0.2, 1.0) {1};

  % New sample
  \node[bufNew] at (1.5, 0.95) {3};

  % Radius circle around new sample
  \draw[radius] (1.5, 0.95) circle (0.65cm);

  % Small radius indicator
  % \draw[->, >=stealth, line width=0.5pt, orange!80!black] (1.5, 0.95) -- (2.16, 1.1);
  \node[anchor=west, color=orange!80!black] at (2.2, 1.1) {$\rho$};
\end{scope}

% =========================================================
% ARROW between panels
% =========================================================
\draw[-{Stealth[length=4mm, width=3mm]}, line width=1.5pt] (2.4, 0.85) -- (3.3, 0.85);
\node[anchor=south, align=center] at (2.8, 1.1) {Insert};
\end{tikzpicture}%
}
\end{minipage}%
\hfill
\begin{minipage}[c]{0.605\textwidth}
\caption{\textbf{Pseudocount via replay buffer neighbors.} Each buffer entry stores a count of its neighbors within radius $\rho$ (\textit{in standardized feature space}). When a new sample is inserted, its count is set to the number of neighbors within $\rho$ plus one (for itself), and each neighbor's count is incremented.}
\label{fig:pseudocount}
\end{minipage}
\vspace*{-5pt}
\end{figure}

\begin{tcolorbox}[colback=gray!10, colframe=gray!40, boxrule=0.5pt, arc=2pt, left=6pt, right=6pt, top=4pt, bottom=4pt]
When goals may be reselected at any step, classic methods are prohibitively expensive while our pseudocount remains tractable. By moving the cost from query to insertion time and standardizing across features, we obtain a novelty estimator that is accurate yet lightweight. We validate its accuracy in Appendix~\ref{app:pseudocount_ablation} and report wall-clock costs in Appendix~\ref{app:compute}.
\end{tcolorbox}

\subsection{Summary and Related Work}
\label{subsec:summary}
\begin{wrapfigure}{l}{0.45\linewidth}
\vspace*{-14pt}
\begin{algorithm}[H]
\setstretch{0.94}
\footnotesize
\DontPrintSemicolon
\SetKwFunction{FSelectGoal}{SelectGoal}
\SetKwProg{Fn}{Function}{:}{\KwRet}
\caption{SUN Exploration \\(During One Episode)}\label{alg:sun}
\Fn{\FSelectGoal{$s_t, \mathcal{D}$}}{
    $s_{t_{\mathrm{sel}}} \leftarrow s_t$\;
    $\mathcal{C}_t \sim \mathcal{D}$\;
    \Return $\argmax_{g \in \mathcal{C}_t} \mathrm{SUN}(g \spacedmid s_t)$\;
}
\BlankLine
\For{$t = 0 \ldots T$}{
    $\texttt{begin} \leftarrow t = 0$\;
    $\texttt{reached} \leftarrow \|s_t - g_t\| < \eta$\;
    $\texttt{adapt} \leftarrow \Vtheta(s_t, g_t) < \Vtheta(s_{t_{\mathrm{sel}}}, g_t)$\;
    \uIf{\texttt{begin} \KwOr \texttt{reached} \KwOr \texttt{adapt}}{
        $g_t \leftarrow$ \FSelectGoal{$s_t, \mathcal{D}$}\;
    }
    \Else{
        $g_t \leftarrow g_{t-1}$\;
    }
    $a_t \sim \pi(\cdot \mid s_t, g_t)$\;
    $s_{t+1} \sim \probmodel(\cdot \mid s_t, a_t)$\;
    \texttt{Update(}$\mathcal{D}, s_t, a_t, s_{t+1}$\texttt{)}
}
\end{algorithm}
\vspace*{-18pt}
\end{wrapfigure}
Algorithm \ref{alg:sun} summarizes SUN exploration. 
At the beginning of an episode, the agent selects the goal $g_0$ according to Eq.~\eqref{eq:sun-score}, with candidates (and their pseudocounts) sampled from the replay buffer. The initial state $s_0$ is saved as $s_{t_\mathrm{sel}}$ (state at selection time). At every timestep, if $\Vtheta(s_t, g_t) < \Vtheta(s_{t_\mathrm{sel}}, g_t)$ or if the agent has reached the current goal, a new goal $g_t$ is selected from a fresh batch of candidates and $s_{t_\mathrm{sel}}$ is updated; if not, $g_t$ is kept.
Then, the agent acts to explore with $a_t \sim \pi(\cdot \spacedmid s_t, g_t)$, the new sample is inserted into the replay buffer, and pseudocounts are updated. 
The policy and the SVF are trained with any off-policy algorithm (Appendix~\ref{app:notes}).
\\[2pt]
This scheme can be applied to classic reward-driven RL: the goal-conditioned policy drives exploration to collect environment rewards, and task-specific value function and policy are trained off-policy.
% At the beginning of an episode, the agent selects the goal $g_0$ according to Eq.~\eqref{eq:sun-score}, with candidates (and their pseudocounts) sampled from the replay buffer. The initial state $s_0$ is saved as $s_{t_\mathrm{sel}}$ (state at selection time). At every timestep, if $\Vtheta(s_t, g_t) < \Vtheta(s_{t_\mathrm{sel}}, g_t)$ or the current goal is reached, a new goal $g_t$ is selected from a fresh set of candidates and $s_{t_\mathrm{sel}}$ is updated; if not, $g_t$ is kept.
% Then, the agent acts $a_t \sim \pi(\cdot \spacedmid s_t, g_t)$, the new sample is inserted into the replay buffer, and pseudocounts are updated. 

\textbf{AdaGoal and DISCOVER.} Throughout Sections~\ref{sec:prelim}--\ref{sec:method}, we have discussed how SUN relates to prior work along several axes: reward-free exploration~\citep{hazan2019provably,tarbouriech2020improved,tarbouriech2022adaptive}, GCRL~\citep{schaul2015universal}, hindsight relabeling~\citep{andrychowicz2017hindsight,eysenbach2022contrastive,zheng2024contrastive}, and count- or density-based novelty~\citep{bellemare2016unifying,pong2020skewfit,pitis2020maximumentropy,burda2019exploration}. Here, we focus on the two methods closest in spirit to SUN: AdaGoal~\citep{tarbouriech2022adaptive} and DISCOVER~\citep{diaz-bone2025discover}. Both select goals using an ensemble of SVFs $\{V^{\raisebox{-1pt}{$\scriptstyle\theta_1$}}, \ldots, V^{\raisebox{-1pt}{$\scriptstyle\theta_K$}}\}$, with mean $\mu(s, g)$ and standard deviation $\sigma(s, g)$. Formally,

{%
\setlength{\abovedisplayskip}{-6pt}%
\setlength{\abovedisplayshortskip}{0pt}%
\setlength{\belowdisplayskip}{-2pt}%
\setlength{\belowdisplayshortskip}{-2pt}%
% \begin{align*}
% \text{SUN} \quad & g_t = \argmax_{g \in \mathcal{C}_t} \; \Vtheta(s_t, g) \, / \, n_g \\
% \text{DISCOVER}\footnotemark \quad & g_t = \argmax_{g \in \mathcal{C}_0} \; \mu(s_0, g) + \beta\, \sigma(s_0, g) \\
% \text{AdaGoal}\footnotemark \quad & g_t = \argmax_{g \in \mathcal{C}_0} \; \sigma(s_0, g)
% \end{align*}
\begin{align*}
\small
\underbrace{g_t = \argmax_{g \in \mathcal{C}_t} \; \Vtheta(s_t, g) \, / \, n_g \,}_{\text{SUN}\vphantom{\footnotemark[1]}}
\qquad
\underbrace{g_t = \argmax_{g \in \mathcal{C}_0} \; \mu(s_0, g) + \beta\, \sigma(s_0, g) \,}_{\text{DISCOVER}\footnotemark}
\qquad
\underbrace{g_t = \argmax_{g \in \mathcal{C}_0} \; \sigma(s_0, g)}_{\text{AdaGoal}\footnotemark}
\end{align*}
}%
\footnotetext[4]{This is the pure-exploration variant of DISCOVER (referred to as ``achievability + novelty'' in its paper). Its full objective, designed for GCRL, additionally combines this with a direction term toward a task-specific goal.}%
% \footnotetext[5]{Two notes. First, this is the deep-RL variant of AdaGoal, where $\sigma$ approximates a prediction-error upper bound. The tabular version further constrains the $\argmax$ to the expected hitting time as a proxy for reachability, but the deep-RL variant does not. Second, \citeauthor {tarbouriech2022adaptive} apply the $\softmax$ operator rather than $\max$. Here, we use $\max$ for consistency with SUN and DISCOVER, but all methods can use $\softmax$ as well.}%
\footnotetext[5]{This is the deep-RL variant of AdaGoal, where $\sigma$ approximates a prediction-error. The tabular version constrains the $\argmax$ to the expected hitting time as a proxy for reachability, but the deep-RL variant does not.}%

The mean $\mu$ serves as a reachability signal, and the disagreement $\sigma$ as an epistemic-uncertainty signal (goals on which the ensemble disagrees are those the agent cannot reliably reach). AdaGoal explicitly frames this as ``selecting uncertain goals'' and proves PAC guarantees; DISCOVER frames the same quantity as ``novelty'' since uncertainty correlates with under-exploration.
\\[1pt]
Two main differences distinguish SUN from this line of work. \textit{First}, AdaGoal and DISCOVER select the goal once at the beginning of the episode ($g_t$ is selected once from $\mathcal{C}_0$). SUN instead continuously monitors the value of the current goal, and reselects it from a freshly-resampled candidate set $\mathcal{C}_t$ if needed. \textit{Second}, SUN replaces the ensemble with a lightweight pseudocount, removing the cost of training and querying $K$ critics, and introduces an effective balance between reachability and novelty. Section~\ref{sec:experiments} shows these differences matter: episodic commitment fails when goals become unreachable mid-episode, DISCOVER is sensitive to $\beta$, and AdaGoal has no effective reachability proxy.
\\[2pt]
\textbf{Proto-Goals.} \citet{bagaria2023scaling} also combine novelty and reachability, but \textit{sequentially}: first, goal candidates are sampled proportionally to a count-based novelty; then, the one with highest SVF is pursued. Thus, reachability cannot recover from a novelty draw that misses, and novelty cannot override a reachability $\argmax$. The authors report that local reachability hurt performance by biasing selection toward easy goals, until an additional timescale-stratification mechanism was introduced --- exactly the bias our multiplicative indicator avoids (Section~\ref{subsec:ablation_results}).
Similarly, their goal is selected once per episode and pursued until achieved, and the authors list finer-grained goal switching~\citep{pislar2022when} as future work, which is precisely what our adaptive strategy provides.
\\[2pt]
\textbf{Directed Exploration.} Closest to our pseudocount is the episodic novelty module of NGU~\citep{badia2020never}, which also estimates counts with a nearest-neighbor density estimator. 
Two differences matter. \textit{First}, its memory is cleared at every episode, so its counts measure {within-episode} novelty, while SUN leverages {lifetime} novelty. \textit{Second}, it recomputes the kernel sum at query time, which is affordable only for a memory at most one episode long.
% The normalization of \citet{badia2020never} is nonetheless a sensible alternative to our fixed radius $\rho$ wherever the difficulty is one of scale rather than of anisotropy; we discuss this in Appendix~\ref{app:pseudocount_ablation}.

\begin{figure}[t]
    \centering
    \setlength{\fboxsep}{0pt}%
    \setlength{\fboxrule}{0.3pt}%
    % 0.3x2x5pt of black border in control envs, 2x12pt of spacing between all pics = 27pt
    \def\envpanelw{\dimexpr(\linewidth-27pt)/13\relax}%
    \def\envname#1{\begin{minipage}{\envpanelw}\centering\fontsize{5.5pt}{6pt}\selectfont{#1}\end{minipage}}%
    \envname{3Room}\hfill
    \envname{4RoomStuck}\hfill
    \envname{GridMaze}\hfill
    \envname{L.Lander}\hfill
    \envname{M.Car}\hfill
    \envname{Pendulum}\hfill
    \envname{Acrobot}\hfill
    \envname{C.Pole}\hfill
    \envname{P.Maze-S}\hfill
    \envname{P.Maze-H}\hfill
    \envname{A.Maze-S}\hfill
    \envname{A.Maze-H}\hfill
    \envname{A.Push-H}
    \\[0pt]
    % ---- Environment renders ----
    \begin{minipage}{0.98\envpanelw}
    \resizebox{\linewidth}{!}{%
\begin{tikzpicture}[
  empty/.style={fill=black},
  separator/.style={gray!60, line width=2.4pt},
  start/.style={draw=cyan, line width=0.3pt, fill=cyan},
  goal/.style={draw=green!70!black, line width=0.3pt, fill=green!70!black},
]
  \def\cs{0.20}
  \foreach \row [count=\y from 0] in {
    {.,.,.,.,.,.,.,.},
    {.,.,.,.,.,.,.,.},
    {.,.,.,.,.,.,.,.},
    {.,.,.,.,.,.,.,.},
    {.,.,.,.,.,.,.,.},
    {.,.,.,.,.,.,.,.},
    {.,.,.,.,.,.,.,.},
    {.,.,.,.,.,.,.,.},
    {.,.,.,.,.,.,.,.},
    {.}
  } {
    \ifnum\y<9
      \foreach \cell [count=\x from 0] in \row {
        \fill[empty] (\x*\cs, -\y*\cs) rectangle ++(\cs, -\cs);
      }
    \fi
  }
  % Starting positions (white border keeps the fill inside the tile grid)
  \fill[fill=cyan] (0*\cs, -2*\cs) rectangle ++(\cs, -\cs);
  \fill[fill=cyan] (0*\cs, -3*\cs) rectangle ++(\cs, -\cs);
  \fill[fill=cyan] (0*\cs, -8*\cs) rectangle ++(\cs, -\cs);
  % Gray lines separating the three rooms (drawn last, on top of tiles)
  \draw[separator] (0, -3*\cs) -- (8*\cs, -3*\cs);
  \draw[separator] (0, -6*\cs) -- (8*\cs, -6*\cs);
\end{tikzpicture}%
}%
    \end{minipage}%
    \hfill%
    \begin{minipage}{\envpanelw}
    \resizebox{\linewidth}{!}{%
\begin{tikzpicture}[
  empty/.style={fill=black},
  start/.style={fill=cyan},
  wall/.style={fill=gray!60},
  arrow/.style={->, >=stealth, red, line width=1.2pt},
  goal/.style={fill=green!70!black},
  qmark/.style={yellow, font=\fontsize{4}{4}\selectfont\bfseries},
]
  \def\cs{0.20}
  \foreach \row [count=\y from 0] in {
    {W,W,W,W,W,W,W,W,W,W,W,W,W},
    {W,.,.,.,.,.,W,.,.,.,.,.,W},
    {W,.,.,.,.,.,W,.,.,.,.,.,W},
    {W,Q,Q,Q,.,.,.,.,.,.,.,.,W},
    {W,W,.,W,W,W,W,W,.,.,.,.,W},
    {W,.,.,.,.,.,.,W,.,.,.,.,W},
    {W,.,.,.,.,.,.,W,.,.,.,.,W},
    {W,.,.,.,.,.,.,W,W,W,.,W,W},
    {W,.,.,.,.,.,.,W,.,.,.,.,W},
    {W,.,.,.,.,.,.,W,.,.,.,.,W},
    {W,.,.,.,.,.,.,.,.,.,.,.,W},
    {W,.,.,.,.,.,.,W,.,.,.,.,W},
    {W,W,W,W,W,W,W,W,W,W,W,W,W},
    {.}
  } {
    \ifnum\y<13
      \foreach \cell [count=\x from 0] in \row {
        \ifthenelse{\equal{\cell}{W}}{
          \fill[wall] (\x*\cs, -\y*\cs) rectangle ++(\cs, -\cs);
        }{
          \fill[empty] (\x*\cs, -\y*\cs) rectangle ++(\cs, -\cs);
          \ifthenelse{\equal{\cell}{Q}}{
            \node[qmark] at (\x*\cs + \cs/2, -\y*\cs - \cs/2) {?};
          }{}
        }
      }
    \fi
  }
  % (row 4, col 2): down
  \draw[arrow] (2*\cs + \cs/2, -4*\cs - 0.02) -- (2*\cs + \cs/2, -5*\cs + 0.02);
  % Column 3, rows 5-10: left
  \foreach \r in {5,...,10} {
    \draw[arrow] (4*\cs - 0.02, -\r*\cs - \cs/2) -- (3*\cs + 0.02, -\r*\cs - \cs/2);
  }
  % (row 5, col 4): left
  \draw[arrow] (5*\cs - 0.02, -5*\cs - \cs/2) -- (4*\cs + 0.02, -5*\cs - \cs/2);
  % Column 4, rows 6-10: right
  \foreach \r in {6,...,10} {
    \draw[arrow] (4*\cs + 0.02, -\r*\cs - \cs/2) -- (5*\cs - 0.02, -\r*\cs - \cs/2);
  }
  % Row 11, cols 3-4: right
  \foreach \c in {3,4} {
    \draw[arrow] (\c*\cs + 0.02, -11*\cs - \cs/2) -- (\c*\cs + \cs - 0.02, -11*\cs - \cs/2);
  }
  % (row 10, col 7): left
  \draw[arrow] (8*\cs - 0.02, -10*\cs - \cs/2) -- (7*\cs + 0.02, -10*\cs - \cs/2);
  \fill[goal] (11*\cs, -11*\cs) rectangle ++(\cs, -\cs);
  \fill[start] (1*\cs, -1*\cs) rectangle ++(\cs, -\cs);
\end{tikzpicture}%
}%
    \end{minipage}%
    \hfill%
    \begin{minipage}{\envpanelw}
    \resizebox{\linewidth}{!}{%
\begin{tikzpicture}[
  empty/.style={fill=black},
  wall/.style={fill=gray!60},
  start/.style={fill=cyan},
  goal/.style={fill=green!70!black},
]
  \def\cs{0.20}
  \foreach \row [count=\y from 0] in {
    {.,.,.,.,.,.,.,.,.,.,.,.},
    {.,W,.,.,.,.,.,.,.,.,.,.},
    {W,.,.,.,.,W,W,W,W,.,W,.},
    {.,W,.,.,W,W,.,.,.,.,W,.},
    {.,.,.,W,W,.,.,.,.,.,W,.},
    {.,.,W,W,.,.,.,.,.,.,W,.},
    {.,W,W,.,.,.,.,.,W,W,W,.},
    {.,.,W,.,.,.,.,.,.,.,W,.},
    {.,.,.,.,.,W,W,W,.,.,W,.},
    {W,W,.,.,W,.,.,.,.,.,W,.},
    {.,.,W,.,W,.,.,.,.,.,W,.},
    {.,.,W,.,.,.,.,.,.,.,.,.},
    {.}
  } {
    \ifnum\y<12
      \foreach \cell [count=\x from 0] in \row {
        \ifthenelse{\equal{\cell}{W}}{
          \fill[wall] (\x*\cs, -\y*\cs) rectangle ++(\cs, -\cs);
        }{
          \fill[empty] (\x*\cs, -\y*\cs) rectangle ++(\cs, -\cs);
        }
      }
    \fi
  }
  % Starting position (bottom-left)
  \fill[start] (0*\cs, -11*\cs) rectangle ++(\cs, -\cs);
  % Goal position (row 1, col 10)
  \fill[goal] (10*\cs, -1*\cs) rectangle ++(\cs, -\cs);
\end{tikzpicture}%
}%
    \end{minipage}%
    \hfill%
    \begin{minipage}{\envpanelw}
    \fbox{\includegraphics[width=\dimexpr\linewidth-2\fboxrule\relax,keepaspectratio]{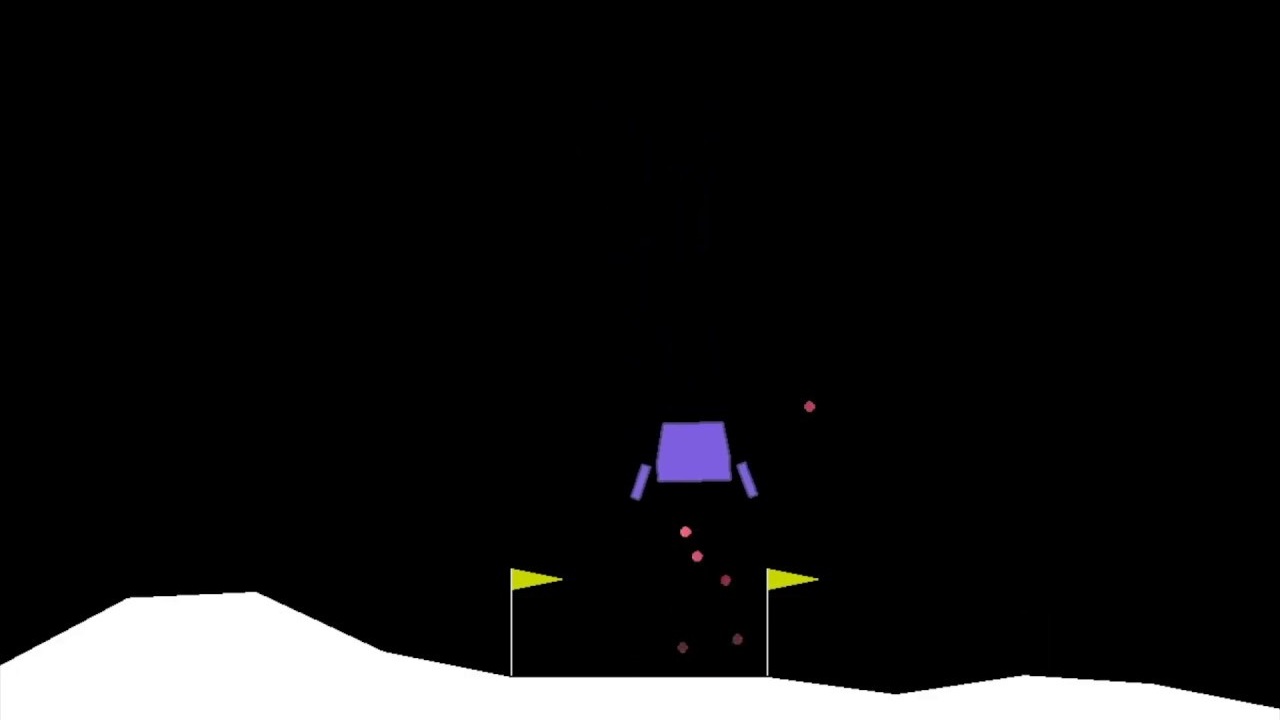}}%
    \end{minipage}%
    \hfill%
    \begin{minipage}{\envpanelw}
    \fbox{\includegraphics[width=\dimexpr\linewidth-2\fboxrule\relax,keepaspectratio]{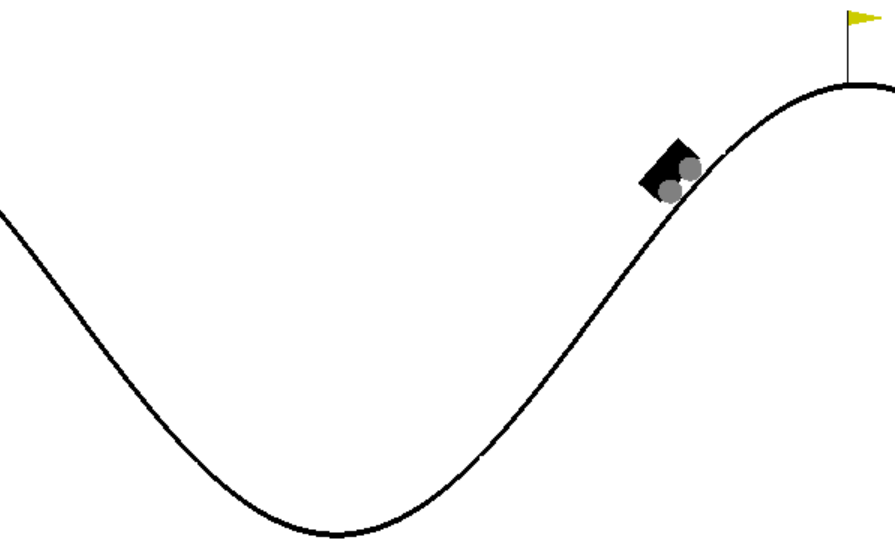}}%
    \end{minipage}%
    \hfill%
    \begin{minipage}{\envpanelw}
    \fbox{\includegraphics[width=\dimexpr\linewidth-2\fboxrule\relax,keepaspectratio]{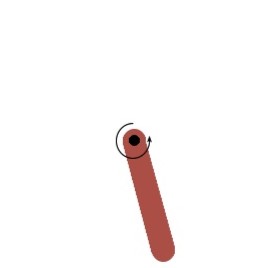}}%
    \end{minipage}%
    \hfill%
    \begin{minipage}{\envpanelw}
    \fbox{\includegraphics[width=\dimexpr\linewidth-2\fboxrule\relax,keepaspectratio]{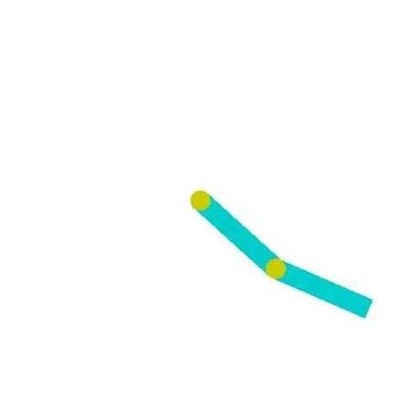}}%
    \end{minipage}%
    \hfill%
    \begin{minipage}{\envpanelw}
    \fbox{\includegraphics[width=\dimexpr\linewidth-2\fboxrule\relax,keepaspectratio]{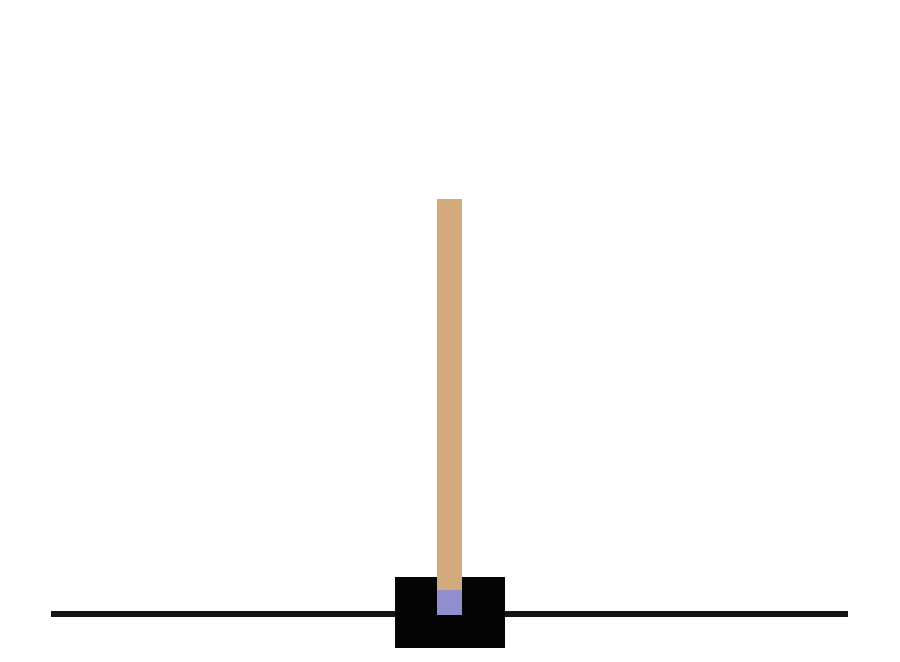}}%
    \end{minipage}%
    \hfill%
    \begin{minipage}{\envpanelw}
    \includegraphics[width=\linewidth,keepaspectratio]{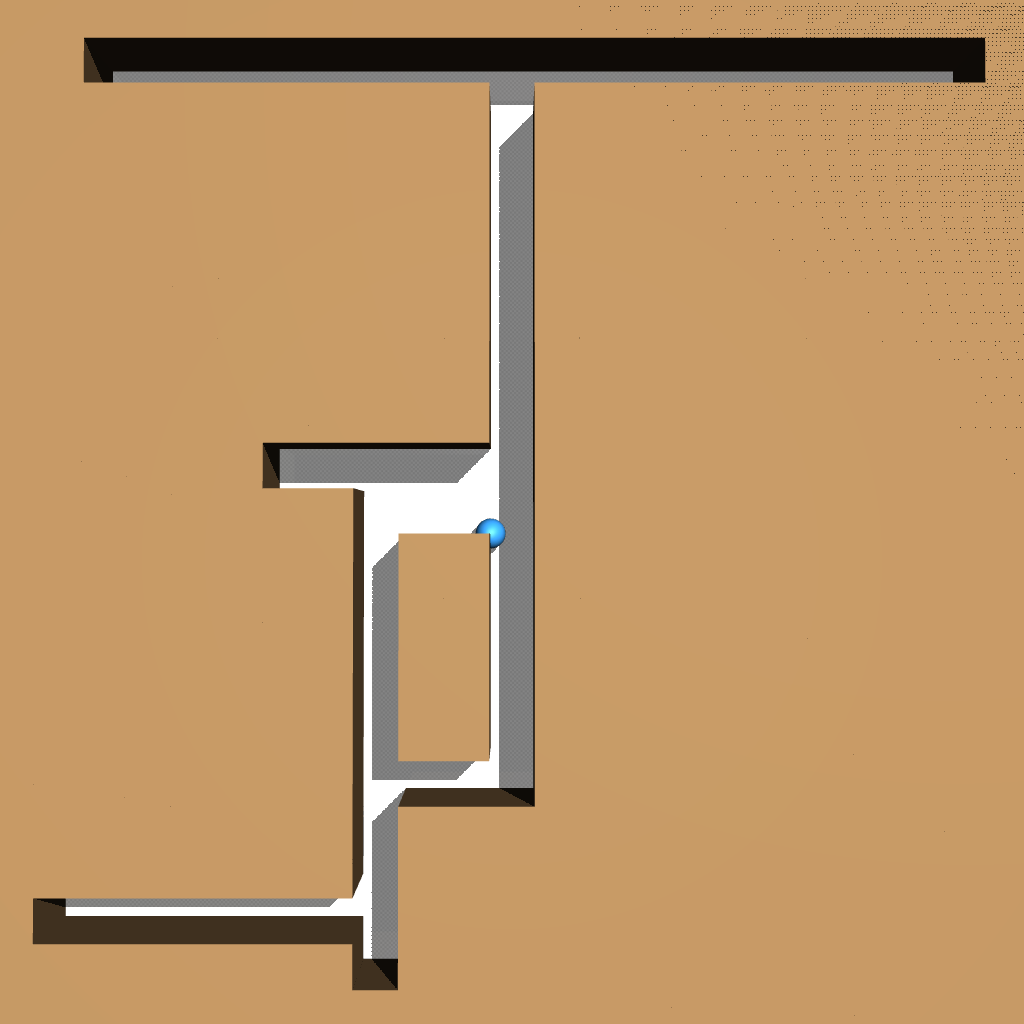}%
    \end{minipage}%
    \hfill%
    \begin{minipage}{\envpanelw}
    \includegraphics[width=\linewidth,keepaspectratio]{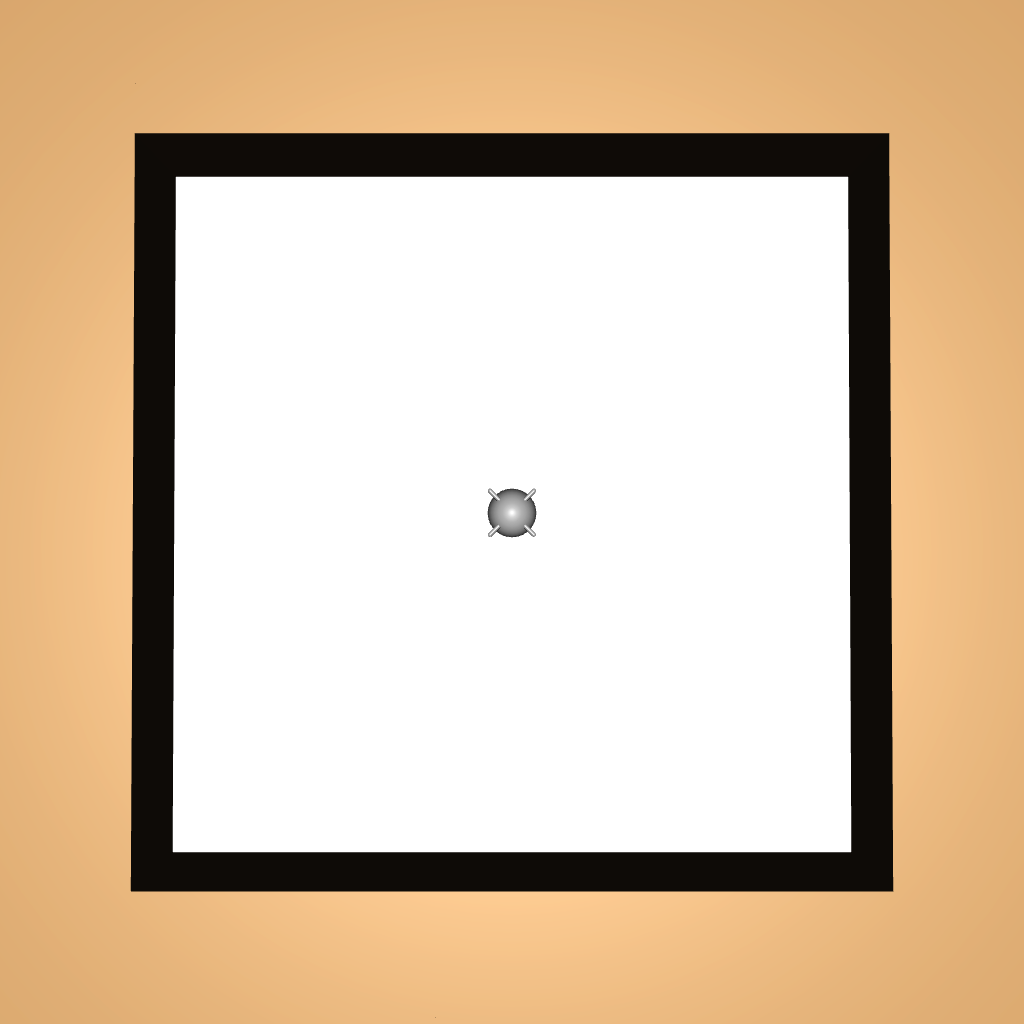}%
    \end{minipage}%
    \hfill%
    \begin{minipage}{\envpanelw}
    \includegraphics[width=\linewidth,keepaspectratio]{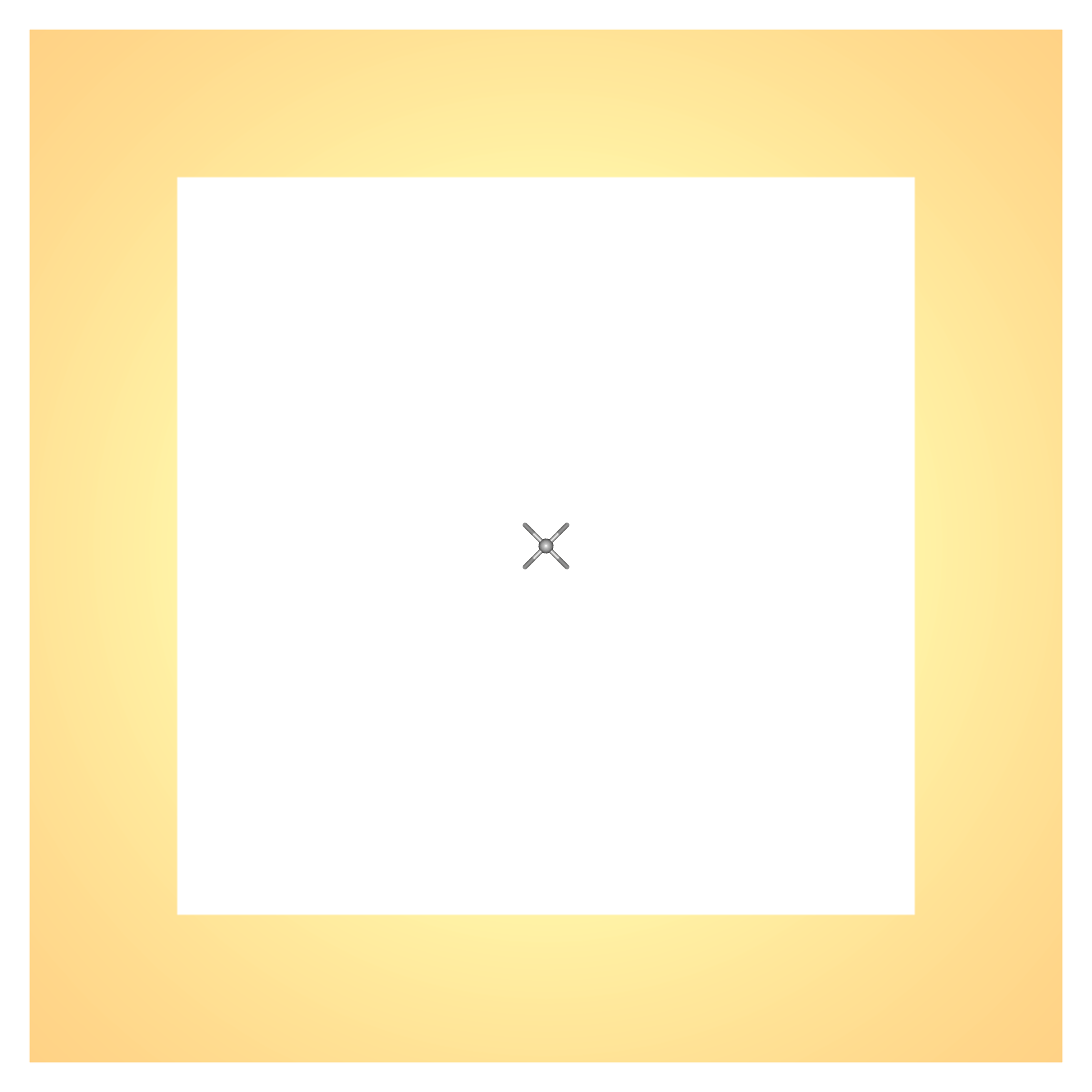}%
    \end{minipage}%
    \hfill%
    \begin{minipage}{\envpanelw}
    \includegraphics[width=\linewidth,keepaspectratio]{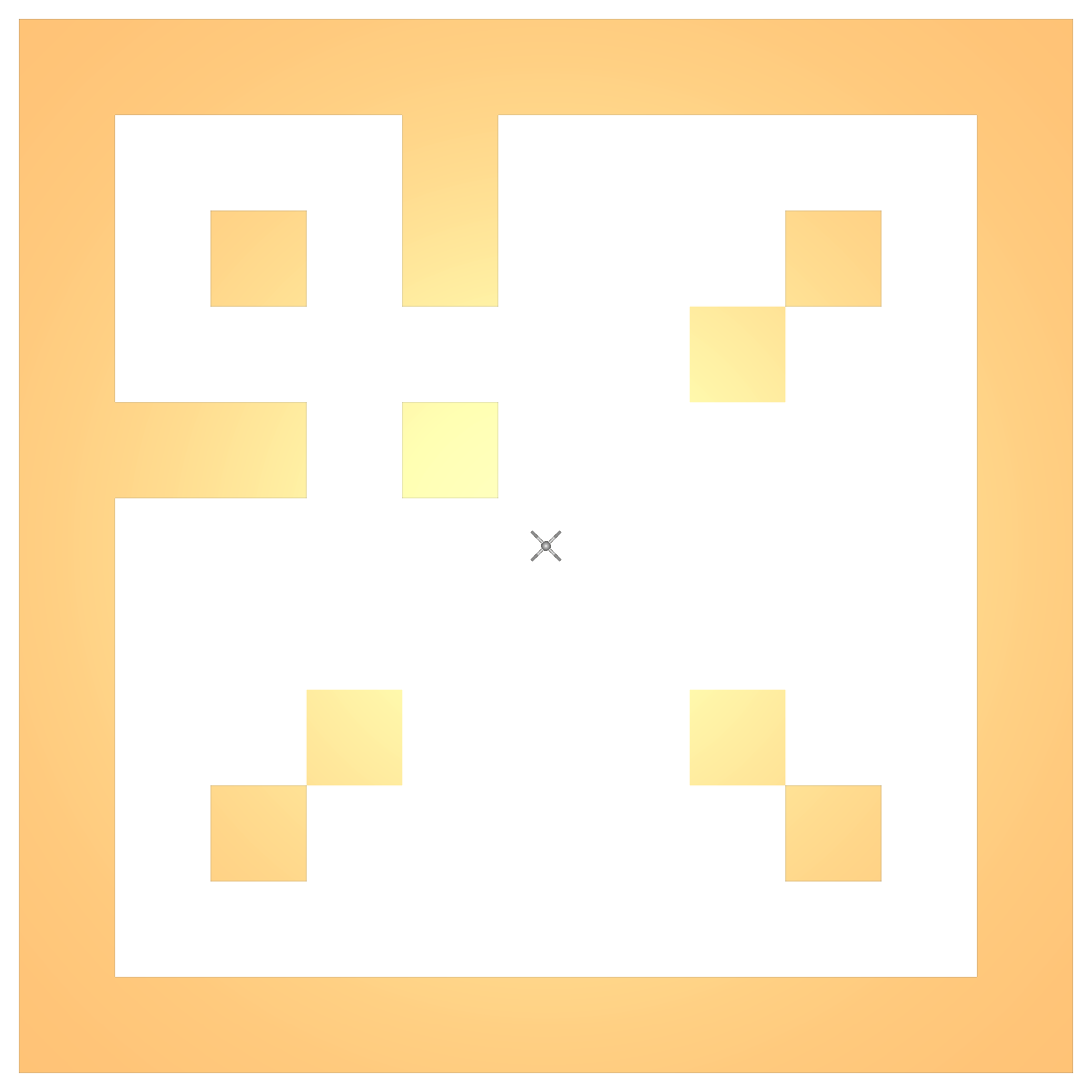}%
    \end{minipage}%
    \hfill%
    \begin{minipage}{\envpanelw}
    \includegraphics[width=\linewidth,keepaspectratio]{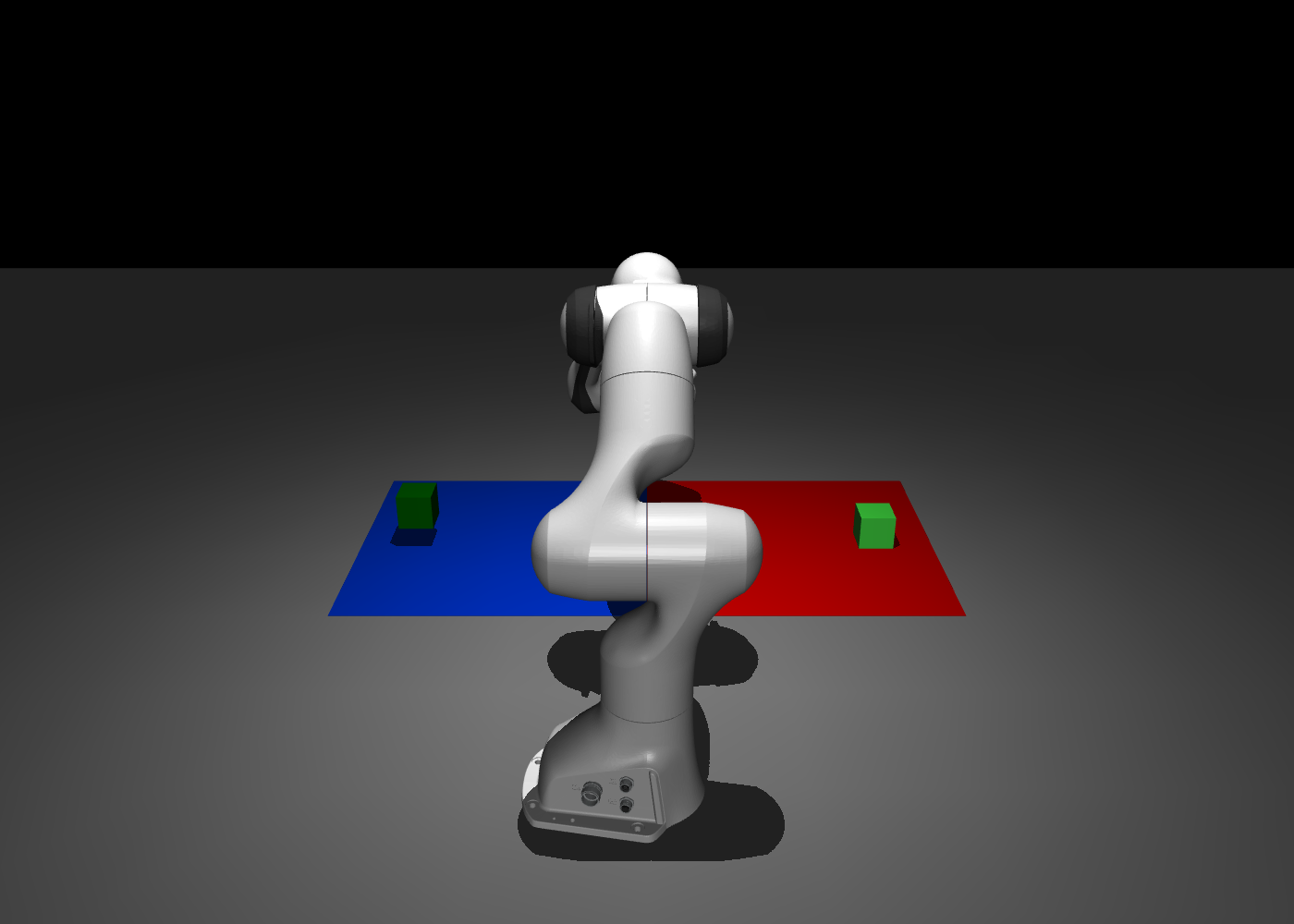}%
    \end{minipage}
    \caption{\label{fig:envs}\textbf{Environments.} Control tasks are standard benchmarks. \textit{Gridworlds are a novel contribution of this paper.} Blue tiles mark starting positions and green tiles mark terminal states. Tiles with a yellow \stt{?} have a 50\% chance of triggering a random movement; tiles with red arrows admit only the corresponding movement. In \stt{FourRoomStuck}, this creates a trap: the agent can accidentally enter the bottom-left room, where any action that does not match the red arrow leaves the agent in place.}
\vspace*{-1.4em}
\end{figure}

\section{Experiments}
\label{sec:experiments}

We select benchmarks that illustrate three challenges in RL exploration: (a) unreachable states or irreversible transitions, (b) large or unbounded goal spaces, and (c) hard exploration (hard-to-reach states, obstacles, mazes). 
Goal dimensionality ranges from two (most environments) to eight (\stt{LunarLander(Full)}).
Details are in Figure \ref{fig:envs} and Appendix~\ref{app:envs}.

\begin{itemize}[leftmargin=*, itemsep=0pt, topsep=-2pt]
\item \textbf{Gridworlds} (a, c): \stt{ThreeRoom} is a larger variant of Figure~\ref{fig:two_room_intro}; the agent spawns in a random isolated room, with higher spawn rate in the first two. In \stt{FourRoomStuck}, the bottom-left room cannot be exited or traversed freely, and some transitions are stochastic. \stt{GridMaze} features a narrow passage near the agent's starting state.
\item \textbf{Classic control}~\citep{towers2024gymnasium} (a, b, c): \stt{MountainCar} is a well-known hard-exploration benchmark. In \stt{CartPole}, episodes terminate quickly when the pole falls, making the corners of the state space hard to reach. \stt{LunarLander} has an unbounded state space (the agent can fly arbitrarily high). \stt{Acrobot} has unreachable configurations, and \stt{Pendulum} hard-to-reach ones.
\item \textbf{GCRL control}~\citep{bortkiewicz2025accelerating} (b, c): \stt{PointMaze-S/H} and \stt{AntMaze-S/H} are locomotion environments of increasing difficulty featuring corridors and dead-ends. \stt{ArmPush-H} is a manipulation task where a Franka Panda pushes a cube on an unbounded plane.  
\end{itemize}

% In Gridworlds, the goal space is the whole grid. In control tasks, it is a 2D positional subspace (e.g., the lander position in \stt{LunarLander}). The only exceptions are \stt{Pendulum} (3D), \stt{PointMaze-H} (4D) and \stt{Acrobot} (4D). %, where angular positions are encoded with their sine and cosine. 
% To evaluate exploration with higher-dimensional goal spaces, we include \stt{LunarLander(Full)} (8D).
% \\[2pt]
% % \clearpage
\textbf{Evaluations.} \textit{First}, we compare SUN against a random uniform exploration baseline, AdaGoal~\citep{tarbouriech2022adaptive} and DISCOVER~\citep{diaz-bone2025discover}. Not only are these two close to SUN, but they have achieved state-of-the-art performance and outperformed algorithms like MEGA \citep{pitis2020maximumentropy}. SUN, AdaGoal, and DISCOVER all use HER~\citep{andrychowicz2017hindsight} to learn their SVFs.  
\textit{Second}, we ablate SUN components, i.e., its indicator and goal-selection strategy. \textit{Finally}, we analyze why AdaGoal, DISCOVER, and non-adaptive goal-selection fail.
% More evaluations are in the Appendix. 
% Full details of the algorithms are in Appendix~\ref{app:hyper}.
% \\[2pt]
\begin{tcolorbox}[colback=gray!10, colframe=gray!40, boxrule=0.5pt, arc=2pt, left=6pt, right=6pt, top=3pt, bottom=3pt]
We evaluate two metrics over the goal space: \textbf{coverage} (fraction of space visited) and Shannon \textbf{entropy} of goal visits normalized to $[0, 1]$. We also report heatmaps for a qualitative analysis. In continuous spaces, we discretize them into 50 bins per dimension to compute these metrics.
In \stt{LunarLander(Full)}, where the goal is eight-dimensional, we report the entropy approximated via Kozachenko-Leonenko $k$-NN (Appendix \ref{app:envs}).
All results are averaged across \textbf{20 seeds}.
% Plots show means (thick line) with 95\% confidence intervals (shaded area) across \textbf{20 seeds}; heatmaps show log-scale counts over the first seed only.
\end{tcolorbox}

\begin{figure}[!t]
    \vspace*{-6pt}
    % \centering
    \includegraphics[trim={0 2.0em 0 0}, clip, width=\linewidth]{\detokenize{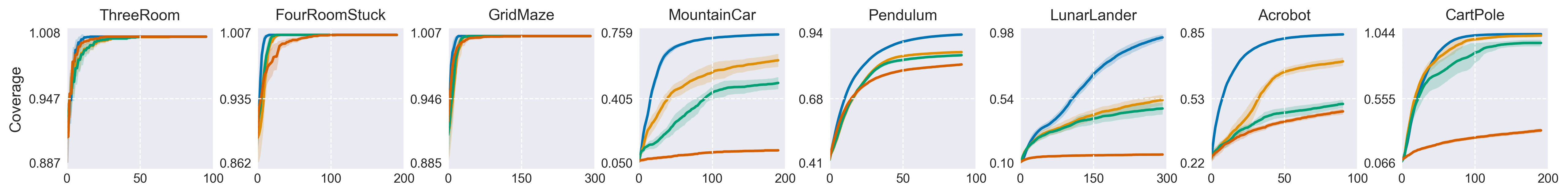}}
    \\[1pt]
    \includegraphics[trim={0 0 0 2.2em}, clip, width=\linewidth]{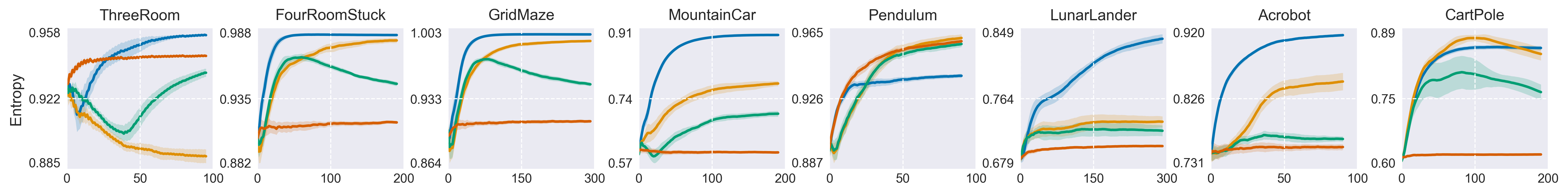}
    \\[-2pt]
    \begin{minipage}[b]{0.63\linewidth}
    \includegraphics[trim={0 2.0em 0 0}, clip, width=\linewidth]{\detokenize{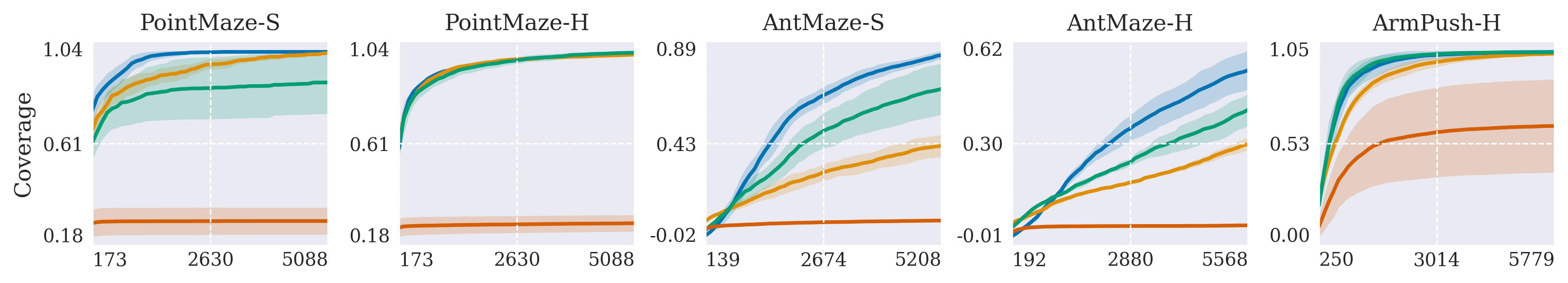}}%
    \\[1pt]
    \includegraphics[trim={0 0 0 2.2em}, clip, width=\linewidth]{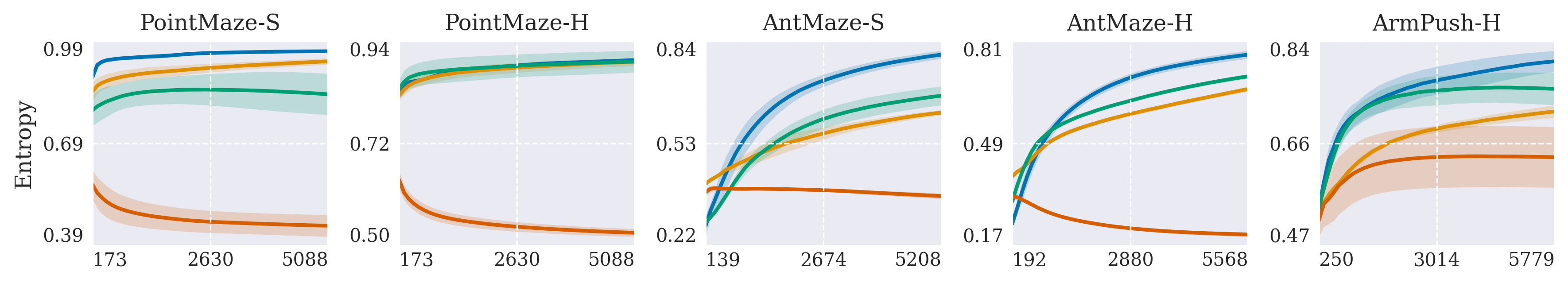}%
    \end{minipage}
    \hfill
    \begin{minipage}[b]{0.36\linewidth}
    \centering
    \hfill
    \includegraphics[trim={0em 0.1em 0 0.1em}, clip, width=0.8\linewidth]{\detokenize{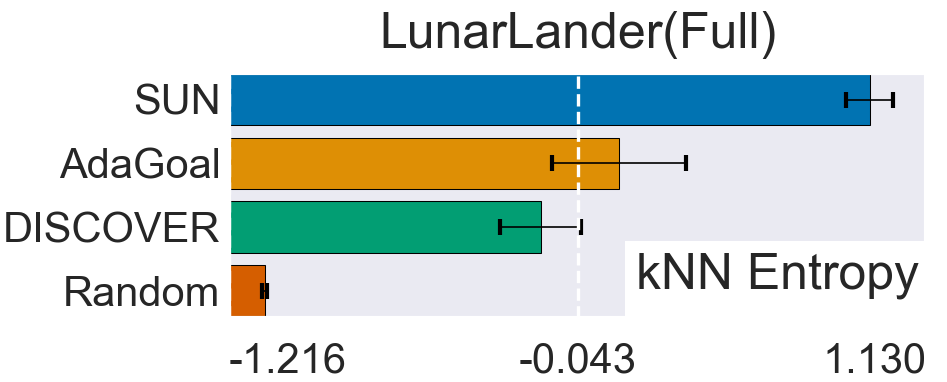}}\hfill
    \hfill
    \captionsetup{skip=-2pt}
    \subcaption*{\tiny Means (thick lines) with 95\% confidence intervals (shaded areas) across {20 seeds}; the $x$-axis shows training steps (thousands). \stt{LunarLander(Full)} reports kNN entropy at the end of training (1M steps).}
    \end{minipage}
    \captionsetup{skip=4pt}
    \caption{\label{fig:main_results}\textbf{Main results}. \textit{SUN outperforms all baselines in all environments.} Note that full coverage is easy in Gridworlds, but uniform exploration (high entropy) is not.}
    \vspace*{-7pt}
\end{figure}

\subsection{SUN vs Baselines}
\label{subsec:main_results}
% \textbf{Quantitative results.}
Figure~\ref{fig:main_results} shows that the two metrics are complementary: coverage measures whether a state has been visited at least once, while entropy measures how uniformly visits are distributed. A method can saturate coverage while still concentrating most of its visits in small regions of the state space, which is what happens in Gridworlds: all methods reach perfect coverage, yet their entropy values differ substantially, with SUN at the top. 
% \\[1pt]
Indeed, Gridworlds are challenging due to their non-uniform initial-state distributions and unreachable goals (\stt{ThreeRoom}), stochastic transitions and irreversible traps (\stt{FourRoomStuck}), and bottlenecks that must be traversed to reach the rest of the grid (\stt{GridMaze}). Coverage alone hides these difficulties; entropy reveals them.
\\[1pt]
This is confirmed by Figure~\ref{fig:visitation_heatmaps}, showing the visit distributions at the end of training (black regions are unvisited or unreachable). SUN always covers a broader portion of the goal space and produces smoother visit patterns, unlike AdaGoal and DISCOVER. E.g., in \stt{FourRoomStuck} their visits are non-uniform inside the bottom-left room; in \stt{GridMaze} they cluster near the bottom-left corner; in control environments they fail to explore far from the initial state.
This confirms the entropy ranking of Figure~\ref{fig:main_results}: methods can cover many goals while still clustering their visits.
\\[1pt]
% , as reflected by the larger gap in the area-under-the-curve bars. 
Continuous-control environments strengthen SUN's advantage: its curves rise faster and plateau only at (near-)full coverage. One exception is \stt{Pendulum}, the only case where SUN's entropy decreases. The reason is structural: reaching some pendulum configurations requires accumulating momentum through swings, thus the agent must revisit the same low-momentum positions over and over. 
Rare states are therefore rare precisely because they can only be reached by re-traversing common ones many times; visiting the tail of the distribution comes at the cost of re-visiting the mode.
% This explains why SUN entropy decreases even though its coverage keeps increasing. 
% \\[1pt]
Importantly, SUN also attains the best performance on \stt{LunarLander(Full)}, despite its higher dimensionality. 
This supports the proposed pseudocount as an effective tool for estimating visits and encoding novelty.
% \\[1pt]
% The heatmaps also display some of the challenges of the control tasks discussed earlier. In \stt{Pendulum}, states corresponding to the pendulum's downright position are the most visited by SUN, though at different velocities. In \stt{CartPole}, goals near the start state (the middle cart position) dominate the visitation of all methods, and only SUN reaches the edges of the goal space.

\begin{figure}[!t]
    \centering
    \includegraphics[trim={0 0 5pt 0}, clip, width=0.617\linewidth]{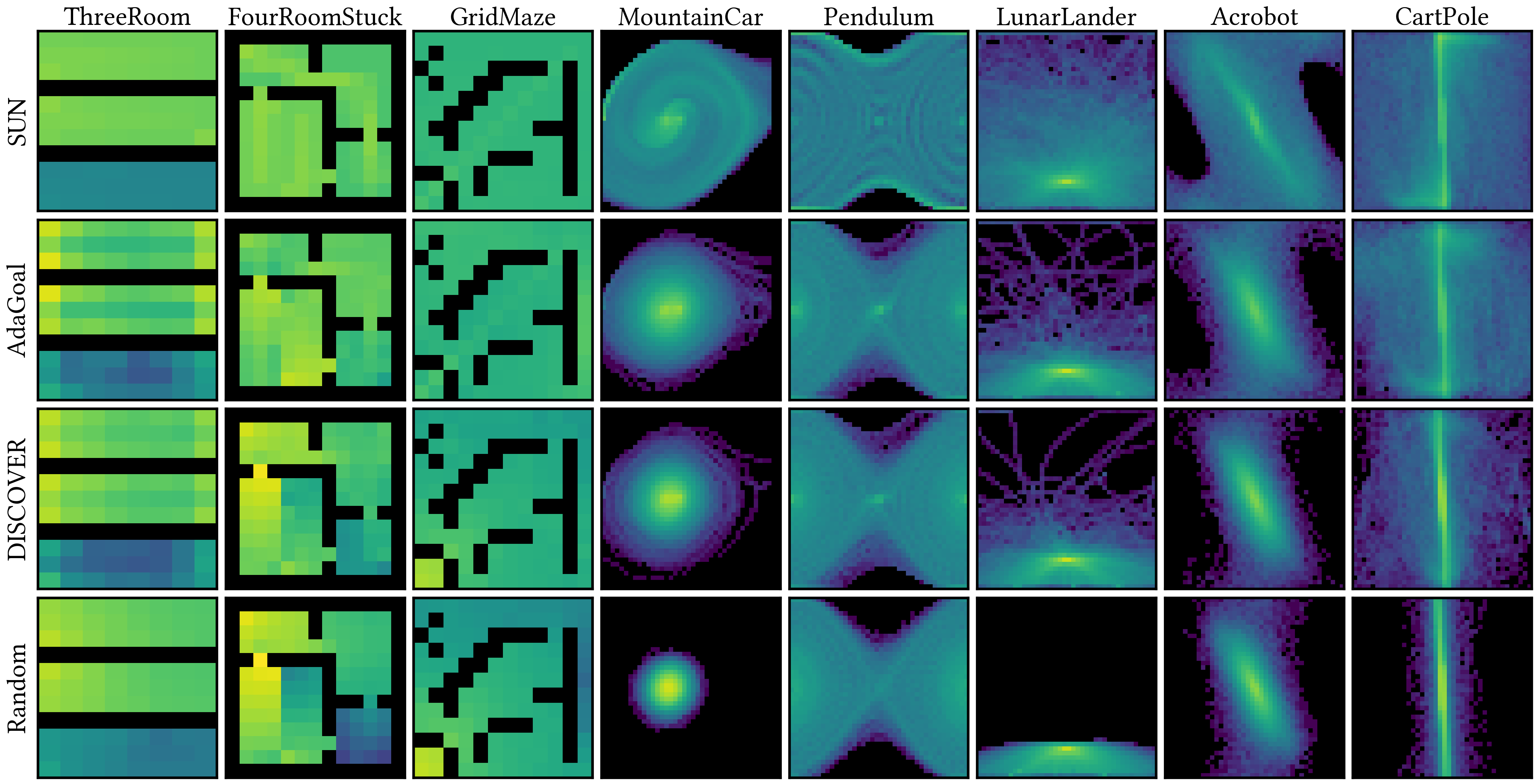}%
    \hfill
    \includegraphics[trim={14pt 0 0 0}, clip, width=0.382\linewidth]{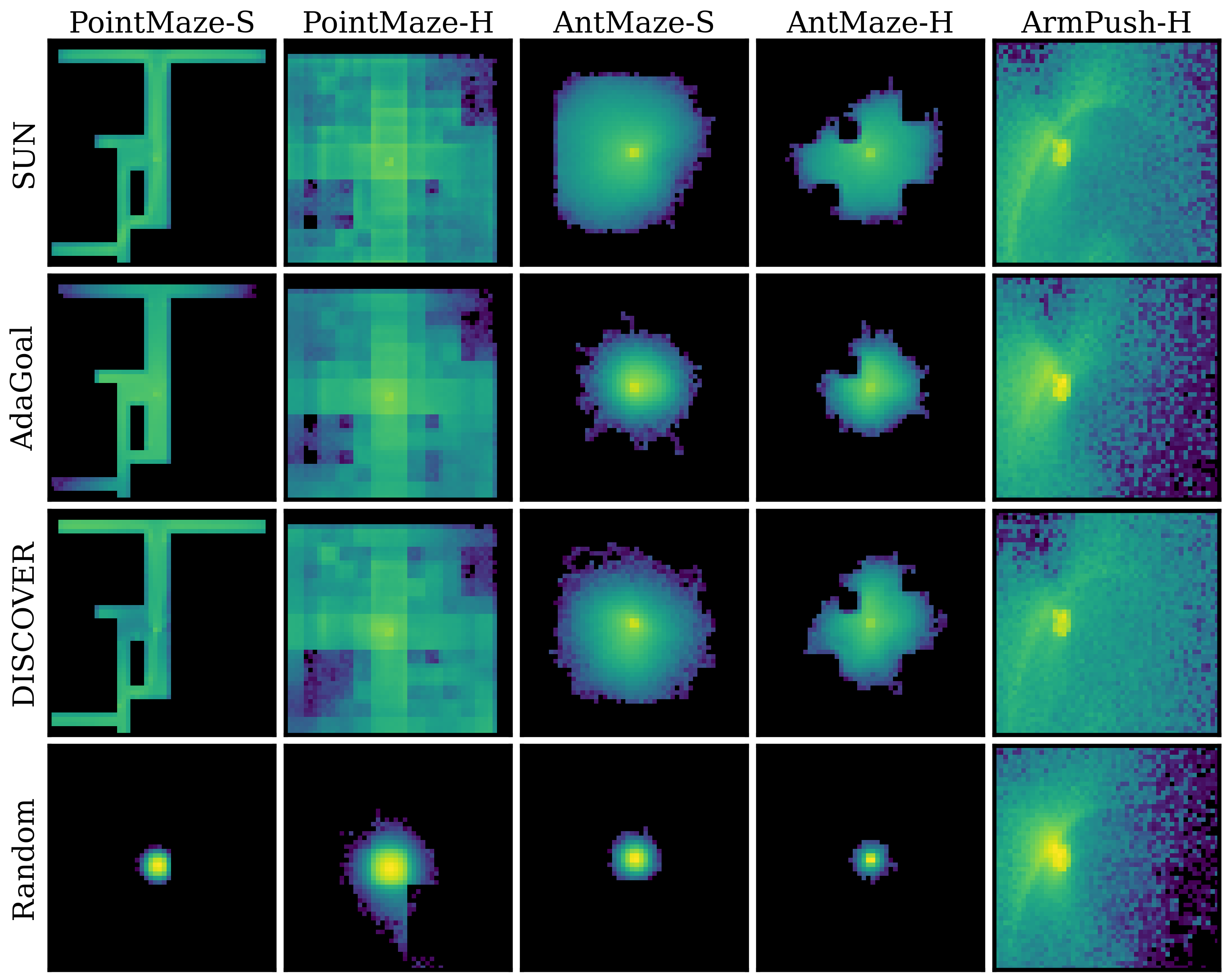}
    \captionsetup{skip=4pt}
    \caption{\label{fig:visitation_heatmaps}\textbf{Log-scale visitation heatmaps} of positional goals at the end of training (first seed). Details about positional goals are in Appendix \ref{app:envs}. 
    % \stt{PointMaze-H} and \stt{ArmPush-H} goal spaces are actually larger (Appendix \ref{app:envs}), and the heatmaps show only planar positions. 
    % \stt{LunarLander(Full)} is not included. 
    \textit{SUN is the only algorithm not clustering its visits, except in the regions corresponding to the episode's starting state (that are naturally visited more).}}
    \vspace*{-8pt}
\end{figure}

\subsection{SUN Ablations}
\label{subsec:ablation_results}
SUN has three components: the multiplicative indicator, the adaptive goal-selection, and the pseudocounts. 
Figure~\ref{fig:ablation_score} ablates indicators, and reinforces the main motivation of this paper: \textit{exploration is ineffective when guided by reachability or novelty alone; both must be considered.} 
Figure~\ref{fig:ablation_selection} ablates goal-selection and validates the importance of the adaptive strategy, highlighting the shortcomings of episodic\footnote{More specifically, we implement DISCOVER's episodic goal-selection, where the agent enters a ``random exploration phase'' if the goal is reached before the end of the episode~\citep{diaz-bone2025discover}. The rationale is that a goal selected for its high novelty is itself a point of interest, so nearby states are likely to be novel as well.} and per-step.\footnote{In \emph{per-step} selection, a fresh batch of candidates is sampled at every step, and the goal changes whenever the batch contains a candidate with a higher SUN score (Eq.~\eqref{eq:sun-score}) than the current one.}
% \\[1pt]
All plots are in Appendix \ref{app:ablation_full} and \ref{app:extra_details}. We ablate pseudocounts in Appendix \ref{app:pseudocount_ablation}.
Figure~\ref{fig:unreach_vs_perfect} confirms that SUN rejects unreachable goals. 
Averaged over the two Gridworlds, where the reachable set is known exactly, all SUN variants select unreachable goals significantly fewer times than the baselines. This is the empirical counterpart of Theorem~\ref{thm:unreachable}: the guarantee is exact only for the true $V^\pi$, and with an approximate $\Vtheta$ the rejection is imperfect, yet \textit{the reachability factor still cuts the rate by roughly eight times with respect to novelty alone}. AdaGoal and DISCOVER both sit well above every SUN variant, confirming that ensemble disagreement is not an effective reachability proxy. Among SUN variants, {Episodic} has the highest rate: committing to a goal for the whole episode means an unreachable one cannot be discarded (see Appendix \ref{app:extra_details} for more details).

\begin{figure}[!t]
% \vspace*{-5pt}
\begin{minipage}[c]{0.4\textwidth}
\centering
\includegraphics[width=\linewidth]{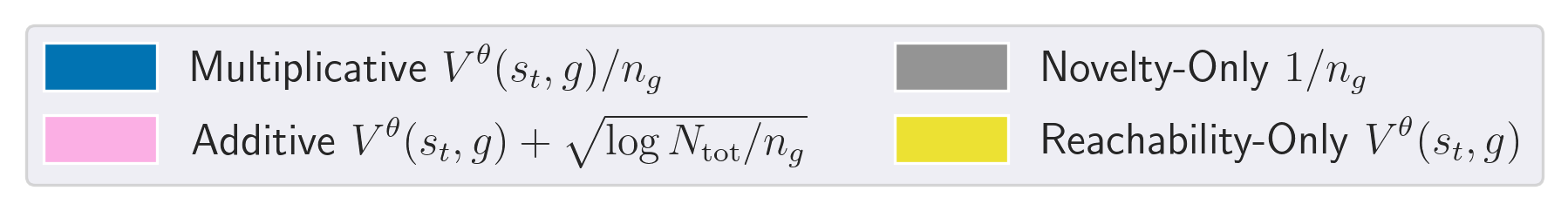}
\\[-2pt]
\includegraphics[width=\linewidth]{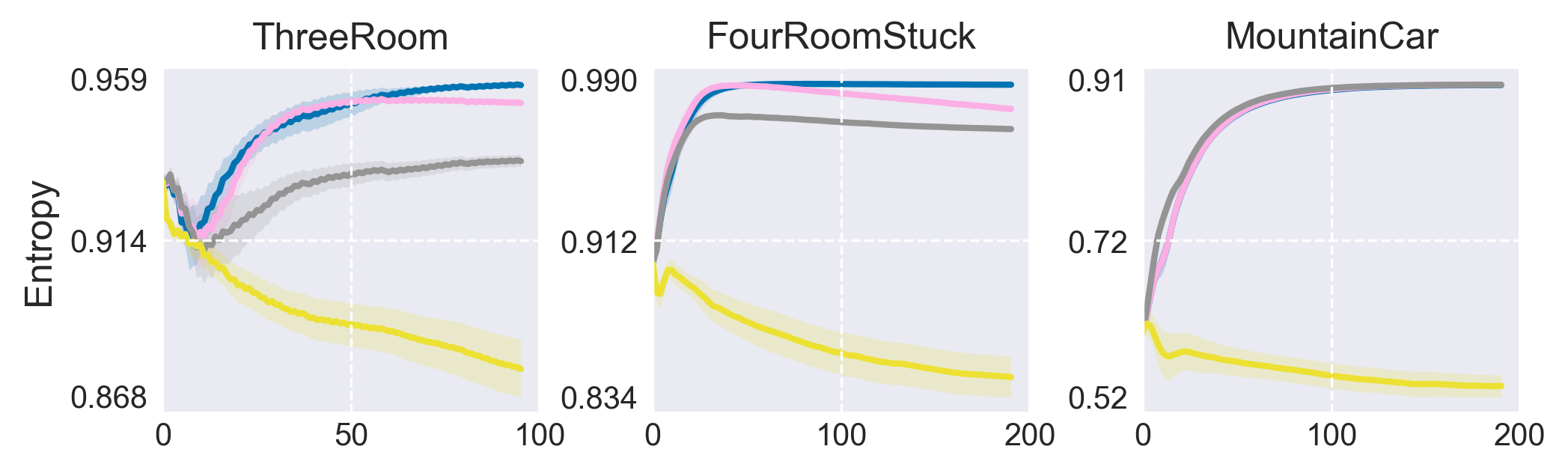}
\end{minipage}%
\hfill
\begin{minipage}[c]{0.59\textwidth}
\caption{\label{fig:ablation_score}\textbf{Indicator ablation}. Driven by \textit{reachability only}, the agent cannot explore (the most reachable state is the current one). \textit{Novelty-only} suffices if goals are always reachable (\stt{MountainCar}), but fails otherwise. The \textit{additive indicator} saturates: as counts increase, the novelty bonus shrinks and reachability dominates the selection.}
\end{minipage}
\\[-7pt]
% \end{figure}
% \begin{figure}[!b]
\begin{minipage}[c]{0.4\textwidth}
\centering
\includegraphics[trim={0 2em 0 0}, clip, width=\linewidth]{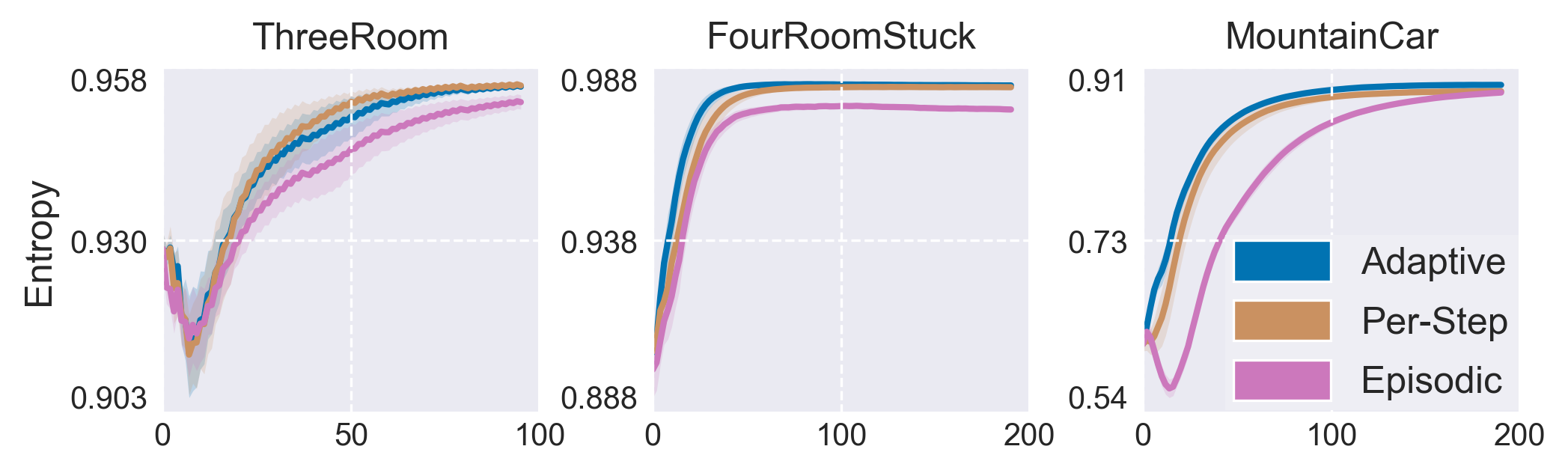}
\\[1pt]
\includegraphics[trim={0 0 0 2.em}, clip, width=\linewidth]{\detokenize{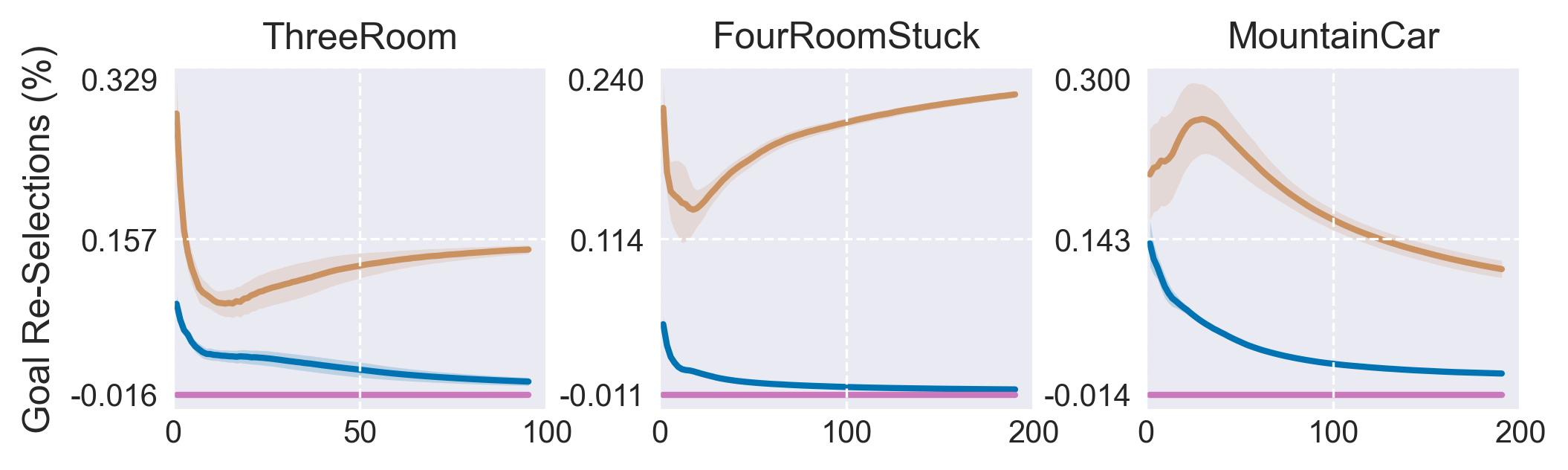}}
\end{minipage}\hfill
\begin{minipage}[c]{0.59\textwidth}
\centering
\caption{\label{fig:ablation_selection}\textbf{Goal-selection ablation}. \textit{Episodic} selection fails when goals become unreachable mid-episode (as in \stt{FourRoomStuck}). \textit{Per-step} selection avoids this failure mode, but its reselection rate never drops to zero, undermining overall performance. In contrast, the reselection rate of \textit{adaptive} selection drops to zero, indicating that the agent has learned to reach the goals it selects.}
\end{minipage}
\\[-6pt]
% \end{figure}
% \begin{figure}[!b]
\begin{minipage}[c]{0.34\textwidth}    
    \centering
    \includegraphics[width=\linewidth]{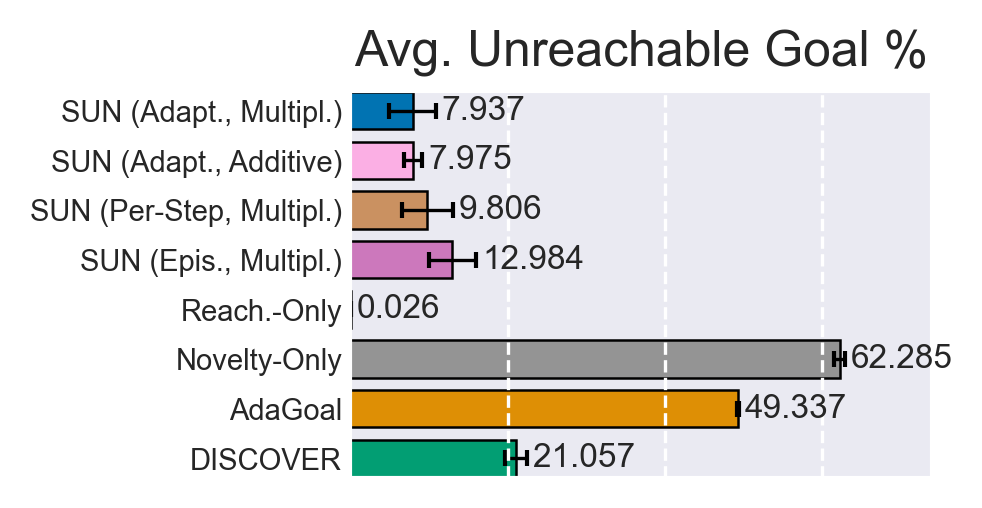}%
\end{minipage}%
\hfill
\begin{minipage}[c]{0.655\textwidth}
\caption{\label{fig:unreach_vs_perfect}\textbf{Selection rates of unreachable goals} averaged over \stt{ThreeRoom} and \stt{FourRoomStuck}. All \textit{SUN variants} select few of them, and only early in training (see Appendix~\ref{app:extra_details}). \textit{Novelty-only} and \textit{AdaGoal} select them throughout, since neither has an effective reachability proxy (unlike \textit{DISCOVER}). \textit{Reachability-only} never selects any, but only because it simply targets the current state.}
\end{minipage}
% \vspace*{-10pt}
\end{figure}

\clearpage

\subsection{Why AdaGoal and DISCOVER Fail}
\label{subsec:sun_better}
The results so far are clear: SUN attains the best entropy and coverage, and the ablations validate the importance of its components. 
Yet another result stands out (Figure \ref{fig:iqm_curves} and Table \ref{tab:iqm_poi}): \textit{all SUN versions are better than AdaGoal and DISCOVER} --- strong evidence of SUN's indicator importance. 
Figure~\ref{fig:unreach_vs_perfect} gives part of the answer, but is limited to Gridworlds, where the reachable set is known. Figure~\ref{fig:failure_analysis} therefore shows the algorithms' goals selected over training in \stt{MountainCar}.
\textit{DISCOVER} (last row) is biased toward reachability: it mostly selects goals close to the starting state (the center of the heatmap). This may stem from its coefficient $\beta$, which must balance the scales of the reachability and novelty terms, an issue the multiplicative indicator does not face.
\textit{AdaGoal} (fifth row), conversely, is biased toward novelty: it selects goals far from the starting state, ``skipping'' intermediate ones. This is a sign that reachability is ignored, and it is not unexpected. As noted in Section~\ref{subsec:summary}, we evaluate the deep-RL variant of AdaGoal, which --- unlike the tabular version --- does not constrain goal-selection by the estimated goal-hitting time.
\citeauthor{tarbouriech2022adaptive} argue that this is approximated implicitly by the disagreement of the critic ensemble. Our results suggest that this does not always hold. Novelty alone is a reasonable signal for goals the agent can actually reach, since visiting them resolves the disagreement at little cost. Unreachable goals, however, resolve only after enough failed attempts for

% \begin{figure}[!t]
%     \begin{minipage}[t]{0.62\linewidth}
%     \vspace{0pt}
%     \begin{minipage}[t]{0.33\linewidth}\vspace{0pt}\includegraphics[width=\linewidth]{\detokenize{plots/rliable/sa_pct_all_tasks_curves_iqm.png}}\end{minipage}%
%     \begin{minipage}[t]{0.33\linewidth}\vspace{0pt}\includegraphics[width=\linewidth]{\detokenize{plots/rliable/sa_h_all_tasks_curves_iqm.png}}\end{minipage}%
%     \begin{minipage}[t]{0.30\linewidth}\vspace{0pt}\includegraphics[width=\linewidth]{\detokenize{plots/rliable/legend_vertical.png}}\end{minipage}
%     \end{minipage}\hfill
%     \begin{minipage}[t]{0.38\linewidth}
%     \vspace{0pt}
%     \caption{\label{fig:iqm_curves}\textbf{Interquartile mean (IQM)} of coverage (left) and entropy (right) over all runs of all environments but \stt{LunarLander(Full)}, computed at 20 evenly spaced points\parfillskip=0pt}
%     \end{minipage}\par\vspace*{-5pt}
%     \noindent points of training. The x-axis is the fraction of training, since environments are trained for different numbers of steps, and both statistics lie in $[0, 1]$, so no further normalization is needed. Shaded areas are 95\% stratified-bootstrap confidence intervals (50{,}000 resamples of the seeds within each environment). \textit{All SUN versions explore faster than the baselines and reach higher coverage and entropy throughout training.}
% \end{figure}

\begin{figure}[!t]
    \includegraphics[width=\linewidth]{\detokenize{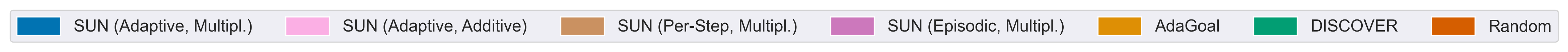}}
    \\[-1.2em]
    \begin{minipage}[t]{0.53\linewidth}\vspace{3pt}
    \includegraphics[width=0.49\linewidth]{\detokenize{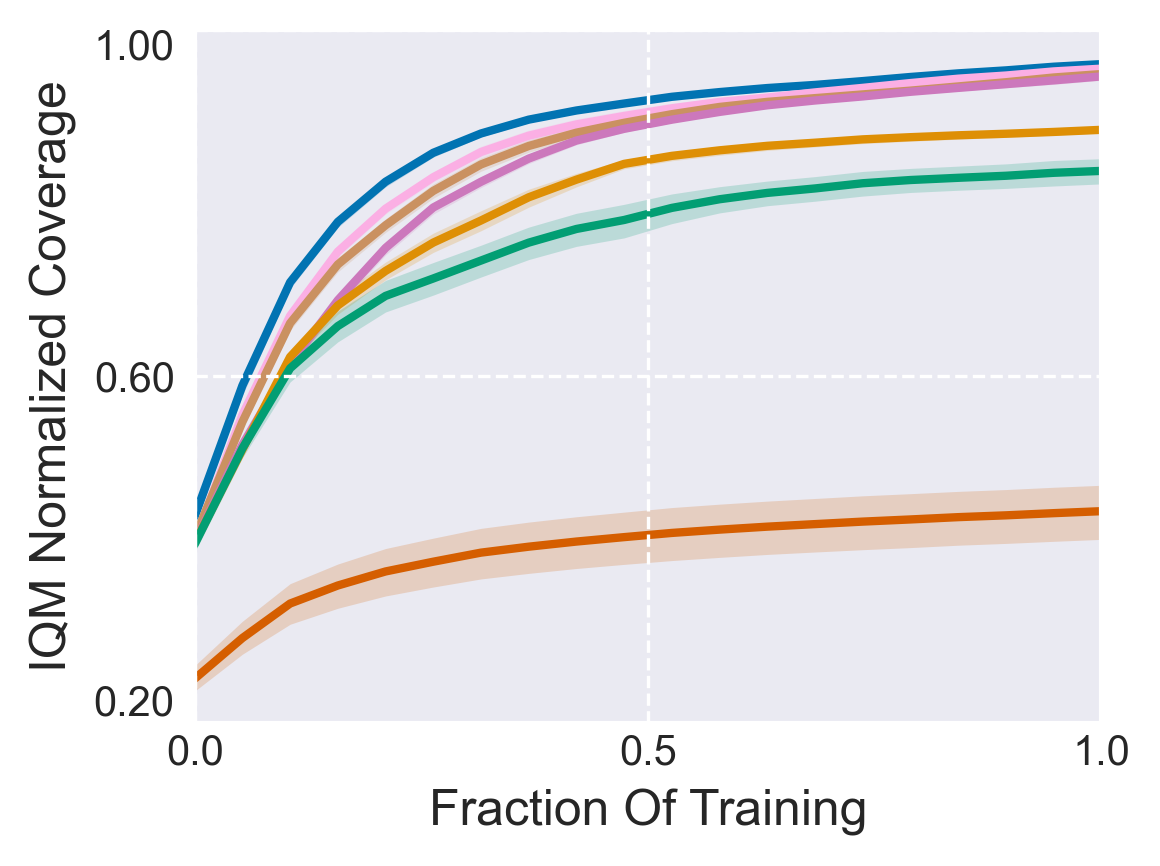}}
    \includegraphics[width=0.49\linewidth]{\detokenize{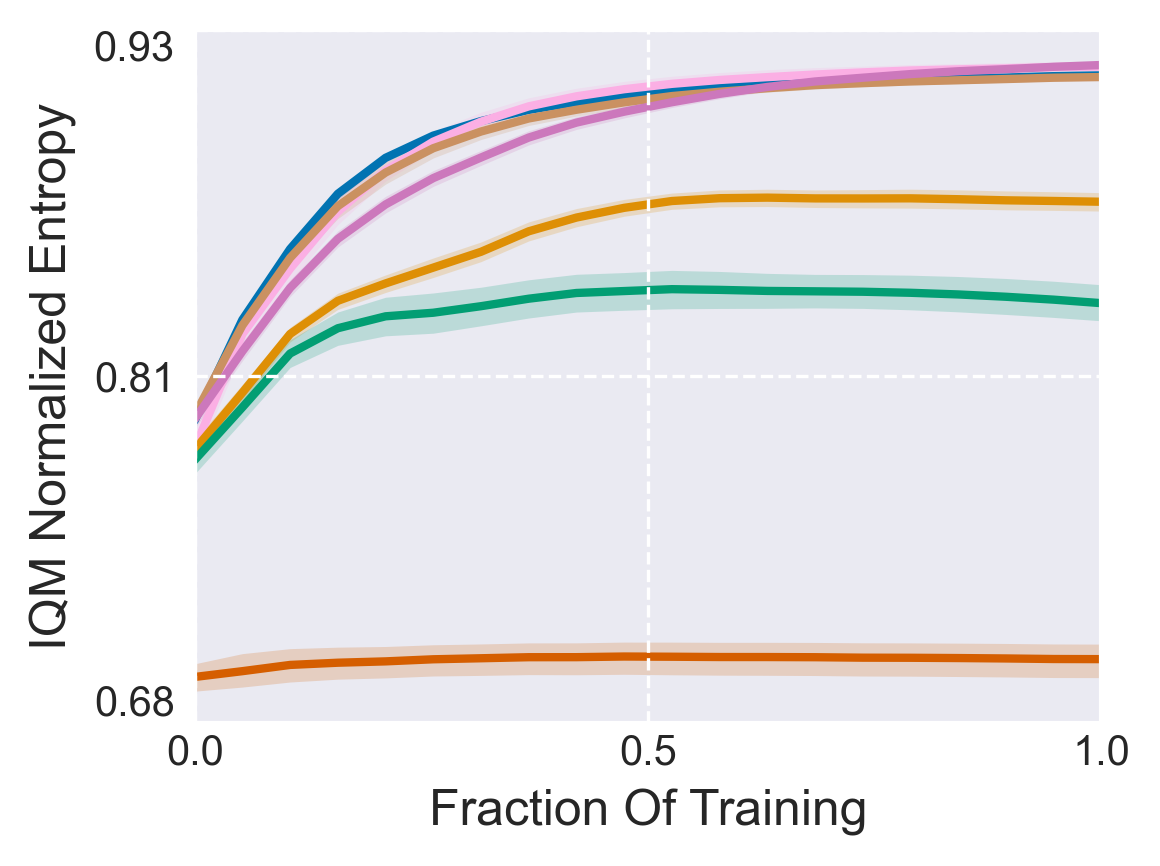}}
    \end{minipage}\hfill
    \begin{minipage}[t]{0.46\linewidth}\vspace{0pt}
    \caption{\label{fig:iqm_curves}\textbf{Interquartile mean (IQM)} of coverage (left) and entropy (right) over all runs of all environments but \stt{LunarLander(Full)}, computed at 20 evenly spaced points of training. The x-axis is the fraction of training, since environments are trained for different numbers of steps; both statistics lie in $[0,1]$, so no\parfillskip=0pt}
    \end{minipage}\par\vspace*{-5pt}
    \noindent further normalization is needed. Shaded areas are 95\% percentile intervals from a stratified bootstrap (50{,}000 resamples of the seeds within each environment). \textit{All SUN versions explore faster than the baselines and reach higher coverage and entropy throughout training.}
\end{figure}

\begin{table}[!t]
    \vspace*{-1em}
    \centering
    \caption{\label{tab:iqm_poi}
    \textbf{Aggregate performance over all environments} but \stt{LunarLander(Full)}, following the protocol of \citet{agarwal2021deep}.
    For each run, we take the area under its coverage and entropy curves, normalized by the training length (both statistics lie in $[0,1]$, so no further normalization is needed).
    \textbf{IQM}: interquartile mean over all runs of all environments.
    \textbf{PoI vs.\ X}: probability of improvement, i.e., the probability that a run of the row method scores higher than a run of baseline X in the same environment, averaged over environments (0.5 means no difference).
    Brackets are 95\% percentile intervals the same stratified bootstrap above.
    \textbf{Bold}: the whole interval lies above 0.5, i.e., the row method outperforms the baseline (note that intervals are narrow because each score is an area under a curve, which averages out most seed-to-seed variation).
    \textit{All SUN versions outperform all baselines on both metrics.}}

    \resizebox{\linewidth}{!}{%
    \begin{tabular}{lcccc}
    \toprule
     & IQM & PoI vs.\ AdaGoal & PoI vs.\ DISCOVER & PoI vs.\ Random \\
    \midrule

    \multicolumn{5}{c}{\textbf{Coverage}} \\
    \midrule

    SUN (Adaptive, Multipl.) & 0.860 [0.858, 0.863] & \textbf{0.88 [0.84, 0.93]} & \textbf{0.85 [0.79, 0.90]} & \textbf{0.93 [0.90, 0.96]} \\
    SUN (Adaptive, Additive) & 0.841 [0.836, 0.846] & \textbf{0.73 [0.66, 0.79]} & \textbf{0.69 [0.64, 0.74]} & \textbf{0.88 [0.84, 0.93]} \\
    SUN (Per-Step, Multipl.) & 0.835 [0.830, 0.839] & \textbf{0.72 [0.66, 0.78]} & \textbf{0.71 [0.65, 0.77]} & \textbf{0.93 [0.90, 0.96]} \\
    SUN (Episodic, Multipl.) & 0.825 [0.821, 0.828] & \textbf{0.66 [0.59, 0.72]} & \textbf{0.63 [0.58, 0.69]} & \textbf{0.91 [0.87, 0.94]} \\
    \midrule
    AdaGoal & 0.787 [0.782, 0.791] & --- & 0.51 [0.44, 0.57] & \textbf{0.86 [0.83, 0.90]} \\
    DISCOVER & 0.742 [0.728, 0.755] & 0.49 [0.43, 0.56] & --- & \textbf{0.88 [0.85, 0.91]} \\
    Random & 0.402 [0.373, 0.428] & 0.14 [0.10, 0.17] & 0.12 [0.09, 0.15] & --- \\

    \midrule
    \multicolumn{5}{c}{\textbf{Entropy}} \\
    \midrule

    SUN (Adaptive, Multipl.) & 0.8907 [0.8898, 0.8915] & \textbf{0.84 [0.80, 0.88]} & \textbf{0.86 [0.82, 0.90]} & \textbf{0.90 [0.88, 0.92]} \\
    SUN (Adaptive, Additive) & 0.8893 [0.8861, 0.8931] & \textbf{0.69 [0.64, 0.75]} & \textbf{0.73 [0.68, 0.79]} & \textbf{0.84 [0.80, 0.89]} \\
    SUN (Per-Step, Multipl.) & 0.8878 [0.8856, 0.8898] & \textbf{0.76 [0.72, 0.81]} & \textbf{0.80 [0.74, 0.85]} & \textbf{0.90 [0.88, 0.92]} \\
    SUN (Episodic, Multipl.) & 0.8817 [0.8799, 0.8833] & \textbf{0.74 [0.69, 0.80]} & \textbf{0.78 [0.72, 0.84]} & \textbf{0.89 [0.87, 0.92]} \\
    \midrule
    AdaGoal & 0.8506 [0.8480, 0.8531] & --- & \textbf{0.60 [0.55, 0.65]} & \textbf{0.82 [0.80, 0.85]} \\
    DISCOVER & 0.8271 [0.8215, 0.8327] & 0.40 [0.35, 0.45] & --- & \textbf{0.83 [0.81, 0.85]} \\
    Random & 0.7051 [0.6994, 0.7098] & 0.18 [0.15, 0.20] & 0.17 [0.15, 0.19] & --- \\

    \bottomrule
    \end{tabular}%
    }
\end{table}

\clearpage

\begin{wrapfigure}[11]{l}{0.47\linewidth}
    \centering
    \includegraphics[trim={0 0 0 2pt}, clip, width=\linewidth]{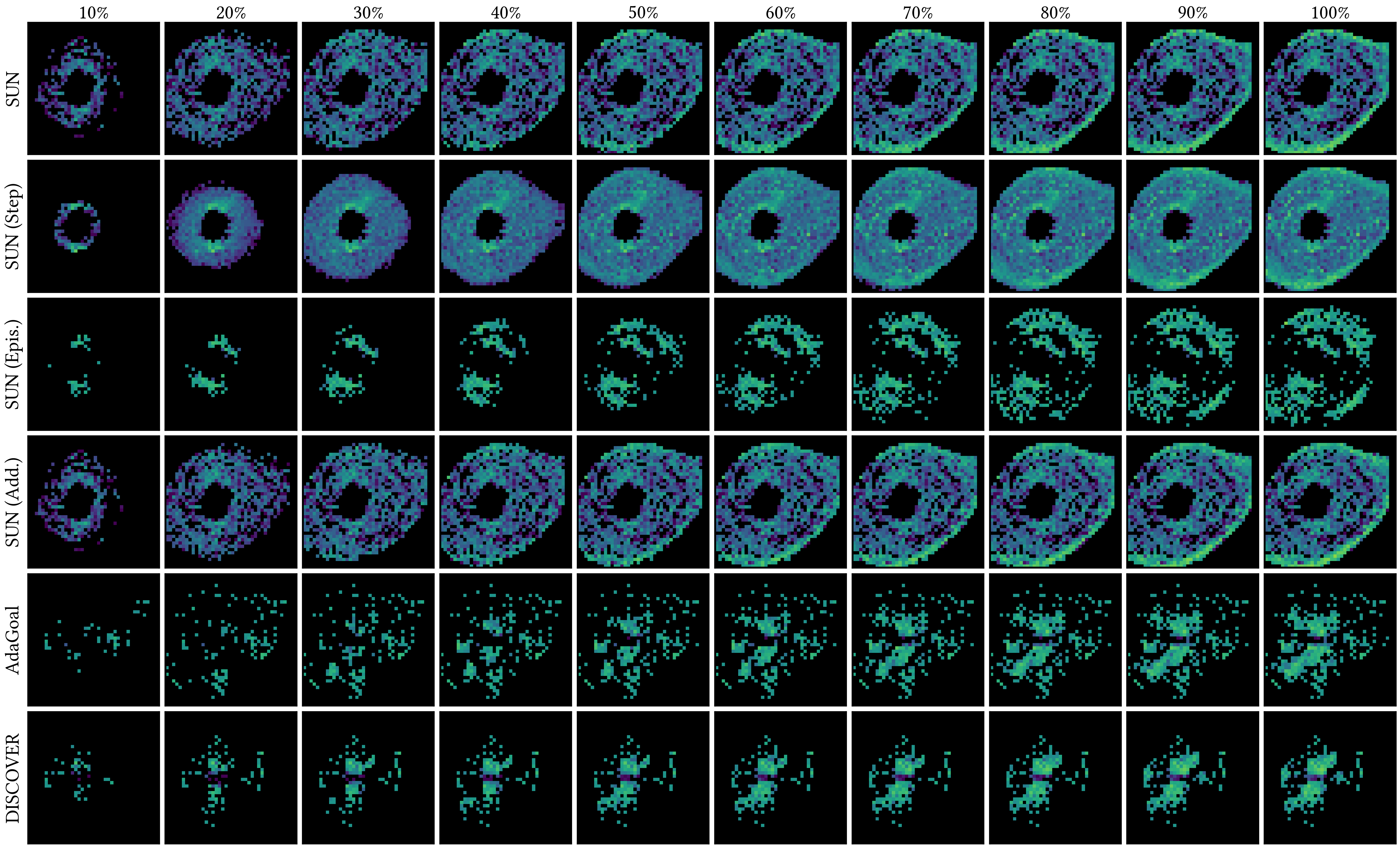}
    \captionsetup{skip=3pt}
    \caption{\label{fig:failure_analysis}\textbf{Goals selected} over time in \stt{M.Car}.}
\end{wrapfigure}
every ensemble member to recognize them as such, and each of those attempts is an episode without novel data. 
\textit{All SUN versions} (first four rows) instead \textit{progressively select goals that are further away}. The contrast is clearest between SUN Episodic and AdaGoal: both select once per episode, yet AdaGoal only aims immediately to distant hard-to-reach states. The progression is therefore intrinsic to the indicator rather than a product of reselection, and it is what balances reachability against novelty.

\vspace*{7pt}

\begin{tcolorbox}[colback=gray!10, colframe=gray!40, boxrule=0.5pt, arc=2pt, left=6pt, right=6pt, top=3pt, bottom=3pt]
The results in this section validate our claims. (1) SUN attains the best entropy and coverage
across all environments, visiting the goal space significantly more uniformly than every baseline.
(2) Its pseudocount effectively approximates novelty even in high-dimensional goal spaces. (3)
Its multiplicative indicator does not vanish and selects goals that are both reachable and novel,
and its adaptive goal-selection outperforms the baselines’ episodic strategy.
\end{tcolorbox}

\section{Discussion}
\label{sec:discussion}
% In this paper, we tackled the challenge of exploration in RL via goal-conditioned policies, focusing on the selection of goals that are simultaneously novel and reachable. We argued that existing methods cannot fully capture the tension between the two and end up neglecting one or the other. We thus proposed SUN, a unified and principled indicator and goal-selection rule grounded in SVFs and counts. 
% We further introduced a lightweight pseudocount that scales to per-step goal-selection in continuous spaces. Finally, we validated SUN on standard and new benchmarks, where it consistently outperforms state-of-the-art methods AdaGoal and DISCOVER.
In this paper, we tackled exploration in RL via GCRL, by selecting goals that are simultaneously novel and reachable. We argued that existing methods cannot fully capture the tension between the two and end up neglecting one or the other. We thus proposed SUN, a unified and principled indicator and goal-selection rule grounded in SVFs and counts, and a lightweight pseudocount that scales to per-step goal-selection in continuous spaces. Then, we validated SUN on standard and new benchmarks, where it consistently outperforms state-of-the-art methods AdaGoal and DISCOVER.
\\[1pt]
\textbf{Strengths.} The strength of SUN lies in its principled design. Its score admits a natural interpretation as the value of a goal-reaching objective weighted by the goal's inverse count, driving the agent toward rarely-visited states \textit{within reach}. This is made possible by our lightweight pseudocount, which avoids the overhead of classical pseudocount and density-based approaches while preserving their accuracy. Our results confirm this: AdaGoal and DISCOVER, which instead score goals by critic-ensemble disagreement, fail to balance the two signals and collapse toward one or the other.
\\[1pt]
\textbf{Limitations and Future Work.}
\textit{First}, although principled, count-based novelty is ineffective in large spaces, such as image observations, \textit{regardless of their pseudocount approximations}. Since most observed states are effectively unique, counts become nearly uniform across the buffer and do not provide a useful novelty signal. In such regimes, exploration likely requires richer signals, such as learned curiosity~\citep{raileanu2020ride, parisi2021interesting} or mutual information between trajectories and learned latents~\citep{eysenbach2019diversity, sharma2020dynamics}. Combining these signals with SUN's reachability factor is a promising direction.
\\[1pt]
\textit{Second}, SUN score does not account for the states traversed to reach a goal. A longer path through many novel states may be preferable to a shorter path to a marginally rarer goal
% (see Appendix \ref{app:pendulum}) 
--- a distinction the current formulation cannot make. Extending SUN indicator to a cumulative form would more directly target a maximum-entropy state distribution~\citep{hazan2019provably,mutti2021task}. 
% We see this as a natural next step.
\\[1pt]
\textit{Third}, the SVF is learned on the binary reward of Eq.~\eqref{eq:w-function}, which is sparse by construction and needs hindsight relabeling. 
Yet the state occupancy that Eq.~\eqref{eq:w-function} defines can be estimated directly by contrastive learning without relabeling~\citep{eysenbach2022contrastive,jiang2025episodic,mohamed2026temporal}. \citet{bae2024tldr} further use a quasimetric variant~\citep{wang2023optimal} as the reward of the goal-conditioned policy. Supplying SUN's reachability factor this way is a natural next step.
% \textit{Third}, the SVF is learned on the binary reward of Eq.~\eqref{eq:w-function}, which is sparse by construction and needs hindsight relabeling to train. 
% Yet the discounted state occupancy that Eq.~\eqref{eq:w-function} defines can be estimated directly by contrastive learning on state-action pairs and their sampled future states~\citep{eysenbach2022contrastive}, without bootstrapping and without relabeling. 
% Recent methods estimate it this way~\citep{jiang2025episodic,mohamed2026temporal}, and \citet{bae2024tldr} use a quasimetric variant~\citep{wang2023optimal} as the reward of the goal-conditioned policy. Supplying SUN's reachability factor this way is a natural next step, though the resulting estimates are distances or log-ratios rather than values in $[0,1]$, so whether they retain the exact zero on unreachable goals is an open question.
\\[1pt]
\textit{Finally}, SUN draws goals from the replay buffer, thus it can only propose goals it has already encountered. 
A natural way to relax this is to exploit the geometry of the goal space. Recent work on temporal distances and quasimetric value functions~\citep{myers2024learning} learns goal-space metrics that generalize beyond observed pairs; Laplacian-style representations~\citep{shehmar2026laplacian} provide a similar latent geometry. \citet{bae2024tldr} already select goals in such a space but without a reachability term conditioned on the current state. Separately, studies on out-of-distribution generalization in goal-conditioned RL~\citep{yang2023what} characterize the design choices that enable extrapolation to unseen goals. 
Combining SUN's reachability with such learned metrics is a promising direction.
% \textit{Finally}, SUN draws candidate goals from the replay buffer, thus it can only propose goals it has already encountered. 
% A natural way to relax this is to exploit the geometry of the goal space. Recent work on temporal distances and quasimetric value functions~\citep{myers2024learning} learns goal-space metrics that generalize beyond observed pairs; Laplacian-style representations~\citep{shehmar2026laplacian} provide a similar latent geometry. \citet{bae2024tldr} already select goals in such a space, by a nearest-neighbor entropy estimate over a replay minibatch, but without a reachability term conditioned on the current state. Separately, studies on out-of-distribution generalization in goal-conditioned RL~\citep{yang2023what} characterize the design choices that enable extrapolation to unseen goals. 
% Combining SUN's reachability factor with such learned metrics --- sampling goals as points in a continuous metric space rather than from the buffer --- is a promising direction.
% \\[2pt]
% \textit{Finally}, we do not provide regret or sample-complexity bounds, because candidates are sampled from replay rather than enumerated. Extending these structural results to a PAC-style analysis, as in AdaGoal~\citep{tarbouriech2022adaptive}, remains open.

\clearpage

\begin{ack}
This research was supported by grants from the European Laboratory for Learning and Intelligent Systems (ELLIS) and Finnish IT Center for Science (CSC).
\end{ack}

\setlength{\bibsep}{5pt}
\bibliography{main.bib}
\bibliographystyle{abbrvnat}

\clearpage
\appendix

\startcontents[appendices]
\section*{Appendices}

\begingroup
\hypersetup{
  colorlinks=true,
  linkcolor=linkblue
}
\printcontents[appendices]{}{1}{\setlength{\parskip}{3pt}\setcounter{tocdepth}{2}}
\endgroup

\clearpage
\begin{appendix}

\setlength{\abovedisplayskip}{7pt}%
\setlength{\abovedisplayshortskip}{7pt}%
\setlength{\belowdisplayskip}{7pt}%
\setlength{\belowdisplayshortskip}{7pt}%

\section{Theoretical Properties}
\label{app:theory}

This section analyzes the SUN selection rule under an oracle: the SVF $V^\pi(s, g)$ and the count $n_g$ are exact, the replay buffer is frozen, and the goal-conditioned policy $\pi$ acts to reach $g$. The analysis is structural --- it characterizes what SUN does given perfect estimates, not what it learns from samples. Within this scope, we establish three properties (Sections~\ref{subsec:thm-countbonus}--\ref{subsec:thm-unreachable}) that together formalize how SUN balances novelty and reachability, and we draw connections to existing methods through a log-space decomposition (Section~\ref{subsec:thm-logspace}). The three properties turn out to need little of the oracle, and we state at each the conditions it actually uses; it is the closed form of Section~\ref{subsec:thm-logspace}, which needs (C2) and (C4), and the stability result of Section~\ref{subsec:thm-perstep}, which needs those together with (C3), that rest on the oracle. Section~\ref{subsec:thm-stochastic} extends all three to stochastic dynamics, and Section~\ref{subsec:thm-perstep} analyzes the adaptive selection rule of Section~\ref{subsec:adaptive_selection}.

\subsection{Setup and Notation}
\label{subsec:thm-setup}

Let $\mu(g)$ denote the replay marginal over goals --- the empirical distribution of states stored in the buffer. The pseudocount $n_g$ from Section~\ref{subsec:pseudo} is a finite-sample estimate of $\mu(g)$ up to a normalization constant; under the oracle, we treat $n_g$ as exact and $\mu(g) > 0$ for every candidate. We define the \emph{log-rarity potential}
\begin{equation}
  \Phi(g) \;\triangleq\; \log \mu(g),
\end{equation}
which is monotone in $n_g$: small $n_g$ corresponds to small $\mu(g)$, hence small (very negative) $\Phi(g)$. The SUN novelty score $\nu(g) = 1/n_g$ in Section \ref{sec:method} therefore corresponds to $-\Phi(g)$ in log-space, up to a constant.

Throughout the analysis, $\pi$ denotes the goal-conditioned policy targeting $g$, and $\tau_g \triangleq \inf\{t \geq 0 : s_t = g\}$ is its (random) hitting time of $g$ from a starting state $s$. The SVF (Eq.~\ref{eq:w-function}) under $\pi$ admits the equivalent form
\begin{equation}
  V^{\pi}(s, g) = \EV_{\pi}\!\left[\gamma^{\tau_g} \,\ind{\scriptstyle\{\tau_g < \infty\}}\right],
\end{equation}
when $g$ is terminal (i.e., once reached, the episode is considered ended). This identity will be used repeatedly. 
\begin{assumption}[Oracle setting] From now on, we assume the following conditions are true.
\label{ass:ideal}
\begin{enumerate}[leftmargin=*, itemsep=1pt, topsep=0pt, label=\textup{(C\arabic*)}]
\item The state space $\statespace$ is finite.
\item Transitions are deterministic.
\item The replay marginal $\mu$, the SVF $V^\pi$, and the count $n_g$ are frozen during the analysis.
\item Under $\pi$, if $g$ is reachable the trajectory hits it after $k_\pi(s, g) - 1 \in \{0, 1, 2, \dots\}$ steps, where $k_\pi(s, g)$ counts the states on the path from $s$ to $g$ inclusive (so $k_\pi(g, g) = 1$), and the episode is considered ended at the first hit (i.e., goal-reaching transitions are terminal). If $g$ is unreachable, $k_\pi(s, g) = \infty$. We further assume $\pi$ is optimal for its own goal, i.e.\ $k_\pi(s, g) = \min_{\pi'} k_{\pi'}(s, g)$ for all $s$; this is what licenses the shortest-path inequality in the proof of Proposition~\ref{thm:goal-stability}(ii).
\end{enumerate}
\end{assumption}

Assumption~\ref{ass:ideal} mirrors the structural setup used in the autonomous-exploration literature~\citep{lim2012autonomous,tarbouriech2020improved,tarbouriech2022adaptive}. Section~\ref{subsec:thm-stochastic} replaces (C2), and the deterministic hitting count of (C4), with stochastic dynamics and almost-sure hitting.

\subsection{Count-Scaled Value}
\label{subsec:thm-countbonus}

The first result links SUN to the count-based intrinsic motivation literature~\citep{bellemare2016unifying}: SUN is exactly the value of a goal-reaching reward scaled by the goal's inverse count.

\begin{theorem}[Count-scaled value]
\label{thm:countbonus}
Under Assumption~\ref{ass:ideal}(C3), for any goal $g$ and any policy $\pi$, define the goal-specific reward $r^g_t \triangleq \ind{\scriptstyle\{s_t = g\}} / n_g$, where $n_g$ is the (frozen) count from Assumption~\ref{ass:ideal}(C3). Then
\begin{equation}
  \EV_\pi\!\left[\sum_{t \geq 0} \gamma^t \, r^g_t \;\Big|\; s_0 = s\right]
  \;=\;
  \frac{V^\pi(s, g)}{n_g}
  \;=\;
  \mathrm{SUN}(g \spacedmid s).
\end{equation}
\end{theorem}

\begin{proof}
By Assumption~\ref{ass:ideal}(C3), $n_g$ is frozen during the analysis and therefore independent of the trajectory. It thus factors out of the expectation:
\begin{align}
  \EV_\pi\!\left[\sum_{t \geq 0} \gamma^t \, r^g_t \;\Big|\; s_0 = s\right]
  &\;=\; \EV_\pi\!\left[\sum_{t \geq 0} \gamma^t \, \frac{\ind{\scriptstyle\{s_t = g\}}}{n_g} \;\Big|\; s_0 = s\right] \\
  &\;=\; \frac{1}{n_g} \, \EV_\pi\!\left[\sum_{t \geq 0} \gamma^t \, \ind{\scriptstyle\{s_t = g\}} \;\Big|\; s_0 = s\right] \\
  &\;=\; \frac{V^\pi(s, g)}{n_g}
  \;=\; \mathrm{SUN}(g \spacedmid s),
\end{align}
where the last equality uses the definition of the SVF (Eq.~\ref{eq:w-function}) and the SUN score (Eq.~\ref{eq:sun-score}).
\end{proof}

The statement needs only (C3). It holds for an arbitrary policy, deterministic or stochastic dynamics, and with or without the terminal-goal convention, since the argument is a change of order of the expectation and a constant factor.

This identity is the simplest interpretation of SUN: maximizing $\mathrm{SUN}(g \spacedmid s)$ over $g$ is equivalent to acting greedily with respect to the value of an exploration reward that pays inversely proportional to how often $g$ has been visited. SUN is therefore not just a heuristic combination of two signals --- it is the value of a single, principled exploration objective. Note that $r^g_t$ pays only at $g$: it is the goal-reaching indicator of Eq.~\ref{eq:w-function} scaled by a constant, not the state-wise bonus $1/n_{s_t}$ of \citet{bellemare2016unifying}, which rewards every state traversed. SUN therefore scores the endpoint of a trajectory, not the states along it, a distinction we return to in the Discussion.

\begin{remark}[Frozen vs. online counts]
The equivalence holds when $n_g$ is treated as fixed during the trajectory --- the standard frozen-replay setting (C3). When counts are updated online (as in our practical algorithm and in the standard count-bonus literature), the surrogate reward becomes non-stationary and the equivalence holds only approximately, with the size of the gap controlled by how much $n_g$ changes over the trajectory.
\end{remark}

\begin{remark}[Form vs. exponent]
Theorem~\ref{thm:countbonus} holds for any $\nu(g) = f(n_g)$, since $n_g$ is frozen and factors out of the expectation. It therefore justifies the multiplicative \textit{form} but not the exponent: $\nu(g) = 1/n_g$ is a choice, with $1/\sqrt{n_g}$~\citep{bellemare2016unifying} the natural alternative.
\end{remark}

\subsection{Hitting-Probability Bound}
\label{subsec:thm-hitting}

The SVF $V^\pi(s, g)$ is a discounted occupancy, but the property we ultimately care about is whether the agent reaches $g$ from $s$ within a reasonable horizon. The next theorem connects the two for the goal-conditioned policy $\pi$, under the terminal-goal convention of Assumption~\ref{ass:ideal}(C4): once the agent reaches $g$, the episode ends. Its proof uses nothing beyond terminality, so we state it directly in that general form. In this setting, $V^{\pi}(s, g) = \EV[\gamma^{\tau_g} \ind{\scriptstyle\{\tau_g < \infty\}}] \leq 1$ (Section~\ref{subsec:thm-setup}), and a high SVF guarantees a high hitting probability over a short horizon.

\begin{theorem}[Short-horizon hitting bound]
\label{thm:hitting}
Under the terminal-goal convention of Assumption~\ref{ass:ideal}(C4), with $\tau_g$ the (possibly random) hitting time and transitions deterministic or stochastic alike, for every horizon $n \geq 0$,
\begin{equation}
  \Pr_{\pi}\!\left[\tau_g \leq n\right]
  \;\geq\;
  V^{\pi}(s, g) \;-\; \gamma^{n+1}.
\end{equation}
In particular, choosing $n$ such that $\gamma^{n+1} \leq V^{\pi}(s, g)/2$ gives $\Pr_{\pi}[\tau_g \leq n] \geq V^{\pi}(s, g)/2$. This horizon is $n = \mathcal{O}\bigl(\log(2/V^{\pi}) / (1 - \gamma)\bigr)$.
\end{theorem}

\begin{proof}
Under the terminal-goal convention of Assumption~\ref{ass:ideal}(C4), $V^{\pi}(s, g)$ admits the equivalent form $V^{\pi}(s, g) = \EV[\gamma^{\tau_g} \ind{\scriptstyle\{\tau_g < \infty\}}]$ derived in Section~\ref{subsec:thm-setup}. We split the expectation by hitting horizon:
\begin{align}
  V^{\pi}(s, g)
  &\;=\; \EV\!\left[\gamma^{\tau_g} \, \ind{\scriptstyle\{\tau_g < \infty\}}\right] \\
  &\;=\; \EV\!\left[\gamma^{\tau_g} \, \ind{\scriptstyle\{\tau_g \leq n\}}\right]
        \,+\, \EV\!\left[\gamma^{\tau_g} \, \ind{\scriptstyle\{n < \tau_g < \infty\}}\right].
\end{align}
We bound each term separately.
\\
For the first term, $\gamma^{\tau_g} \leq 1$ when $\tau_g \leq n$, so
\begin{equation}
  \EV\!\left[\gamma^{\tau_g} \, \ind{\scriptstyle\{\tau_g \leq n\}}\right]
  \;\leq\;
  \EV\!\left[\ind{\scriptstyle\{\tau_g \leq n\}}\right]
  \;=\;
  \Pr_{\pi}[\tau_g \leq n].
\end{equation}
For the second term, $\gamma^{\tau_g} \leq \gamma^{n+1}$ when $\tau_g > n$, so
\begin{equation}
  \EV\!\left[\gamma^{\tau_g} \, \ind{\scriptstyle\{n < \tau_g < \infty\}}\right]
  \;\leq\;
  \gamma^{n+1} \, \Pr_{\pi}[\tau_g > n]
  \;=\;
  \gamma^{n+1} \bigl(1 - \Pr_{\pi}[\tau_g \leq n]\bigr).
\end{equation}
Combining the two bounds gives $V^{\pi}(s, g) \leq \Pr_{\pi}[\tau_g \leq n] \,(1 - \gamma^{n+1}) + \gamma^{n+1}$, and rearranging yields
\begin{equation}
  \Pr_{\pi}[\tau_g \leq n] \;\geq\; \frac{V^{\pi}(s, g) \,-\, \gamma^{n+1}}{1 - \gamma^{n+1}}.
\end{equation}
The weaker form follows because the probability is non-negative and, when the numerator is non-negative, the denominator is at most one.
For the horizon claim, set $\gamma^{n+1} \leq V^{\pi}/2$ and solve: $n+1 \geq \log(2/V^{\pi}) / \log(1/\gamma)$, which gives $n = \mathcal{O}\bigl(\log(2/V^{\pi}) / (1 - \gamma)\bigr)$ since $\log(1/\gamma) \geq 1 - \gamma$ for $\gamma \in (0, 1)$.
\end{proof}

The horizon scales logarithmically in $1/V^{\pi}$ and inversely in $1 - \gamma$: high-SVF goals are hit quickly with high probability, while low-SVF goals require long horizons. This is the formal version of the intuition that the SVF is a reachability signal.

\subsection{Rejection of Unreachable Goals}
\label{subsec:thm-unreachable}

The third result is the dual of Theorem~\ref{thm:hitting}: goals that no policy can reach receive zero score. Since $V^\pi(s, s) = 1$, the maximum score over any candidate set containing the current state $s$ is positive, so SUN never selects them.

\begin{theorem}[Unreachability rejection]
\label{thm:unreachable}
Let $\Pi$ be any class of policies. If $g$ is unreachable from $s$ under every $\pi \in \Pi$ --- that is, $\Pr_\pi[\tau_g < \infty] = 0$ for all $\pi \in \Pi$ --- then
\begin{equation}
  \sup_{\pi \in \Pi} V^\pi(s, g) = 0,
  \qquad
  \mathrm{SUN}(g \spacedmid s) = 0,
\end{equation}
for any $n_g > 0$.
\end{theorem}

\begin{proof}
For any $\pi \in \Pi$, the unreachability hypothesis $\Pr_\pi[\tau_g < \infty] = 0$ means the indicator $\ind{\scriptstyle\{\tau_g < \infty\}}$ is zero almost surely. Hence
\begin{equation}
  V^\pi(s, g) = \EV_\pi\!\left[\gamma^{\tau_g} \, \ind{\scriptstyle\{\tau_g < \infty\}}\right] = 0,
\end{equation}
and $\mathrm{SUN}(g \spacedmid s) = 0/n_g = 0$.
\end{proof}

Like Theorem~\ref{thm:countbonus}, this statement uses none of (C1)--(C4); it needs only $n_g > 0$, which Section~\ref{subsec:thm-setup} assumes. The proof above is written with the terminal-goal form for uniformity, but it can be run from Eq.~\ref{eq:w-function} directly: if $g$ is never visited, every indicator in the sum is zero almost surely, whatever the dynamics and whether or not $g$ is terminal.

This property fails for novelty-only scores ($1/n_g$ alone), which assign maximal value to never-visited unreachable goals --- exactly the failure mode shown in Figure~\ref{fig:two_room_intro}. SUN is therefore guaranteed to ignore goals that the agent cannot reach, regardless of how rare they are. Note that the hypothesis is zero hitting probability under every policy in $\Pi$, which is stronger than the bounded-expected-time notion of reachability defined in the Section \ref{sec:prelim} (and than the radius-$L$ notion of Appendix~\ref{app:pac}): a goal that is reached with small positive probability, or only in infinite expected time, is not rejected by this theorem but down-weighted through $V^\pi$.

\subsection{Log-space Decomposition and Connections to Prior Work}
\label{subsec:thm-logspace}

A useful consequence of the count-bonus form is that $\log \mathrm{SUN}$ decomposes additively:
\begin{equation}
  \log \mathrm{SUN}(g \spacedmid s) \;=\; \log V^\pi(s, g) \,-\, \Phi(g) \,-\, \log N,
  \label{eq:sun-log}
\end{equation}
where $N$ is the buffer size, so that $n_g = N \mu(g)$. The additive constant $\log N$ is the same for every goal and therefore cancels in any comparison between goals; we drop it below.
Under Assumption~\ref{ass:ideal}, $V^{\pi}(s, g) = \gamma^{k_\pi(s, g) - 1}$ for reachable $g$, so the score takes the closed form
\begin{equation}
  \log \mathrm{SUN}(g \spacedmid s) \;=\; -\bigl(k_\pi(s, g) - 1\bigr) \, \kappa \,-\, \Phi(g) \,-\, \log N,
  \qquad \kappa \triangleq -\log \gamma > 0.
  \label{eq:sun-closedform}
\end{equation}
SUN therefore takes the form of a soft Lagrangian: $-\Phi(g)$ is the rarity reward, $(k_\pi - 1)\kappa$ is the discounted distance cost, and $\kappa$ is \emph{derived} from the discount $\gamma$ rather than introduced as a separate hyperparameter.
This is structurally similar to AdaGoal~\citep{tarbouriech2022adaptive}, which solves a hard-constrained version with a user-specified radius. The two are not equivalent --- AdaGoal uses a prediction-error estimate of epistemic uncertainty as its novelty signal, while SUN uses distributional rarity --- but both balance the same two ingredients.

The log-space form also yields a state-dependent admissibility condition: a goal $g$ is preferred over staying at $s$ if and only if
\begin{equation}
  k_\pi(s, g) \;<\; 1 \,+\, \frac{\Phi(s) - \Phi(g)}{\kappa}.
\end{equation}
A goal that is no rarer than the current state ($\Phi(g) \geq \Phi(s)$) is never preferred, and the maximum admissible distance scales with the rarity gap $\Phi(s) - \Phi(g)$. This is a tighter analogue of the ``neither too easy nor too hard'' intuition behind AdaGoal and DISCOVER~\citep{diaz-bone2025discover}, with the trade-off automatically calibrated by the discount.
The calibration is not neutral, however. With $\gamma = 0.99$ (Appendix~\ref{app:hyper}), $\kappa \approx 0.01$: a rarity gap of one nat admits goals about $100$ steps away, and a goal with half the count of the current state is admissible up to about $69$ steps away. Within the horizons of our environments, the oracle indicator therefore acts largely as a filter on the reachable set and orders reachable goals by rarity; the finer balance between distance and rarity observed in Section~\ref{sec:experiments} arises where $\Vtheta$ underestimates distant goals relative to $\gamma^{k_\pi - 1}$. The exchange rate is fixed by $\gamma$ rather than tuned, but it is not free: changing it also changes the horizon of value learning.

\subsection{Extension to Stochastic Dynamics}
\label{subsec:thm-stochastic}

Theorems~\ref{thm:countbonus} and~\ref{thm:unreachable} hold beyond deterministic dynamics (Theorem~\ref{thm:hitting} is already stated in that setting). Replace (C2) with stochastic transitions, and the deterministic hitting count of (C4) with the random hitting time $\tau_g$; the terminal-goal convention of (C4) is retained. The equivalent SVF form $V^{\pi}(s, g) = \EV[\gamma^{\tau_g} \ind{\scriptstyle\{\tau_g < \infty\}}]$ then continues to hold, and both theorems transfer verbatim. Note that (C3) is retained: Theorem~\ref{thm:countbonus} requires a frozen $n_g$ to factor out of the expectation, and holds only approximately otherwise.
The closed form $V^{\pi} = \gamma^{k_\pi - 1}$, and with it Eq.~\ref{eq:sun-closedform}, no longer holds in general --- $k_\pi$ becomes a random hitting time --- but the log-space decomposition (Eq.~\ref{eq:sun-log}) is unchanged.

\subsection{Adaptive Goal-Selection: Theoretical Consistency and Practical Motivation}
\label{subsec:thm-perstep}

A natural concern with the SUN formulation is that the SVF $V^{\pi}(s, g)$ is defined under a policy $\pi$ that pursues $g$ for the entire trajectory, while SUN's adaptive strategy may reselect the goal mid-episode based on a value-consistency check. We show that this concern dissolves in two ways. In the deterministic, oracle setting, the SUN indicator is monotone along any trajectory that pursues its own $\argmax$: the value never decreases, so the adaptive check never fires, and the goal remains fixed for the entire episode (Proposition~\ref{thm:goal-stability}). In stochastic or approximate settings, monotonicity can fail --- but this failure is precisely what the adaptive check is designed to detect.

\begin{proposition}[Value monotonicity and goal stability under deterministic dynamics]
\label{thm:goal-stability}
Under Assumption~\ref{ass:ideal} (deterministic dynamics, known $V^\pi$ and $n_g$, frozen replay, and $\pi$ optimal for its own goal), if SUN selects $g^*_t$ at state $s_t$ and the agent takes one step under $\pi_{g^*_t}$ to reach $s_{t+1}$, then:
\begin{enumerate}[leftmargin=*, noitemsep, topsep=2pt]
\item[(i)] the value along the pursued goal is non-decreasing: $V^\pi(s_{t+1}, g^*_t) \geq V^\pi(s_t, g^*_t)$;
\item[(ii)] $g^*_t$ remains a SUN $\argmax$ at $s_{t+1}$ (and the unique one if it was unique at $s_t$):
\begin{equation}
g^*_t \,\in\, \argmax_{g \in \goalspace} \frac{V^\pi(s_{t+1}, g)}{n_g}.
\end{equation}
\end{enumerate}
By (i), the adaptive check $V^\pi(s_t, g^*_t) < V^\pi(s_{t_{\mathrm{sel}}}, g^*_t)$ never fires; by (ii), the argmax is stable. By induction, with ties broken in favor of the current goal, the same goal is pursued throughout the episode until $g^*_t$ is reached.
\end{proposition}

\begin{proof}
Let $k_g \triangleq k_\pi(s_t, g)$. Under deterministic dynamics, one step along the optimal path to $g^*_t$ reduces its hitting time by exactly one: $k_\pi(s_{t+1}, g^*_t) = k_{g^*_t} - 1$. Applying $V^\pi(s, g) = \gamma^{k_\pi(s, g) - 1}$ gives $V^\pi(s_{t+1}, g^*_t) = \gamma^{-1} V^\pi(s_t, g^*_t) \geq V^\pi(s_t, g^*_t)$ (since $\gamma \in (0, 1)$), proving (i).

For (ii), for any other goal $g$, the triangle inequality on hitting times (a one-step transition can increase the distance to $g$ by at most one) gives $k_\pi(s_{t+1}, g) \geq k_g - 1$, hence $V^\pi(s_{t+1}, g) \leq \gamma^{-1} V^\pi(s_t, g)$. By Assumption~\ref{ass:ideal}(C3), $\mu$ is frozen, so the counts $n_g$ do not change between $t$ and $t+1$. (The conclusion survives outside this assumption with exact counts: $n_{g^*_t}$ cannot increment within an episode, since reaching $g^*_t$ would end it, while the count of any other goal can only increase, which only lowers that goal's score. With the fixed-radius pseudocount of Section~\ref{subsec:pseudo}, inserting $s_{t+1}$ also increments $n_{g^*_t}$ whenever $g^*_t$ lies within $\rho$ of $s_{t+1}$ without being reached, so there the statement is only approximate.) Therefore
\begin{align*}
\frac{V^\pi(s_{t+1}, g^*_t)}{n_{g^*_t}} &= \gamma^{-1} \, \frac{V^\pi(s_t, g^*_t)}{n_{g^*_t}}, \\
\frac{V^\pi(s_{t+1}, g)}{n_g} &\leq \gamma^{-1} \, \frac{V^\pi(s_t, g)}{n_g} \quad \forall g.
\end{align*}
By the optimality of $g^*_t$ at $s_t$, the ordering is preserved at $s_{t+1}$: $g^*_t$ remains an $\argmax$, and the unique one if it was unique at $s_t$. Theorem~\ref{thm:unreachable} ensures unreachable goals remain excluded.
\end{proof}

Proposition~\ref{thm:goal-stability} shows that, in the oracle setting, the adaptive check never fires, so adaptive selection and fixed-goal commitment produce the same trajectory. The fixed-goal semantics of $V^{\pi}$ and the adaptive semantics of SUN are therefore consistent.
\\[2pt]
Part (ii) is an $\argmax$ over the whole goal space, and its proof applies verbatim to an $\argmax$ over any fixed subset containing $g^*_t$: the inequality $V^\pi(s_{t+1}, g)/n_g \leq \gamma^{-1} V^\pi(s_t, g)/n_g$ holds for every $g$, so the ordering is preserved within the subset as well. Per-step selection therefore coincides with the other two whenever its candidate set is $\goalspace$ itself, or is drawn once and kept for the episode. This is not the per-step variant we evaluate: there a fresh batch $\mathcal{C}_t \subset \goalspace$ is drawn at every step (Section~\ref{subsec:adaptive_selection}), and a batch that happens to miss $g^*_t$ forces a different selection even when $\Vtheta$ is exact. Per-step selection is therefore unstable by construction and not only through approximation error, which is consistent with its re-selection rate never reaching zero in Figure~\ref{fig:ablation_selection}.

\textbf{Stochastic dynamics.} Under stochastic transitions, the next state $s_{t+1}$ is random and there is no deterministic hitting count $k_\pi$. The one-step relation $V^\pi(s_{t+1}, g^*_t) = \gamma^{-1} V^\pi(s_t, g^*_t)$ of part (i) then holds only \emph{in expectation} over $s_{t+1}$ (for $s_t \neq g^*_t$, the Bellman equation under the terminal convention gives $V^\pi(s_t, g^*_t) = \gamma \, \EV[V^\pi(s_{t+1}, g^*_t)]$), and a realized transition may push the agent to a state where $V^\pi(s_{t+1}, g^*_t) < V^\pi(s_{t_{\mathrm{sel}}}, g^*_t)$ --- exactly the condition the adaptive check flags for reselection. A fixed-goal policy would continue pursuing $g^*_t$ regardless; per-step reselection would reselect at every step, discarding stable information. The adaptive check reselects only when the SVF signals a value drop, which under stochastic dynamics corresponds to a transition into a genuinely less favorable region.

\textbf{Approximate $\Vtheta$.} The same argument applies when $\Vtheta$ is learned rather than exact. Errors in $\Vtheta$ can cause monotonicity to fail even under deterministic dynamics: if $\Vtheta(s_t, g^*_t)$ was overestimated at selection time, the true value at $s_{t+1}$ may be lower. The adaptive check detects this and triggers reselection, effectively withdrawing commitment when the SVF's estimates are proved unreliable by their own subsequent values.

\begin{tcolorbox}[colback=gray!10, colframe=gray!40, boxrule=0.5pt, arc=2pt, left=6pt, right=6pt, top=4pt, bottom=4pt]
SUN's adaptive strategy is therefore the right choice in both regimes. When $\Vtheta$ is exact and dynamics are deterministic, Proposition~\ref{thm:goal-stability} shows that the check never fires and the agent commits to a single goal per episode, matching fixed-goal commitment exactly. When $\Vtheta$ is inaccurate or dynamics are stochastic, the check acts as an SVF-driven trigger: it withdraws commitment precisely when the SVF signals that the current goal is no longer reliable.
\end{tcolorbox}

\clearpage

\setlength{\abovedisplayskip}{3pt}%
\setlength{\abovedisplayshortskip}{3pt}%
\setlength{\belowdisplayskip}{3pt}%
\setlength{\belowdisplayshortskip}{3pt}%

\section{Pseudocount Radius and Standardization}
\label{app:pseudocount}

The pseudocount of Section~\ref{subsec:pseudo} relies on a single scalar radius $\rho$ in standardized feature space. We elaborate here on how the scales are estimated, how $\rho$ relates to the goal dimensionality, and why the scales must be held fixed.
% \\[2pt]

\textbf{Robust per-feature scale.} A natural choice is to divide each feature by its standard deviation $\sigma_m$, so that a single radius $\rho$ has the same meaning across features. Two problems arise. First, visitation distributions are heavy-tailed --- a few excursions to a distant part of the state space inflate $\sigma_m$ for the whole buffer --- so the standard deviation does not reflect the spread of the region the agent actually occupies. We therefore use the interquartile range, rescaled so that it matches the standard deviation for Gaussian data,
\begin{equation}
\setlength{\abovedisplayskip}{5pt}%
\setlength{\abovedisplayshortskip}{5pt}%
\setlength{\belowdisplayskip}{5pt}%
\setlength{\belowdisplayshortskip}{5pt}%
  \hat\sigma_m \;=\; \frac{Q_{75}(s_m) - Q_{25}(s_m)}{1.349},
\end{equation}
where the quantiles are taken over the stored buffer entries. Second, a nearly constant feature has $\hat\sigma_m \approx 0$, and dividing by it amplifies tiny variations, so two points that are almost identical in dimension $m$ end up far apart. We therefore floor $\hat\sigma_m$ at a small fraction of the median scale across the non-degenerate features,
\begin{equation}
\setlength{\abovedisplayskip}{5pt}%
\setlength{\abovedisplayshortskip}{5pt}%
\setlength{\belowdisplayskip}{5pt}%
\setlength{\belowdisplayshortskip}{5pt}%
  \tilde\sigma_m \;=\; \max\!\left(\hat\sigma_m, \; \zeta \cdot \mathrm{median}\{\hat\sigma_{m'} : \hat\sigma_{m'} > 0\}\right),
  \qquad \zeta = 10^{-2},
\end{equation}
and fall back to $\tilde\sigma_m = 1$ for features that are exactly constant across the buffer. This caps the ratio between the widest and narrowest scale without forcing genuinely thin directions to inherit the spread of wide ones. The squared distance used by the pseudocount is then
\begin{equation}
\setlength{\abovedisplayskip}{5pt}%
\setlength{\abovedisplayshortskip}{5pt}%
\setlength{\belowdisplayskip}{5pt}%
\setlength{\belowdisplayshortskip}{5pt}%
  \|s - s'\|^2_{/\tilde\sigma} \;=\; \sum_m \frac{(s_m - s'_m)^2}{\tilde\sigma_m^2}.
\end{equation}
% \\[1pt]

\textbf{Radius and dimensionality.} Standardization makes $\rho$ comparable across the features of one goal space, but not across goal spaces of different dimensionality. The ball of radius $\rho$ in $M$ dimensions has volume proportional to $\rho^M$, so the fraction of the buffer it contains falls off geometrically in $M$: halving $\rho$ costs a factor $4$ in two dimensions and a factor $2^8 = 256$ in eight. A radius that leaves counts informative in a two-dimensional goal space can therefore leave almost every entry at $n_g = 1$ in a higher-dimensional one, where $\nu(g) = 1/n_g$ is uniform across candidates and the indicator reduces to the SVF alone. This is why $\rho$ must grow with $M$, and why \stt{LunarLander(Full)}, whose goal space is eight-dimensional, uses the largest radius of the suite (Appendix~\ref{app:hyper}). It is a property of fixed-radius density estimation rather than of our estimator: any kernel with a fixed bandwidth requires the bandwidth to grow with the dimension.
% \\[2pt]

\textbf{Why the scales must be fixed.} Inserting a sample updates only the counts of its own neighbors, so every stored $n_g$ equals the number of buffer entries within $\rho$ of $g$ only as long as $\tilde\sigma$ is the same at every insertion. If the scales change, two stored entries can become neighbors under the new metric while neither is a neighbor of the newly inserted sample --- neither count is corrected and the error is never recovered. With $\rho = 0.1$ in one dimension, for instance, the points $0$ and $1$ are two standardized units apart and each hold a count of one; after inserting $100$ the scale grows and their separation falls below $\rho$, yet neither lies within $\rho$ of the new point and both counts stay at one. Recomputing the scales at every insertion would therefore make $n_g$ depend on the order in which samples arrived.
% \\[2pt]

\textbf{Schedule.} We avoid this as follows. First, no counts are maintained during warm-up (i.e., random exploration, see Section \ref{app:hyper}). When the warm-up ends, the scales are computed from all the samples collected so far and every count is built from them in a single pass, which is quadratic in the buffer size but runs at a point where the buffer is still small. The scales are then recomputed every $10^3$ insertions and compared with the previous estimate; once the largest relative per-feature change falls below $10^{-2}$ they are frozen for the remainder of training, and all counts are rebuilt once under them. 
% Between recomputations the scales are constant, so $n_g$ can drift only across the refreshes that precede the freeze, and the rebuild at the freeze removes that drift. Where the visited region keeps expanding, the tolerance may not be met before training ends; the scales then continue to be refreshed and $n_g$ is an approximation of the fixed-radius count under the current metric rather than an exact one. Appendix~\ref{app:pseudocount_ablation} shows that the resulting novelty signal tracks exact counts accurately.
% \\[2pt]

\textbf{Eviction.} When the buffer is full, inserting a sample evicts one, and the evicted entry's neighbors are decremented before the new entry's neighbors are incremented, at the same per-step cost. In our experiments the buffer is large enough that no eviction occurs.

\clearpage

\section{Practical Notes}
\label{app:notes}
% Here, we discuss a few practical details related to Algorithm \ref{alg:sun}.

\textbf{SVF reward.} In RL literature, the reward in Eq.~\eqref{eq:w-function} is sometimes replaced by alternatives that target the same quantity through different reward shapings. For example, returning $-1$ until $s_i$ is reached and $0$ thereafter recovers a negative-distance interpretation~\citep{schaul2015universal,andrychowicz2017hindsight}. However, those rewards performed worse in our experiments.
% % Continuous variants based on similarity kernels have also been used~\citep{pong2018temporal}. 
% The condition $s_k = s_i$ can also be relaxed to apply over a feature map rather than raw states, yielding the {successor features} of~\citet{barreto2017successor}, which generalize the SVFs to settings where exact state matching is impractical.
\\[1pt]
\textbf{Goal relabeling.}
Training the SVF is a self-supervised process: given tuples $(s_t, a_t, g_t, s_{t+1})$, $V^\pi$ can be trained with TD learning using the reward in Eq.~\eqref{eq:w-function}, i.e., $r_t = \ind{\scriptstyle{\{s_t = g\}}}$. Effective training requires tuples in which the agent both reaches and fails to reach the goal, i.e., \textit{positive} and \textit{negative} samples. Early in training, however, the agent rarely reaches its commanded goal $g_t$, so the replay buffer contains almost exclusively negatives. Hindsight Experience Replay (HER)~\citep{andrychowicz2017hindsight} addresses this imbalance by relabeling $g_t$ with a state \textit{from the trajectory itself}: the selected state fires the reward and yields positive samples. We adopt HER's ``future'' strategy: given a trajectory of $T$ steps, $g_t$ is sampled from $\{s_t, s_{t+1}, \ldots, s_T\}$. Figure~\ref{fig:her-relabel} illustrates the procedure.
Note that training is off-policy by construction: the relabeled goal is not the goal under which the trajectory was collected.
% Furthermore, future relabeling only provides positive training samples --- negatives are still needed. Common strategies either relabel part of the mini-batch, or duplicate it keeping the original on-policy goals as negatives and relabeling the copies with positives.
% \\[2pt]
% The choice of negatives matters as much as that of positives. A natural alternative is to sample \textit{states} from the buffer as negative goals, rather than reusing the on-policy goals, but this biases the negatives toward frequently visited states, whereas the goals the SVF most needs to learn to reject are the rare ones. Consider \stt{ThreeRoom}, where the agent spends most of its time in the first two rooms. Sampling negatives from stored states would make third-room goals almost never appear. Those, however, are precisely the goals selected most often early in training, when the SVF is still inaccurate and novelty dominates the selection rule: the agent spends its early steps attempting, and failing, to reach the third room. If negatives are drawn from visited states, it is never trained on those failures, and never learns that the third room is unreachable from the other two.
\\[1pt]
\textbf{Replay buffer eviction.} If the buffer is large enough (as in our experiments), no data eviction occurs when new samples are inserted (Algorithm \ref{alg:sun}:17). If eviction happens (e.g.\ in a small FIFO buffer), the principled option is to decrement the counts of the evicted sample's neighbors, preserving exactness at the same per-step cost as insertion. %\footnote{This is why $n_{g_{t-1}}$ is computed at every step in Algorithm \ref{alg:sun}, and not read from the buffer ($g_{t-1}$ may have been evicted).} 
A simpler approximation is to leave counts as-is, but this introduces an upward bias on the evicted sample's neighbors. % --- their stored counts no longer reflect the current buffer. 
% The effect of this approximation on the SUN ranking depends on $n_g$: rarer goals are disproportionately deflated by the nonlinearity of $1/n_g$. Whether this approximation is acceptable in practice depends on the eviction rate.
\\[1pt]
\textbf{Action-level noise.} Ideally, the exploration policy would act greedily with respect to the SVF (Algorithm \ref{alg:sun}:15). However, since $\Vtheta$ is inaccurate early in training, it may be appropriate to inject a small amount of noise into $\pi(a \spacedmid s_t, g_t)$. More details are in Appendix \ref{app:hyper}.

\begin{figure}[!t]
\centering
\begin{minipage}[c]{0.41\linewidth}
\centering
\resizebox{\linewidth}{!}{
\begin{tikzpicture}[
  trajnode/.style={circle, draw=black!50, fill=gray!15, minimum size=6mm, font=\small, inner sep=0pt},
  posnode/.style={circle, draw=green!50!black, fill=green!25, minimum size=6mm, font=\small, line width=0.7pt, inner sep=0pt},
  goalnode/.style={circle, draw=blue!70!black, fill=blue!20, minimum size=6mm, font=\small, line width=0.7pt, inner sep=0pt},
  negnode/.style={circle, draw=red!60!black, fill=red!15, minimum size=6mm, font=\small, line width=0.7pt, inner sep=0pt},
  mainarrow/.style={-{Stealth[length=2mm]}, line width=0.8pt, black},
  subarrow/.style={-{Stealth[length=1.5mm]}, line width=0.5pt, gray!70},
  branch/.style={dashed, line width=0.5pt, gray!60},
  threshold/.style={draw=green!50!black, dashed, line width=0.5pt, fill=green!8, fill opacity=0.4},
]
% Main trajectory (horizontal, top row)
\node[trajnode] (m1) at (0, 3.2) {$s_1$};
\node[trajnode] (m2) at (1.2, 3.2) {$s_2$};
\node[trajnode] (m3) at (2.4, 3.2) {$s_3$};
\node[trajnode] (m4) at (3.6, 3.2) {$s_4$};
\node[trajnode] (m5) at (4.8, 3.2) {$s_5$};
\node[font=\Large, gray!70] at (5.4, 3.2) {$...$};
\node[negnode] (mg) at (6.1, 3.2) {$g$};
\draw[mainarrow] (m1) -- (m2);
\draw[mainarrow] (m2) -- (m3);
\draw[mainarrow] (m3) -- (m4);
\draw[mainarrow] (m4) -- (m5);
% Column 1 (below s_1): sub-trajectory s_1 s_2 s_3 s_4, goal = s_4
\node[trajnode] (c1n1) at (0, 2.4) {$s_1$};
\node[trajnode] (c1n2) at (0, 1.6) {$s_2$};
\node[trajnode] (c1n3) at (0, 0.8) {$s_3$};
\node[goalnode] (c1n4) at (0, 0.0) {$s_4$};
\draw[subarrow] (c1n1) -- (c1n2);
\draw[subarrow] (c1n2) -- (c1n3);
\draw[subarrow] (c1n3) -- (c1n4);
\draw[branch] (m1) -- (c1n1);
% Column 2 (below s_2): sub-trajectory s_2 s_3, goal = s_3
\node[trajnode] (c2n1) at (1.2, 2.4) {$s_2$};
\node[goalnode] (c2n2) at (1.2, 1.6) {$s_3$};
\draw[subarrow] (c2n1) -- (c2n2);
\draw[branch] (m2) -- (c2n1);
% Column 3 (below s_3): sub-trajectory s_3 s_4 s_5, goal = s_5
\node[trajnode] (c3n1) at (2.4, 2.4) {$s_3$};
\node[trajnode] (c3n2) at (2.4, 1.6) {$s_4$};
\node[goalnode] (c3n3) at (2.4, 0.8) {$s_5$};
\draw[subarrow] (c3n1) -- (c3n2);
\draw[subarrow] (c3n2) -- (c3n3);
\draw[branch] (m3) -- (c3n1);
% Column 4 (below s_4): single state s_4, goal = s_4
\node[goalnode] (c4n1) at (3.6, 2.4) {$s_4$};
\draw[branch] (m4) -- (c4n1);
% Column 5 (below s_5): single state s_5, goal = s_5
\node[goalnode] (c5n1) at (4.8, 2.4) {$s_5$};
\draw[branch] (m5) -- (c5n1);
\end{tikzpicture}}
\end{minipage}%
\hfill
\begin{minipage}[c]{0.55\linewidth}
\caption{\textbf{Example of HER ``future'' relabeling.} The agent explores states $s_1 \ldots s_5$ (top row, black arrows) while trying to reach goal $g$ (red). To provide \textit{positive} rewards, each state $s_t$ is assigned a goal $g_t$ (blue) among future states within the same trajectory. For example, $s_1$ is assigned $g_1 \leftarrow s_4$. 
TD targets for $s_t$ are then computed according to the sub-trajectory from $s_t$ to $g_t$ (gray downwards arrows).}
\label{fig:her-relabel}
\end{minipage}
\end{figure}

\section{Source Code, Compute Details, and Runtimes}
\label{app:compute}

We ran parallelized experiments on SLURM-based clusters. 
For Gridworlds and Classic control, code is in PyTorch and ran on AMD Turin 9965 CPUs.
For GCRL control, code is in JAX and ran on a mix of NVIDIA V100 (32\,GB), A100 (80\,GB) and GH200 (96\,GB) GPUs. Operations are JIT-compiled and execute on a single GPU per run. 
Code at \url{https://github.com/sparisi/rl-anvil}
\\[1pt]
Wall-clock time (Table \ref{tab:runtime}) varies with environment dimensionality and training steps.
Note that the comparison is not like-for-like: AdaGoal, DISCOVER, and SUN Episodic score a larger batch of candidates once per episode (Section \ref{app:hyper}), whereas SUN scores a smaller batch whenever its adaptive check fires --- often early in training, and progressively less as the agent is trained (Section~\ref{subsec:ablation_results}).
% \\[1pt]
Furthermore, AdaGoal and DISCOVER run in comparable time to SUN despite training an ensemble of critics: in PyTorch (Gridworlds and Classic control) they update a single randomly selected critic per step (Section \ref{subsec:dqn_details}), and in JAX (GCRL control) the critic updates are parallelized.

\begin{table}[!b]
\centering
\captionsetup{skip=2pt}
\caption{\small\textbf{Average wall-clock runtime}, in minutes.}
\label{tab:runtime}
\resizebox{\textwidth}{!}{%
\begin{tabular}{lrrrrrrrrrrrrrr}
\toprule
 & ThreeR. & FourR.Stuck & GridMaze & M.Car & Pend. & L.Lander & L.Lander(Full) & Acrobot & C.Pole & P.Maze-S & P.Maze-H & A.Maze-S & A.Maze-H & ArmPush-H \\
\midrule
SUN (Adapt., Multipl.) & 76.4 & 164.6 & 268.4 & 211.2 & 106.3 & 457.9 & 1715.2 & 119.5 & 233.9 & 15.5 & 16.8 & 29.5 & 36.7 & 116.3 \\
SUN (Adapt., Additive) & 91.0 & 158.3 & 256.9 & 230.0 & 141.3 & 438.6 & --- & 167.9 & 256.4 & 15.5 & 17.3 & 29.6 & 36.6 & 114.7 \\
SUN (Per-Step, Multipl.) & 102.6 & 215.7 & 284.2 & 247.4 & 149.2 & 471.5 & --- & 167.4 & 294.3 & 15.3 & 16.4 & 29.3 & 36.6 & 115.0 \\
SUN (Epis., Multipl.) & 97.2 & 161.8 & 261.1 & 223.7 & 139.1 & 436.4 & --- & 166.0 & 248.6 & 15.2 & 16.3 & 29.3 & 36.7 & 115.8 \\
Reach.-Only & 79.3 & 184.2 & 288.6 & 231.0 & 133.8 & 452.1 & --- & 141.8 & 263.8 & 15.8 & 17.7 & 29.7 & 37.0 & 112.4 \\
Novelty-Only & 83.9 & 178.5 & 275.1 & 239.0 & 138.7 & 480.5 & --- & 156.1 & 267.6 & 14.1 & 14.4 & 29.5 & 36.8 & 112.3 \\
AdaGoal & 94.9 & 225.3 & 292.0 & 225.5 & 119.2 & 493.7 & 1859.3 & 138.6 & 215.9 & 14.8 & 14.1 & 24.8 & 35.6 & 110.0 \\
DISCOVER & 96.7 & 222.1 & 270.7 & 213.6 & 107.6 & 533.2 & 2034.9 & 125.3 & 198.9 & 14.8 & 14.1 & 25.0 & 35.6 & 110.5 \\
Random & 7.7 & 15.4 & 33.8 & 18.9 & 11.0 & 53.7 & 247.1 & 10.7 & 22.6 & 10.9 & 12.4 & 18.4 & 26.5 & 102.0 \\
\bottomrule
\end{tabular}%
}
% \vspace*{-5pt}
\end{table}

\clearpage

\section{Environment Details}
\label{app:envs}

\textbf{Gridworlds.} Novel environments shown in Figure~\ref{fig:gridworlds}.
% \citep{Parisi_Gym-Gridworlds}. 
The observation is a one-hot encoding of the agent's tile. 
The goal space is $\statespace \!\times\! \actionspace$, i.e., the agent should do \textit{every action in all states}.
\vspace*{-3pt}
\begin{itemize}[leftmargin=*, itemsep=-2pt, topsep=0pt]
\item \stt{ThreeRoom}: three rooms separated by walls. The agent spawns non-uniformly across three positions (cyan tiles): 47.5\% chance in the first room, 47.5\% in the second, 5\% in the third. There are four actions: left, right, up, down. Episode horizon: 100 steps.
\item \stt{FourRoomStuck}: a variation of the classic four-room~\citep{sutton1999between}. The bottom-left room can be entered but not exited, and cannot be traversed freely due to one-way tiles. 
The three tiles above its north entrance (marked with a yellow question mark) are noisy, i.e., movement is randomized with 50\% probability.
There are four actions: left, right, up, down. Episode horizon: 200 steps. Episodes end when the agent reaches the green tile.
\item \stt{GridMaze}: maze with nine actions (left, right, up, down, up-left, down-left, up-right, down-right, stay). Exploration is hard due to the narrow passage near the starting position, that can be traversed only with ``up-right''. Episode horizon: 200 steps. Episodes end on action ``stay'' in the green tile.
\end{itemize}

\textbf{Classic control.} Open-sourced classic RL benchmarks~\citep{towers2024gymnasium} shown in Figure \ref{fig:classic_control}. The goal space is $\smash{\tilde{\statespace}} \!\times\! \actionspace$, where $\smash{\tilde{\statespace}}$ is a subset of the state space that depends on the environment.%
\footnote{Gymnasium's~\citep{towers2024gymnasium} documented bounds are ``advisory'' and do not correspond to the actual region the agent can visit. For example, \stt{CartPole}'s documented bounds are $[-4.8, 4.8]$ and $[-24^\circ, 24^\circ]$, but episodes terminate if the cart leaves $[-2.4, 2.4]$ or the pole falls outside $[-12^\circ, 12^\circ]$. Similarly, \stt{LunarLander} episodes terminate if $x$ leaves $[-1, 1]$, despite documented bounds of $[-2.5, 2.5]$.}
\vspace*{-3pt}
\begin{itemize}[leftmargin=*, itemsep=-2pt, topsep=0pt]
\item \stt{LunarLander}: the $(x, y)$ coordinate of the lander, in $[-1, 1] \times [-0.59, \infty]$ (unbounded).
\item \stt{LunarLander(Full)}: full eight-dimensional state space.
\item \stt{MountainCar}: the position $x \in [-1.2, 0.6]$ and the velocity $\dot x \in [-0.07, 0.07]$.
\item \stt{Pendulum}: the whole state space, i.e., the sine and cosine of the pendulum angle (both bounded in $[-1, 1]$) and its angular velocity (bounded in $[-8, 8]$). The default one-dimensional action is continuous and bounded; here we discretize it into eight equally spaced discrete actions.
\item \stt{Acrobot}: the sine and cosine of the joint angles, for a total of four dimensions bounded in $[-1, 1]$.
\item \stt{CartPole}: the $x$ coordinate of the cart and the angle of the pole, in $[-2.4, 2.4] \times [-12^\circ, 12^\circ]$.
\end{itemize}

\textbf{GCRL Control.} Open-sourced GCRL benchmarks~\citep{bortkiewicz2025accelerating} shown in Figure~\ref{fig:jax_environments}. The goal space is $\smash{\tilde\statespace}$; the action is continuous and not part of the goal space.
\vspace*{-3pt}
\begin{itemize}[leftmargin=*, itemsep=-2pt, topsep=0pt]
\item \stt{PointMaze-S}: a point-mass agent navigates a maze. The goal space is its $(x, y)$ planar position, bounded by the maze layout, i.e., $x \in [-11, 12], y \in [-11, 12]$. 
\item \stt{PointMaze-H}: the area (and the goal space) the agent navigates is four-dimensional, i.e., $[-6, 7]^4$. Its heatmaps only show the first two dimensions.
\item \stt{AntMaze-S/H}: like \stt{PointMaze-S}, but the agent is an ant-like quadruped. Goal bounds are $x \in [-17.5, 17.5], y \in [-17.5, 17.5]$ (S) and $x \in [-27.5, 27.5], y \in [-27.5, 27.5]$ (H).
\item \stt{ArmPush-H}: a Franka Panda pushes a green cube on a plane. The goal space is the cube's planar $(x, y)$ unbounded position. The blue/red region of the plane in Figure~\ref{fig:jax_environments}, located at $x \in [-0.45, -0.35], y \in [0.60, 0.70]$, is where the cube spawns at the beginning of an episode. %Its heatmaps only show the first two dimensions.
\end{itemize}

\begin{figure}[!t]
\centering
\begin{minipage}[c]{0.58\linewidth}
\centering
\begin{subfigure}[b]{0.26\textwidth}
\centering
\resizebox{\linewidth}{!}{%
\begin{tikzpicture}[
  empty/.style={fill=black, draw=white, line width=0.3pt},
  separator/.style={gray!60, line width=2.4pt},
  start/.style={draw=cyan, line width=0.3pt, fill=cyan},
  goal/.style={draw=green!70!black, line width=0.3pt, fill=green!70!black},
]
  \def\cs{0.20}
  \foreach \row [count=\y from 0] in {
    {.,.,.,.,.,.,.,.},
    {.,.,.,.,.,.,.,.},
    {.,.,.,.,.,.,.,.},
    {.,.,.,.,.,.,.,.},
    {.,.,.,.,.,.,.,.},
    {.,.,.,.,.,.,.,.},
    {.,.,.,.,.,.,.,.},
    {.,.,.,.,.,.,.,.},
    {.,.,.,.,.,.,.,.},
    {.}
  } {
    \ifnum\y<9
      \foreach \cell [count=\x from 0] in \row {
        \fill[empty] (\x*\cs, -\y*\cs) rectangle ++(\cs, -\cs);
        \draw[empty] (\x*\cs, -\y*\cs) rectangle ++(\cs, -\cs);
      }
    \fi
  }
  % Starting positions (white border keeps the fill inside the tile grid)
  \filldraw[fill=cyan, draw=white, line width=0.3pt] (0*\cs, -2*\cs) rectangle ++(\cs, -\cs);
  \filldraw[fill=cyan, draw=white, line width=0.3pt] (0*\cs, -3*\cs) rectangle ++(\cs, -\cs);
  \filldraw[fill=cyan, draw=white, line width=0.3pt] (0*\cs, -8*\cs) rectangle ++(\cs, -\cs);
  % Gray lines separating the three rooms (drawn last, on top of tiles)
  \draw[separator] (0, -3*\cs) -- (8*\cs, -3*\cs);
  \draw[separator] (0, -6*\cs) -- (8*\cs, -6*\cs);
\end{tikzpicture}%
}%
\caption{\stt{ThreeRoom}}
\label{fig:env-threeroom}
\end{subfigure}
\hfill
\begin{subfigure}[b]{0.38\textwidth}
\centering
\resizebox{\linewidth}{!}{%
\begin{tikzpicture}[
  empty/.style={fill=black, draw=white, line width=0.3pt},
  wall/.style={fill=gray!60, draw=white, line width=0.3pt},
  arrow/.style={->, >=stealth, red, line width=1.2pt},
  start/.style={fill=cyan, draw=white, line width=0.3pt},
  goal/.style={fill=green!70!black, draw=white, line width=0.3pt},
  qmark/.style={yellow, font=\fontsize{4}{4}\selectfont\bfseries},
]
  \def\cs{0.20}
  \foreach \row [count=\y from 0] in {
    {W,W,W,W,W,W,W,W,W,W,W,W,W},
    {W,.,.,.,.,.,W,.,.,.,.,.,W},
    {W,.,.,.,.,.,W,.,.,.,.,.,W},
    {W,Q,Q,Q,.,.,.,.,.,.,.,.,W},
    {W,W,.,W,W,W,W,W,.,.,.,.,W},
    {W,.,.,.,.,.,.,W,.,.,.,.,W},
    {W,.,.,.,.,.,.,W,.,.,.,.,W},
    {W,.,.,.,.,.,.,W,W,W,.,W,W},
    {W,.,.,.,.,.,.,W,.,.,.,.,W},
    {W,.,.,.,.,.,.,W,.,.,.,.,W},
    {W,.,.,.,.,.,.,.,.,.,.,.,W},
    {W,.,.,.,.,.,.,W,.,.,.,G,W},
    {W,W,W,W,W,W,W,W,W,W,W,W,W},
    {.}
  } {
    \ifnum\y<13
      \foreach \cell [count=\x from 0] in \row {
        \ifthenelse{\equal{\cell}{W}}{
          \fill[wall] (\x*\cs, -\y*\cs) rectangle ++(\cs, -\cs);
          \draw[wall] (\x*\cs, -\y*\cs) rectangle ++(\cs, -\cs);
        }{
          \ifthenelse{\equal{\cell}{G}}{
            \fill[goal] (\x*\cs, -\y*\cs) rectangle ++(\cs, -\cs);
            \draw[goal] (\x*\cs, -\y*\cs) rectangle ++(\cs, -\cs);
          }{
            \fill[empty] (\x*\cs, -\y*\cs) rectangle ++(\cs, -\cs);
            \draw[empty] (\x*\cs, -\y*\cs) rectangle ++(\cs, -\cs);
            \ifthenelse{\equal{\cell}{Q}}{
              \node[qmark] at (\x*\cs + \cs/2, -\y*\cs - \cs/2) {?};
            }{}
          }
        }
      }
    \fi
  }
  % (row 4, col 2): down
  \draw[arrow] (2*\cs + \cs/2, -4*\cs - 0.02) -- (2*\cs + \cs/2, -5*\cs + 0.02);
  % Column 3, rows 5-10: left
  \foreach \r in {5,...,10} {
    \draw[arrow] (4*\cs - 0.02, -\r*\cs - \cs/2) -- (3*\cs + 0.02, -\r*\cs - \cs/2);
  }
  % (row 5, col 4): left
  \draw[arrow] (5*\cs - 0.02, -5*\cs - \cs/2) -- (4*\cs + 0.02, -5*\cs - \cs/2);
  % Column 4, rows 6-10: right
  \foreach \r in {6,...,10} {
    \draw[arrow] (4*\cs + 0.02, -\r*\cs - \cs/2) -- (5*\cs - 0.02, -\r*\cs - \cs/2);
  }
  % Row 11, cols 3-4: right
  \foreach \c in {3,4} {
    \draw[arrow] (\c*\cs + 0.02, -11*\cs - \cs/2) -- (\c*\cs + \cs - 0.02, -11*\cs - \cs/2);
  }
  % (row 10, col 7): left
  \draw[arrow] (8*\cs - 0.02, -10*\cs - \cs/2) -- (7*\cs + 0.02, -10*\cs - \cs/2);
  % Starting position (top-left, just inside the wall)
  \fill[start] (1*\cs, -1*\cs) rectangle ++(\cs, -\cs);
\end{tikzpicture}%
}%
\caption{\stt{FourRoomStuck}}
\label{fig:env-fourroomstuck}
\end{subfigure}
\hfill
\begin{subfigure}[b]{0.33\textwidth}
\centering
\resizebox{\linewidth}{!}{%
\begin{tikzpicture}[
  empty/.style={fill=black, draw=white, line width=0.3pt},
  wall/.style={fill=gray!60, draw=white, line width=0.3pt},
  start/.style={fill=cyan, draw=white, line width=0.3pt},
  goal/.style={fill=green!70!black, draw=white, line width=0.3pt},
]

  \def\cs{0.20}
  \foreach \row [count=\y from 0] in {
    {.,.,.,.,.,.,.,.,.,.,.,.},
    {.,W,.,.,.,.,.,.,.,.,.,.},
    {W,.,.,.,.,W,W,W,W,.,W,.},
    {.,W,.,.,W,W,.,.,.,.,W,.},
    {.,.,.,W,W,.,.,.,.,.,W,.},
    {.,.,W,W,.,.,.,.,.,.,W,.},
    {.,W,W,.,.,.,.,.,W,W,W,.},
    {.,.,W,.,.,.,.,.,.,.,W,.},
    {.,.,.,.,.,W,W,W,.,.,W,.},
    {W,W,.,.,W,.,.,.,.,.,W,.},
    {.,.,W,.,W,.,.,.,.,.,W,.},
    {.,.,W,.,.,.,.,.,.,.,.,.},
    {.}
  } {
    \ifnum\y<12
      \foreach \cell [count=\x from 0] in \row {
        \ifthenelse{\equal{\cell}{W}}{
          \fill[wall] (\x*\cs, -\y*\cs) rectangle ++(\cs, -\cs);
          \draw[wall] (\x*\cs, -\y*\cs) rectangle ++(\cs, -\cs);
        }{
          \fill[empty] (\x*\cs, -\y*\cs) rectangle ++(\cs, -\cs);
          \draw[empty] (\x*\cs, -\y*\cs) rectangle ++(\cs, -\cs);
        }
      }
    \fi
  }
  % Starting position (bottom-left)
  \filldraw[start] (0*\cs, -11*\cs) rectangle ++(\cs, -\cs);
  % Goal position (row 1, col 10)
  \filldraw[goal] (10*\cs, -1*\cs) rectangle ++(\cs, -\cs);
\end{tikzpicture}%
}%
\caption{\stt{GridMaze}}
\label{fig:env-gridmaze}
\end{subfigure}
\end{minipage}
\hfill
\begin{minipage}[c]{0.4\linewidth}
\caption{\label{fig:gridworlds}\textbf{Gridworlds.} Black tiles are empty; gray tiles are walls; cyan tiles are starting positions; green tiles are terminal positions where the episode ends. Red arrows mark one-way tiles where only the action matching the arrow succeeds. In yellow \stt{?} tiles, movement is randomized with 50\% probability.}
\end{minipage}
% \end{figure}
% \begin{figure}[!t]
% \centering
\\[6pt]
\setlength{\fboxsep}{0pt}%
\setlength{\fboxrule}{0.4pt}% border thickness, adjust to taste
\begin{subfigure}[b]{0.17\textwidth}
\centering
\fbox{\includegraphics[trim={1em 0 1em 0}, clip, width=\linewidth,keepaspectratio]{fig/env/lunar}}
\caption{\stt{LunarLander}}
\label{fig:env-lunar}
\end{subfigure}
\hfill
\begin{subfigure}[b]{0.17\textwidth}
\centering
\fbox{\includegraphics[width=\linewidth,keepaspectratio]{fig/env/mountain}}
\caption{\stt{MountainCar}}
\label{fig:env-mountain}
\end{subfigure}
\hfill
\begin{subfigure}[b]{0.13\textwidth}
\centering
\fbox{\includegraphics[trim={0 0 0 2.7em}, clip, width=\linewidth,keepaspectratio]{fig/env/pendulum}}
\caption{\stt{Pendulum}}
\label{fig:env-pendulum}
\end{subfigure}
\hfill
\begin{subfigure}[b]{0.12\textwidth}
\centering
\fbox{\includegraphics[trim={0 0 0 2em}, clip, width=\linewidth,keepaspectratio]{fig/env/acrobot}}
\caption{\stt{Acrobot}}
\label{fig:env-acrobot}
\end{subfigure}
\hfill
\begin{subfigure}[b]{0.13\textwidth}
\centering
\fbox{\includegraphics[width=\linewidth,keepaspectratio]{fig/env/cart}}
\caption{\stt{CartPole}}
\label{fig:env-cart}
\end{subfigure}
\hfill
\begin{minipage}[b]{0.22\textwidth}
\caption{\label{fig:classic_control}\textbf{Classic control} environments with continuous states and discrete actions.}
\end{minipage}
\end{figure}

\clearpage

\begin{figure}[!t]
\centering
\begin{subfigure}[b]{0.17\textwidth}
\centering
\includegraphics[width=\linewidth,keepaspectratio]{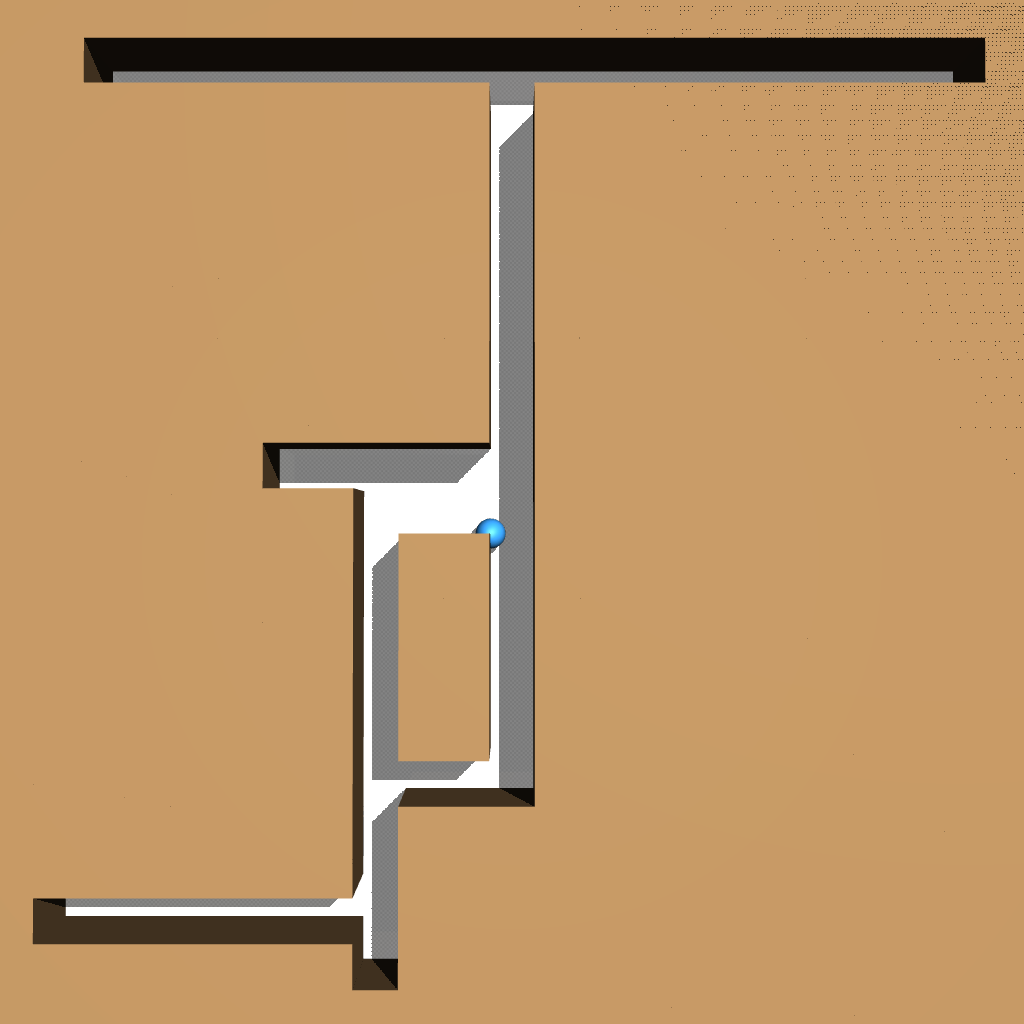}
\caption{\stt{PointMaze-S}}
\label{fig:env-point_maze_s}
\end{subfigure}
\hfill
\begin{subfigure}[b]{0.17\textwidth}
\centering
\includegraphics[width=\linewidth,keepaspectratio]{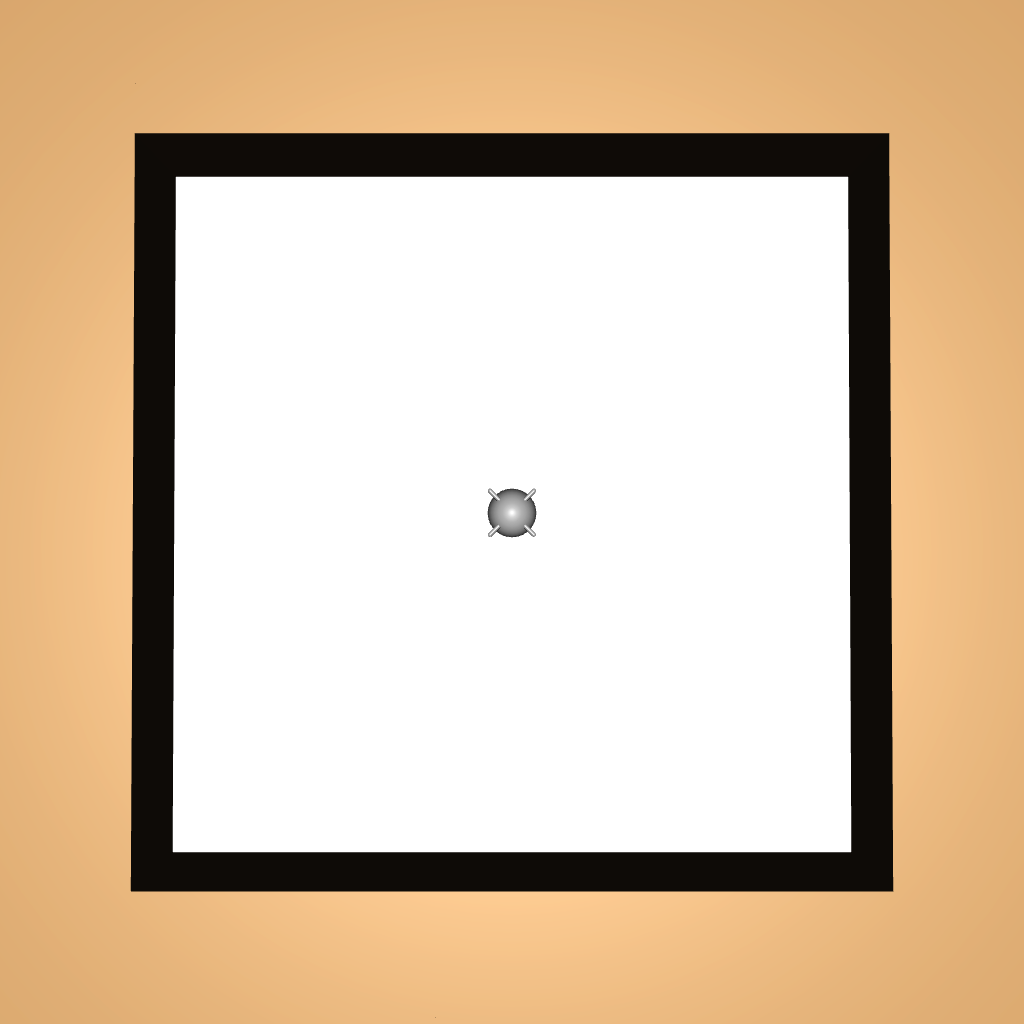}
\caption{\stt{PointMaze-H}}
\label{fig:env-point_maze_h}
\end{subfigure}
\hfill
\begin{subfigure}[b]{0.17\textwidth}
\centering
\includegraphics[width=\linewidth,keepaspectratio]{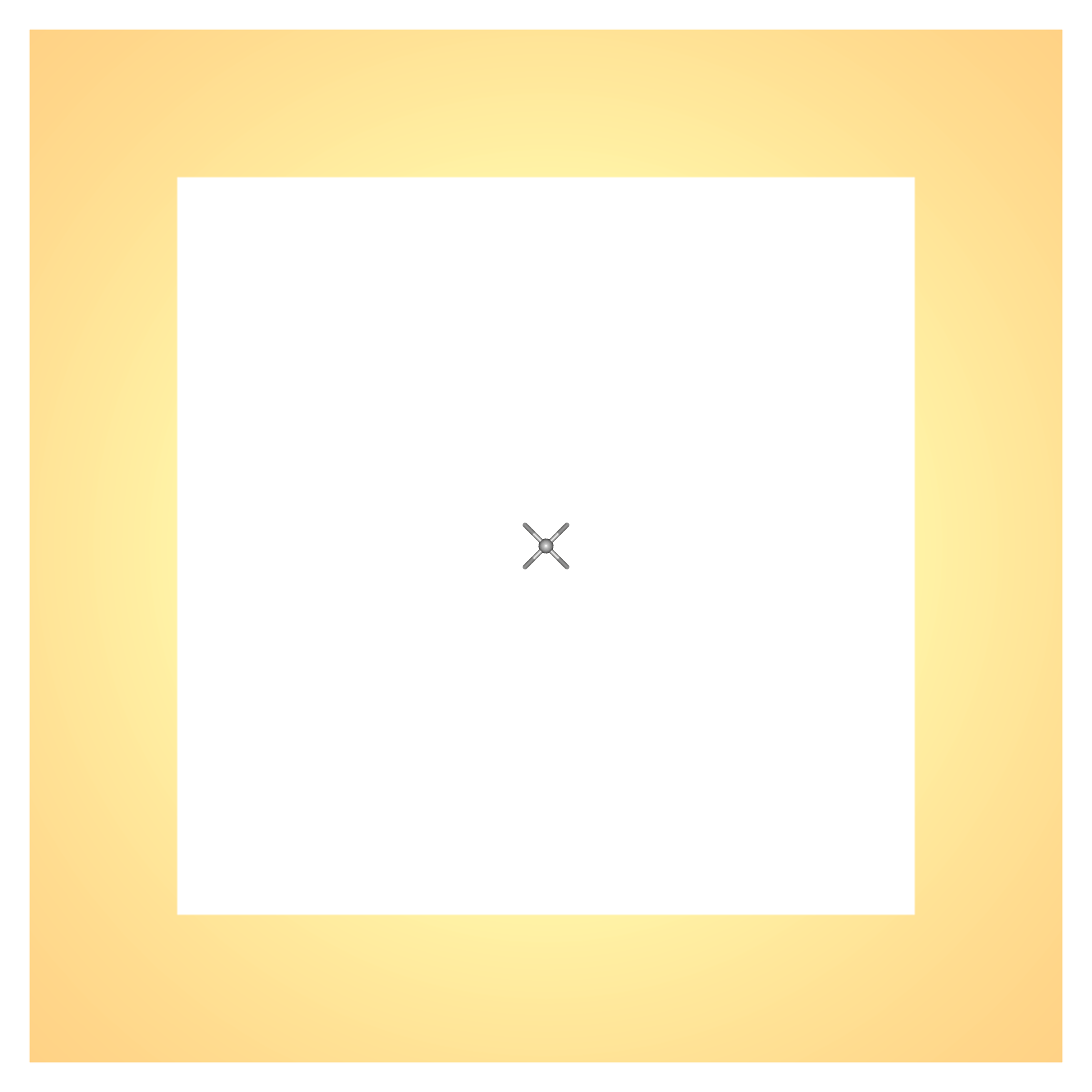}
\caption{\stt{AntMaze-S}}
\label{fig:env-ant_maze_s}
\end{subfigure}
\hfill
\begin{subfigure}[b]{0.17\textwidth}
\centering
\includegraphics[width=\linewidth,keepaspectratio]{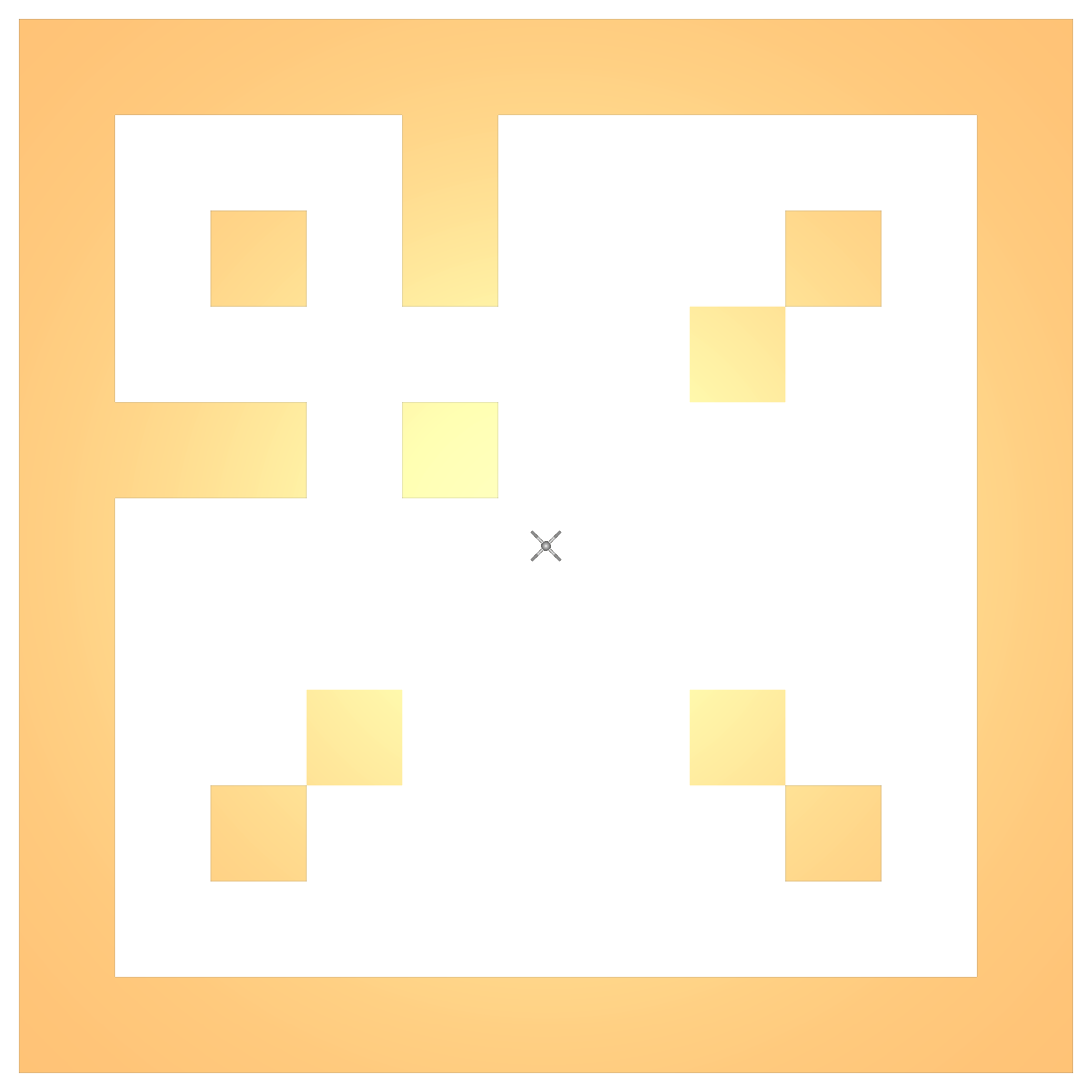}
\caption{\stt{AntMaze-H}}
\label{fig:env-ant_maze_h}
\end{subfigure}
\hfill
\begin{subfigure}[b]{0.24\textwidth}
\centering
\includegraphics[width=\linewidth,keepaspectratio]{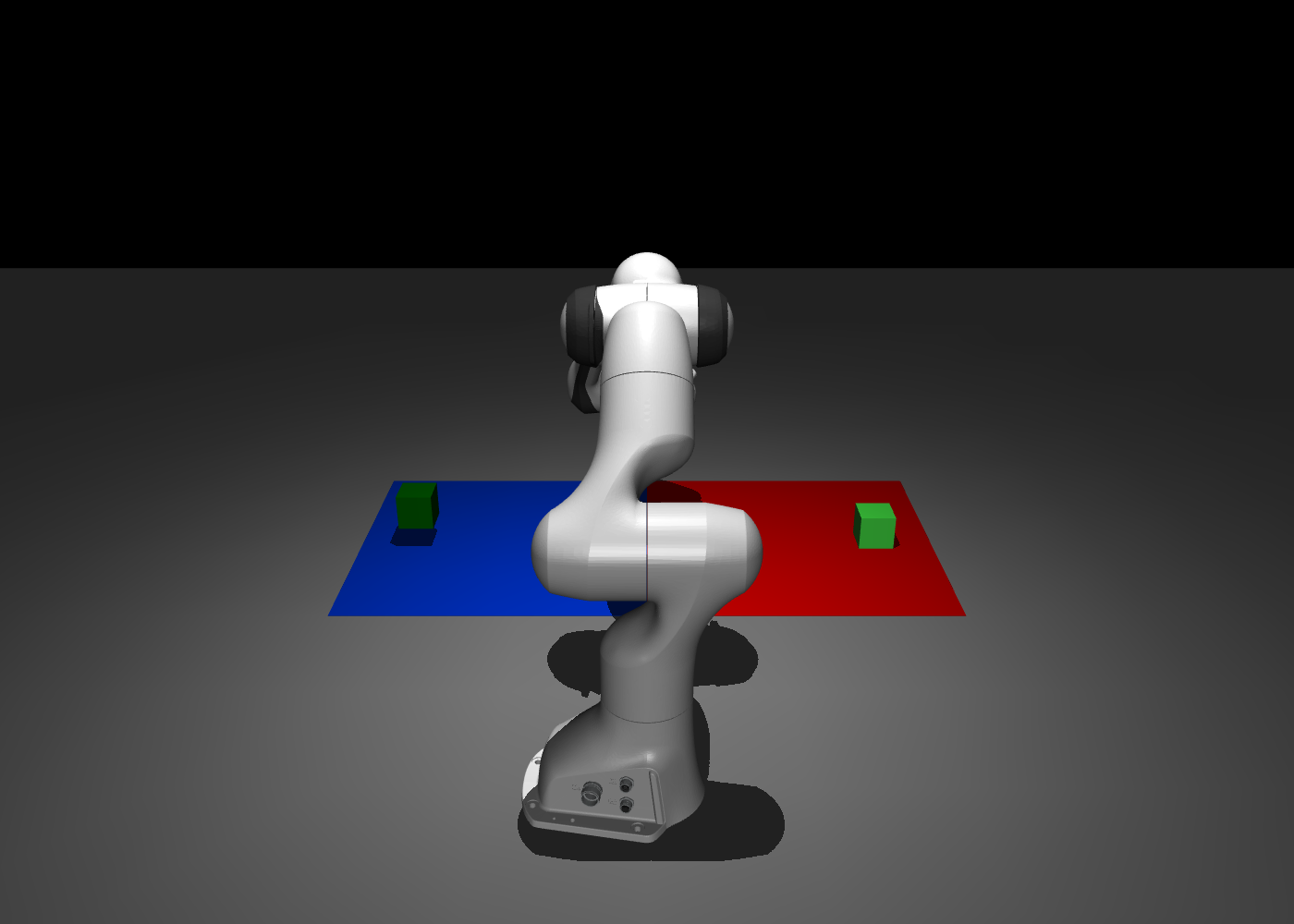}
\caption{\stt{ArmPush-H}}
\label{fig:env-arm_push_h}
\end{subfigure}
\caption{\label{fig:jax_environments}\textbf{GCRL environments} with continuous states and actions. \stt{PointMaze-H} is actually four-dimensional, and Figure \subref{fig:env-point_maze_h} shows the two-dimensional projection.}
\end{figure}

\textbf{Evaluation metrics.} In Gridworlds, Shannon entropy and coverage can be computed exactly, since the set of visitable states is finite and known. In control tasks, we discretize the continuous goal space into 50 bins per dimension, which gives accurate estimates of both metrics. These estimates can be conservative, though: \stt{Acrobot} and \stt{MountainCar}, for example, have unreachable position-velocity configurations, yet we normalize entropy and coverage as if the entire binned space were visitable.
Note also that the goal spaces of \stt{LunarLander} and \stt{ArmPush-H} are unbounded. In the former, the $y$ position has no upper limit, and we compute metrics using $[-0.25, 10]$ as its bounds. In the latter, both $x$ and $y$ are unbounded, and we compute metrics using $[-1.00, 1.00] \times [0.15, 1.15]$.
\\[2pt]
For \stt{LunarLander(Full)}, the goal space is eight-dimensional: six continuous coordinates and two binary ones. Here no choice of bin resolution yields an informative occupancy distribution. A coarse grid assigns distant states to a common cell, so the counts say little about how dispersed the buffer is. A fine grid fails in the opposite way: in six dimensions almost every state falls in a cell of its own, every occupied cell holds a single sample, and the occupancy distribution is uniform by construction. The normalized entropy then reduces to $\log(\text{occupied cells}) / \log(\text{total cells})$, a monotone function of coverage that adds nothing to it.
Figure~\ref{fig:lunar_full_stats} shows this with 50 bins per continuous coordinate. Coverage stays close to zero in all algorithms, so the occupied cells carry comparable mass and entropy is determined almost entirely by how many of them there are. A method whose visits cluster can then score a higher entropy than one whose visits are more dispersed, purely by occupying marginally more cells (Figure~\ref{fig:kl_entropy_example}).
\\[2pt]
This is why in Figure~\ref{fig:main_results}, and again in Figure~\ref{fig:lunar_full_stats}, we report an approximate continuous differential entropy via the Kozachenko-Leonenko $k$-NN estimator~\citep{kozachenko1987sample}. It measures the density around each sample rather than within a fixed grid, so it keeps tracking how the buffer redistributes throughout training and separates methods cleanly. We evaluate it on the whole buffer, with neighbors found by a $k$-d tree:
\begingroup
\setlength{\abovedisplayskip}{2pt}
\setlength{\belowdisplayskip}{2pt}
\begin{equation}
\hat H
= \psi(n) - \psi(k) + \log c_d + \frac{d}{n} \sum_{i=1}^n \log \rho_k(i),
\label{eq:kl_entropy}
\end{equation}
\endgroup
where $n$ is the number of samples, $\rho_k(i)$ is the Euclidean distance from $s_i$ to its $k$-th nearest neighbour ($k = 10$), $\psi$ is the digamma function, and $\smash{c_d = \pi^{d/2} / \Gamma(d/2 + 1)}$ is the volume of the unit ball in $\mathbb{R}^d$.
% \\[2pt]
Note that the goal space of \stt{LunarLander(Full)} has two binary coordinates (the leg-contact indicators). We treat them as discrete variables alongside the action. The buffer is therefore split into groups, one per combination of action and contact state, and Eq.~\eqref{eq:kl_entropy} is applied within each group over the remaining $d = 6$ continuous coordinates. The reported entropy is the Shannon entropy of the groups plus the average of the within-group estimates, weighted by group size.

% \begin{figure}[!b]
% \centering
% \begin{minipage}[c]{0.4\textwidth}
% \centering
% \includegraphics[width=\linewidth]{fig/entropy_kl.pdf}
% \end{minipage}\hfill
% \begin{minipage}[c]{0.56\textwidth}
% % \tiny%
% \centering%
% \begin{tabular}{lrrr}
% \toprule
% & $\sigma$ (spread) & $H_\textrm{Shannon}$ & $H_\textrm{k-NN}$ \\
% \midrule
% tight fill & 0.0004 & 2.1972 & $-10.6512$ \\
% wide fill  & 0.0020 & 2.1972 & $-7.4323$ \\
% \bottomrule
% \end{tabular}
% \end{minipage}
% \caption{\label{fig:kl_entropy_example}Two distributions occupying the \emph{same} nine bins with equal
% counts. The insets zoom into one bin to reveal the tight vs.\ wide fill.
% Shannon entropy is identical ($\log 9$); $k$-NN differential entropy differs by >3 nats (note that can be negative).}
% \end{figure}

\begin{figure}[!b]
\centering
\begin{minipage}[c]{0.42\textwidth}
\centering
\includegraphics[width=\linewidth]{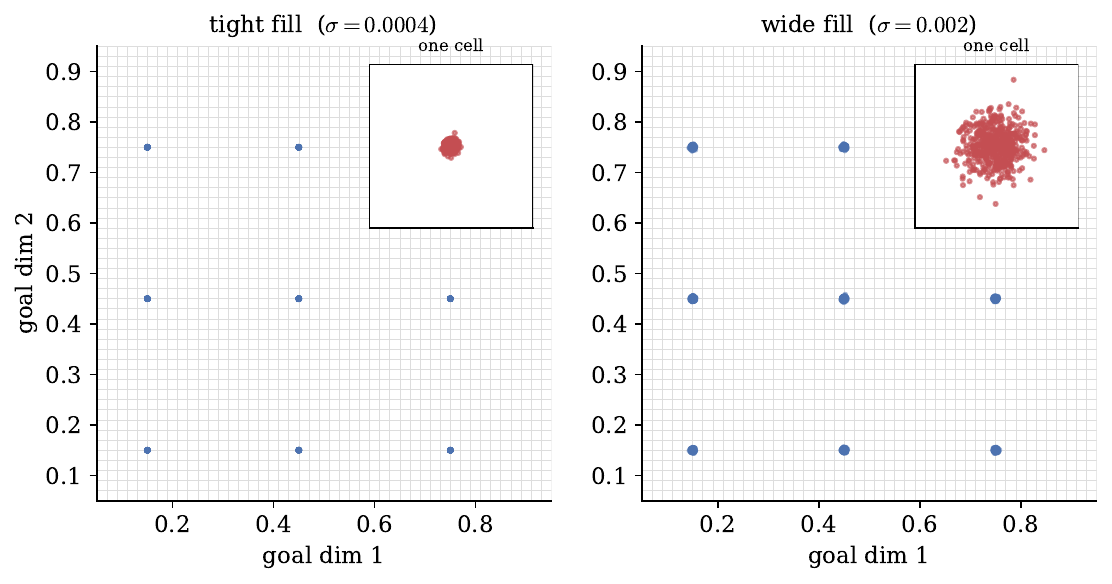}
\end{minipage}\hfill
\begin{minipage}[c]{0.57\textwidth}
\begin{minipage}[t]{0.42\linewidth}
\vspace{3pt}
\centering
\resizebox{\textwidth}{!}{%
\begin{tabular}{@{}lrr@{}}
\toprule
& $H_\textrm{Shannon}$ & $H_\textrm{k-NN}$ \\
\midrule
tight fill & 2.1972 & $-10.6512$ \\
wide fill & 2.1972 & $-7.4323$ \\
\bottomrule
\end{tabular}\par}
\end{minipage}\hfill
\begin{minipage}[t]{0.54\linewidth}
\setlength{\abovecaptionskip}{1pt}%
\caption{\label{fig:kl_entropy_example}\textbf{Binned entropy is blind to spread.} Two distributions occupying the \emph{same} nine\parfillskip=0pt}
\end{minipage}\par
\noindent bins with equal counts but different spread $\sigma$. The insets zoom into one bin to reveal the tight vs.\ wide fill. Shannon entropy is identical ($\log 9$); $k$-NN differential entropy differs by >3 nats (note that the latter can be negative).
\end{minipage}
\end{figure}

\begin{figure}[!t]
\centering
\includegraphics[width=0.395\linewidth]{\detokenize{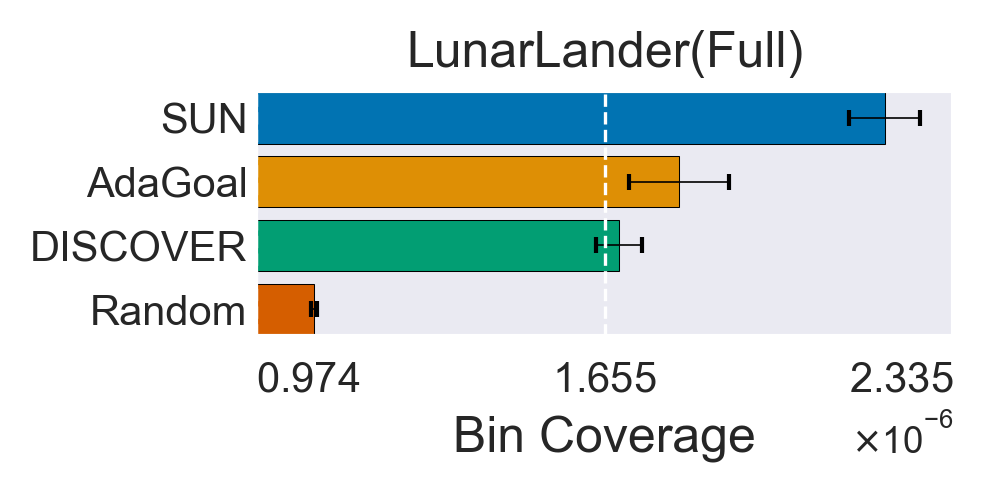}}\hfill
\includegraphics[trim={6em 0 0 0}, clip, width=0.295\linewidth]{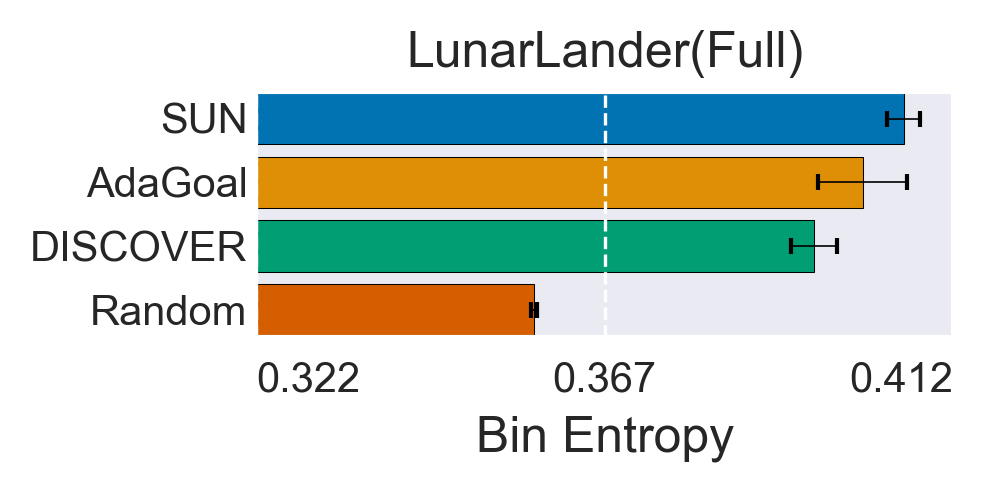}\hfill
\includegraphics[trim={6em 0 0 0}, clip, width=0.295\linewidth]{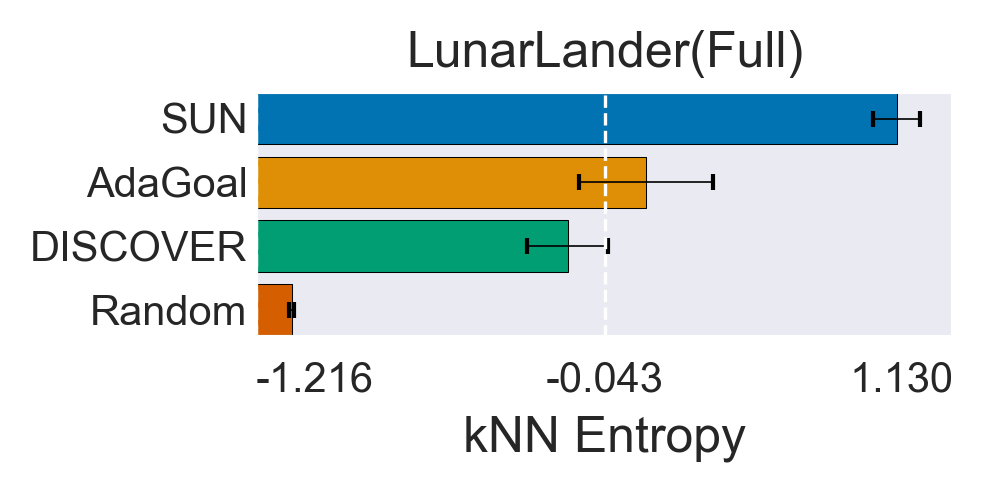}
\caption{\label{fig:lunar_full_stats}\textbf{\stt{LunarLander(Full)} results}, computed with 50 bins (left, center) and with Eq.~\eqref{eq:kl_entropy} (right). All algorithms occupy a negligible fraction of the binned space, so every occupied cell holds roughly one sample and the binned entropy is determined by how many cells are occupied rather than by how the visits are distributed among them --- it therefore adds nothing to coverage (see also Figure~\ref{fig:kl_entropy_example}). The $k$-NN estimator measures the density around each sample instead, and separates the methods clearly, with SUN attaining the highest entropy.}
\end{figure}

\section{Training Hyperparameters}
\label{app:hyper}
\begin{tcolorbox}[colback=gray!10, colframe=gray!40, boxrule=0.5pt, arc=2pt, left=6pt, right=6pt, top=4pt, bottom=4pt]
Hyperparameters have not been tuned, with the exception of the pseudocount radius $\rho$ (see Appendix \ref{app:pseudocount_ablation}). For GCRL control, we used DISCOVER official codebase. For Gridworlds and Classic control, we implemented a simple version of DQN and used common hyperparameters without any tuning.
\end{tcolorbox}

\textbf{Replay buffer.} Before training starts, the replay buffer is ``warmed up'' with data collected with a random-action policy. For Gridworlds and Classic control, we collect 10,000 samples (except \stt{ThreeRoom}, where we collect 5,000). For GCRL control, we collect 1,000.
The replay buffer then stores all samples collected until the end of training, i.e., no data is ever evicted.
\\[2pt]
\textbf{Goal-selection.} SUN and its ablated versions draw a pool of 256 candidates from the replay buffer at every re-selection. Episodic algorithms (SUN Episodic, DISCOVER, AdaGoal) instead draw 2500 candidates once, at the beginning of the episode.
% For GCRL control the pool is smaller and set per environment: 256 candidates on \stt{PointMaze-S/H}, 512 on \stt{AntMaze-S/H} and \stt{ArmPush-H} (the only exception is SUN, that draws 512 candidates on \stt{PointMaze-H}). \simone{this is bad, SUN should not draw more. And episodic should draw 2500}
\\[2pt]
\textbf{Goal-reached.} In GCRL tasks, the threshold $\eta$ in Algorithm~\ref{alg:sun}:7 is given by the environment. In Classic control, it is the bin width of the discretized counts used for the heatmaps. Gridworlds states are discrete and no threshold is needed.
\\[2pt]
\textbf{Algorithm-specific.} SUN's pseudocounts are computed using $\rho = 0.25$ ($0.5$ in \stt{L.Lander(Full)}).
DISCOVER's novelty coefficient is $\beta = 10$ (for Gridworlds and Classic control) and $\beta = 1$ (for GCRL control).

\begin{table}[!b]
\centering
\setlength{\tabcolsep}{4pt}
\renewcommand{\arraystretch}{1.01}
\captionsetup{skip=2pt}
\caption{\textbf{Hyperparameters} used in our experiments.}
\label{tab:hyperparams}

\begin{minipage}[t]{0.465\textwidth}
\vspace{0pt}
\centering
\resizebox{\textwidth}{!}{%
\begin{tabular}{ll}
\toprule
\textbf{Gridworlds and Classic control} & \\
\midrule
Discount factor $\gamma$           & $0.99$ \\
Trace factor $\lambda$             & $0.95$ \\
Target network copy frequency      & $1$ \\
Target network Polyak coefficient  & $0.001$ \\
Minibatch size                     & 16 \\
TD($\lambda$) horizon $T$          & 16 \\
Update frequency                   & 1 step \\
Updates per step                   & $1$ \\
Clip reward                        & False \\
Optimizer                          & AdamW \\
Learning rate                      & $10^{-3}$ \\
Loss                               & Huber \\
Gradient norm clipped at           & $1.0$ \\
\bottomrule
\end{tabular}%
}
\end{minipage}%
\hfill
\begin{minipage}[t]{0.525\textwidth}
\vspace{0pt}
\centering
\resizebox{\textwidth}{!}{%
\begin{tabular}{ll}
\toprule
\textbf{GCRL control} & \\
\midrule
Discount factor $\gamma$           & $0.99$ \\
Trace factor $\lambda$             & $0.95$ \\
Target network Polyak coefficient (critic) & $0.005$ \\
Target network Polyak coefficient (actor)  & $5\!\!\times\!\!10^{-7}$ \\
Target networks copy frequency     & $2$ \\
Policy delay                       & $2$ \\
Minibatch size                     & 256 \\
TD($\lambda$) horizon $T$          & 3 \\
Update frequency                   & 1 step \\
Updates per step                   & $1$ \\
Target-policy smoothing std        & $0.2$ \\
Target-action noise clip           & $0.5$ \\
Optimizer                          & AdamW \\
Learning rate                      & $10^{-3}$ \\
Critic loss                        & MSE \\
Gradient norm clipped at           & --- \\
\bottomrule
\end{tabular}%
}
\end{minipage}

\end{table}

\clearpage

\subsection{Gridworlds and Classic Control}
\label{subsec:dqn_details}
The goal is composed of a state component $g_s$ (environment-specific, see Section \ref{app:envs}), and an action component $g_a$. 
That is, we explicitly learn SVFs whose goal is to perform specific actions in specific states. 
Because actions are discrete, we learn $\Qtheta(s, a, g_s, g_a)$: the Q-network takes state and the goal-state as input, and outputs the action-value for every goal-action. 
\\[2pt]
$\Qtheta$ is trained with Double DQN \citep{hasselt2010double} with TD($\lambda$) using Watkins's cutting traces \citep{watkins1992q} with two practical modifications motivated by the bounded scale of the SVFs in $[0, 1]$.\footnote{When $(g_s, g_a) = (s, a)$ the reward in Eq. \eqref{eq:w-function} is 1 and the transition is terminal, otherwise the reward is 0.}
\textit{First}, to mitigate overestimation bias, TD targets clip $\max_a \Qtheta(s', a, g_s, g_a)$ to 1. \textit{Second}, we relax the strict $\arg\max$ used by Watkins's cutting: an action is considered greedy if its Q-value lies within $1\!-\!\gamma$ of the maximum. 
This tolerance matches the scale of one-step Bellman backups, and prevents traces from being cut too aggressively by neural network approximation error, to which a strict $\arg\max$ is overly sensitive. 
\\[2pt]
The policy in Algorithm \ref{alg:sun} is $\varepsilon$-greedy ($\varepsilon \!=\! 0.1$) with respect to the critic $\Qtheta$. Action-level noise is needed with neural-network approximators, and is aligned with DISCOVER official implementation (with built-in noise via Gaussian perturbations on the actor's output, see below).
% Note that this does not go against what claimed in Section \ref{sec:method}: the noise introduced is minimal, constant, and not tuned (unlike classic methods where it is either learned or carefully decayed).
\\[2pt]
AdaGoal and DISCOVER use four critics, each with its own target network (DQN-style). At every update, one randomly selected critic is updated, using the target network of another randomly selected critic.\footnote{Updating all critics at every step performed worse on average and increased wall-clock time by $\times3.5$.}
Their policy is $\varepsilon$-greedy with respect to one fixed critic per episode, selected randomly. 
%\footnote{\citeauthor{tarbouriech2022adaptive} apply the $\softmax$ operator rather than $\max$. Here, we use $\max$ for consistency with SUN and DISCOVER, but all methods can use $\softmax$ as well.}
\\[2pt]
For training, we uniformly sample 16 batches from the replay buffer, and append all the following $T-1$ samples, for a total of 256 samples. These sequences may have samples from different consecutive trajectories, but are kept separate thanks to truncation flags.
Then, we relabel goals with HER (Section~\ref{app:notes}) with one modification. The original HER ``future'' strategy assigns each timestep its own future goal, yielding $T$ sub-sequences of \textit{varying} length. For more efficient batch training, we implement a ``segmented'' variant: we concatenate contiguous sub-sequences (each with its own future goal) so that their total length matches the original $T$ (Figure~\ref{fig:her-segmented}). Relabeling is repeated 4 times per sequence, so each state is trained against 4 future goals. These samples are all positives; to balance them, we draw 4 random negatives from the replay buffer (one per full sequence of length $T = 16$). As a result, each DQN update uses $16\!\times\!16\!\times\!4\!\times\!2 = 2,048$ data points.

\begin{figure}[!t]
\centering
\begin{minipage}[c]{0.42\linewidth}
\centering
\resizebox{\linewidth}{!}{
\begin{tikzpicture}[
  trajnode/.style={circle, draw=black!50, fill=gray!15, minimum size=6mm, font=\small, inner sep=0pt},
  goalnode/.style={circle, draw=blue!70!black, fill=blue!20, minimum size=6mm, font=\small, line width=0.7pt, inner sep=0pt},
  negnode/.style={circle, draw=red!60!black, fill=red!15, minimum size=6mm, font=\small, line width=0.7pt, inner sep=0pt},
  mainarrow/.style={-{Stealth[length=2mm]}, line width=0.8pt, black},
  assign/.style={-{Stealth[length=1.5mm]}, line width=0.5pt, blue!60, dashed},
  seg/.style={rounded corners=2pt, line width=0.7pt},
]
% Main trajectory (top row): s_0 ... s_6, goal g
\foreach \i in {0,...,6} {
  \node[trajnode] (m\i) at (\i*1.0, 3.0) {$s_\i$};
}
\node[font=\Large, gray!70] at (6.7, 3.0) {...};
\node[negnode] (mg) at (7.4, 3.0) {$g$};
\foreach \i [evaluate=\i as \j using int(\i+1)] in {0,...,5} {
  \draw[mainarrow] (m\i) -- (m\j);
}

% Segment brackets + assigned goals (below)
% Segment [0,2] -> goal s_2
\draw[seg, blue!60] (-0.3, 2.35) -- (-0.3, 2.15) -- (2.3, 2.15) -- (2.3, 2.35);
\node[goalnode] (g0) at (1.0, 1.6) {$s_2$};
\draw[assign] (1.0, 2.1) -- (g0);

% Segment [3,4] -> goal s_4
\draw[seg, blue!60] (2.7, 2.35) -- (2.7, 2.15) -- (4.3, 2.15) -- (4.3, 2.35);
\node[goalnode] (g1) at (3.5, 1.6) {$s_4$};
\draw[assign] (3.5, 2.1) -- (g1);

% Segment [5,6] -> goal s_6
\draw[seg, blue!60] (4.7, 2.35) -- (4.7, 2.15) -- (6.3, 2.15) -- (6.3, 2.35);
\node[goalnode] (g2) at (5.5, 1.6) {$s_6$};
\draw[assign] (5.5, 2.1) -- (g2);
\end{tikzpicture}}
\end{minipage}%
\hfill
\begin{minipage}[c]{0.55\linewidth}
\caption{\textbf{Example of HER ``segmented future'' relabeling.} The trajectory is split into \textit{segments}: a \textit{future} cut timestep $k$ is sampled, its state $s_k$ (blue) is assigned as the goal to all timesteps in the segment, and the next segment starts at $k+1$. Here, the cuts fall at $2, 4, 6$.}
\label{fig:her-segmented}
\end{minipage}
\vspace*{-1.em}
\end{figure}

% \begin{figure}[t]
% \begin{minipage}{\columnwidth}
% \begin{wrapfigure}{l}{0.481\columnwidth}
%   \vspace{-\baselineskip}
%   \vspace*{-3pt}
%   \centering
%   \input{fig/segmented_her}
%   \vspace{-\baselineskip}
%   \vspace*{-10pt}
% \end{wrapfigure}
% \refstepcounter{figure}\label{fig:her-segmented}%
% \noindent\textbf{Figure \thefigure:} \textbf{Example of HER ``segmented future'' relabeling.} The trajectory is split into contiguous \textit{segments}: a \textit{future} cut timestep $k$ is sampled, its state $s_k$ (blue) is assigned as the goal to all timesteps in the segment, and the next segment starts at $k+1$. In this example the cuts fall at $2, 4, 6$.
% \end{minipage}
% \end{figure}

\subsection{GCRL Control}
\label{subsec:td3_details}
We build on DISCOVER official implementation, and learn $\Vtheta(s, g_s)$ and a goal-conditioned policy $\pi(a \spacedmid s, g_s)$. The action (continuous) is not part of the goal.
% \\[2pt]
$\Vtheta$ is trained with TD3 \citep{fujimoto2018addressing} with TD($\lambda$) targets \textit{without} Watkins's cutting traces (with continuous actions, $\argmax$ is infeasible) and \textit{without} clipping $\Vtheta(s', g_s)$ to 1. 
% \\[2pt]
Once the goal is selected, the policy explores with noise $\mathcal{N}(0, 0.4)$ added to the action, as in the official DISCOVER implementation.
\\[2pt]
AdaGoal and DISCOVER use an ensemble of four critics, each with its own target network (DQN-style). All are trained at every update, with different random batches, and with TD targets from its own target network. 
There is one policy, trained against the average value returned by all critics.
\\[2pt]
All algorithms relabel goals with HER (Section~\ref{app:notes}), without modification. At each update, a batch of 256 samples is drawn uniformly from the replay buffer. For a randomly selected half, we sample a future goal from within the next $T_{\max} = 50$ steps of the original trajectory (or until termination, whichever comes first) and compute TD($\lambda$) targets up to that goal. For the other half, we sample the relabel goal uniformly from the buffer and compute one-step TD targets.\footnote{The original DISCOVER implementation uses one-step TD targets rather than TD($\lambda$). In our experiments, however, eligibility traces improved the performance of all algorithms.} %, consistent with the multi-step HER of \citet{yang2021bias}.} 
Note that, unlike in DQN, the intermediate steps between a batch point and its relabeled goal are used only to compute the TD($\lambda$) target, \textit{not} for gradient updates. Each TD3 update therefore uses exactly 256 data points.

\clearpage

\subsection{Network Architecture}
\label{subsec:network}

Figure \ref{fig:network} shows the network architecture used by all algorithms.
We observed that the following components improved the performance of \textit{all} algorithms.
\begin{itemize}[leftmargin=*, itemsep=0.5pt, topsep=1pt]
\item The Gaussian encoding helped mildly in all environments but significantly in \stt{LunarLander}. We suspect this is due to \stt{LunarLander}'s distinctive observation space: it is unbounded, asymmetric, and its elements have differing scales.
\item Feature fusion helped significantly in all environments. Concatenation alone forces the network to learn the interaction between state and goal from scratch, whereas the element-wise product $f_s \odot f_g$ makes the alignment of corresponding features directly available to the downstream layers.
\item Having separate layers before the fusion was necessary. Encoding state and goal with separate streams and combining them late is standard in goal- and task-conditioned value functions~\citep{borsa2019universal,eysenbach2022contrastive}. One may be tempted to use a single shared set of layers when the observation space and the goal space are identical (e.g., in Gridworlds, \stt{Pendulum}, or \stt{MountainCar}). Yet this made learning accurate SVFs significantly harder, especially in Gridworlds where some states are unreachable. With shared layers, states and goals are mapped by the same function, so two tiles that look similar receive similar representations regardless of which is the agent's position and which the target. The network then generalizes reachability across them, and assigns non-zero value to goals that are near in observation space but unreachable in the environment.
\item SiLU activation performed better than ReLU and LeakyReLU in all environments.
\item Combining LayerNorm with near-zero initialization and no bias in all but the last linear layer improved stability, in line with the findings of \citet{elsayed2026streaming}.
\end{itemize}

\begin{figure}[!t]
\begin{minipage}[c]{0.345\textwidth}
\centering
% Requires: \usepackage{tikz}
%           \usepackage{caption}   % for \captionof (already loaded by subcaption)
%           \usetikzlibrary{positioning, arrows.meta, calc, backgrounds, fit}

% ── Shared TikZ styles ──────────────────────────────────────────────────
\tikzset{
  io/.style    = {font=\scriptsize},
  enc/.style   = {draw, rounded corners=3pt, fill=cyan!12,
                  minimum width=1.6cm, minimum height=0.55cm,
                  align=center, font=\scriptsize},
  sub/.style   = {draw, rounded corners=2pt, fill=orange!18,
                  minimum width=1.5cm, minimum height=0.45cm,
                  align=center, font=\scriptsize},
  subv/.style  = {draw, rounded corners=2pt, fill=violet!18,
                  minimum width=1.5cm, minimum height=0.45cm,
                  align=center, font=\scriptsize},
  mlin/.style  = {draw, rounded corners=2pt, fill=violet!30,
                  minimum width=1.15cm, minimum height=0.5cm,
                  align=center, font=\tiny},
  maxbox/.style= {draw, rounded corners=2pt, fill=violet!30,
                  minimum width=1.5cm, minimum height=0.45cm,
                  align=center, font=\scriptsize},
  group/.style = {draw, dashed, rounded corners=4pt, inner sep=4pt},
  bigroup/.style = {draw, dashed, rounded corners=5pt,
                    fill=violet!4, inner sep=6pt},
  fus/.style   = {draw, rounded corners=3pt, fill=teal!12,
                  minimum width=4.4cm, minimum height=0.75cm,
                  align=center, font=\scriptsize},
  albl/.style  = {font=\scriptsize, inner sep=1pt},
  arr/.style   = {-{Stealth[length=4pt, width=3.5pt]}, semithick,
                  rounded corners=4pt},
  iarr/.style  = {-{Stealth[length=3pt, width=3pt]}, thin},
}

\resizebox{\linewidth}{!}{%
\begin{tikzpicture}

%% ─── State branch (top-left) ───────────────────────────────────────────
\node[io]  (s)     at (0,0) {$s$};
\node[enc] (enc_s) [below=0.25cm of s] {Encoder $\phi_s$};
\node[sub] (s_lin)  [below=0.45cm of enc_s] {Linear};
\node[sub] (s_ln)   [below=0.14cm of s_lin]  {LayerNorm};
\node[sub] (s_silu) [below=0.14cm of s_ln]   {SiLU};
\node[sub] (s_lin2) [below=0.14cm of s_silu] {Linear};
\begin{pgfonlayer}{background}
  \node[group, fit=(s_lin)(s_lin2)] (body_s) {};
\end{pgfonlayer}

%% ─── Goal branch (top-right) ───────────────────────────────────────────
\node[io]  (g)     at (3.6,0) {$g_s$};
\node[enc] (enc_g) [below=0.25cm of g] {Encoder $\phi_g$};
\node[sub] (g_lin)  [below=0.45cm of enc_g] {Linear};
\node[sub] (g_ln)   [below=0.14cm of g_lin]  {LayerNorm};
\node[sub] (g_silu) [below=0.14cm of g_ln]   {SiLU};
\node[sub] (g_lin2) [below=0.14cm of g_silu] {Linear};
\begin{pgfonlayer}{background}
  \node[group, fit=(g_lin)(g_lin2)] (body_g) {};
\end{pgfonlayer}

%% ─── Input arrows ──────────────────────────────────────────────────────
\draw[arr] (s) -- (enc_s);
\draw[arr] (enc_s.south) -- (body_s.north);
\draw[arr] (g) -- (enc_g);
\draw[arr] (enc_g.south) -- (body_g.north);
\draw[iarr] (s_lin) -- (s_ln); \draw[iarr] (s_ln) -- (s_silu); \draw[iarr] (s_silu) -- (s_lin2);
\draw[iarr] (g_lin) -- (g_ln); \draw[iarr] (g_ln) -- (g_silu); \draw[iarr] (g_silu) -- (g_lin2);

%% ─── Fusion (single row, wide box) ─────────────────────────────────────
\node[fus] (fuse) at ($(s_lin2.south)!0.5!(g_lin2.south) + (0,-1.1)$)
  {$\mathrm{concat}(f_{s},\, f_g,\, f_{s} \odot f_g)$};
\draw[arr] (body_s.south) -- ++(0,-0.25) -| ([xshift=-0.6cm]fuse.north);
\draw[arr] (body_g.south) -- ++(0,-0.25) -| ([xshift= 0.6cm]fuse.north);
\node[albl, anchor=east] at ([xshift=-0.08cm, yshift=-0.22cm]body_s.south) {$f_s$};
\node[albl, anchor=west] at ([xshift= 0.08cm, yshift=-0.22cm]body_g.south) {$f_g$};

%% ─── LayerNorm before the maxout unit ──────────────────────────────────
\node[subv] (pre_ln) [below=0.35cm of fuse] {LayerNorm};
\draw[arr] (fuse.south) -- (pre_ln.north);

%% ─── Maxout unit: 4 parallel linears -> max ────────────────────────────
\node[mlin] (mx1) at ([xshift=-1.845cm, yshift=-0.7cm]pre_ln.south) {Linear};
\node[mlin] (mx2) [right=0.08cm of mx1] {Linear};
\node[mlin] (mx3) [right=0.08cm of mx2] {Linear};
\node[mlin] (mx4) [right=0.08cm of mx3] {Linear};
\node[maxbox] (mxmax) at ([yshift=-0.9cm]$(mx1)!0.5!(mx4)$) {Max};
% \node[io, anchor=west] (maxlabel) at ([xshift=0.28cm, yshift=-0.3cm]mx4.east) {Maxout};

%% ─── LayerNorm after max (outside Maxout dash) + lower stack ───────────
\node[subv] (b2_ln0)   [below=0.6cm of mxmax]     {LayerNorm};
\node[subv] (b2_silu0) [below=0.14cm of b2_ln0]   {SiLU};
\node[subv] (b2_drop)  [below=0.14cm of b2_silu0] {Dropout 0.1};
\node[subv] (b2_l1)    [below=0.14cm of b2_drop]  {Linear};
\node[subv] (b2_ln)    [below=0.14cm of b2_l1]    {LayerNorm};
\node[subv] (b2_silu)  [below=0.14cm of b2_ln]    {SiLU};
\node[subv] (b2_l2)    [below=0.14cm of b2_silu]  {Linear};

%% ─── Dashed containers ────────────────────────────────────────────────
\begin{pgfonlayer}{background}
  \node[bigroup, fit=(b2_ln0)(b2_l2)] (sharedbg) {};
  \node[group,   fit=(mx1)(mx4)(mxmax)] (maxoutbox) {};
\end{pgfonlayer}

%% ─── Maxout fan + shared arrows ────────────────────────────────────────
\draw[iarr, rounded corners=4pt] ([xshift=-0.6cm]pre_ln.south) -- ++(0,-0.22) -| (mx1.north);
\draw[iarr, rounded corners=4pt] ([xshift=-0.2cm]pre_ln.south) -- ++(0,-0.22) -| (mx2.north);
\draw[iarr, rounded corners=4pt] ([xshift= 0.2cm]pre_ln.south) -- ++(0,-0.22) -| (mx3.north);
\draw[iarr, rounded corners=4pt] ([xshift= 0.6cm]pre_ln.south) -- ++(0,-0.22) -| (mx4.north);
\draw[iarr, rounded corners=4pt] (mx1.south) -- ++(0,-0.21) -| ([xshift=-0.6cm]mxmax.north);
\draw[iarr, rounded corners=4pt] (mx2.south) -- ++(0,-0.21) -| ([xshift=-0.2cm]mxmax.north);
\draw[iarr, rounded corners=4pt] (mx3.south) -- ++(0,-0.21) -| ([xshift= 0.2cm]mxmax.north);
\draw[iarr, rounded corners=4pt] (mx4.south) -- ++(0,-0.21) -| ([xshift= 0.6cm]mxmax.north);
\draw[arr]  (mxmax.south) -- (b2_ln0.north);
\draw[iarr] (b2_ln0) -- (b2_silu0); \draw[iarr] (b2_silu0) -- (b2_drop); \draw[iarr] (b2_drop) -- (b2_l1);
\draw[iarr] (b2_l1) -- (b2_ln); \draw[iarr] (b2_ln) -- (b2_silu); \draw[iarr] (b2_silu) -- (b2_l2);

%% ─── Output ────────────────────────────────────────────────────────────
\node[io] (qout) [below=0.2cm of b2_l2] {};
\draw[arr] (b2_l2.south) -- (qout.south);
\end{tikzpicture}%
}
\captionsetup{skip=-5pt}
\caption{\textbf{Network architecture.}}
\label{fig:network}
\end{minipage}
\hfill
\begin{minipage}[c]{0.63\textwidth}
In $Q^\theta(s, a, s_g, a_g)$, linear layers size is 64 and the network output shape is $|\actionspace|\! \times\! |\actionspace|$ --- state-action value for every goal-action in a given goal-state. In $\Vtheta(s, s_g)$ and $\pi(a \spacedmid s, s_g)$, linear layers size is 256 and output shapes are 1 and $|\actionspace|$, respectively. The two networks do not share any layer.
\\[2pt]
\textbf{Encoder.}
Gridworlds do not need an encoder because their observation is already appropriate (one-hot encoding of the agent's position).
In continuous control environments, the state and state-goal encoders are a radial basis function layer \citep{poggio1990networks}. The layer places $C$ centers $\mu_1, \ldots, \mu_C$ per input dimension, initialized uniformly between $v_{\min}$ and $v_{\max}$, and learns both the centers and bandwidths $h_c > 0$ end-to-end via gradient descent, jointly with the rest of the network. For each scalar input $x$, the layer computes Gaussian activations $\varphi_c(x) = \exp(-0.5((x-\mu_c)/h_c)^2)$. To prevent vanishing gradients for inputs outside $[v_{\min}, v_{\max}]$, we add a linear leak to the boundary tiles, so that the activation grows linearly with distance once $x$ falls below $\mu_1$ or above $\mu_C$. We use $C = 20$ centers initialized with $v_{\min} = -3$, $v_{\max} = 3$. Environment raw observations are standardized using running mean and standard deviation tracked with Welford's online algorithm.
\\[2pt]
\textbf{Feature fusion.}
We combine the state and goal features by concatenating $f_s$, $f_g$, and their element-wise product $f_s \odot f_g$. %, following the matching architecture used in InferSent~\citep{conneau2017supervised}.
The element-wise product provides a multiplicative interaction term that makes pairwise alignment between corresponding components of $f_s$ and $f_g$ directly available to the downstream layers, complementing the information carried by the concatenation of $f_s$ and $f_g$ alone.
\\[2pt]
\textbf{Maxout.}
After feature fusion we apply a \emph{Maxout} unit \citep{goodfellow2013maxout}, which computes four parallel linear projections of its input and takes the element-wise maximum across them. 
% Maxout acts as a learned, piecewise-linear activation: it is a universal approximator of convex activation functions.
\end{minipage}
\end{figure}

\clearpage

\begin{figure}[t]
    % \centering
    \includegraphics[trim={0 2.0em 0 0}, clip, width=\linewidth]{\detokenize{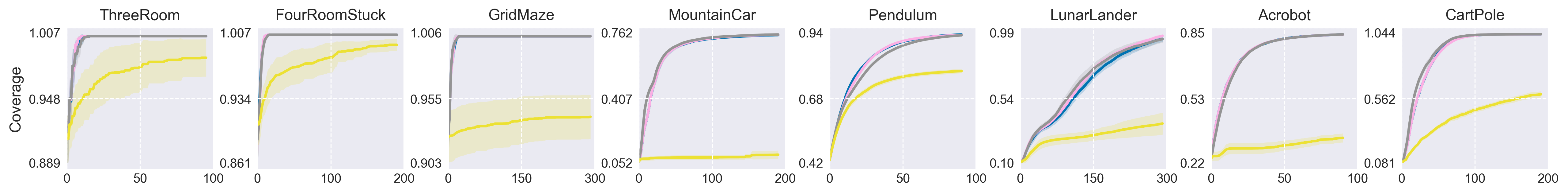}}
    \\[1pt]
    \includegraphics[trim={0 0 0 2.2em}, clip, width=\linewidth]{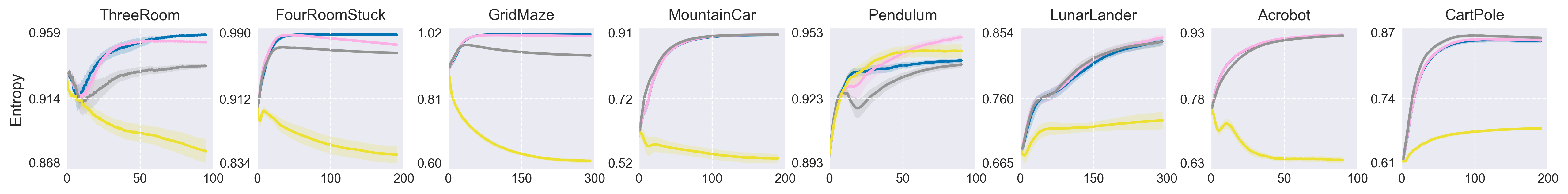}
    \\[1pt]
    \begin{minipage}[b]{0.64\linewidth}
    \includegraphics[trim={0 2.0em 0 0}, clip, width=\linewidth]{\detokenize{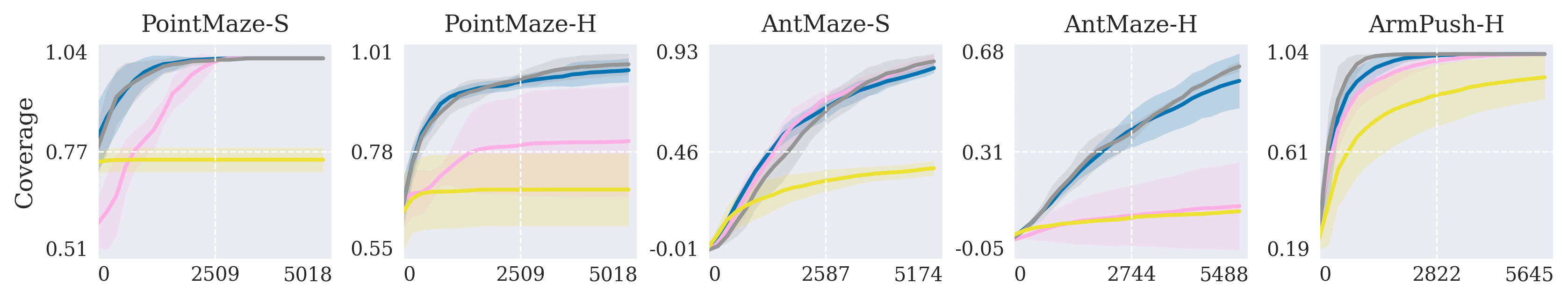}}
    \\[1pt]
    \includegraphics[trim={0 0 0 2.2em}, clip, width=\linewidth]{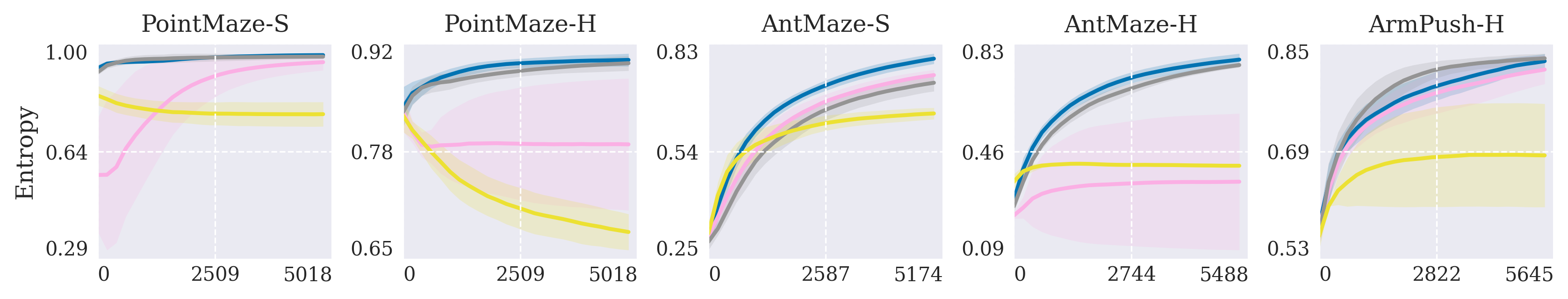}
    \end{minipage}
    \hfill
    \centering
    \includegraphics[width=0.3\linewidth]{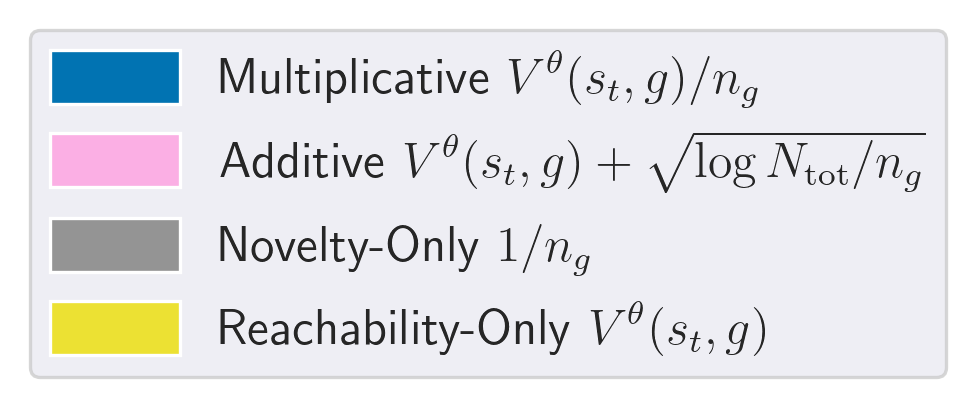}
    \\[2pt]
    \includegraphics[trim={0 0 5pt 0}, clip, width=0.617\linewidth]{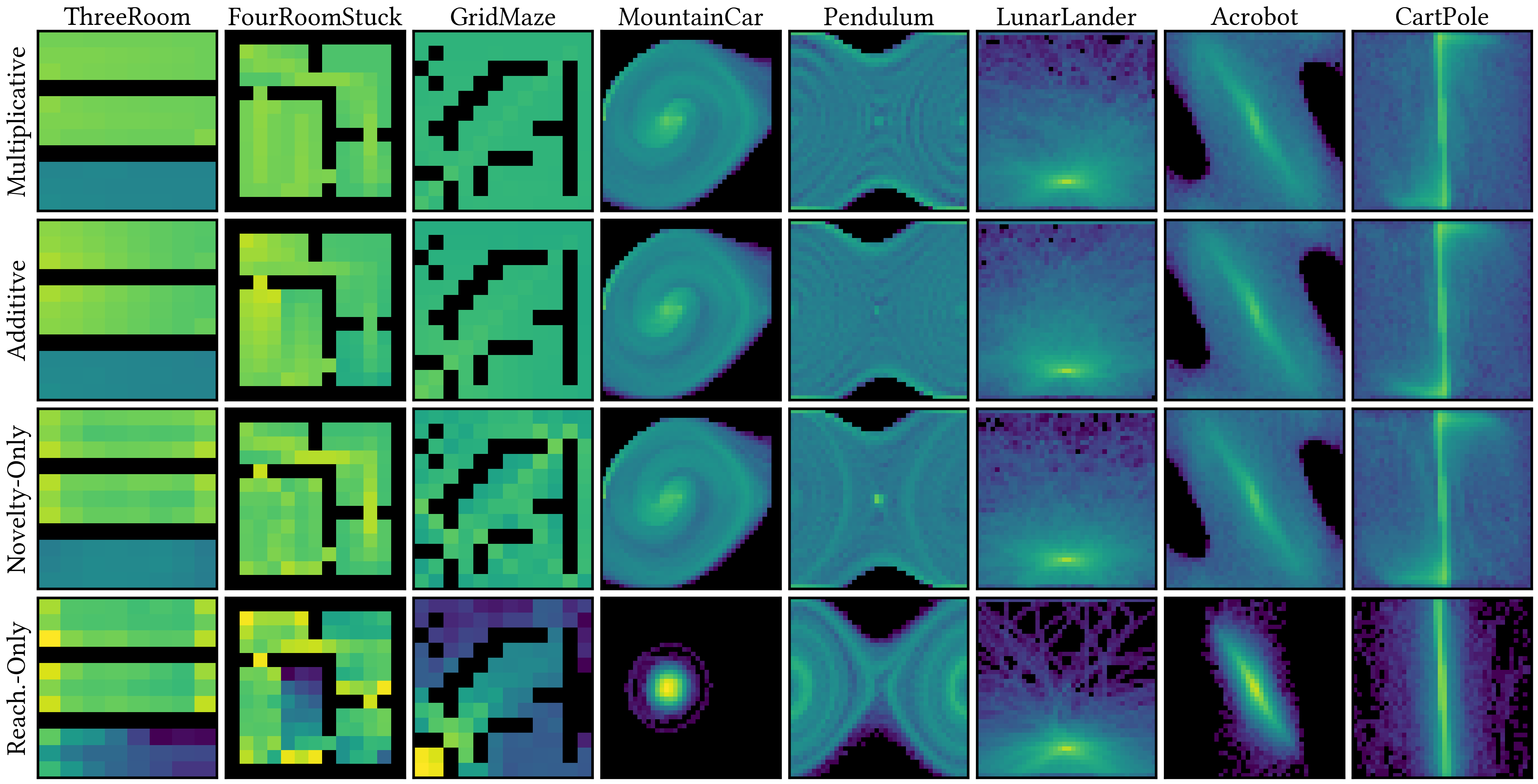}%
    \includegraphics[width=0.3829\linewidth]{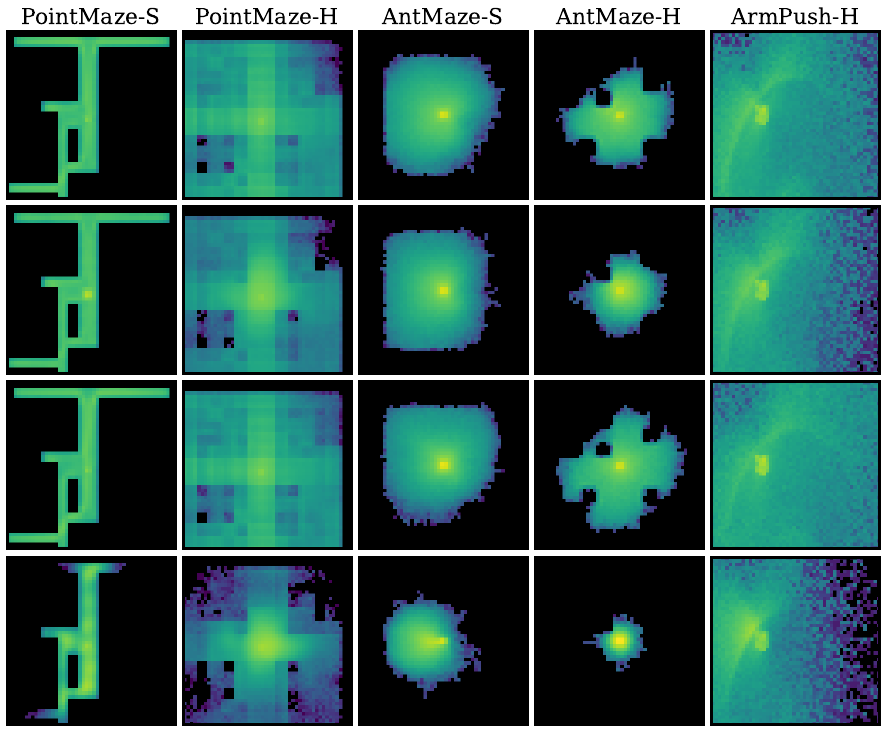}
    \caption{\label{fig:ablation_score_all}\textbf{Indicator ablation.} Full version of Figure \ref{fig:ablation_score} with heatmaps.}
\end{figure}

\begin{figure}[t]
    % \centering
    \includegraphics[trim={0 2.0em 0 0}, clip, width=\linewidth]{\detokenize{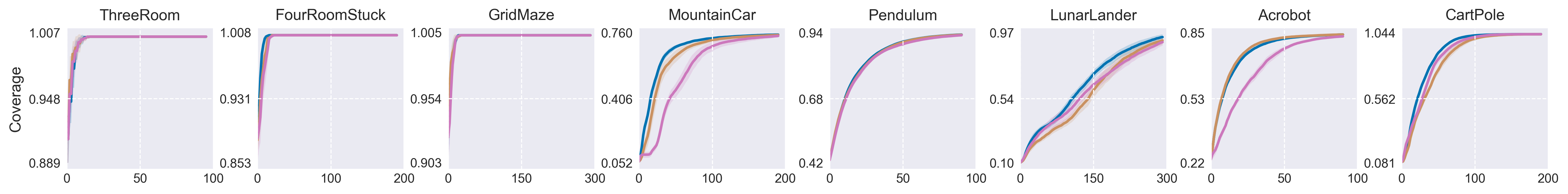}}
    \\[1pt]
    \includegraphics[trim={0 0 0 2.2em}, clip, width=\linewidth]{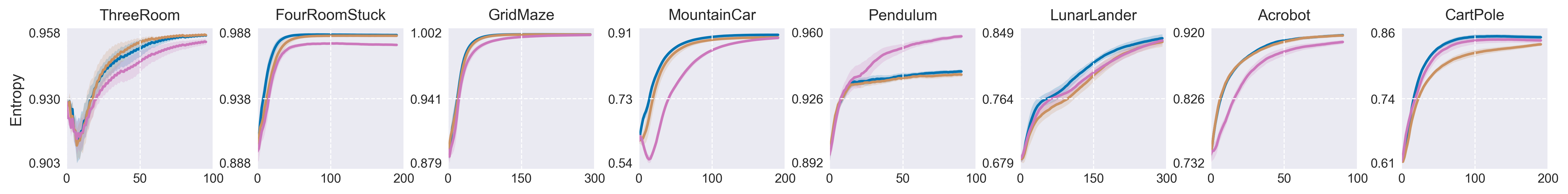}
    \\[1pt]
    \begin{minipage}[b]{0.64\linewidth}
    \includegraphics[trim={0 2.0em 0 0}, clip, width=\linewidth]{\detokenize{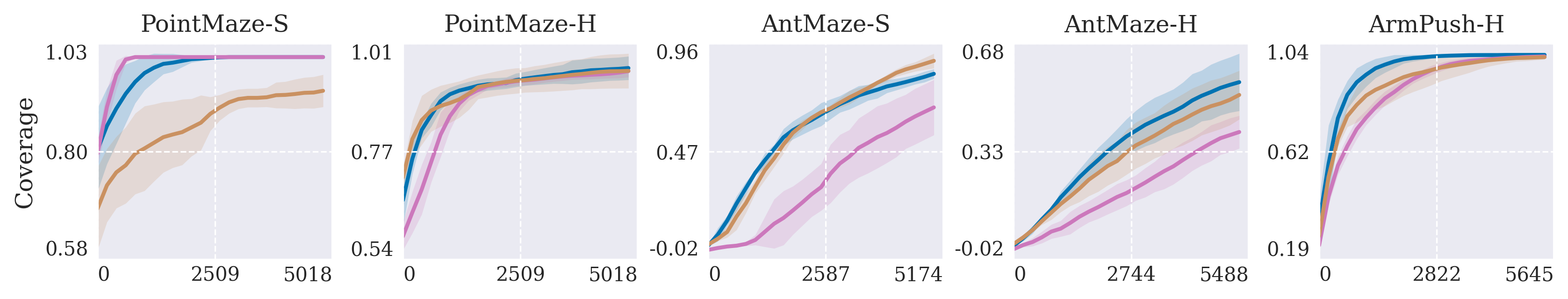}}
    \\[1pt]
    \includegraphics[trim={0 0 0 2.2em}, clip, width=\linewidth]{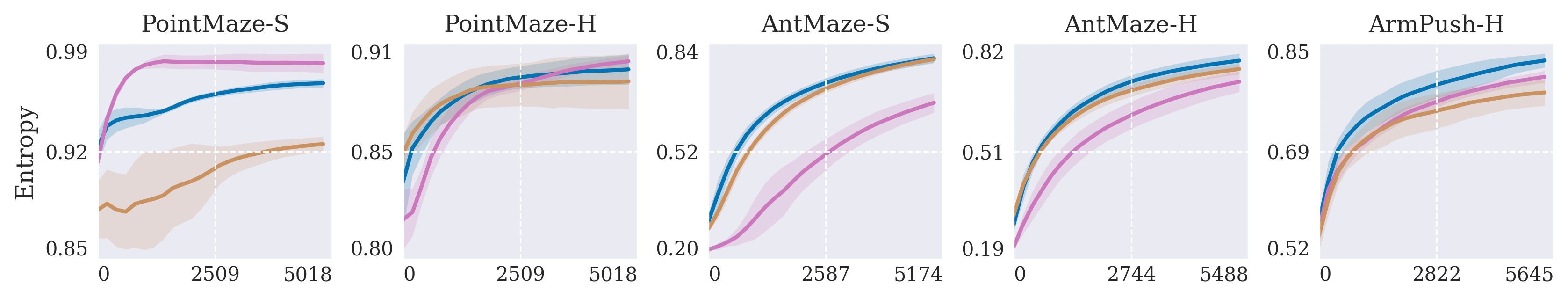}
    \end{minipage}
    \hfill
    \centering
    \includegraphics[width=0.2\linewidth]{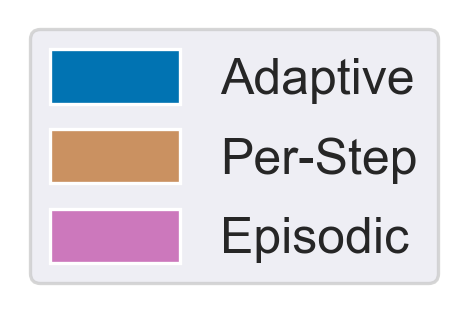}
    \\[2pt]
    \includegraphics[trim={0 0 5pt 0}, clip, width=0.617\linewidth]{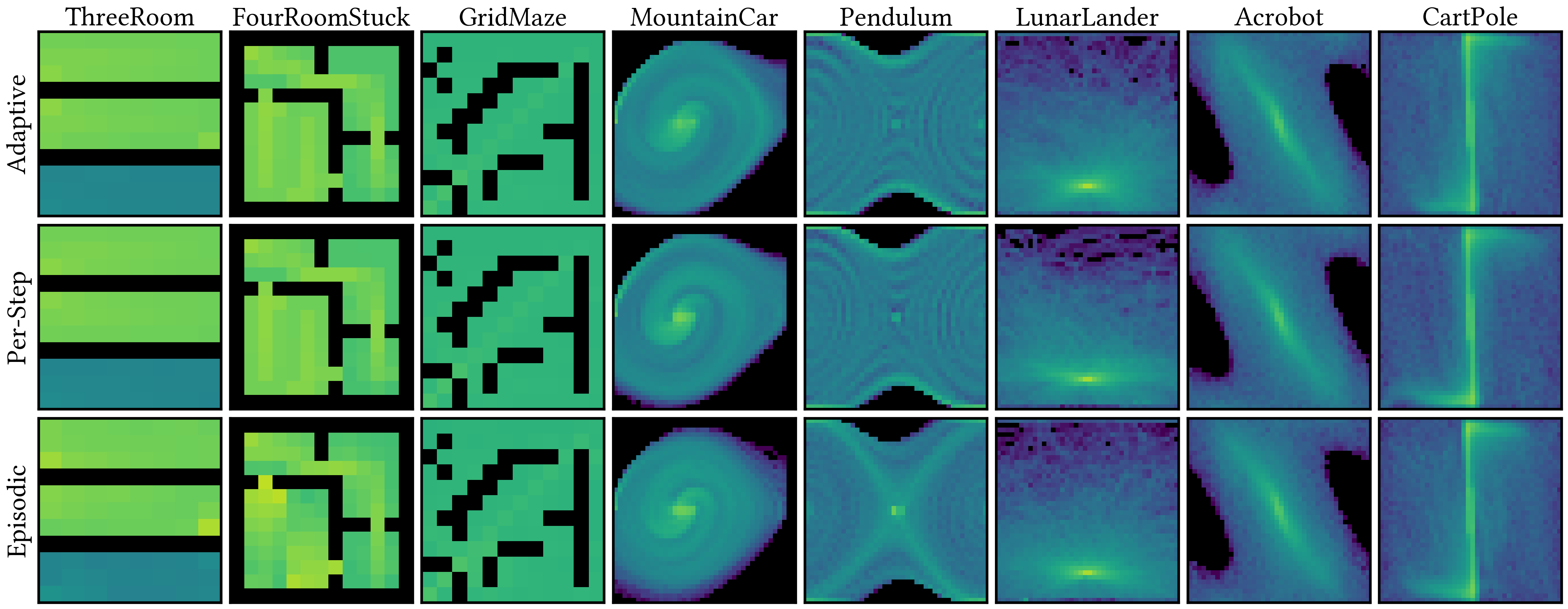}%
    \includegraphics[width=0.3830\linewidth]{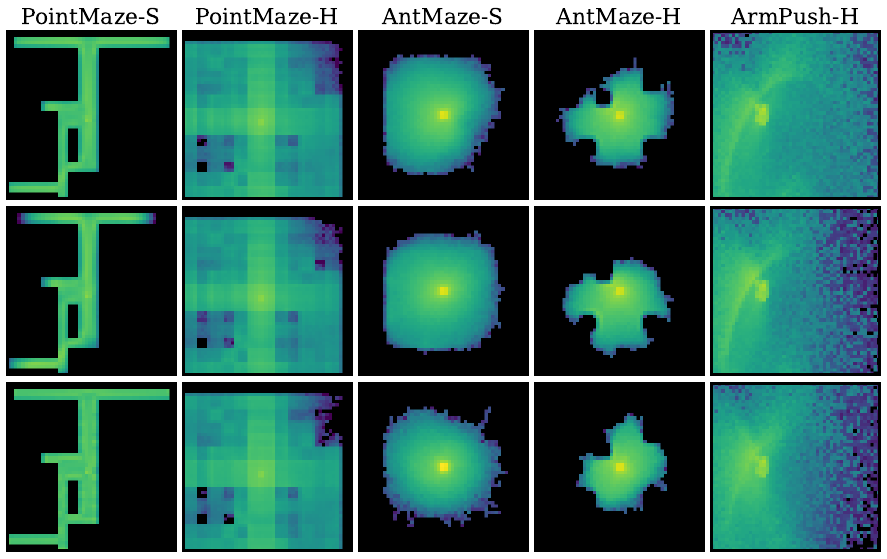}
    \caption{\label{fig:ablation_selection_all}\textbf{Goal-selection ablation.} Full version of Figure~\ref{fig:ablation_selection} with heatmaps.}
\end{figure}

\clearpage

\section{Indicator and Goal-Selection Ablations}
\label{app:ablation_full}
Figures \ref{fig:ablation_score_all} and \ref{fig:ablation_selection_all} report the full training curves and end-of-training visitation heatmaps for the ablations discussed in Section~\ref{subsec:ablation_results}. We exclude \stt{LunarLander(Full)} due to its computational expense. Per-environment goal-selection heatmaps with additional statistics are in Appendix~\ref{app:extra_details}
\\[2pt]
The ablations exhibit the same trends described in Section \ref{subsec:ablation_results}, with SUN performing the best. The only exceptions are \stt{Pendulum} and \stt{PointMaze-S}. In the former, SUN Additive performs best, because its novelty term shrinks as counts grow and reachability dominates the selection: the agent then stops targeting the rare states that can only be visited by repeatedly traversing the same states, which is what drives SUN's entropy down in Section~\ref{sec:experiments}. In the latter, SUN Episodic is the best. \stt{PointMaze-S} is a small environment where most goals can be reached within a third of the episode, so committing to a goal and then exploring randomly around it, as episodic selection does, is effective. The same strategy is detrimental where goals are hard to reach or random actions barely move the agent: in \stt{AntMaze} it is the worst of the three.

\begin{figure}[!b]
    % \vspace*{-1em}
    \begin{center}
    \includegraphics[width=0.63\linewidth]{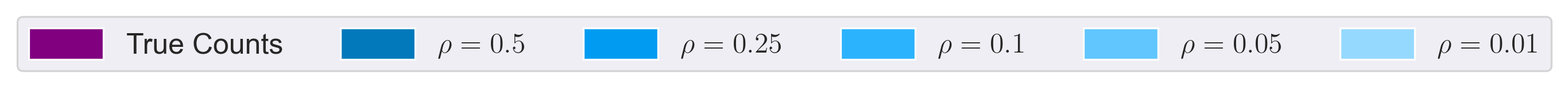}
    \end{center}
    \vspace*{-5pt}
    \includegraphics[trim={0 2.0em 0 0}, clip, width=0.5\linewidth]{\detokenize{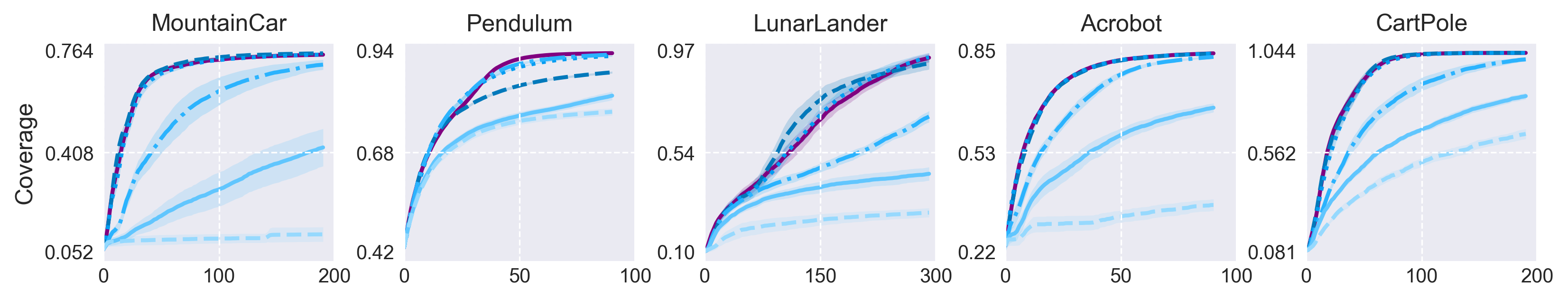}} \hfill
    \includegraphics[trim={2.2em 2.0em 0 0}, clip, width=0.48\linewidth]{\detokenize{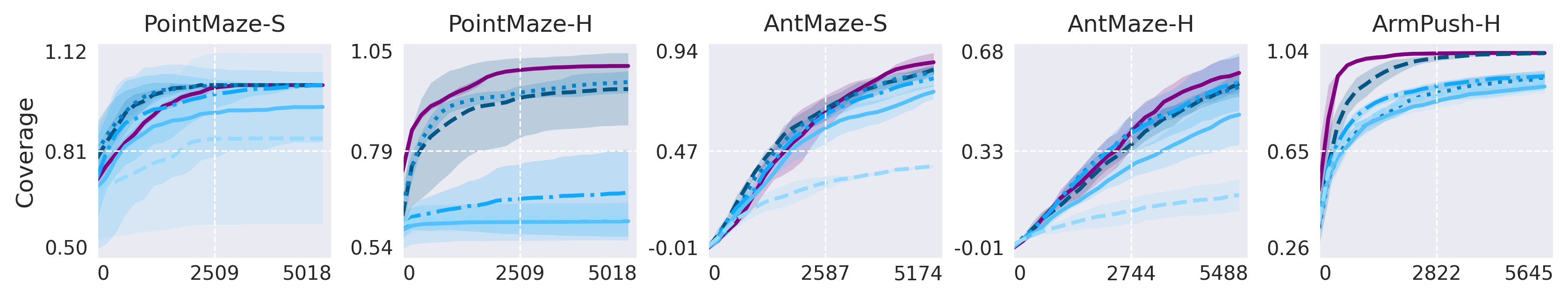}}
    \\[1pt]
    \includegraphics[trim={0 0 0 2.2em}, clip, width=0.5\linewidth]{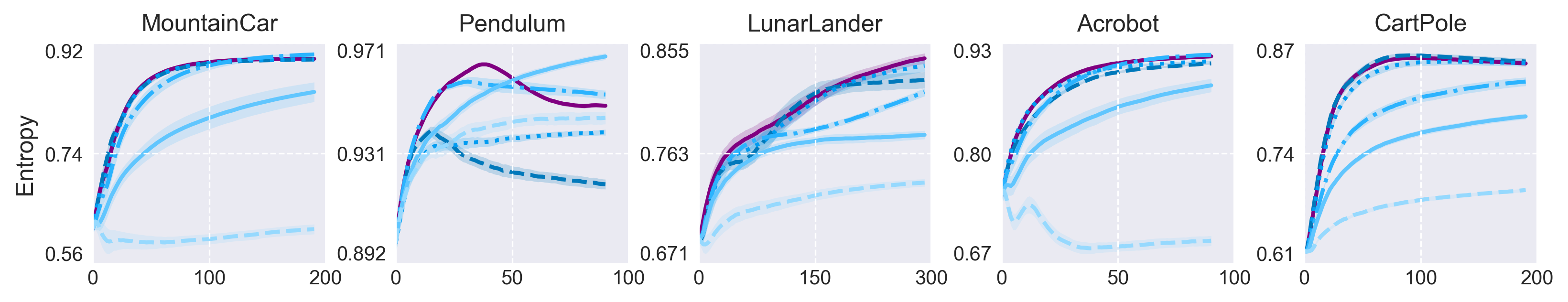}
    \hfill
    \includegraphics[trim={2.2em 0 0 2.2em}, clip, width=0.48\linewidth]{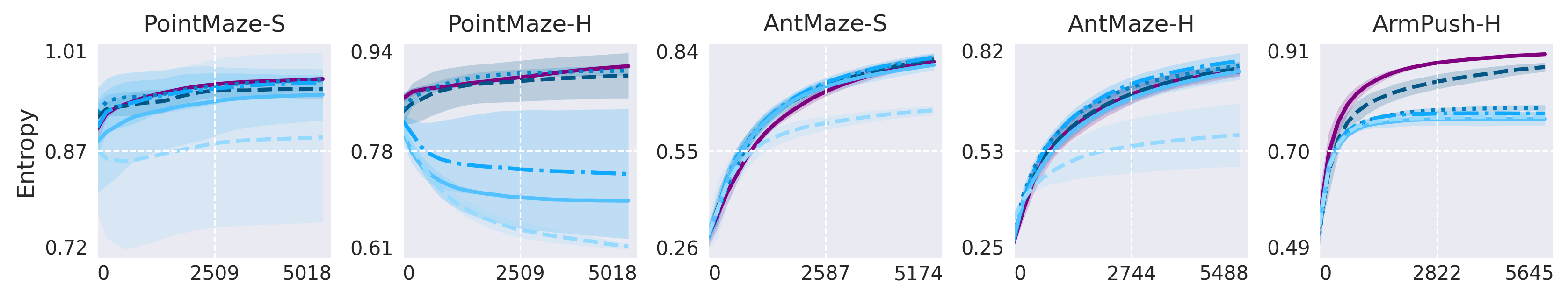}
    \\[-2pt]
    \begin{minipage}[c]{0.19\textwidth}
    \includegraphics[width=\linewidth]{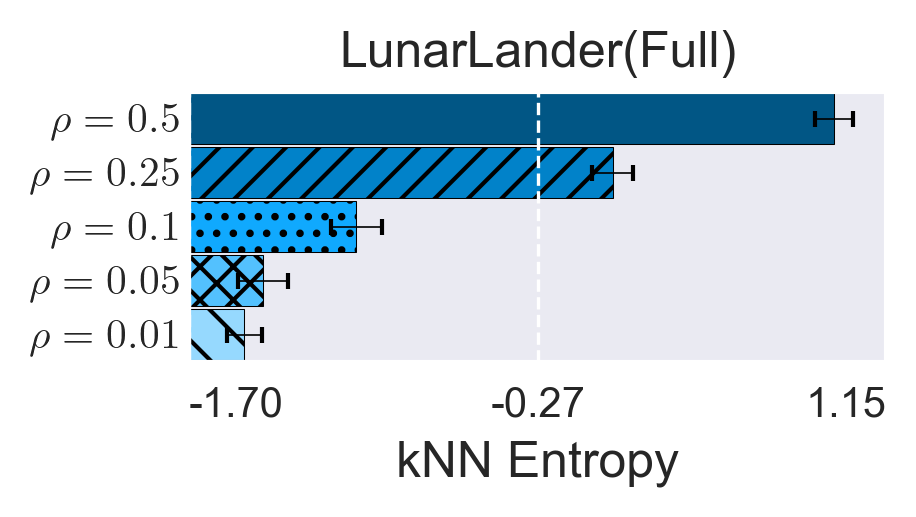}
    \end{minipage}\hfill
    \begin{minipage}[c]{0.8\textwidth}
    \setlength{\abovecaptionskip}{1pt}%
    \caption{\label{fig:pseudo_vs_true_plots}\textbf{Pseudocounts vs.\ true counts.} Gridworlds are omitted, since there pseudocounts coincide with the counts; \stt{LunarLander(Full)} has no true binned counts (see Appendix~\ref{app:envs}). True counts improve coverage and entropy in almost all environments. One exception is \stt{Pendulum}, where entropy drops, the\parfillskip=0pt}
    \end{minipage}\par
    \noindent same behavior discussed in Section~\ref{subsec:main_results}, amplified by exact counting. Another is \stt{PointMaze-S}, where $\rho \in \{0.25, 0.5\}$ attains the fastest coverage growth, although all radii reach full coverage. We attribute this to the narrow corridors of this maze, which make {true counts} sensitive to the bin discretization. \stt{LunarLander(Full)}'s eight-dimensional is more sensitive to $\rho$: performance improves monotonically with the radius, and only $\rho = 0.5$ is adequate.
    % The reason has been discussed in Appendix~\ref{app:pseudocount}: a ball of radius $\rho$ has volume proportional to $\rho^M$, so in higher-dimensional goal spaces a small radius leaves almost every entry with a count of one, causing novelty collapse.
    \stt{Acrobot} and \stt{AntMaze-H} show a similar trend: their goals are four-dimensional, against two or three in the other environments, and they are the only two where entropy \textit{decreases} over training at the smallest radius.
\end{figure}

\section{Pseudocounts Ablation}
\label{app:pseudocount_ablation}

Here we evaluate how $\rho$ affects the pseudocount approximation (Section~\ref{app:pseudocount}). Figure~\ref{fig:pseudo_vs_true_plots} shows SUN performance for varying $\rho$ against \textit{true counts}, i.e., counts obtained by discretizing the goal space with 50 bins per dimension. Figure~\ref{fig:pseudo_vs_true_heatmap} shows instead pseudocounts computed \textit{on the same data} (collected while training with \textit{true counts}). \stt{LunarLander(Full)} has no true counts version (Appendix~\ref{app:envs}), so there we only compare radii against each other.
Results show that our pseudocounts closely track the coverage and entropy of training with true counts across all other environments for most values of $\rho$, except for values that are too small.
This is expected: the estimator is a uniform-kernel density estimate and $\rho$ is its bandwidth, so an extreme value in either direction makes $\nu(g) = 1/n_g$ constant across candidates.
When $\rho$ is too small, the ball around an inserted sample rarely contains any other buffer entry, so most samples end up with small \textit{uniform} counts.
When $\rho$ is too large, the opposite degeneracy occurs: the ball contains a large fraction of the buffer for every entry, so counts become comparable across the visited set.
In both cases, SUN score of Eq.~\eqref{eq:sun-score} reduces to an $\argmax$ over the SVF alone, losing any notion of novelty.
\\[1pt]
In practice only the first regime is reached. Across the suite, entropy and coverage degrade as $\rho$ decreases, and the corresponding heatmaps in Figure~\ref{fig:pseudo_vs_true_heatmap} show zero-count areas at the smallest radii, while in low-dimensional goal spaces performance is largely flat for $\rho \geq 0.1$. This is a consequence of the robust per-feature scale of Section~\ref{app:pseudocount}: because each dimension is normalized by its own interquartile range, a ball of radius $\rho$ covers a comparable fraction of each dimension's support in every environment, and even $\rho = 0.5$ preserves the shape of the visited set --- the mazes' corridors, the spiral of \stt{MountainCar}, the arc of \stt{ArmPush-H}.
\\[1pt]
The per-feature scale does not, however, remove the dependence on the goal-space dimension $M$. A ball of radius $\rho$ has volume proportional to $\rho^M$, so at a fixed radius the fraction of the visited set it contains shrinks exponentially with $M$, and a small radius leaves almost every entry with a count of one. \stt{LunarLander(Full)}, with $M = 8$, is the clearest case: performance improves monotonically with the radius, and only $\rho = 0.5$ is adequate. Even this largest radius stays in the first regime, as no degradation from overly large balls appears. The four-dimensional \stt{Acrobot} and \stt{AntMaze-H} show a milder form of the same effect, being the only two environments where entropy decreases at the smallest radius over training.

\begin{figure}[!t]
    \centering
    \includegraphics[width=0.395\linewidth]{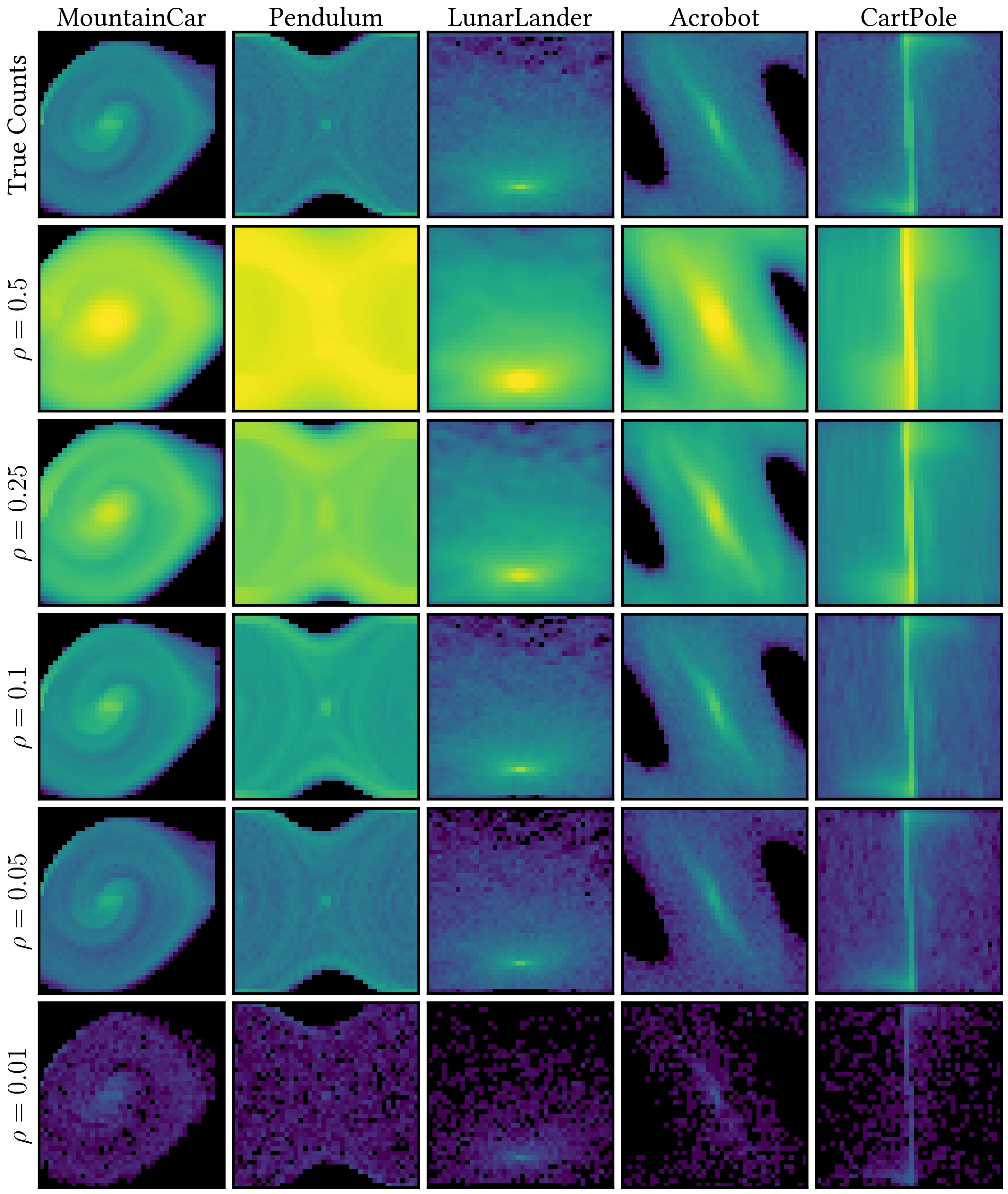}%
    \includegraphics[width=0.385\linewidth]{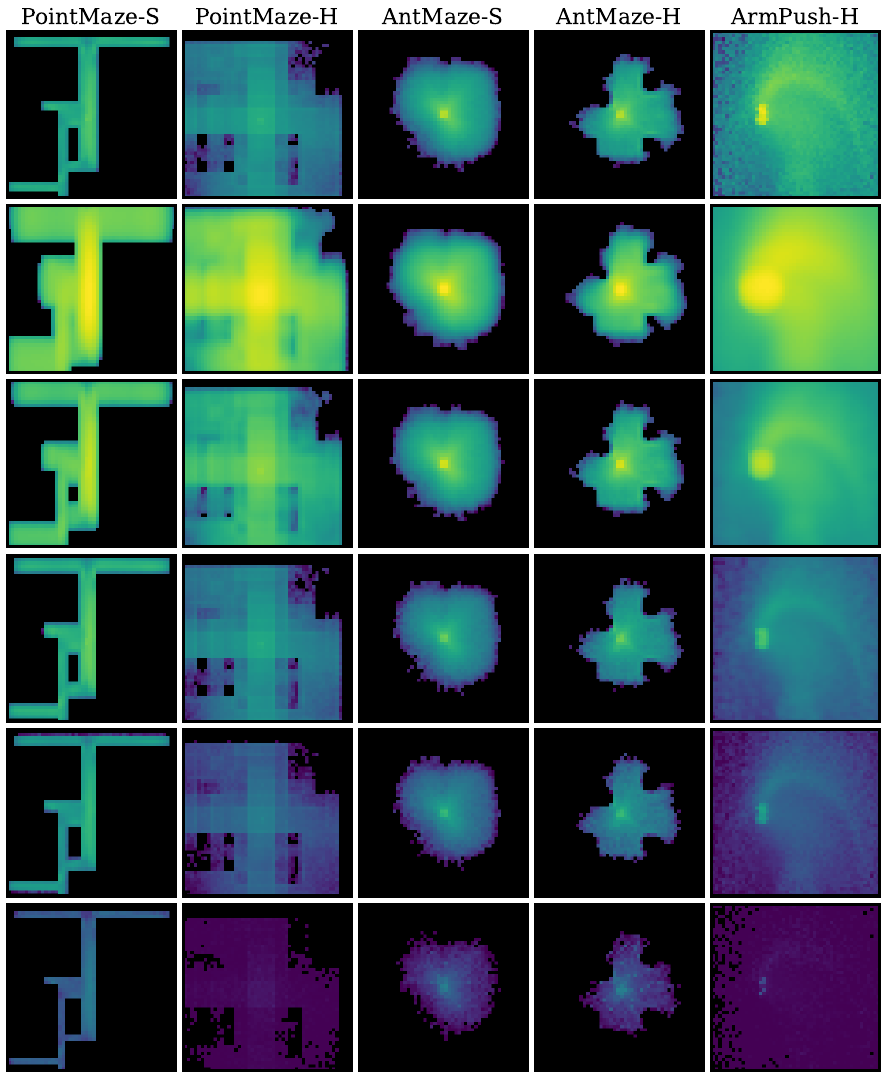}%
    \hfill
    \begin{minipage}[b]{0.21\textwidth}
        \caption{\label{fig:pseudo_vs_true_heatmap}\textbf{Log-scale pseudocount heatmaps.} The top row shows the true visit counts at the end of training (first seed). Pseudocounts (rows below) are computed on that same data at decreasing radii $\rho$. At any radius, pseudocounts preserve the shape of the visited set (e.g., \stt{MountainCar} spiral).}
    \end{minipage}
\end{figure}

% \clearpage

\begin{figure}[!b]
    % \vspace*{-1em}
    \centering
    \includegraphics[width=\linewidth]{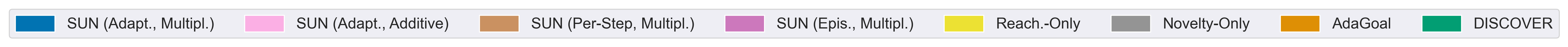}\\[-1pt]
    \includegraphics[width=0.49\linewidth]{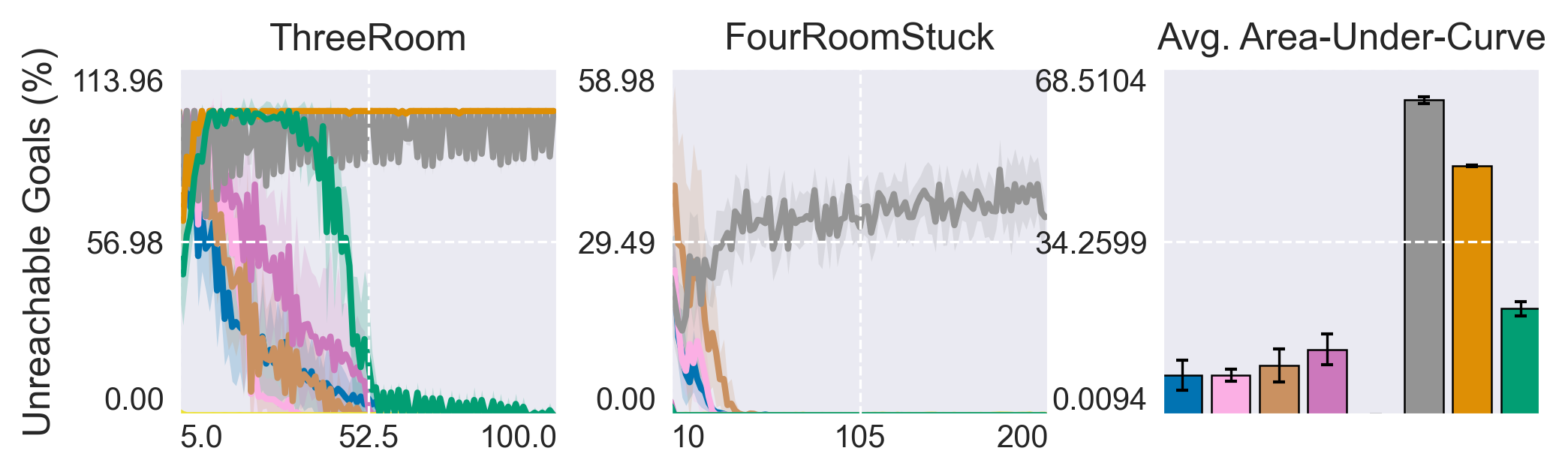}%
    \hfill
    \includegraphics[width=0.49\linewidth]{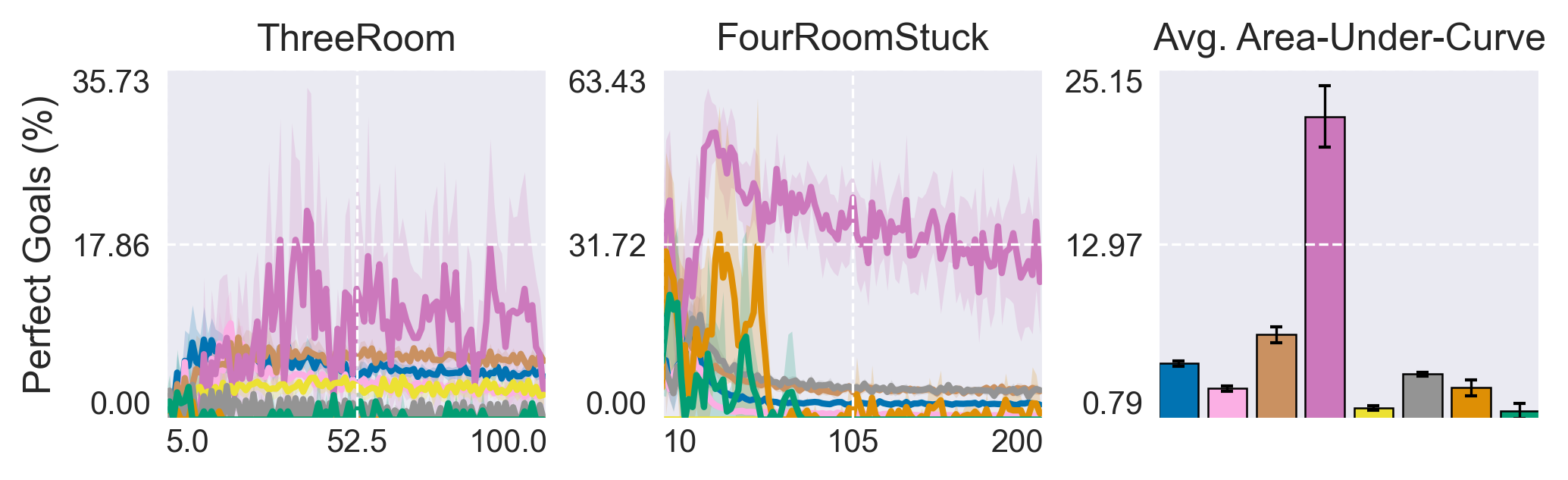}%
    \hfill
\captionsetup{skip=0pt}
\caption{\label{fig:unreach_vs_perfect_full}\textbf{Selection rates of unreachable and perfect goals} in \stt{ThreeRoom} and \stt{FourRoomStuck}.}
\end{figure}

\section{Detailed Analysis Of All Environments}
\label{app:extra_details}
For each environment, we report the following goal-related statistics, characterizing the behavior of SUN (with adaptive, episodic, and per-step goal-selection), AdaGoal, and DISCOVER.
\begin{itemize}[leftmargin=*, itemsep=0pt, topsep=-2pt]
\item \textbf{Goals selected and reached}, measured at four stages of training (25/50/75/100\% of the training steps). All heatmaps share the same log-scale color range.
\item \textbf{Goal success}: the ratio of goals reached to goals selected.
\item \textbf{Steps-to-goal}: the average number of steps taken to reach a goal. %, over all goals successfully reached.
\item \textbf{Goal reselections} (SUN Adaptive and Per-Step): the fraction of steps at which $g_t \neq g_{t-1}$, normalized to $[0, 1]$ by the total number of steps. The two differ in what triggers a reselection. In \emph{adaptive} SUN, the goal is discarded and a fresh candidate batch is sampled only when $\Vtheta(s_t, g_t) < \Vtheta(s_{t_{\mathrm{sel}}}, g_t)$ fires. In \emph{per-step} SUN, a fresh batch is sampled at every step, and the goal changes if the batch has a candidate with a higher SUN score (Eq.~\eqref{eq:sun-score}) than the current one.
\item \textbf{Random exploration} (SUN Episodic, AdaGoal, DISCOVER): the fraction of steps at which the agent acts randomly. In \textit{episodic} goal-selection, this is triggered once the agent reaches its goal, following the original DISCOVER implementation.
\end{itemize}
Figure~\ref{fig:unreach_vs_perfect_full} is the extended version of Figure~\ref{fig:unreach_vs_perfect}, and gives further insight into the quality of the goals selected by each algorithm.

\clearpage

The following important trends emerge.
\begin{itemize}[leftmargin=*, noitemsep, topsep=0pt]
\vspace*{-4pt}
\item Episodic selection fails when goals can become unreachable mid-episode (Figures~\ref{fig:extra_three_room} and~\ref{fig:extra_four_room}).
\item In-episode reselection avoids this problem, but Per-Step reselects far more often than Adaptive, which hurts performance. In almost all environments, Adaptive's reselection rate drops to zero over training, indicating that the agent has learned to reach the goals it commits to.
\item DISCOVER is biased toward reachability: across all heatmaps it selects far more goals near the starting states than the other algorithms, and attains high goal success from the very beginning. This may stem from its coefficient $\beta$, which must balance the scales of the reachability and novelty terms --- a problem SUN sidesteps entirely thanks to its multiplicative indicator.
\item AdaGoal is biased toward novelty at the expense of reachability: across all heatmaps it selects mostly goals far from the starting states, ``skipping'' intermediate ones, and in Figure \ref{fig:unreach_vs_perfect_full} it often selects unreachable goals. 
This is not unexpected. As noted in Section~\ref{subsec:summary}, we evaluate the deep-RL variant of AdaGoal, which --- unlike the tabular version --- does not constrain goal selection by the estimated goal-hitting time. \citeauthor{tarbouriech2022adaptive} argue that this is approximated implicitly by the disagreement of the value ensemble. Our results suggest that this does not always hold. Novelty alone is a reasonable signal for goals the agent can actually reach, since visiting them resolves the disagreement at little cost. Unreachable goals, however, resolve only after enough failed attempts for every ensemble member to recognize them as such, and each of those attempts is an episode without novel data. The rule cannot distinguish the two cases, and where the reachable set is a small fraction of the goal space the latter dominates --- precisely the behavior the reachability constraint was meant to prevent.
\item In Gridworlds, SUN tends to select short-horizon goals, as shown by steps-to-goal tending to one. In control tasks the opposite happens: the metric grows over training, meaning SUN selects goals that are progressively further away. This is expected, since control tasks require coherent action sequences to reach distant states --- e.g., in \stt{MountainCar} the agent must build up momentum.
\item This is consistent with Figure~\ref{fig:unreach_vs_perfect_full}, where SUN learns to never select unreachable goals yet still selects \textit{``imperfect''} ones. This is not unexpected, because a ``perfect'' goal is the least visited one regardless of distance ($\argmin_g \{n_g : V^*(s_t,g) > 0\}$, where $V^*$ is the \textit{optimal} SVF), whereas SUN weighs novelty against reachability ($\smash{\argmax_g \Vtheta(s_t, g) / n_g}$). SUN Episodic selects the most perfect goals of all versions, but this does not translate into better exploration: a goal being reachable in principle does not mean the agent can reach it now, and the least-visited reachable goal is typically the farthest one. Episodic selection commits to it for the rest of the episode, whereas Adaptive discards it as soon as the SVF value drops and replaces it with a closer one.% --- which is why Adaptive attains higher entropy and coverage despite selecting fewer perfect goals.
\end{itemize}

\begin{figure}[!b]
    \centering
    % \\[-1pt]
    \begin{subfigure}[b]{0.30\textwidth}
        \includegraphics[trim=3 0 3 3, clip, width=\linewidth]{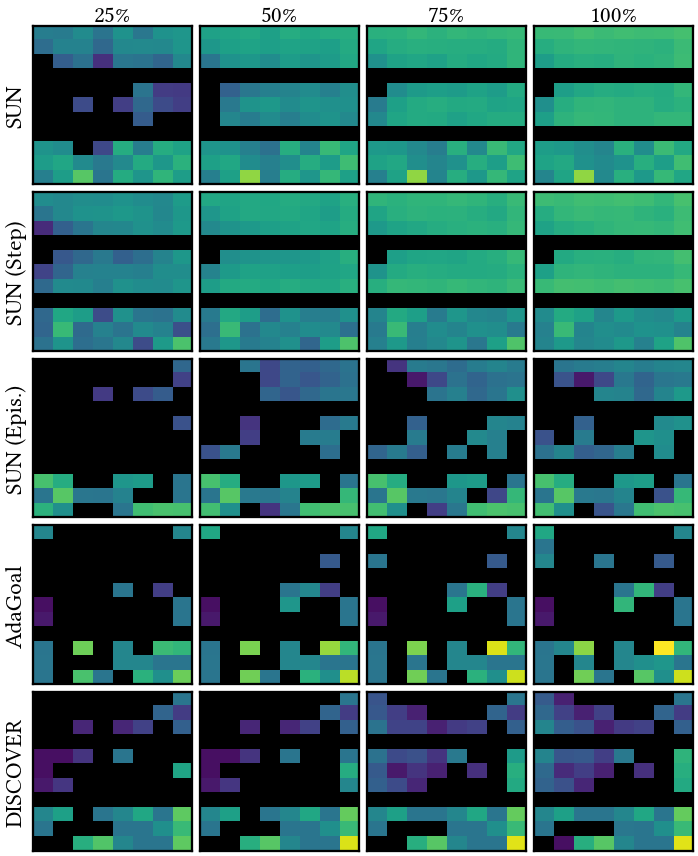}
        \caption{{{Goals selected.}}}
        \end{subfigure}
    \hfill
    \begin{subfigure}[b]{0.30\textwidth}
        \includegraphics[trim=3 0 3 3, clip, width=\linewidth]{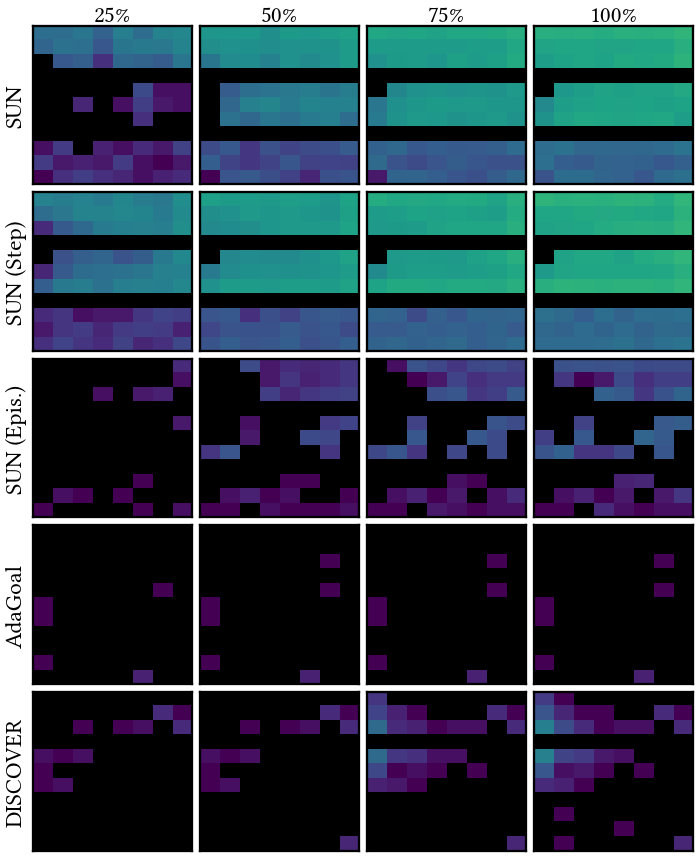}
        \subcaption{{{Goals reached.}}}
    \end{subfigure}
    \hfill
    \begin{subfigure}[b]{0.32\textwidth}
        \includegraphics[width=\linewidth]{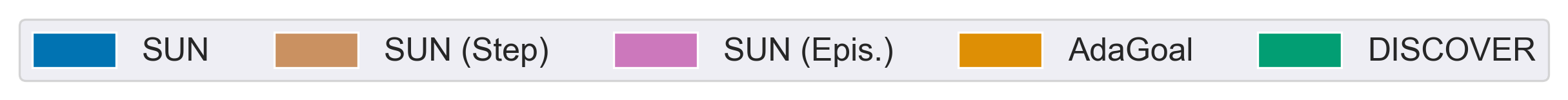}
        \\[-1pt]
        \includegraphics[width=\linewidth]{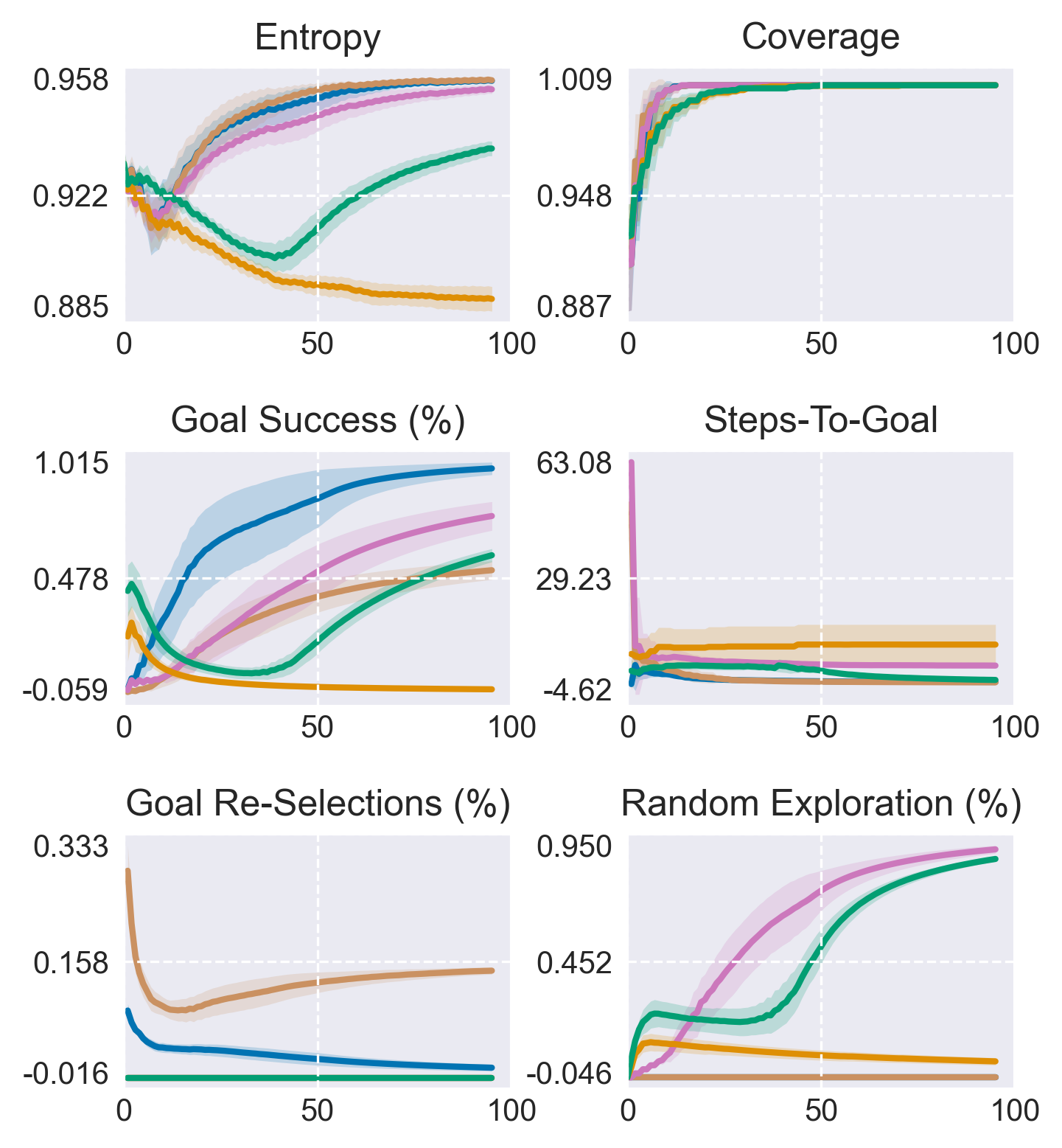}
        \captionsetup{skip=0pt}
        \subcaption{Training curves.}
    \end{subfigure}
    \caption{\label{fig:extra_three_room}\stt{ThreeRoom}. This environment shows the superiority of the SUN indicator, and the importance of properly trading off novelty and reachability. AdaGoal fails because it frequently selects unreachable goals: in (a), it is the only method that does \emph{not} under-select the bottom room (which is often unreachable, see goal success curves in (c)). DISCOVER, conversely, is biased toward reachability: over time, it selects more goals close to the starting states (left of the rooms), and therefore spends most of its time exploring randomly (see random exploration curves in (c)). All three versions of SUN, in contrast, achieve high entropy: over time they select goals uniformly and progressively further to the right (away from starting states).}% The choice of goal-selection strategy matters little here, since the environment is small.}
\end{figure}

\clearpage

\begin{figure}[!b]
    \centering
    \begin{subfigure}[b]{0.31\textwidth}
        \includegraphics[trim=3 0 3 3, clip, width=\linewidth]{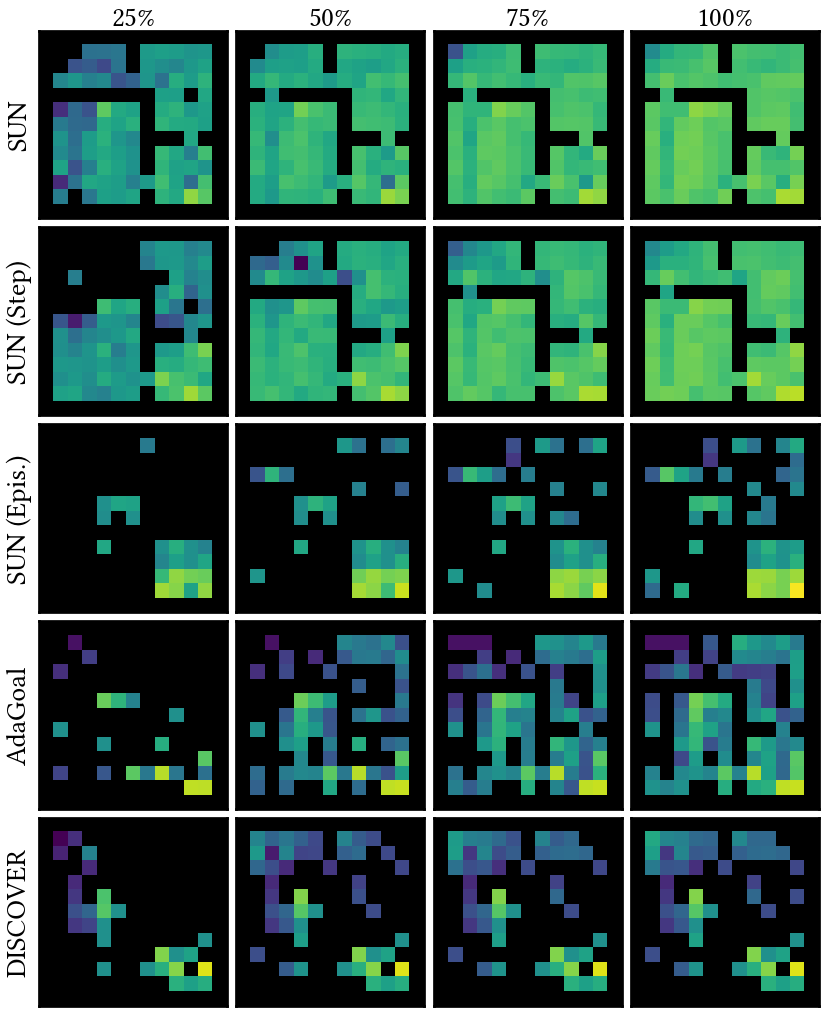}
        \subcaption{{{Goals selected.}}}
        \end{subfigure}
    \hfill
    \begin{subfigure}[b]{0.31\textwidth}
        \includegraphics[trim=3 0 3 3, clip, width=\linewidth]{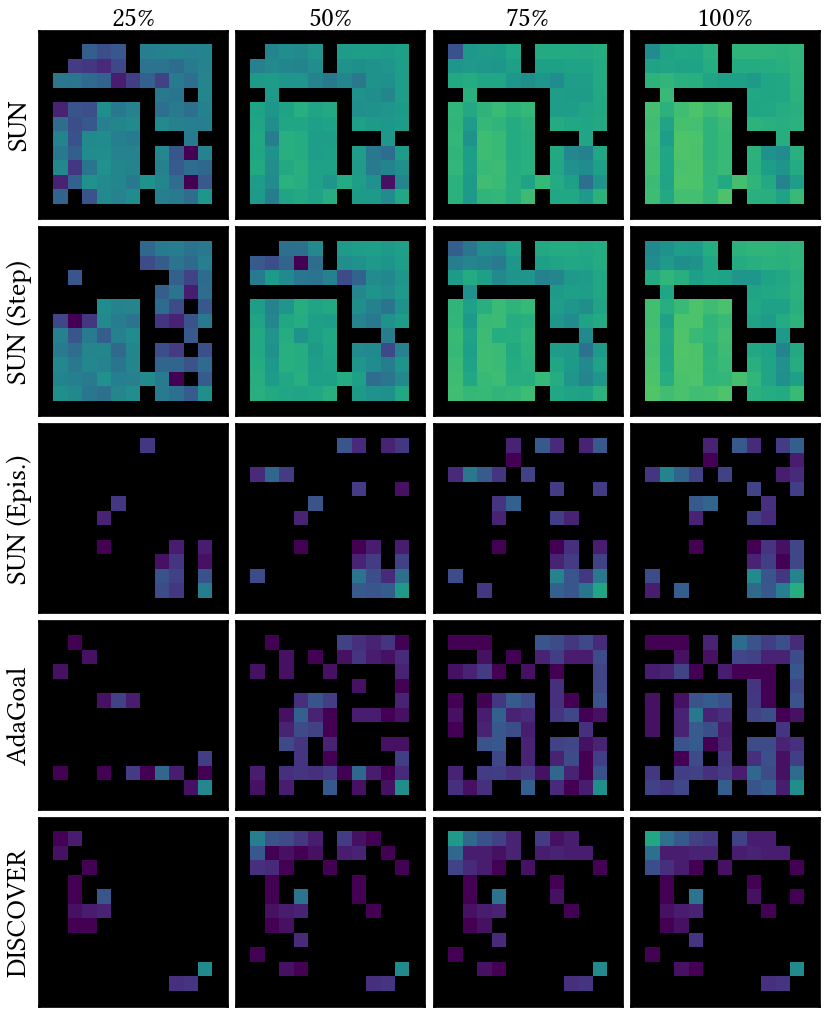}
        \subcaption{{{Goals reached.}}}
    \end{subfigure}
    \hfill
    \begin{subfigure}[b]{0.32\textwidth}
        \includegraphics[width=\linewidth]{plots/goal_stats/legend_horizontal.png}
        \\[-1pt]
        \includegraphics[width=\linewidth]{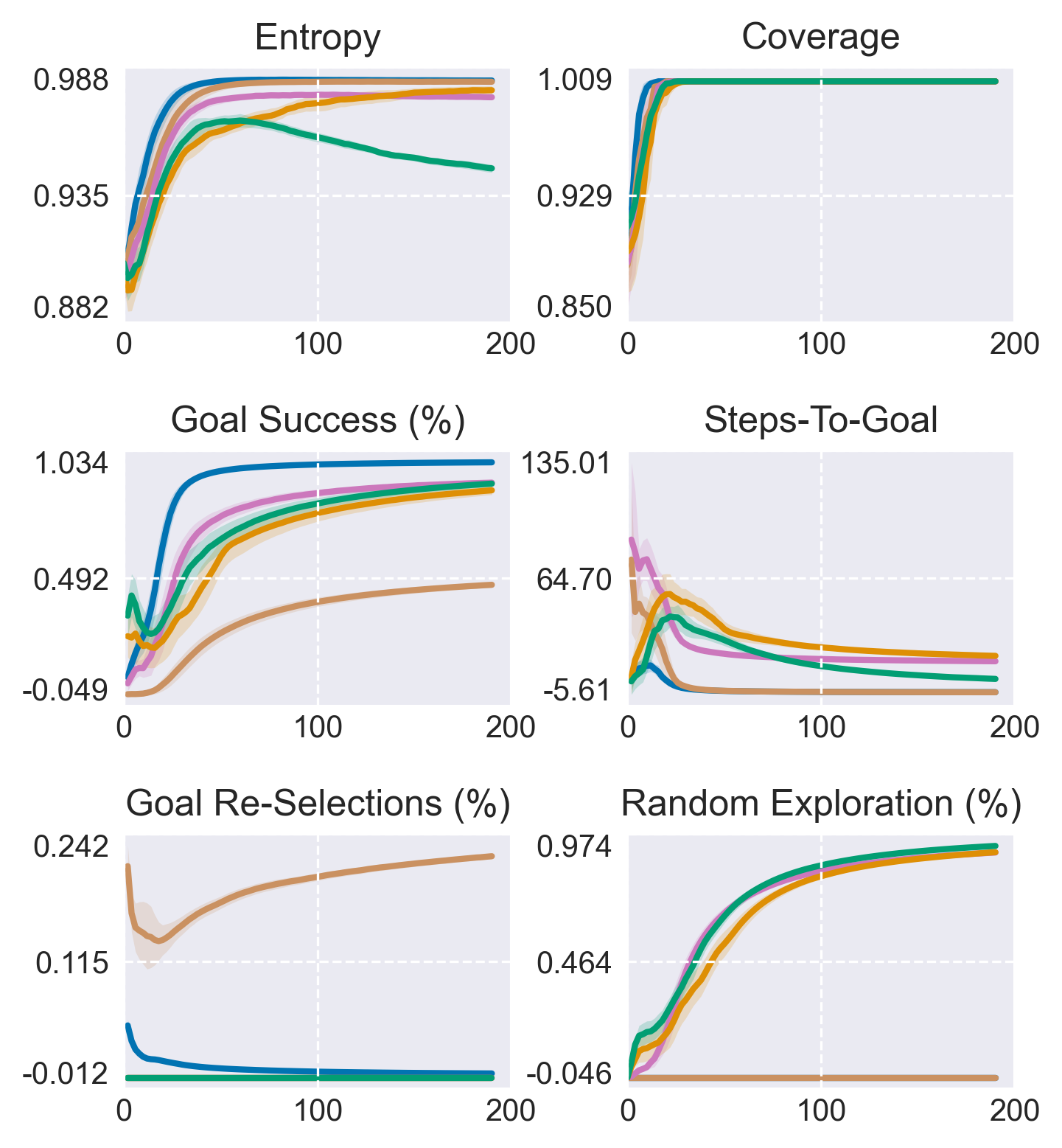}
        \captionsetup{skip=0pt}
        \subcaption{Training curves.}
    \end{subfigure}
    \caption{\label{fig:extra_four_room}\stt{FourRoomStuck}. SUN Episodic, DISCOVER, and AdaGoal perform poorly: the agent can accidentally get trapped in the bottom-left room, at which point the previously selected goal becomes unreachable yet stays fixed for the rest of the episode. SUN Adaptive and Per-Step avoid this through goal reselection, and both attain high entropy. Per-Step entropy increases more slowly, though: it reselects goals far more often than Adaptive, making it unstable early in training.}
    \centering
    \vspace*{2pt}
    \begin{subfigure}[b]{0.31\textwidth}
        \includegraphics[trim=3 0 3 3, clip, width=\linewidth]{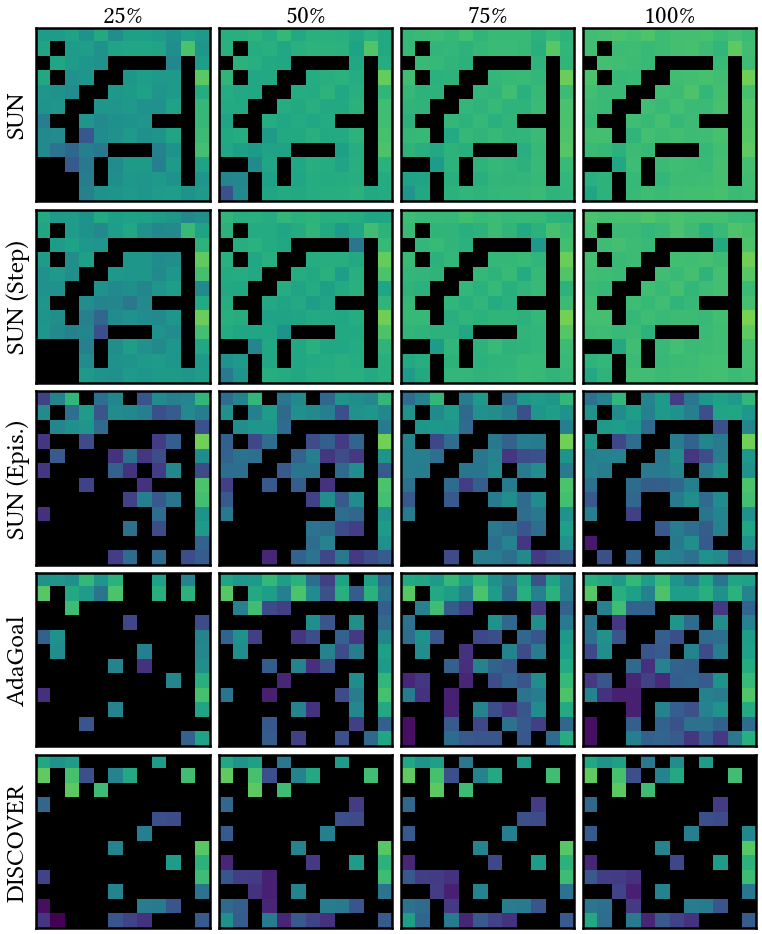}
        \subcaption{{{Goals selected.}}}
        \end{subfigure}
    \hfill
    \begin{subfigure}[b]{0.31\textwidth}
        \includegraphics[trim=3 0 3 3, clip, width=\linewidth]{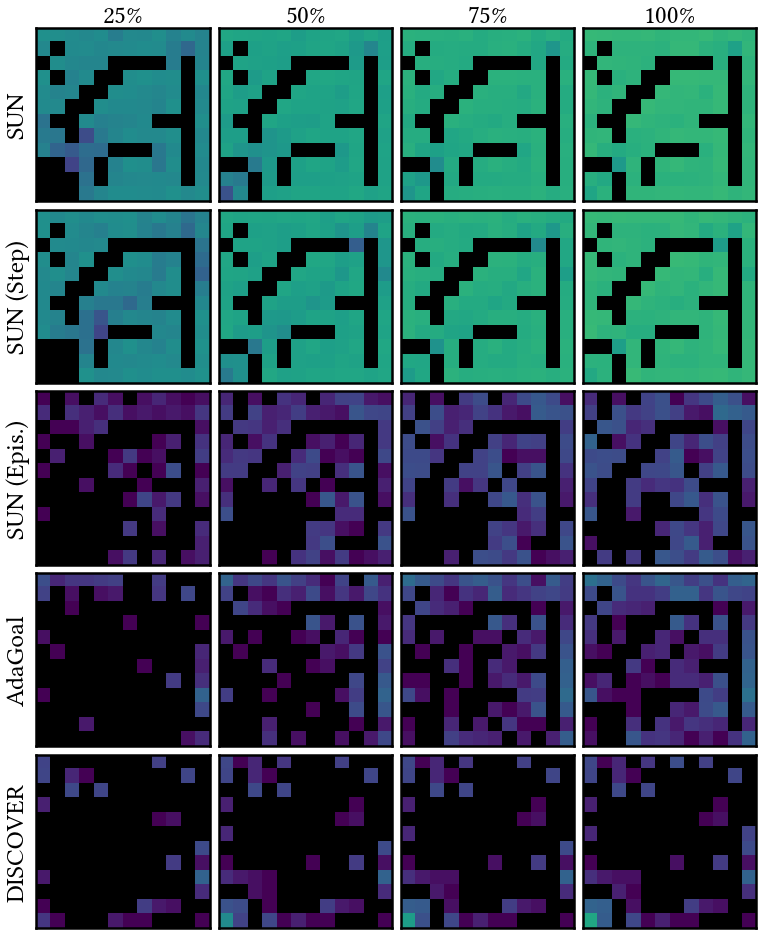}
        \subcaption{{{Goals reached.}}}
    \end{subfigure}
    \hfill
    \begin{subfigure}[b]{0.32\textwidth}
        \includegraphics[width=\linewidth]{plots/goal_stats/legend_horizontal.png}
        \\[-1pt]
        \includegraphics[width=\linewidth]{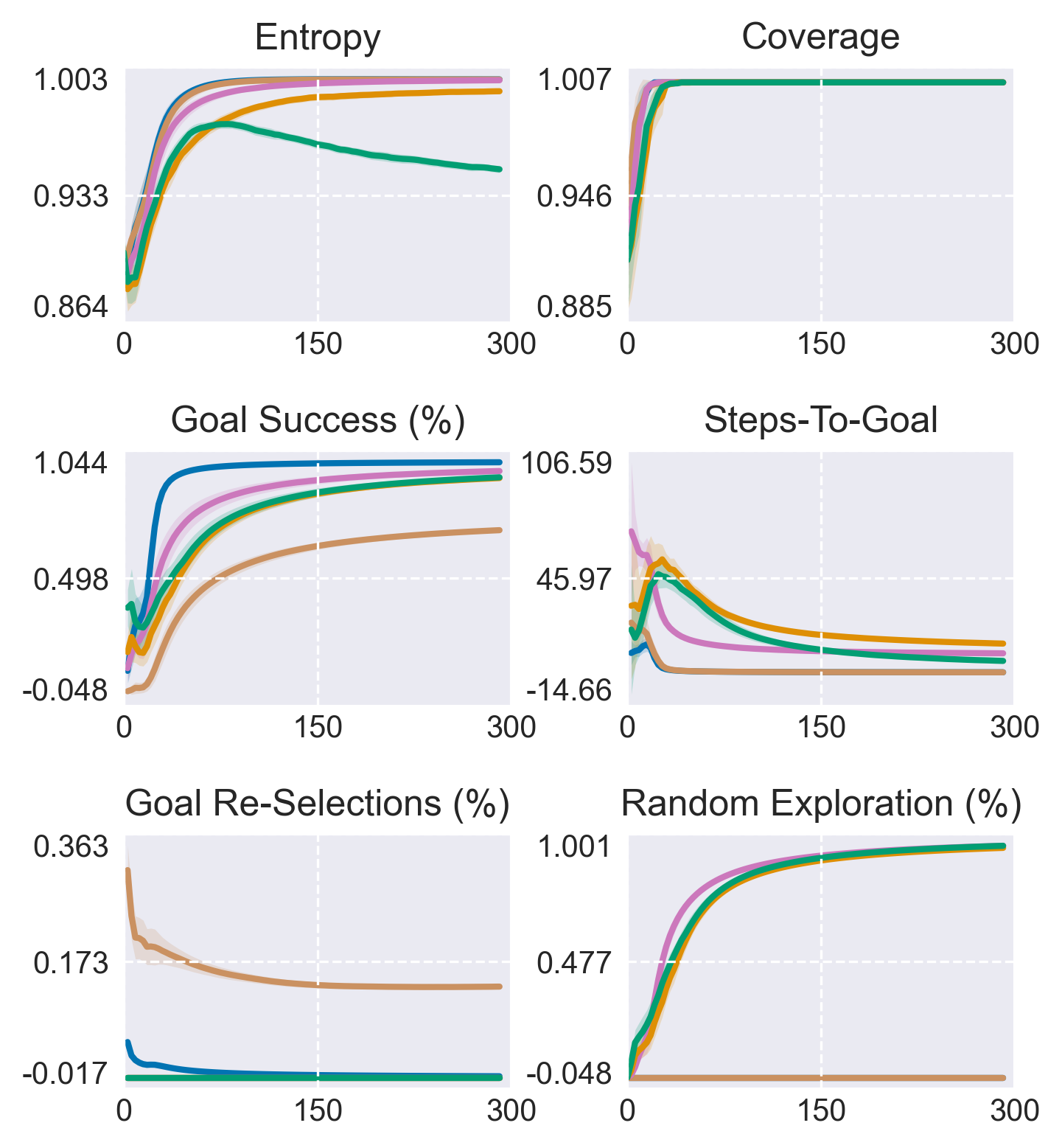}
        \captionsetup{skip=0pt}
        \subcaption{Training curves.}
    \end{subfigure}
    \caption{\label{fig:extra_maze}\stt{GridMaze}. Similar trends to Figure \ref{fig:extra_four_room} emerge, e.g., DISCOVER bias towards easy-to-reach goals and Per-Step higher reselection rate.}
    \centering
    \vspace*{2pt}
    \begin{subfigure}[b]{0.31\textwidth}
        \includegraphics[trim=3 0 3 3, clip, width=\linewidth]{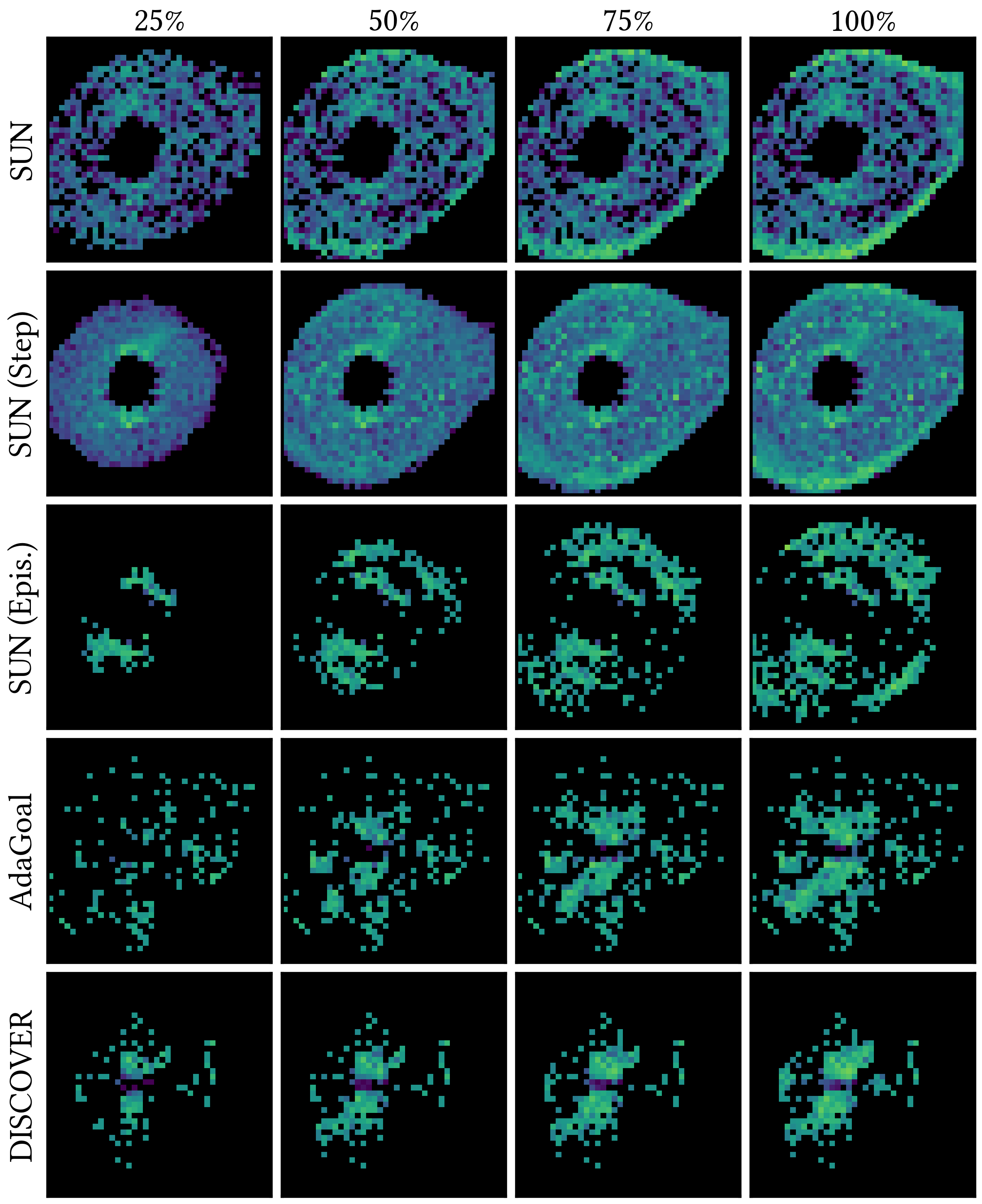}
        \subcaption{{{Goals selected.}}}
        \end{subfigure}
    \hfill
    \begin{subfigure}[b]{0.31\textwidth}
        \includegraphics[trim=3 0 3 3, clip, width=\linewidth]{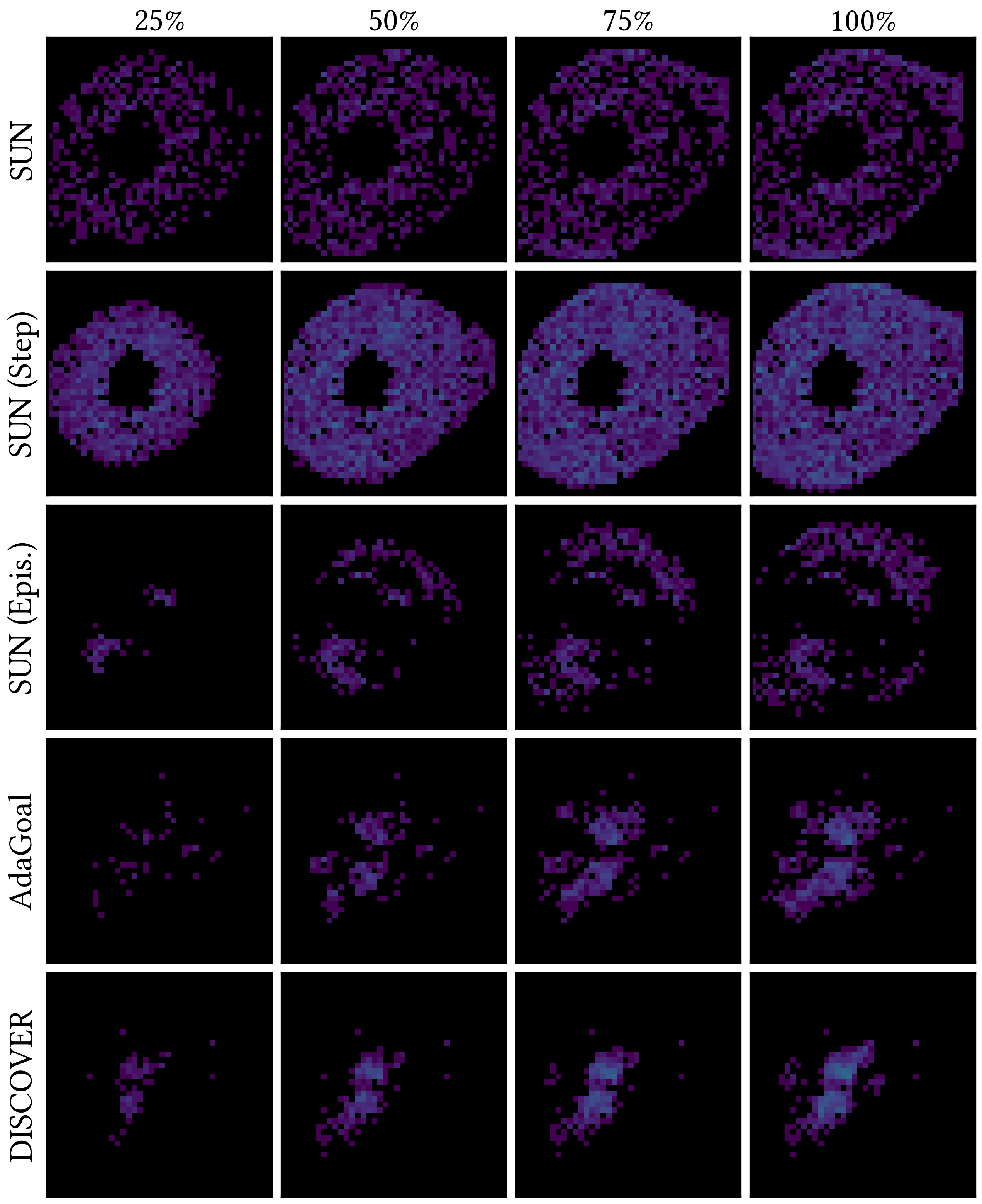}
        \subcaption{{{Goals reached.}}}
    \end{subfigure}
    \hfill
    \begin{subfigure}[b]{0.32\textwidth}
        \includegraphics[width=\linewidth]{plots/goal_stats/legend_horizontal.png}
        \\[-1pt]
        \includegraphics[width=\linewidth]{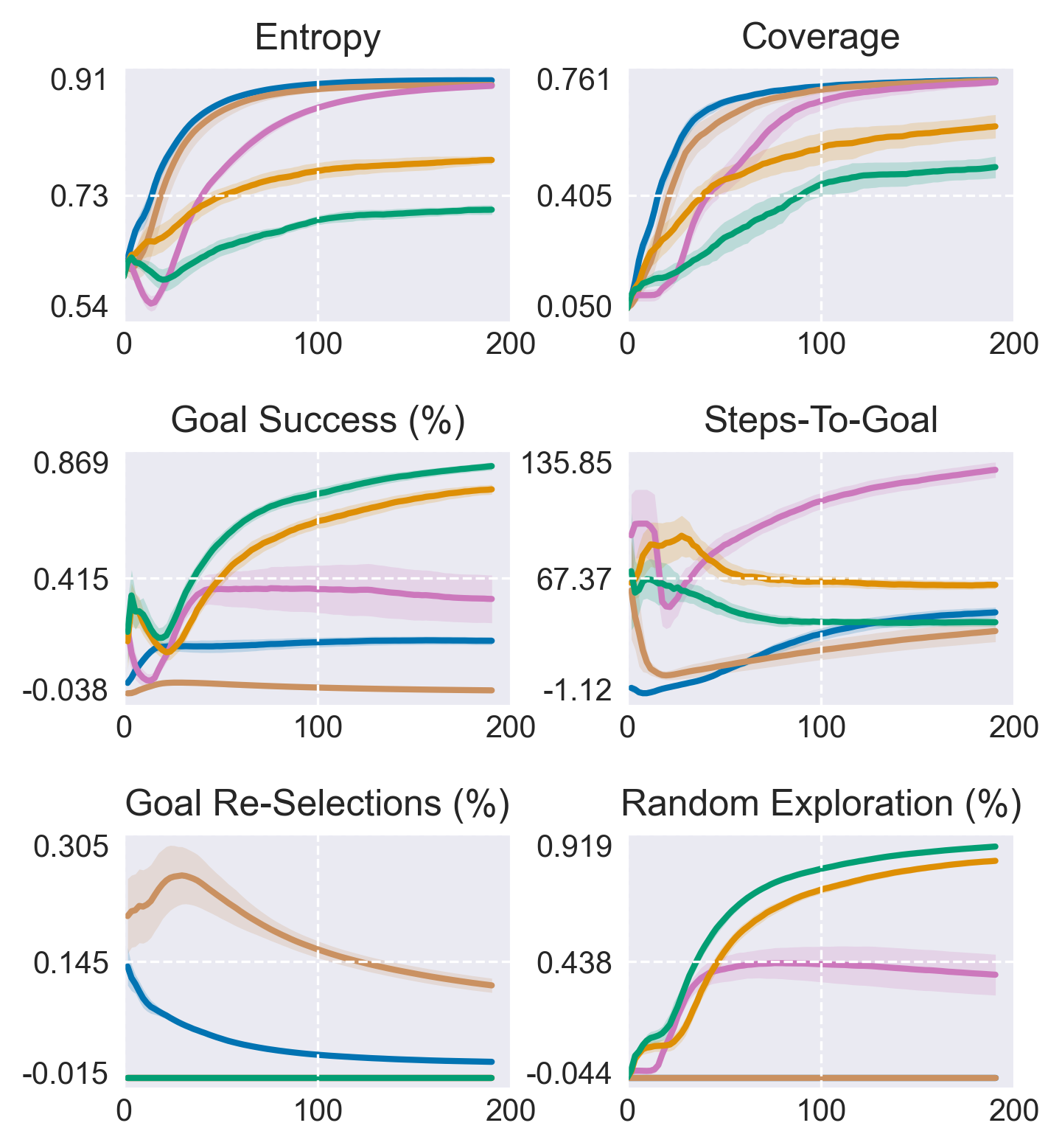}
        \captionsetup{skip=0pt}
        \subcaption{Training curves.}
    \end{subfigure}
    \caption{\label{fig:extra_car}\stt{MountainCar}. The main difference from Gridworlds is that SUN's steps-to-goal increases over training. This is expected, as \stt{MountainCar} requires coherent action sequences to build up the momentum needed to escape the valley. This environment further highlights DISCOVER's reachability bias, as most of its goals lie close to the agent's starting position.}
    \centering
\end{figure}

\clearpage

\begin{figure}
    \begin{subfigure}[b]{0.31\textwidth}
        \includegraphics[trim=3 0 3 3, clip, width=\linewidth]{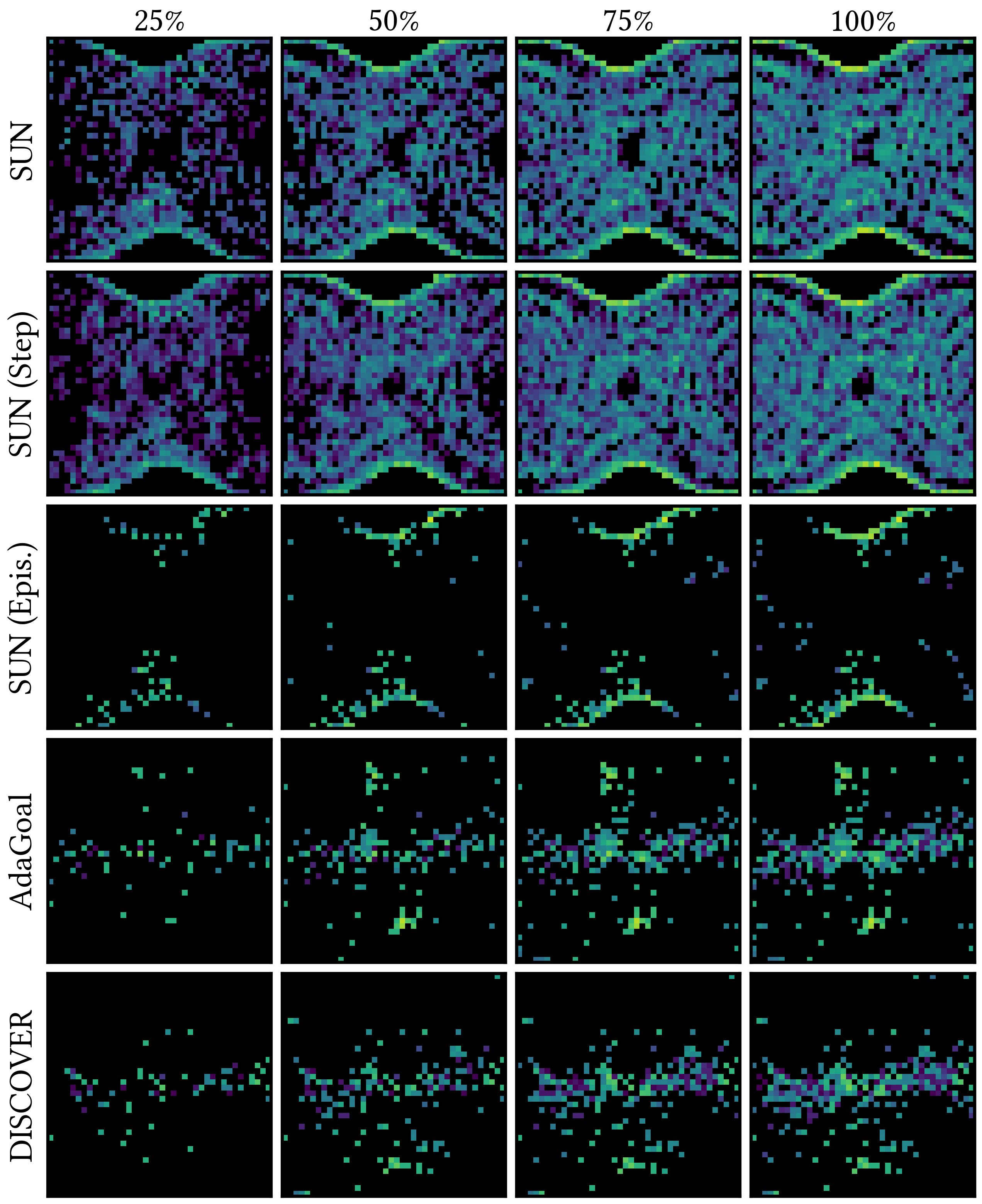}
        \subcaption{{{Goals selected.}}}
        \end{subfigure}
    \hfill
    \begin{subfigure}[b]{0.31\textwidth}
        \includegraphics[trim=3 0 3 3, clip, width=\linewidth]{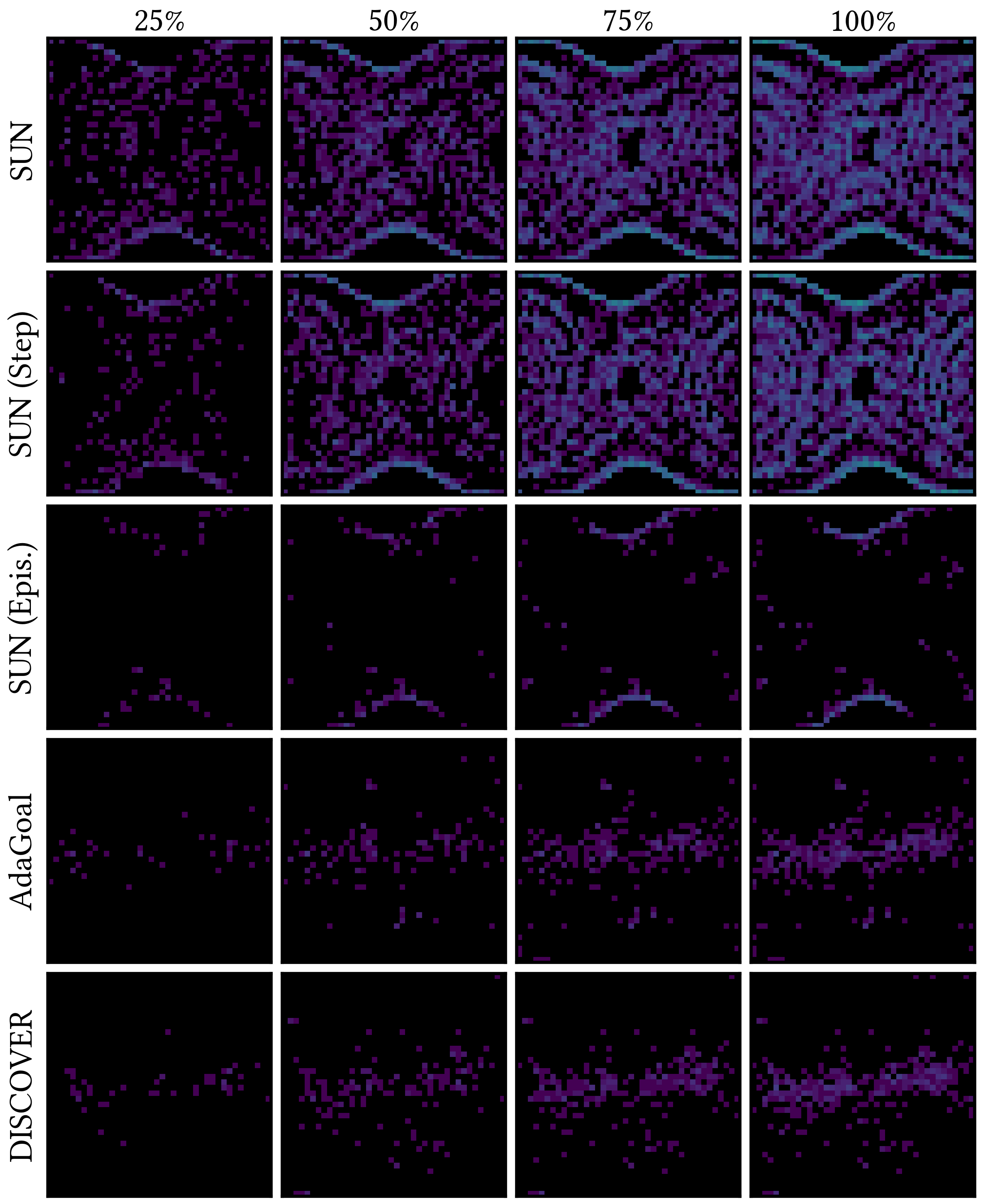}
        \subcaption{{{Goals reached.}}}
    \end{subfigure}
    \hfill
    \begin{subfigure}[b]{0.32\textwidth}
        \includegraphics[width=\linewidth]{plots/goal_stats/legend_horizontal.png}
        \\[-1pt]
        \includegraphics[width=\linewidth]{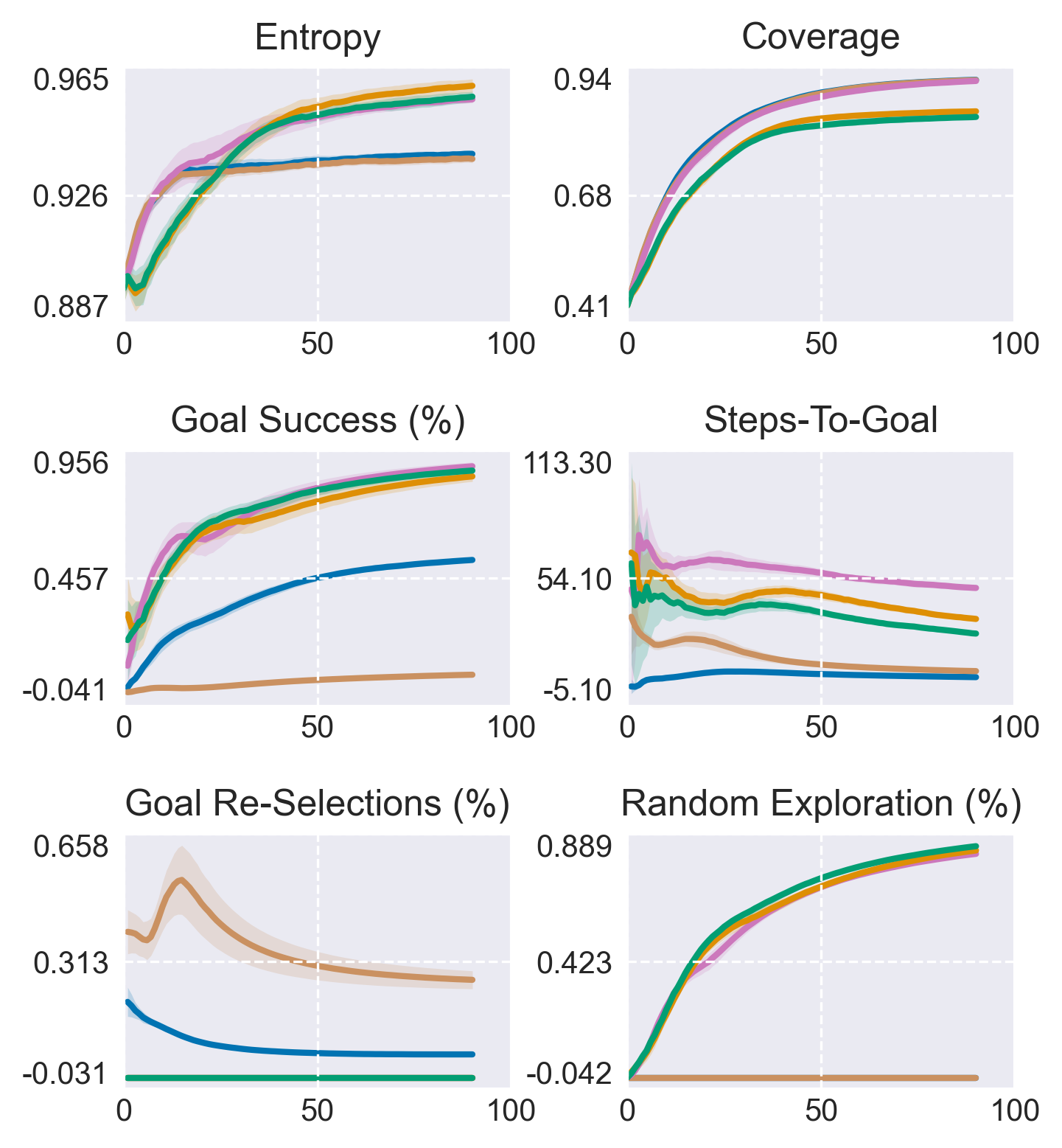}
        \captionsetup{skip=0pt}
        \subcaption{Training curves.}
    \end{subfigure}
    \caption{\label{fig:extra_pendulum}\stt{Pendulum}. As discussed in Section~\ref{sec:experiments}, the distinctive trait of this environment is the presence of hard-to-reach goals that can only be visited by repeatedly traversing easy-to-reach ones, causing entropy to decay even as coverage increases.}
    \vspace*{2pt}
    \centering
    \begin{subfigure}[b]{0.31\textwidth}
        \includegraphics[trim=3 0 3 3, clip, width=\linewidth]{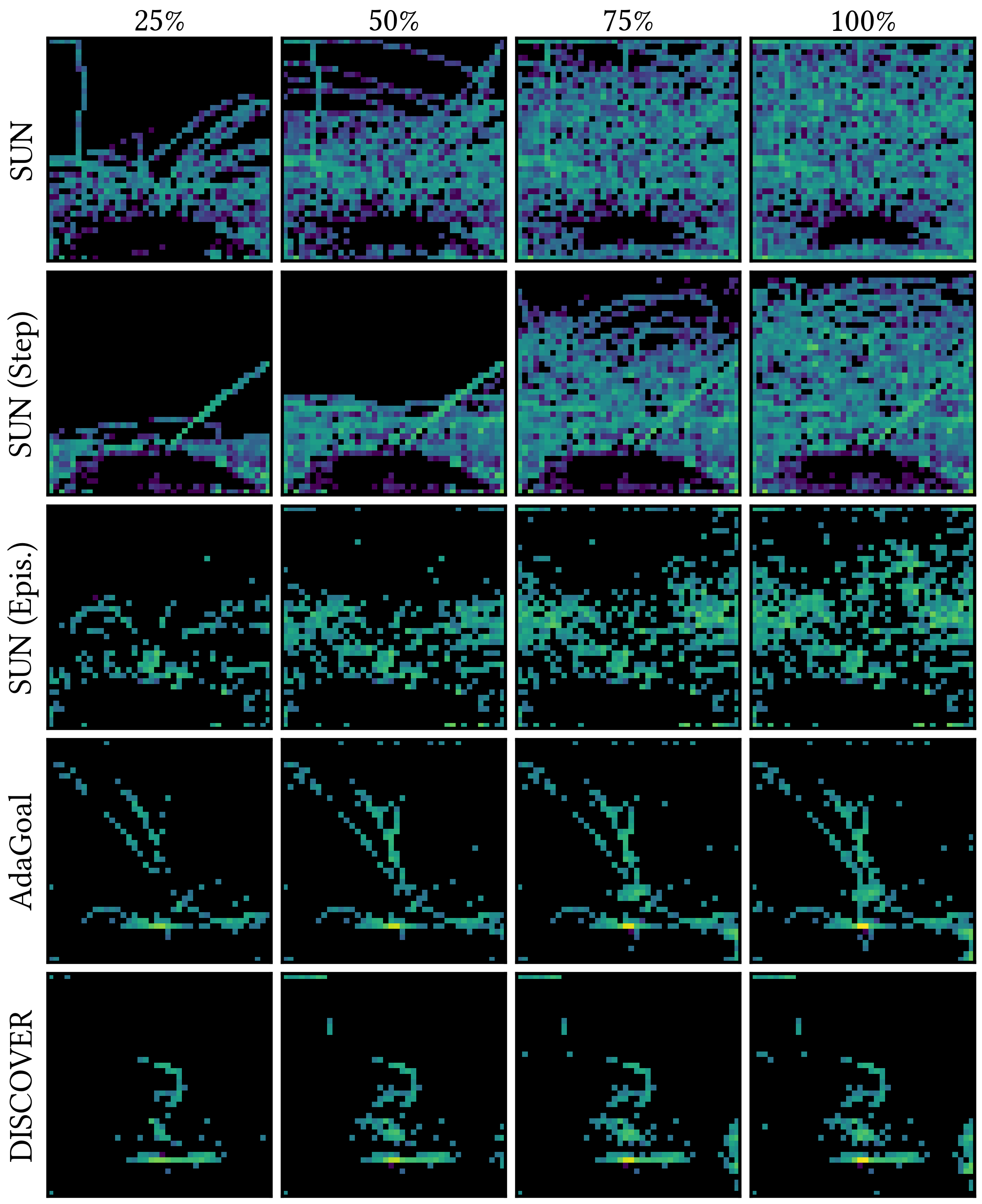}
        \subcaption{{{Goals selected.}}}
        \end{subfigure}
    \hfill
    \begin{subfigure}[b]{0.31\textwidth}
        \includegraphics[trim=3 0 3 3, clip, width=\linewidth]{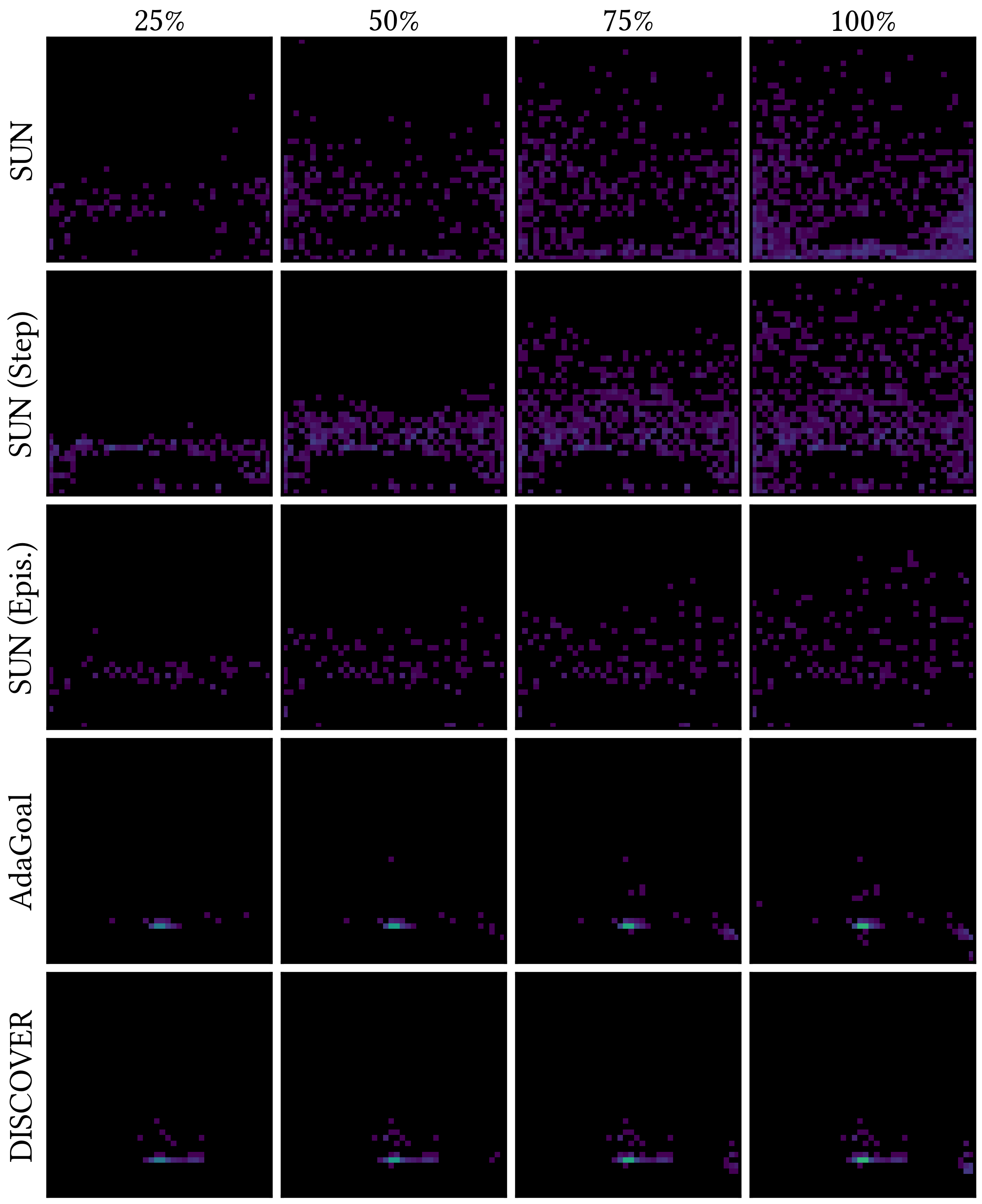}
        \subcaption{{{Goals reached.}}}
    \end{subfigure}
    \hfill
    \begin{subfigure}[b]{0.32\textwidth}
        \includegraphics[width=\linewidth]{plots/goal_stats/legend_horizontal.png}
        \\[-1pt]
        \includegraphics[width=\linewidth]{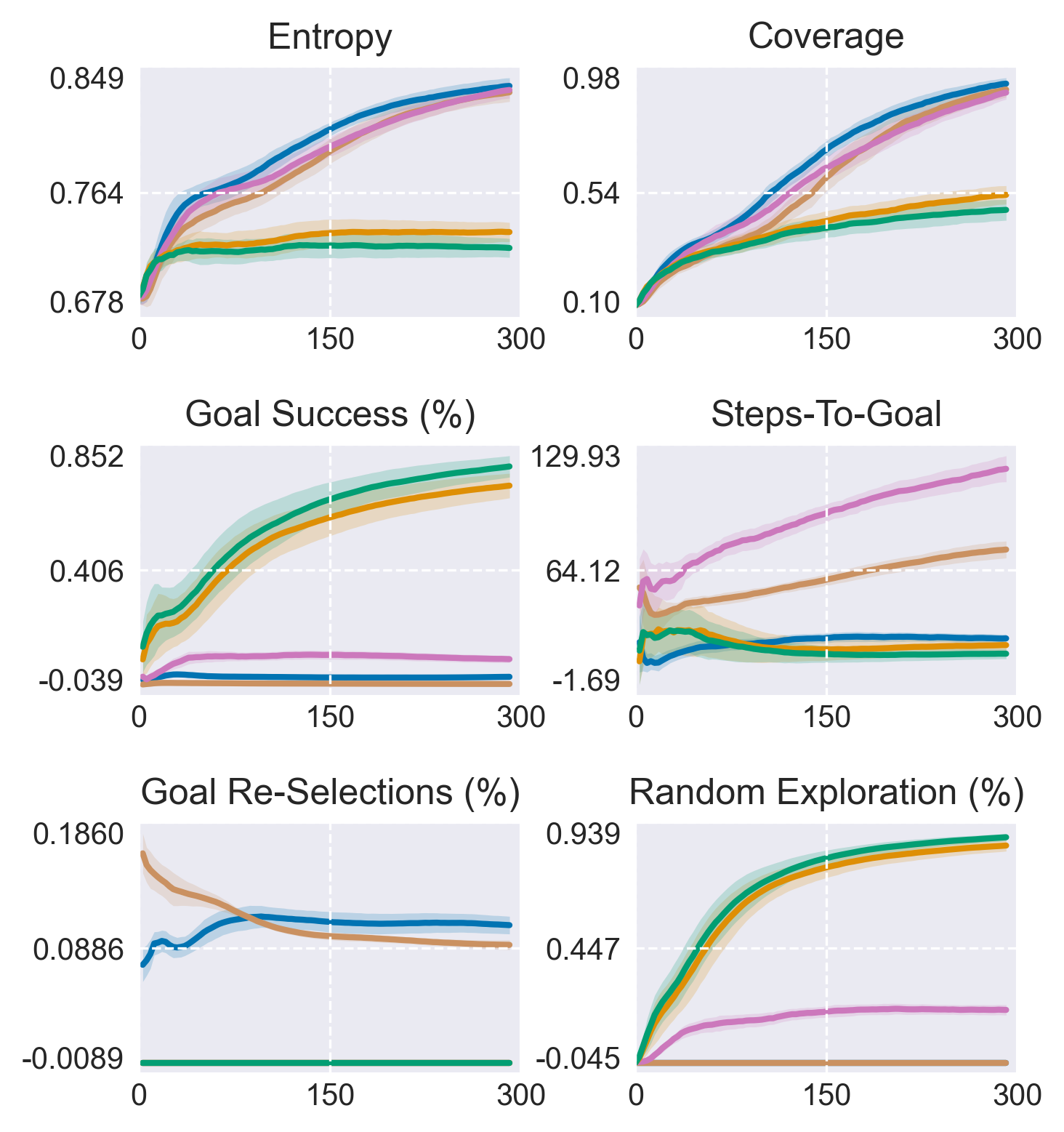}
        \captionsetup{skip=0pt}
        \subcaption{Training curves.}
    \end{subfigure}
    \caption{\label{fig:extra_lunar}\stt{LunarLander}. Among the Classic control tasks, this is the most challenging: it has harder dynamics, higher dimensionality, and a larger goal space --- reasons why SUN's reselection rate does not tend to zero. The difference between SUN and AdaGoal/DISCOVER is more apparent here: SUN's selected goals are spread more uniformly across the whole space. DISCOVER selects mostly goals near the starting state, whereas AdaGoal often picks goals too far away for the agent to reach (e.g., the isolated goal at the top of its heatmaps), confirming its bias toward novelty.}
    \centering
    \vspace*{2pt}
    \begin{subfigure}[b]{0.31\textwidth}
        \includegraphics[trim=3 0 3 3, clip, width=\linewidth]{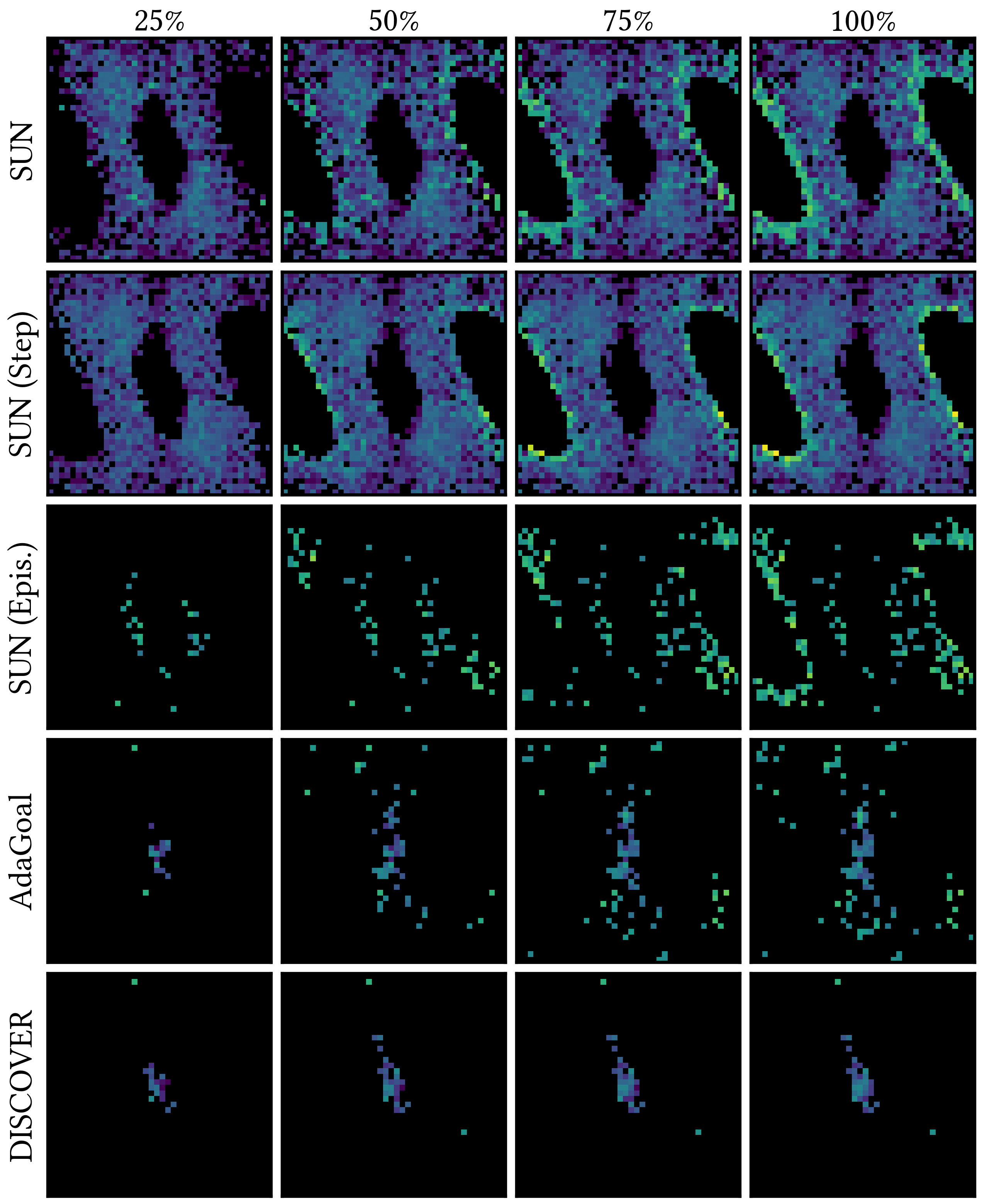}
        \subcaption{{{Goals selected.}}}
        \end{subfigure}
    \hfill
    \begin{subfigure}[b]{0.31\textwidth}
        \includegraphics[trim=3 0 3 3, clip, width=\linewidth]{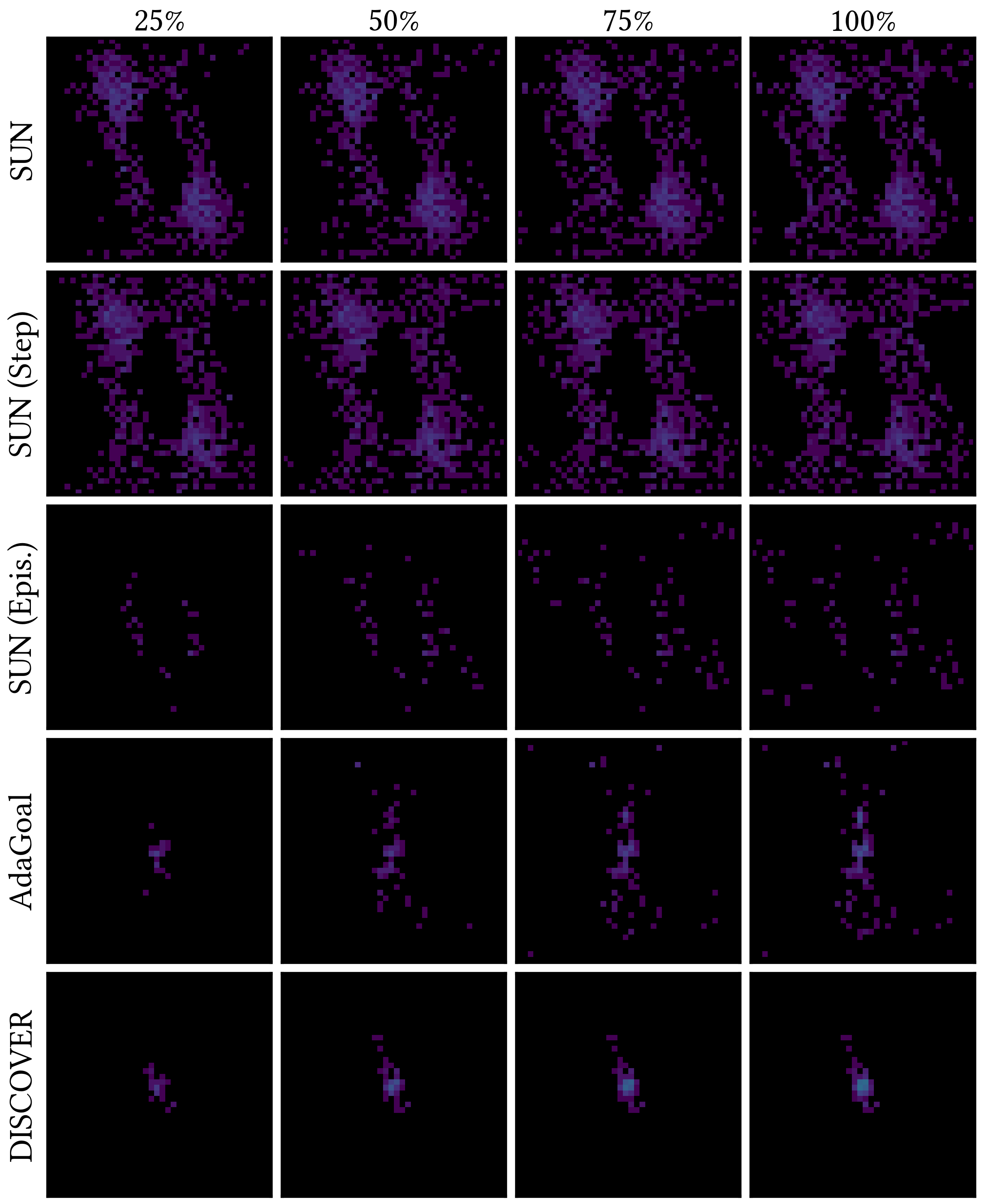}
        \subcaption{{{Goals reached.}}}
    \end{subfigure}
    \hfill
    \begin{subfigure}[b]{0.32\textwidth}
        \includegraphics[width=\linewidth]{plots/goal_stats/legend_horizontal.png}
        \\[-1pt]
        \includegraphics[width=\linewidth]{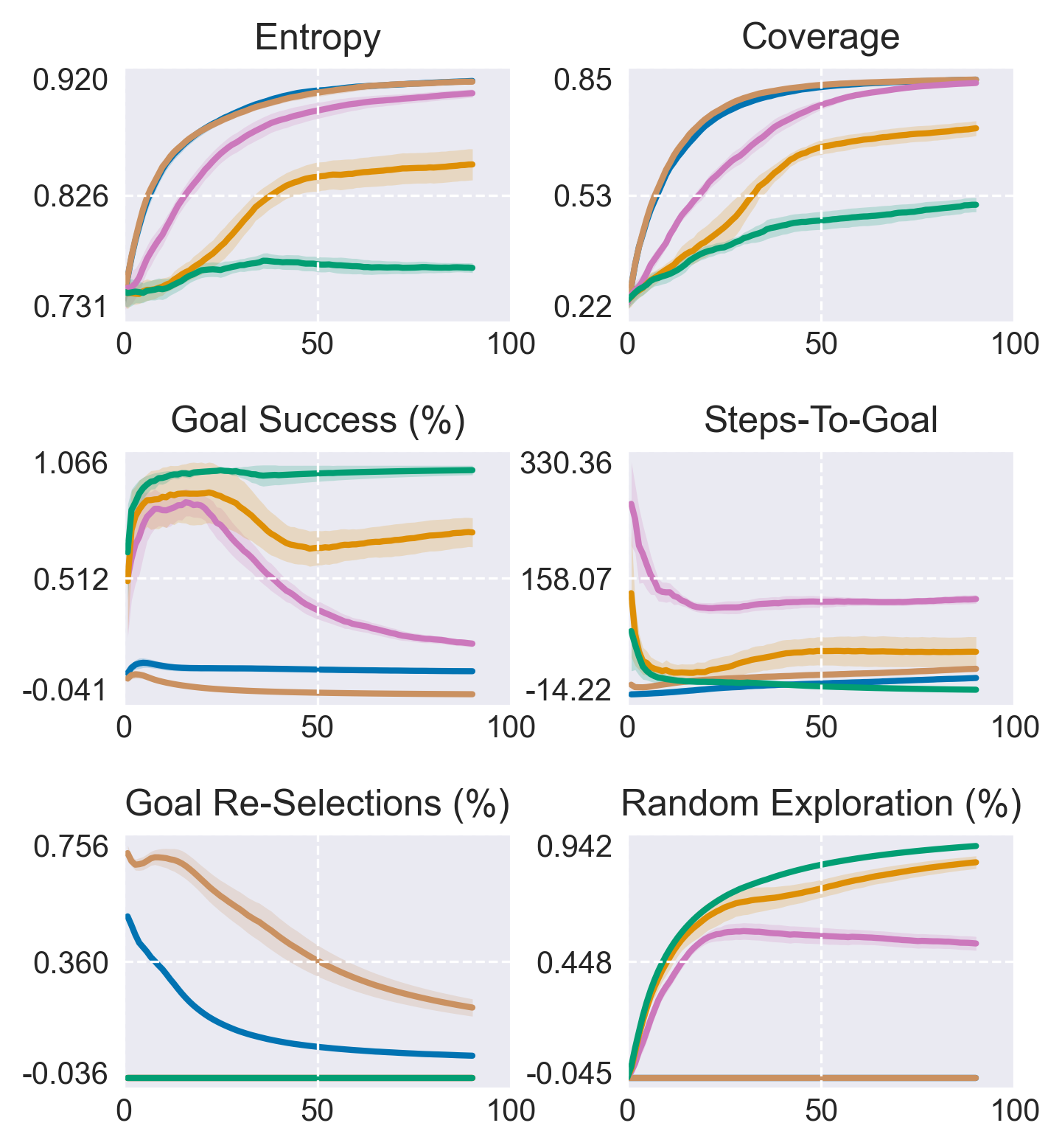}
        \captionsetup{skip=0pt}
        \subcaption{Training curves.}
    \end{subfigure}
    \caption{\label{fig:extra_acro}\stt{Acrobot}. All trends discussed so far emerge here as well, most notably the different spread of goals selected by SUN compared to AdaGoal and DISCOVER.}
    \centering
\end{figure}

\clearpage

\begin{figure}
    \begin{subfigure}[b]{0.31\textwidth}
        \includegraphics[trim=3 0 3 3, clip, width=\linewidth]{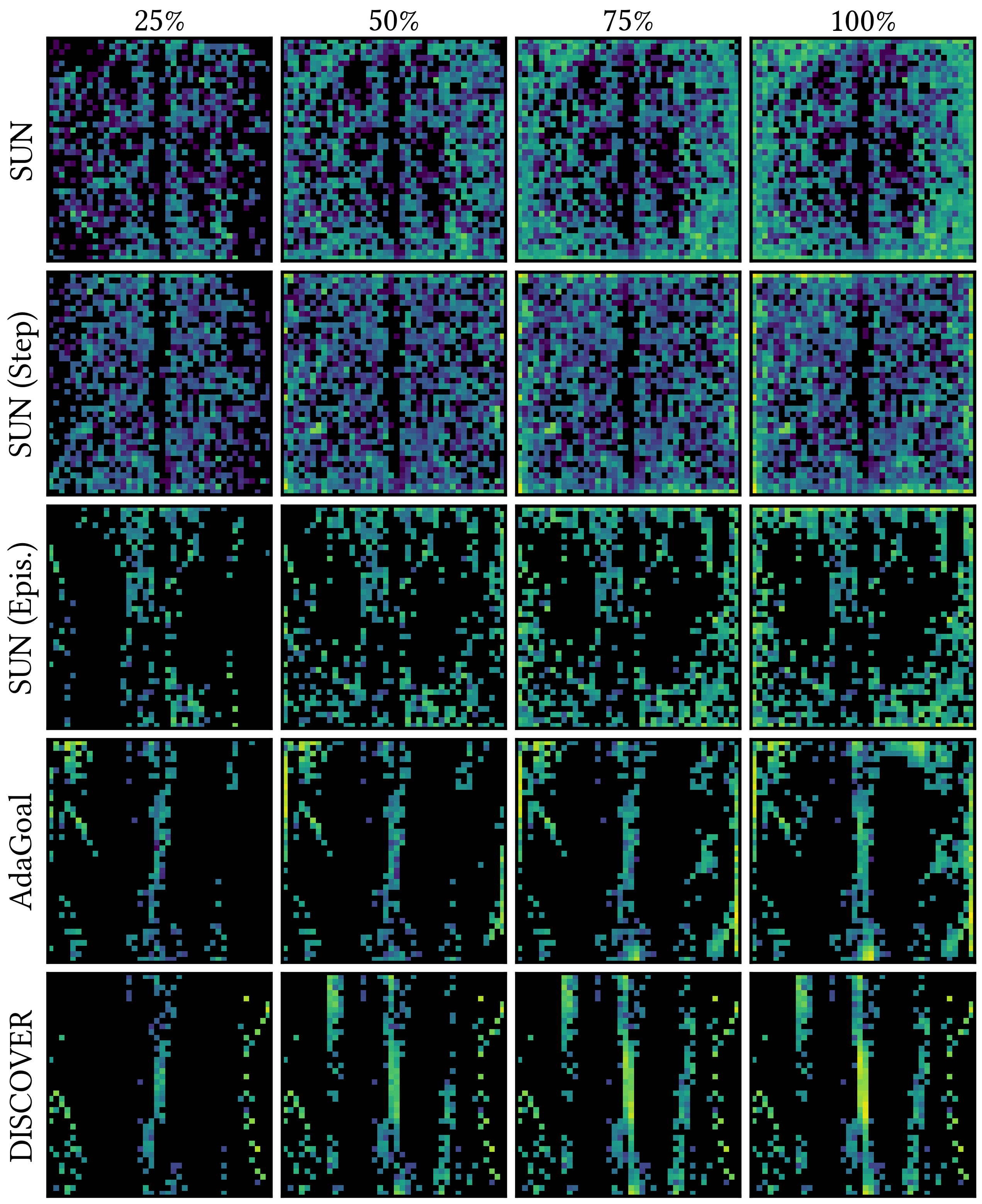}
        \subcaption{{{Goals selected.}}}
        \end{subfigure}
    \hfill
    \begin{subfigure}[b]{0.31\textwidth}
        \includegraphics[trim=3 0 3 3, clip, width=\linewidth]{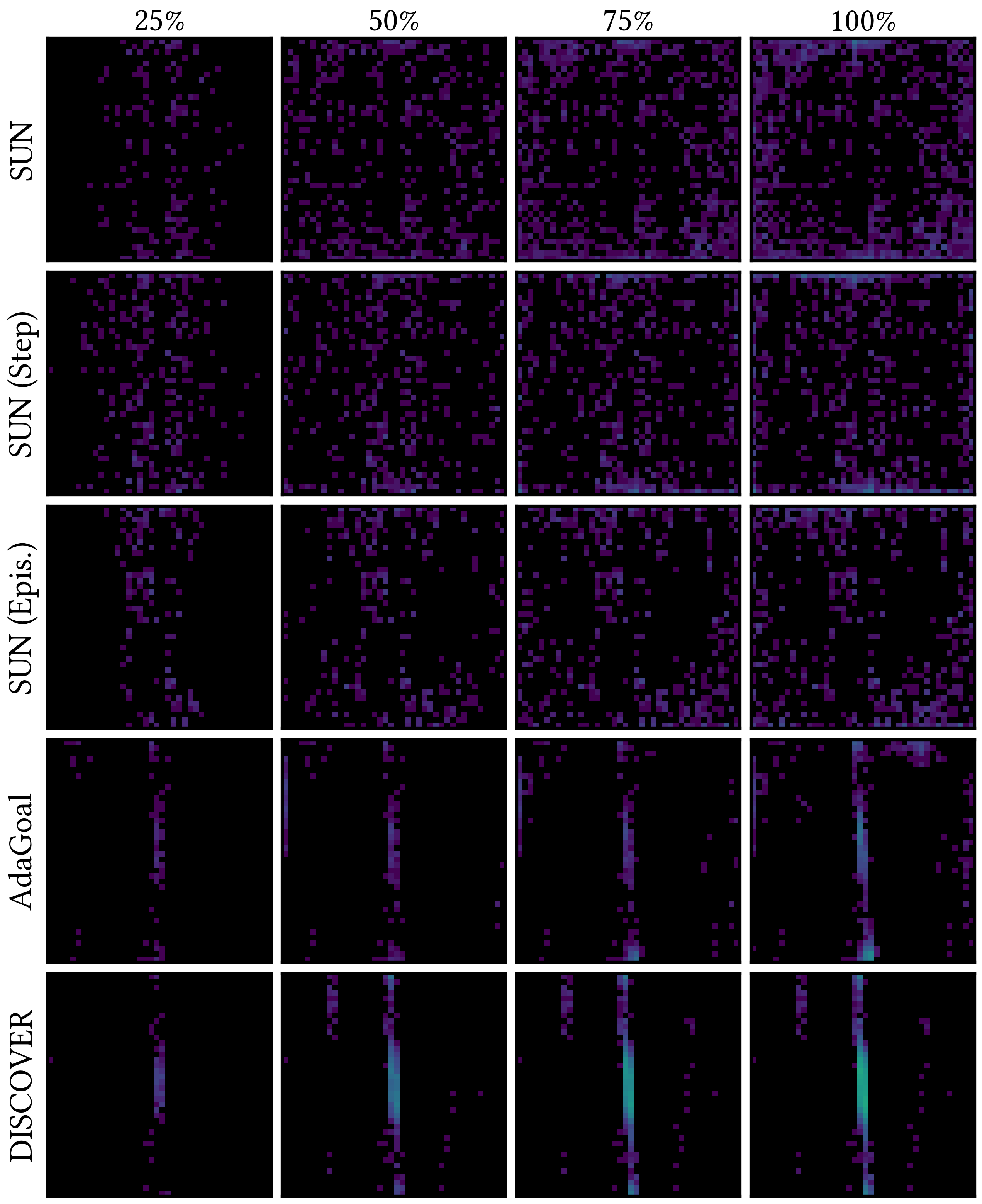}
        \subcaption{{{Goals reached.}}}
    \end{subfigure}
    \hfill
    \begin{subfigure}[b]{0.32\textwidth}
        \includegraphics[width=\linewidth]{plots/goal_stats/legend_horizontal.png}
        \\[-1pt]
        \includegraphics[width=\linewidth]{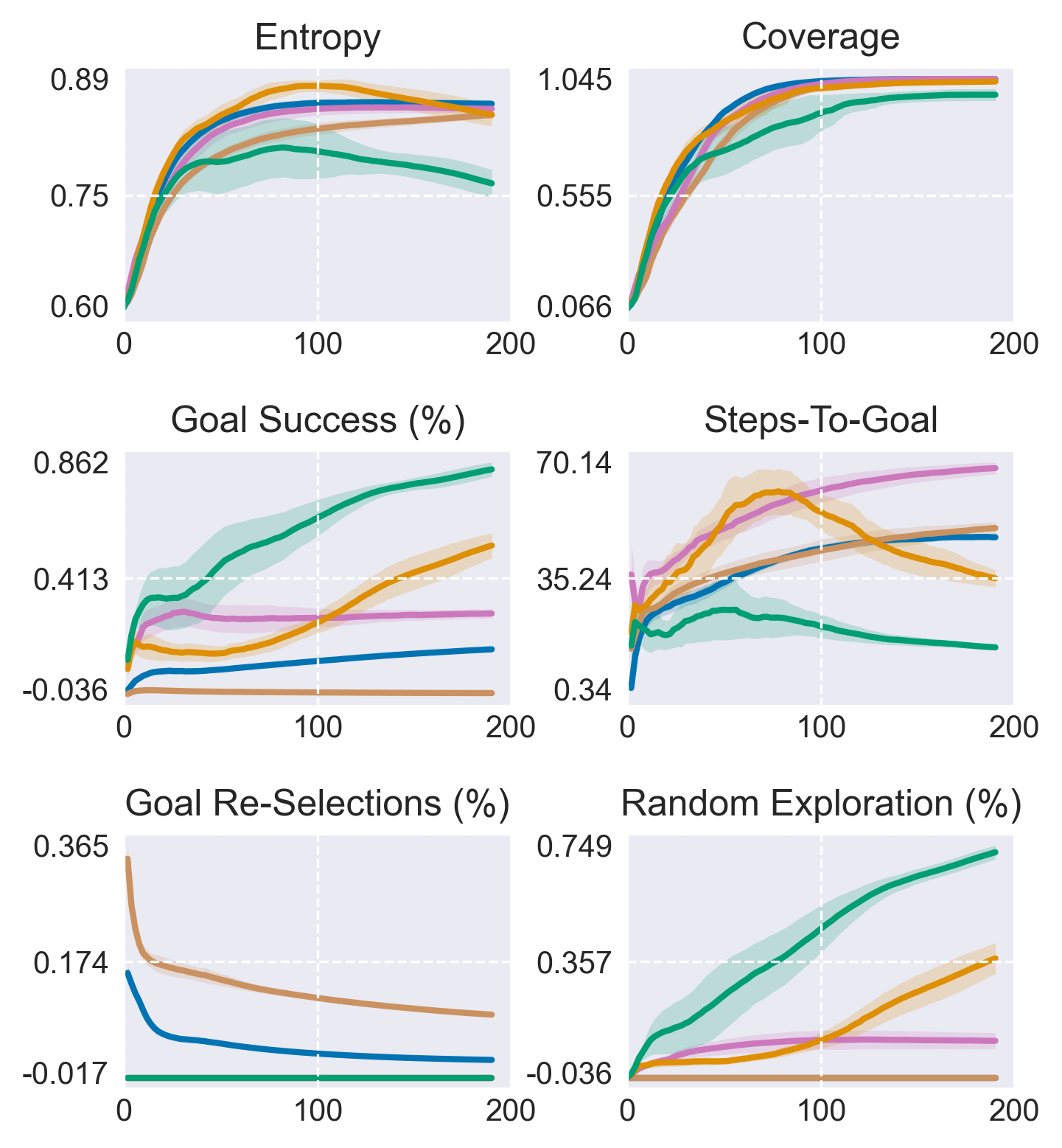}
        \captionsetup{skip=0pt}
        \subcaption{Training curves.}
    \end{subfigure}
    \caption{\label{fig:extra_cart}\stt{CartPole}. All trends discussed so far emerge here as well.}
    \vspace*{2pt}
    \centering
    \begin{subfigure}[b]{0.31\textwidth}
        \includegraphics[trim=3 0 3 3, clip, width=\linewidth]{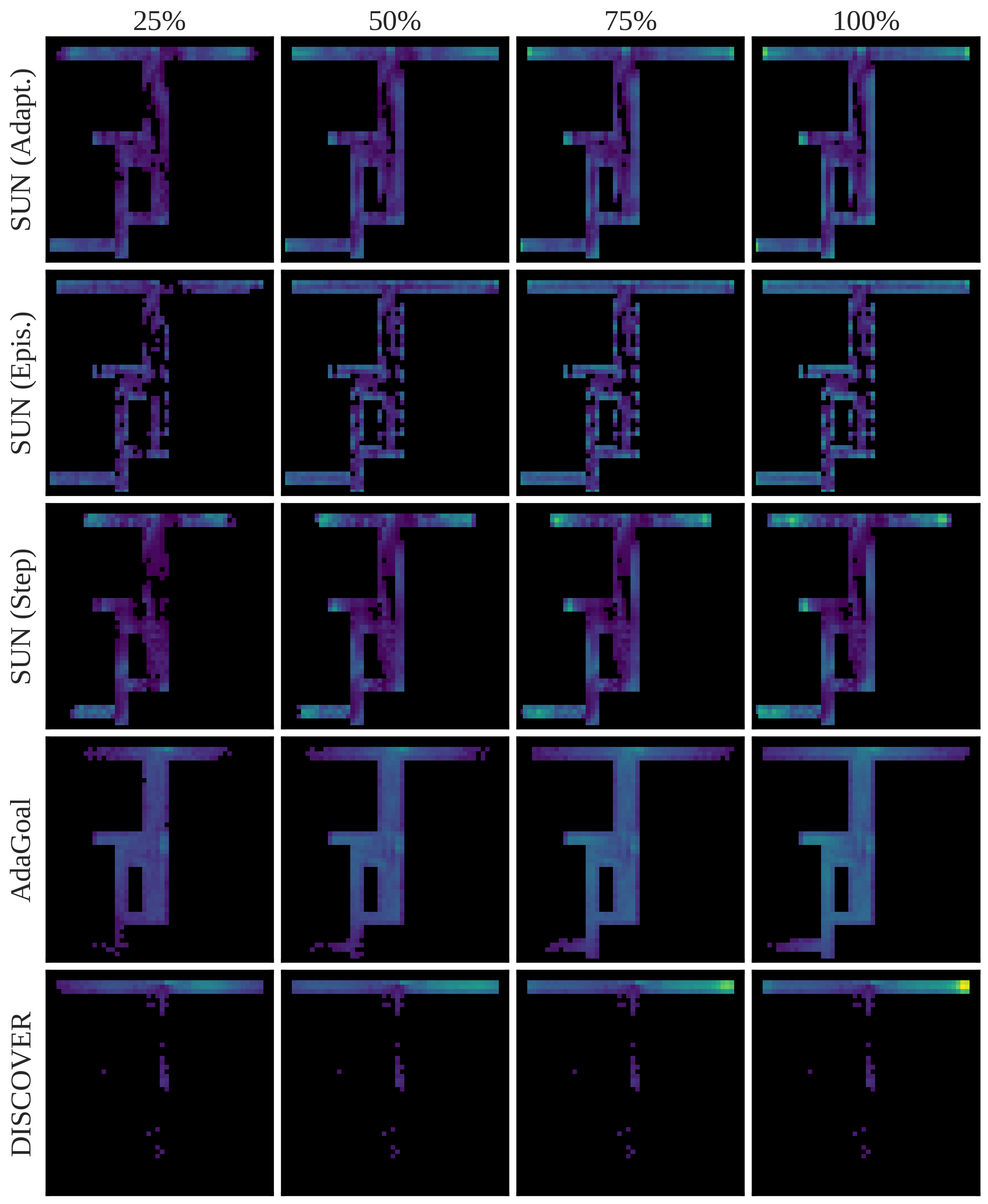}
        \subcaption{{{Goals selected.}}}
        \end{subfigure}
    \hfill
    \begin{subfigure}[b]{0.31\textwidth}
        \includegraphics[trim=3 0 3 3, clip, width=\linewidth]{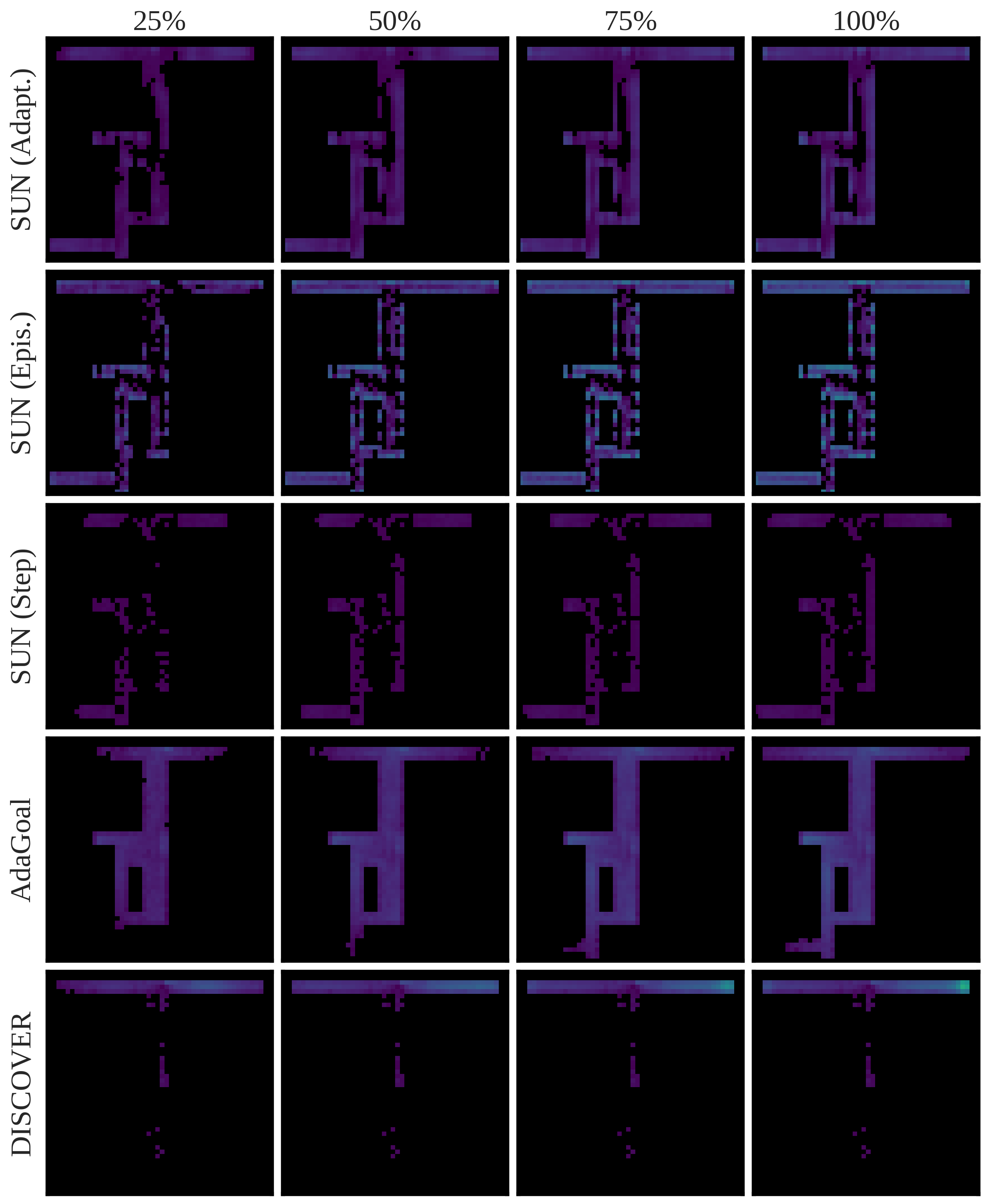}
        \subcaption{{{Goals reached.}}}
    \end{subfigure}
    \hfill
    \begin{subfigure}[b]{0.30\textwidth}
        \includegraphics[width=\linewidth]{plots/goal_stats/legend_horizontal.png}
        \\[-1pt]
        \includegraphics[width=\linewidth]{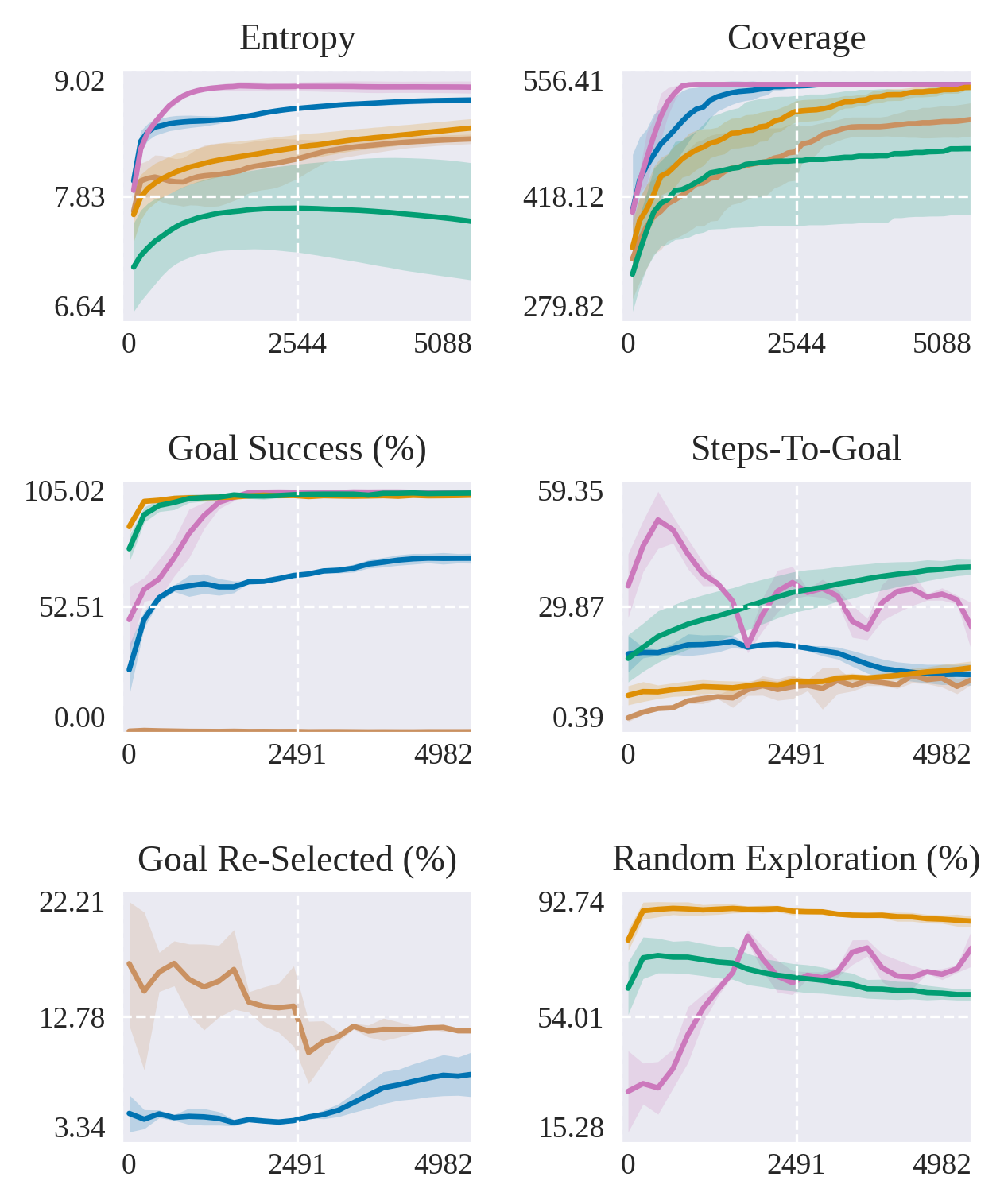}
        \captionsetup{skip=0pt}
        \subcaption{Training curves.}
    \end{subfigure}
    \caption{\label{fig:extra_point_maze_simple}\stt{PointMaze-S}.}
    \vspace*{2pt}
    \centering
    \begin{subfigure}[b]{0.31\textwidth}
        \includegraphics[trim=3 0 3 3, clip, width=\linewidth]{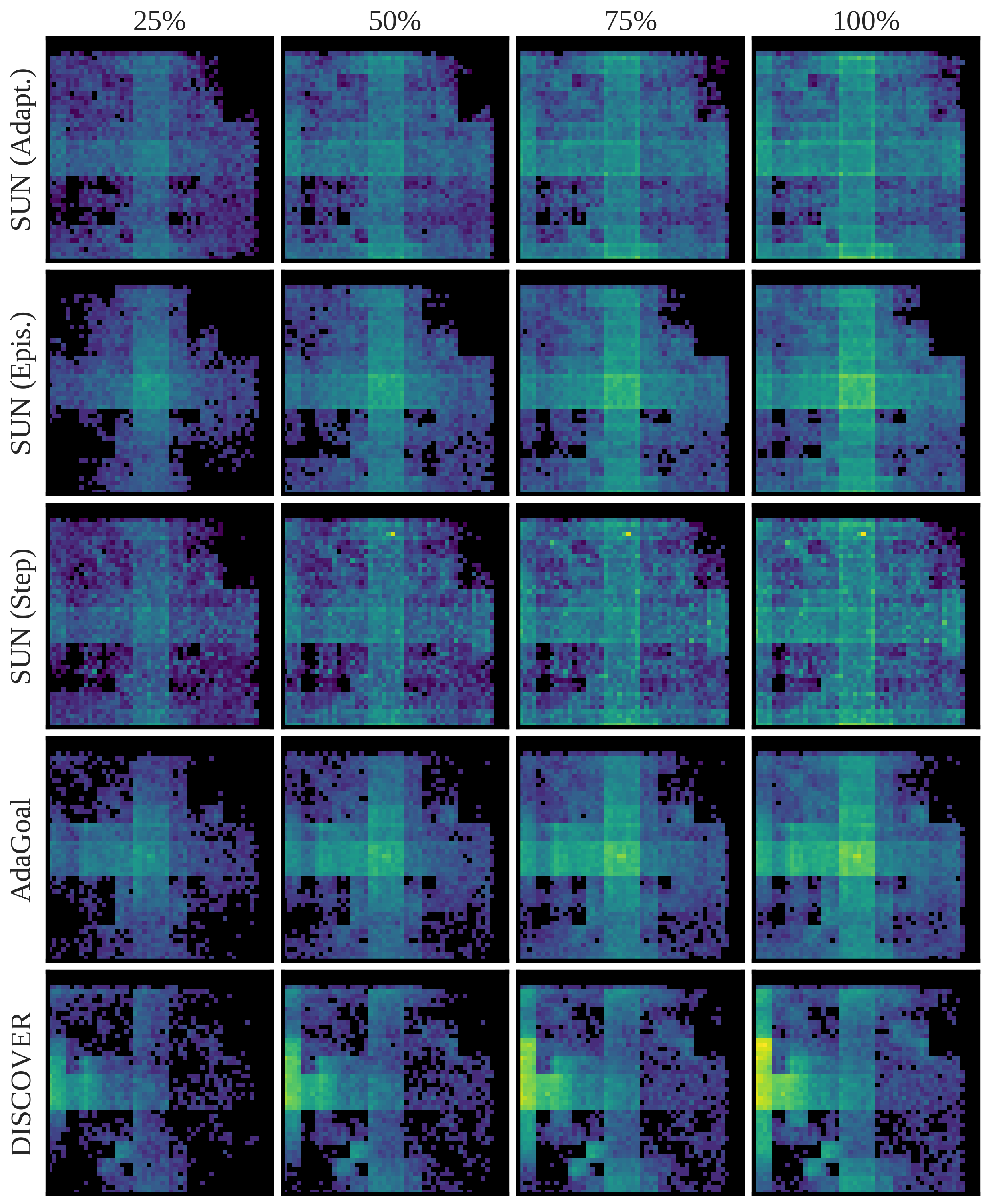}
        \subcaption{{{Goals selected.}}}
        \end{subfigure}
    \hfill
    \begin{subfigure}[b]{0.31\textwidth}
        \includegraphics[trim=3 0 3 3, clip, width=\linewidth]{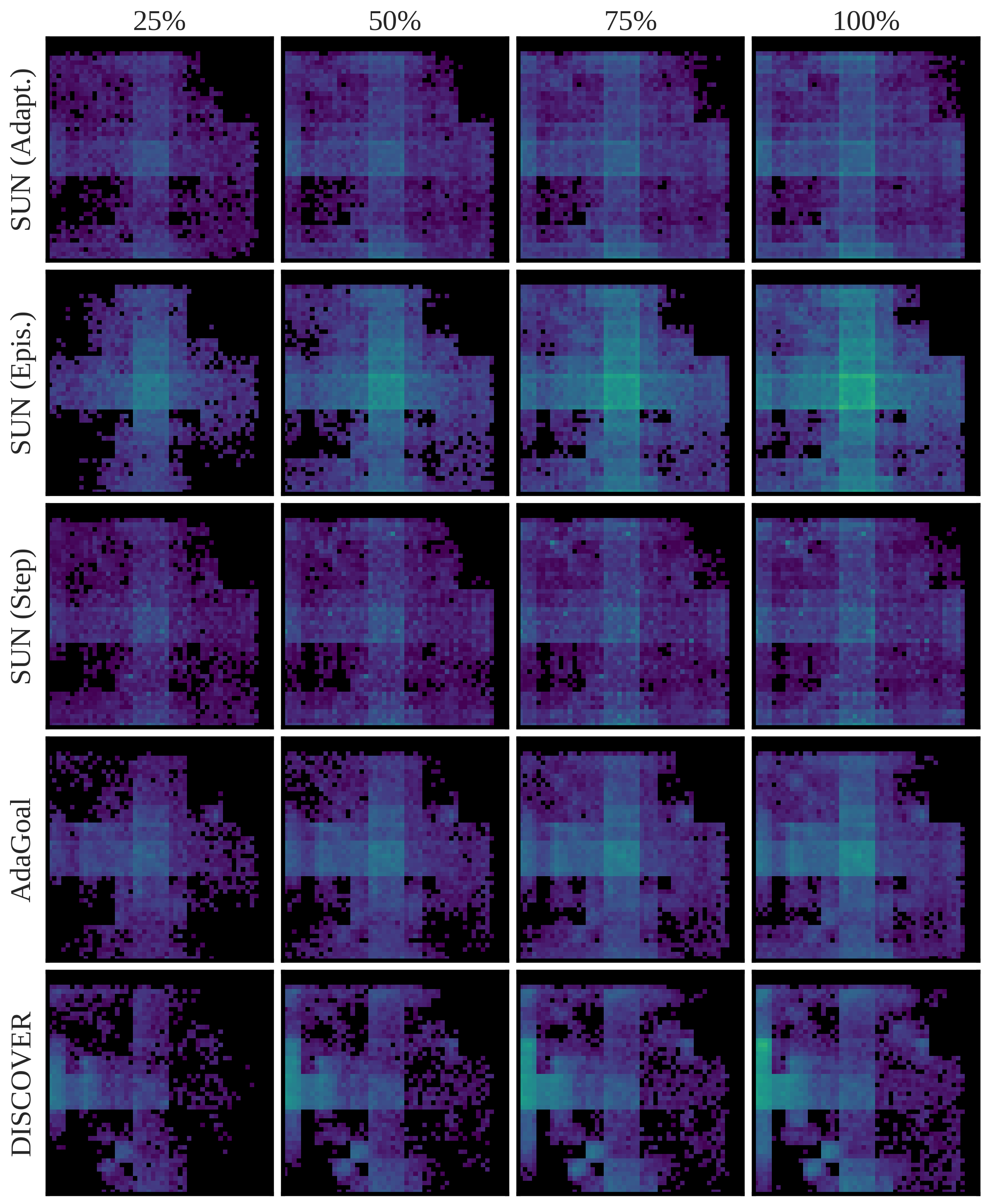}
        \subcaption{{{Goals reached.}}}
    \end{subfigure}
    \hfill
    \begin{subfigure}[b]{0.30\textwidth}
        \includegraphics[width=\linewidth]{plots/goal_stats/legend_horizontal.png}
        \\[-1pt]
        \includegraphics[width=\linewidth]{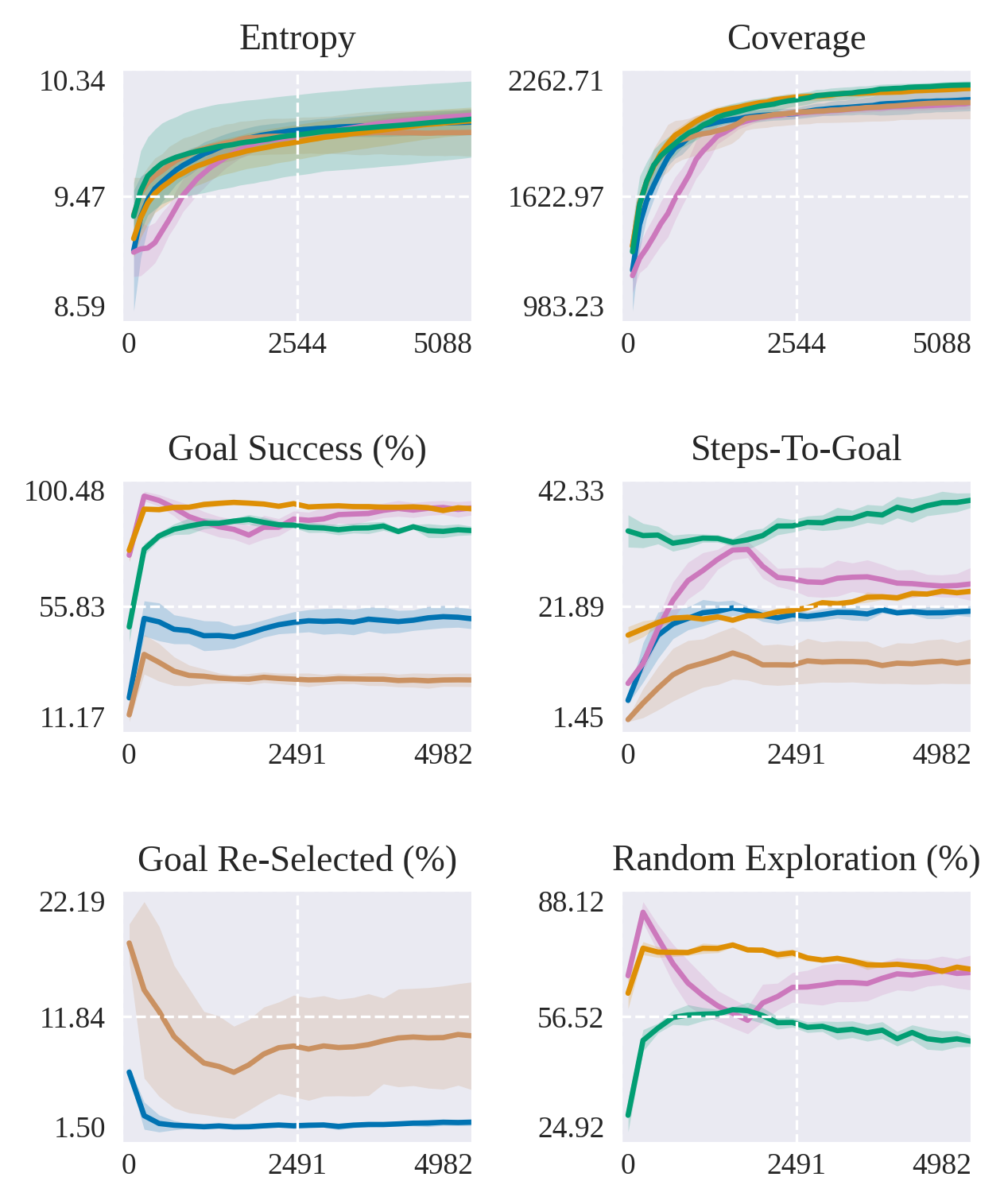}
        \captionsetup{skip=0pt}
        \subcaption{Training curves.}
    \end{subfigure}
    \caption{\label{fig:extra_point_maze_hard}\stt{PointMaze-H}.}
\end{figure}

\begin{figure}[t]
    \centering
    % \\[-1pt]
    \begin{subfigure}[b]{0.31\textwidth}
        \includegraphics[trim=3 0 3 3, clip, width=\linewidth]{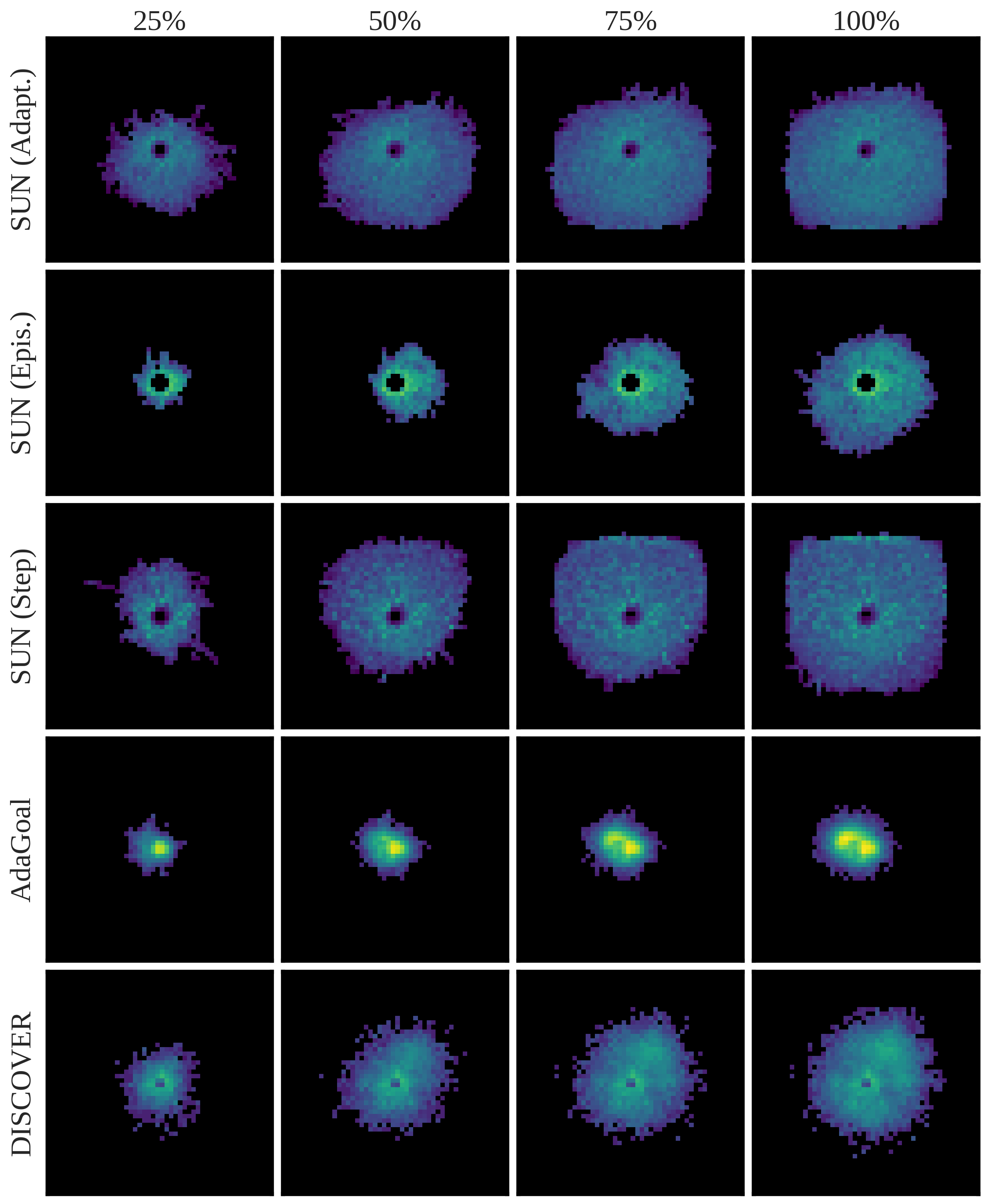}
        \subcaption{{{Goals selected.}}}
        \end{subfigure}
    \hfill
    \begin{subfigure}[b]{0.31\textwidth}
        \includegraphics[trim=3 0 3 3, clip, width=\linewidth]{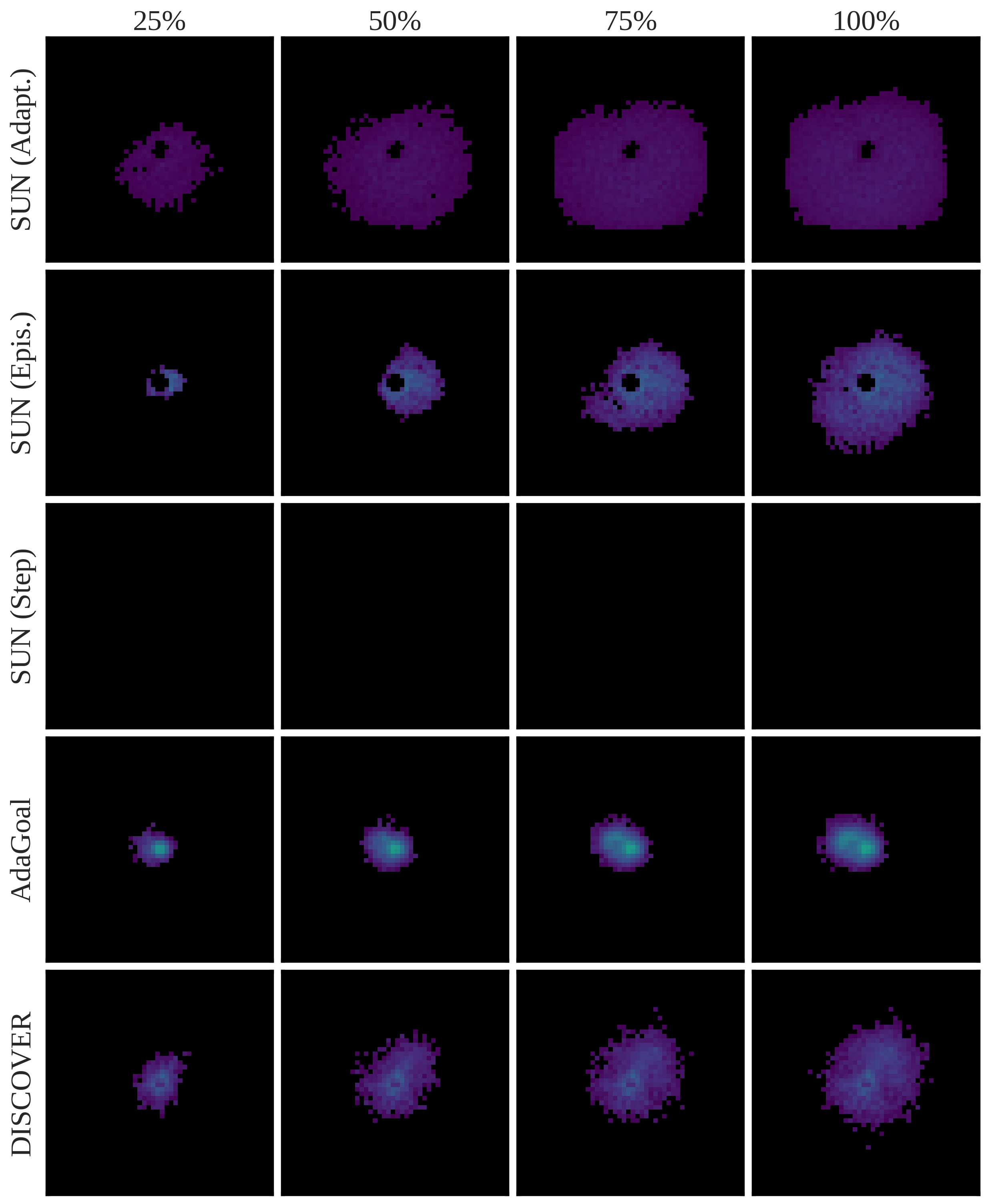}
        \subcaption{{{Goals reached.}}}
    \end{subfigure}
    \hfill
    \begin{subfigure}[b]{0.30\textwidth}
        \includegraphics[width=\linewidth]{plots/goal_stats/legend_horizontal.png}
        \\[-1pt]
        \includegraphics[width=\linewidth]{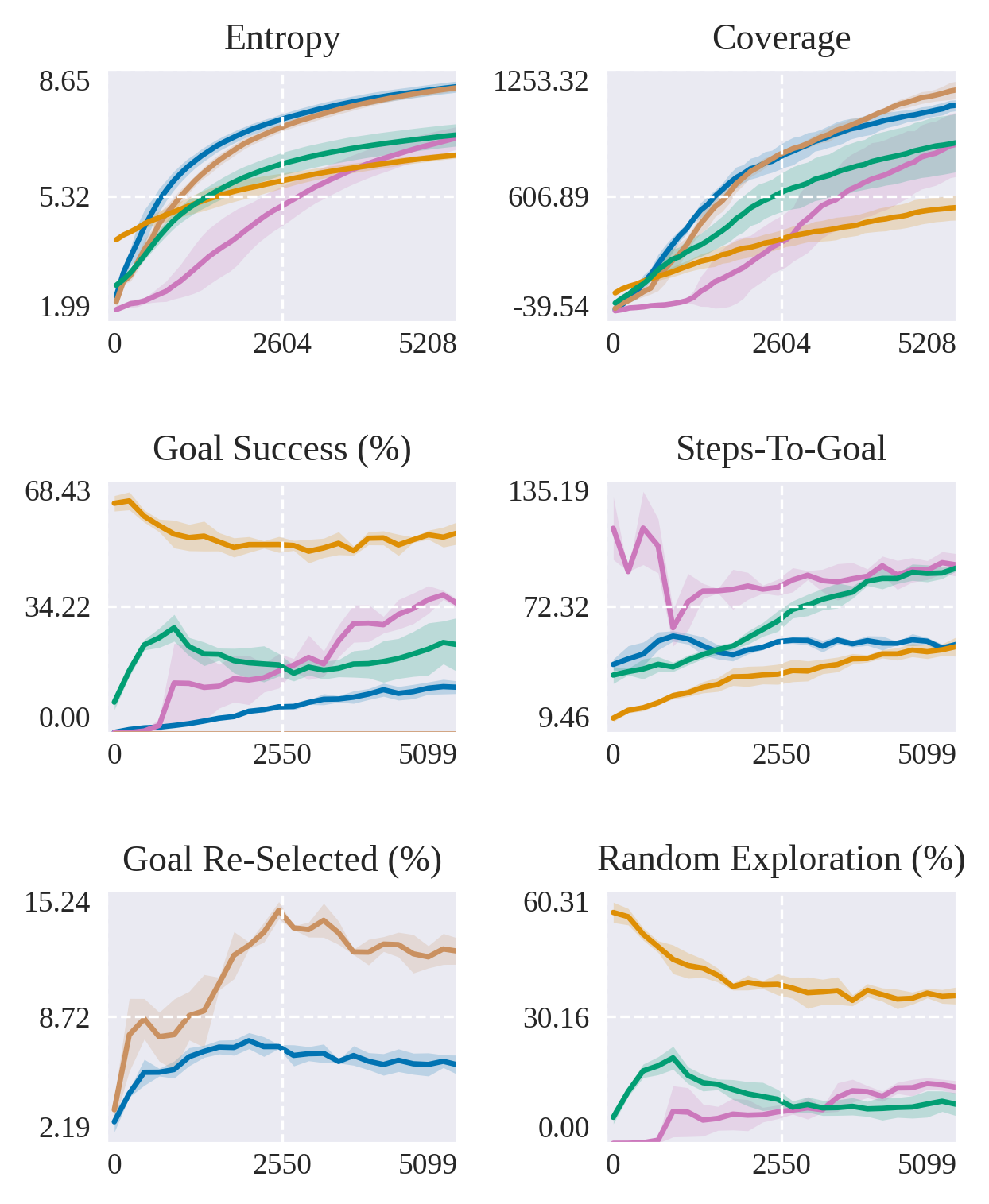}
        \captionsetup{skip=0pt}
        \subcaption{Training curves.}
    \end{subfigure}
    \caption{\label{fig:extra_ant_maze_simple}\stt{AntMaze-S}.}
    \vspace*{2pt}
    \centering
    \begin{subfigure}[b]{0.31\textwidth}
        \includegraphics[trim=3 0 3 3, clip, width=\linewidth]{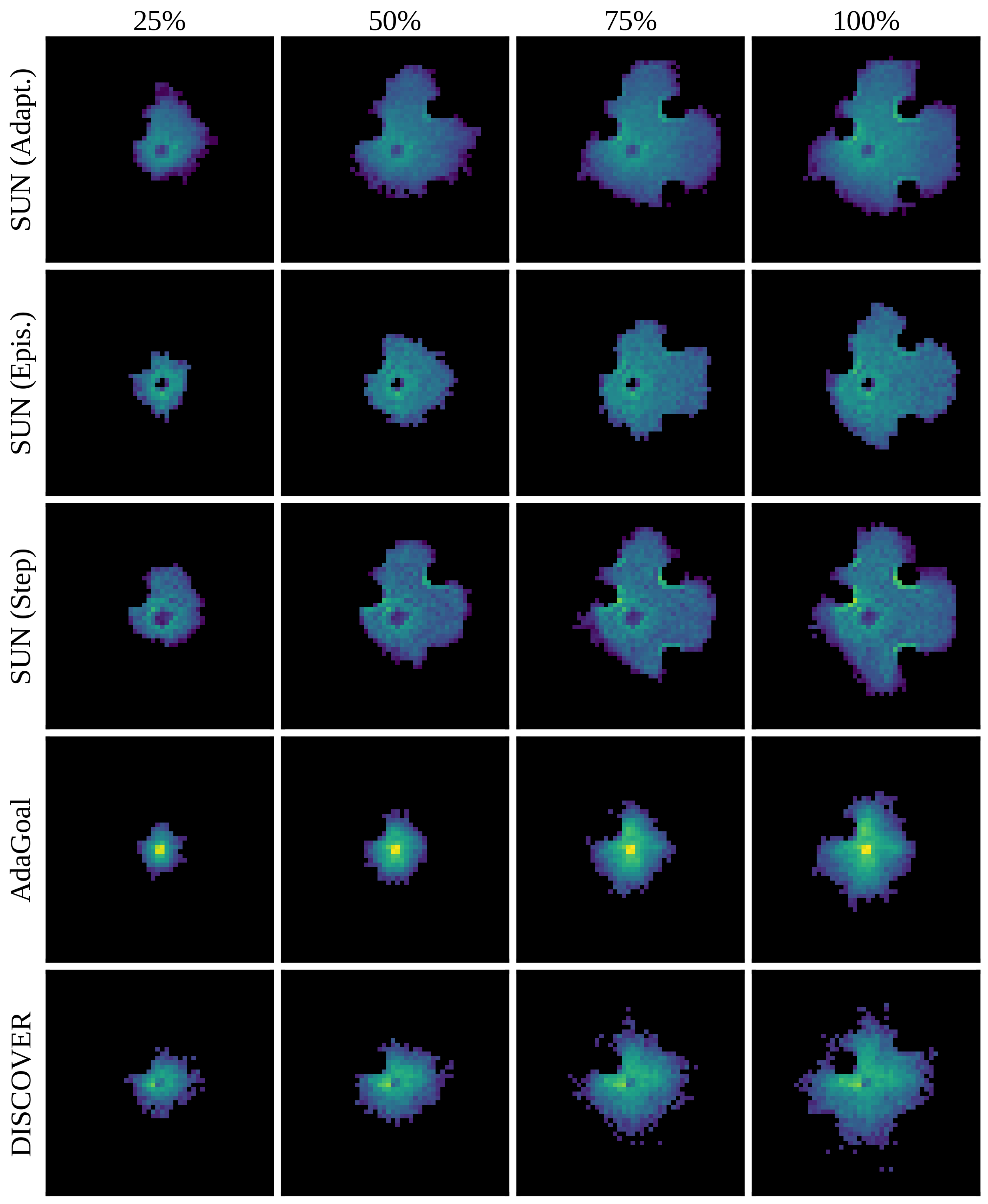}
        \subcaption{{{Goals selected.}}}
        \end{subfigure}
    \hfill
    \begin{subfigure}[b]{0.31\textwidth}
        \includegraphics[trim=3 0 3 3, clip, width=\linewidth]{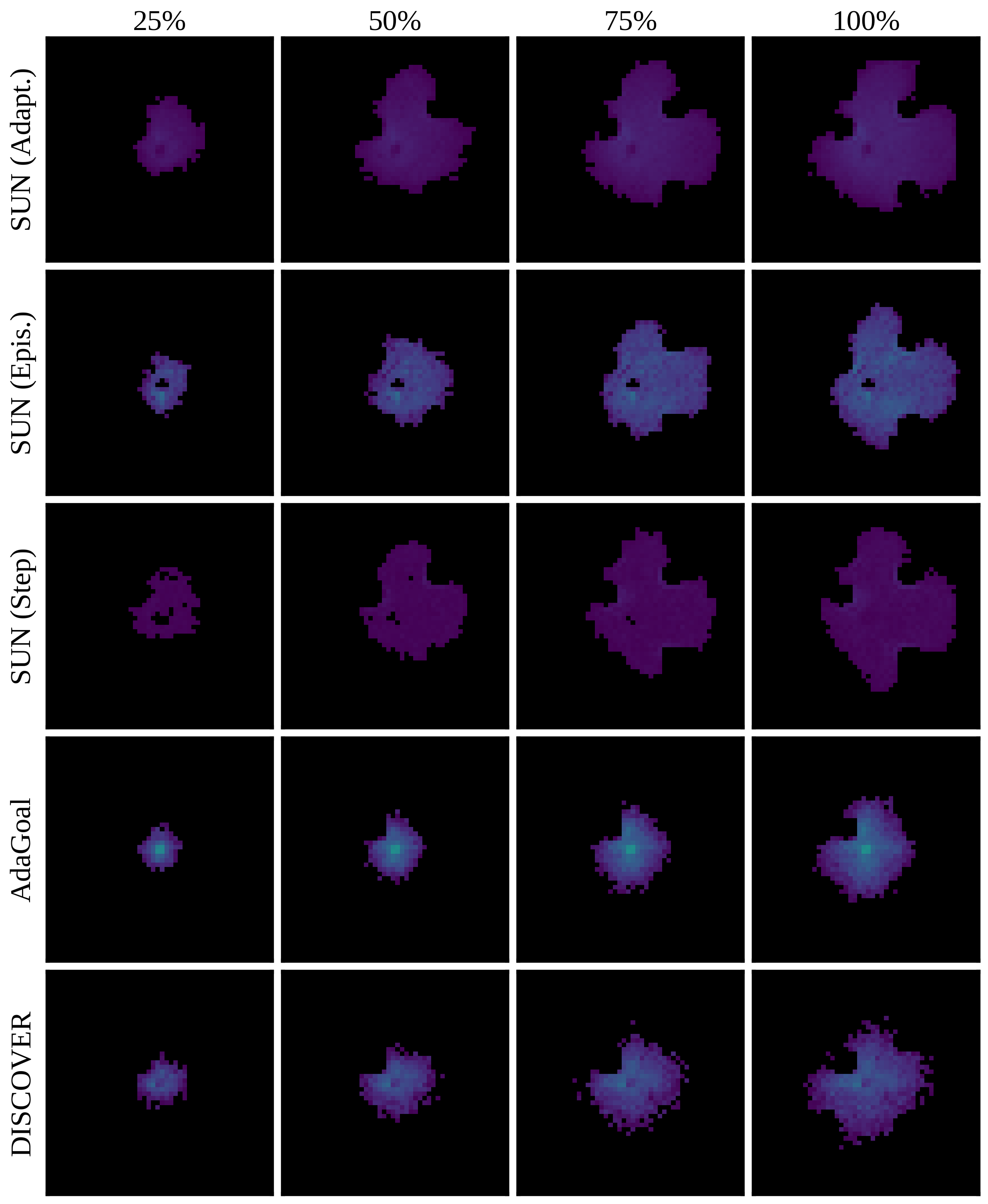}
        \subcaption{{{Goals reached.}}}
    \end{subfigure}
    \hfill
    \begin{subfigure}[b]{0.30\textwidth}
        \includegraphics[width=\linewidth]{plots/goal_stats/legend_horizontal.png}
        \\[-1pt]
        \includegraphics[width=\linewidth]{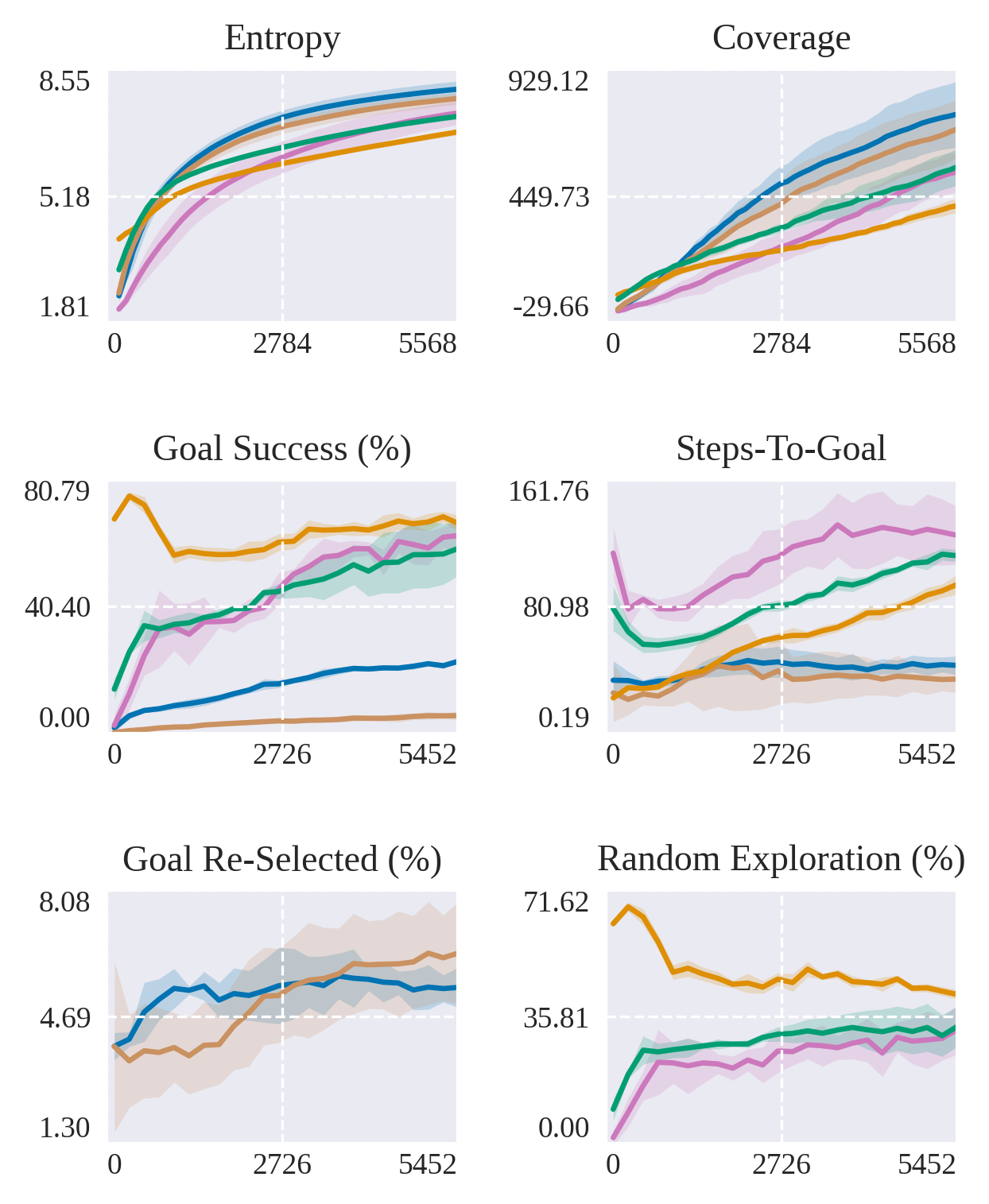}
        \captionsetup{skip=0pt}
        \subcaption{Training curves.}
    \end{subfigure}
    \caption{\label{fig:extra_ant_maze_hard}\stt{AntMaze-H}.}
    \vspace*{2pt}
    \centering
    \begin{subfigure}[b]{0.31\textwidth}
        \includegraphics[trim=3 0 3 3, clip, width=\linewidth]{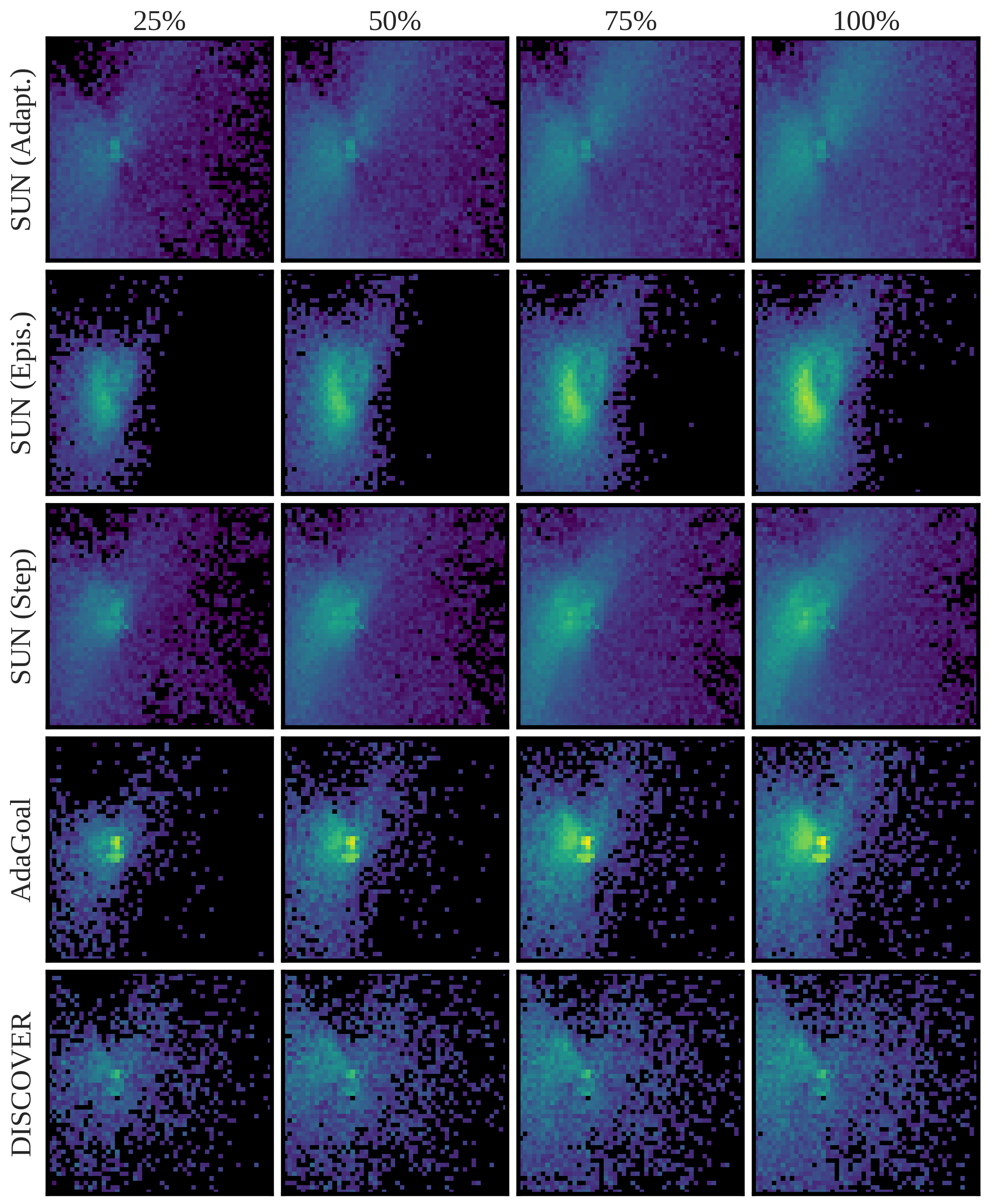}
        \subcaption{{{Goals selected.}}}
        \end{subfigure}
    \hfill
    \begin{subfigure}[b]{0.31\textwidth}
        \includegraphics[trim=3 0 3 3, clip, width=\linewidth]{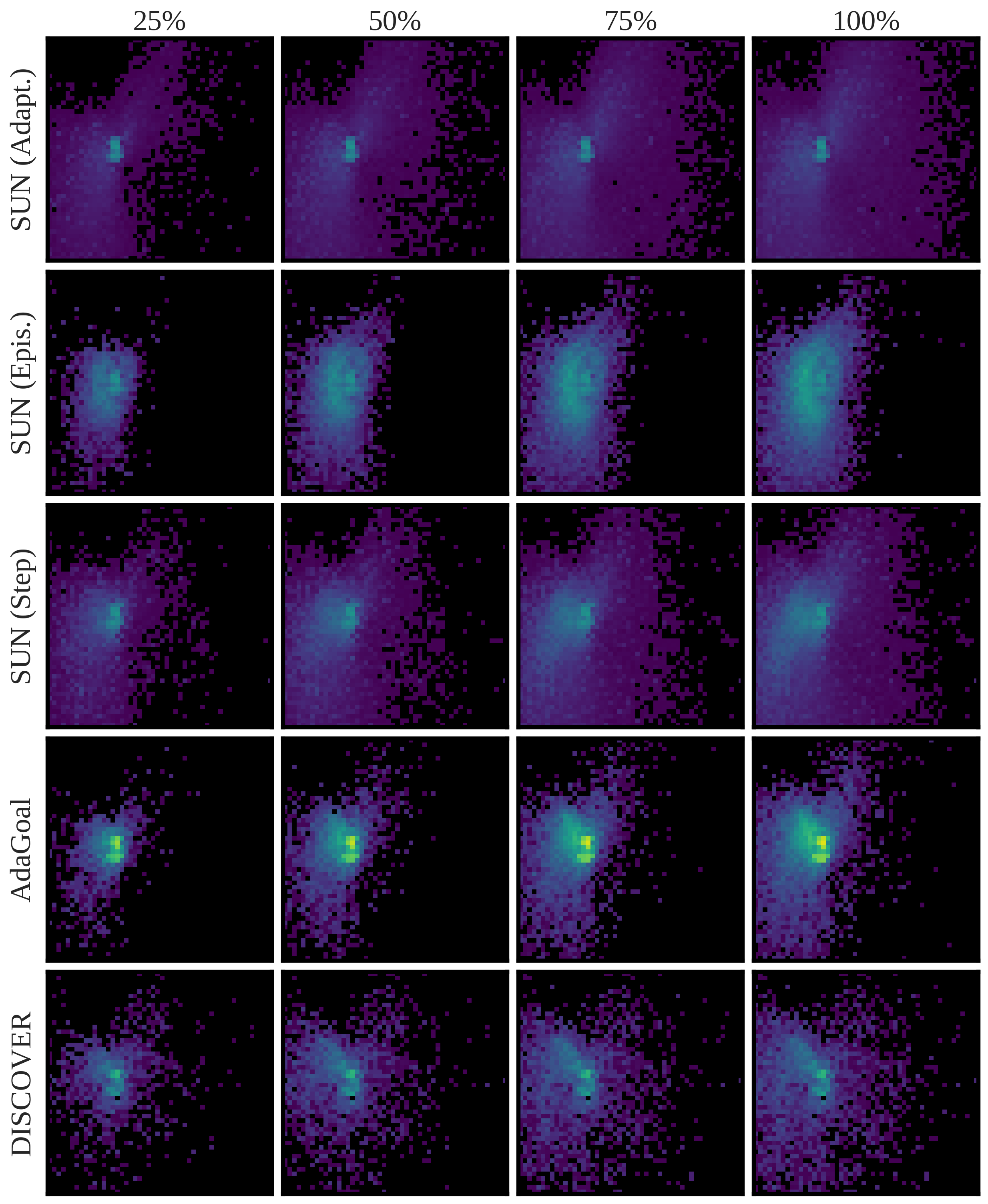}
        \subcaption{{{Goals reached.}}}
    \end{subfigure}
    \hfill
    \begin{subfigure}[b]{0.30\textwidth}
        \includegraphics[width=\linewidth]{plots/goal_stats/legend_horizontal.png}
        \\[-1pt]
        \includegraphics[width=\linewidth]{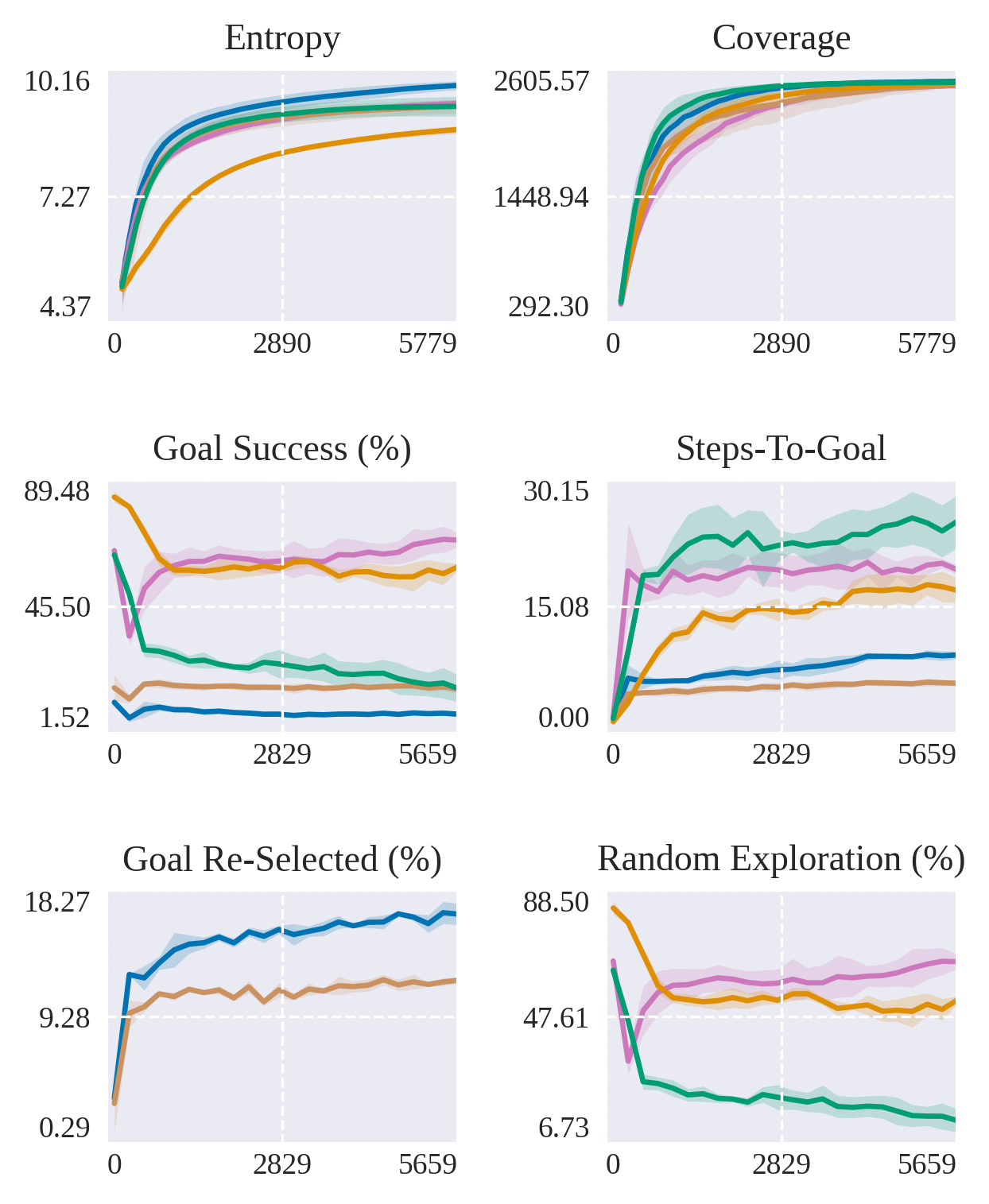}
        \captionsetup{skip=0pt}
        \subcaption{Training curves.}
    \end{subfigure}
    \caption{\label{fig:extra_arm_push_hard}\stt{ArmPush-H}. Reached goals concentrate near the cube's start region (the red/blue area in Figure \ref{fig:jax_environments}). To move it, the agent must first make contact with it.}
\end{figure}

\clearpage

\section{Aggregate Statistics}
\label{app:rliable}

In Section \ref{sec:experiments}, we reported the interquartile mean (IQM) and the probability of improvement (PoI) (Figure~\ref{fig:iqm_curves} and Table~\ref{tab:iqm_poi}).
Here, we complete the picture: Figure~\ref{fig:aggregate_full} adds the median, the mean and the optimality gap to the IQM, and Table~\ref{tab:per_env} reports the per-environment scores the aggregates are computed from.
\\[2pt]
The four aggregates of Figure~\ref{fig:aggregate_full} agree that SUN with adaptive selection and the multiplicative indicator ranks first on both metrics.
They disagree below that, and the disagreement is informative.
The additive version ranks second under the IQM but fifth under the median coverage, because it collapses on some seeds of the two hardest mazes --- its standard deviation exceeds its mean on \stt{AntMaze-H} (Table~\ref{tab:per_env}).
The median and the mean carry those seeds, whereas the IQM trims them away.
For the opposite reason AdaGoal, which is weak but stable, ranks higher under the median than under the IQM.
This is why \citet{agarwal2021deep} recommend the IQM: it reports typical performance without being driven by either tail.
\\[2pt]
Table~\ref{tab:per_env} also shows that the aggregate ordering does not hold environment by environment.
A baseline attains the best score in four of the twenty-six environment-metric pairs: DISCOVER on \stt{ArmPush-H} coverage and \stt{PointMaze-H} entropy, both within one standard error of the best SUN version, and more substantially Random on \stt{Pendulum} entropy and AdaGoal on \stt{CartPole} entropy.
The \stt{Pendulum} case is the behavior discussed in Section~\ref{subsec:main_results}: reaching rare configurations there requires repeatedly traversing common ones, so an agent that visits the tail of the distribution is charged for re-visiting its mode, and a uniformly random policy is not.
Within the SUN versions the ranking is likewise environment-dependent: the additive indicator attains the best coverage on \stt{LunarLander}, per-step selection on \stt{Acrobot}, and episodic selection on \stt{PointMaze-S}, for the reasons discussed in Appendix~\ref{app:ablation_full}.

\begin{figure}[!t]
    \begin{minipage}[t]{0.7\linewidth}\vspace{0pt}
    \includegraphics[width=\linewidth]{\detokenize{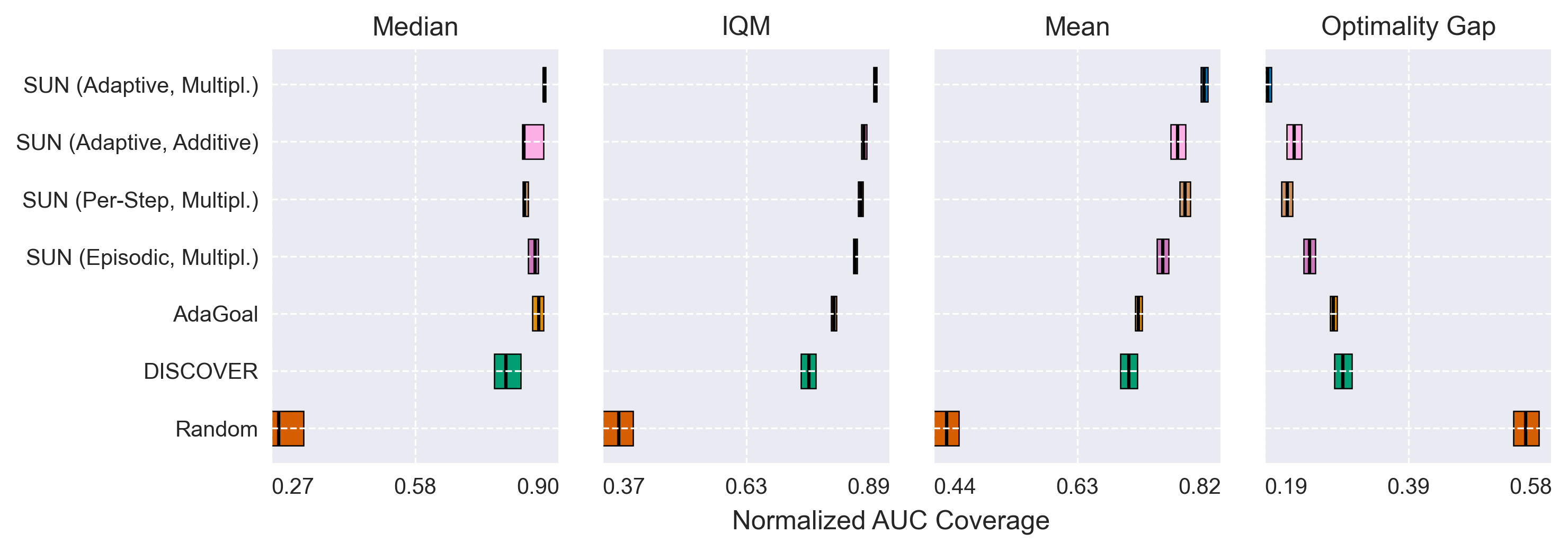}}
    \\[2pt]
    \includegraphics[width=\linewidth]{\detokenize{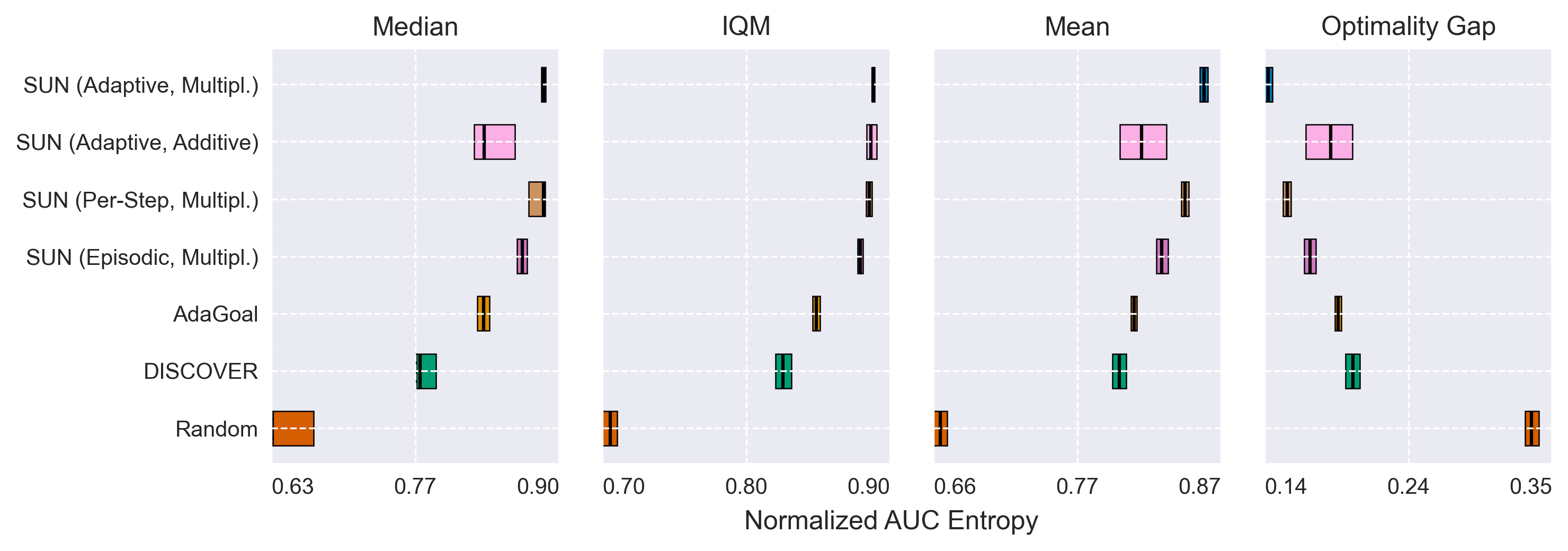}}
    \end{minipage}\hfill
    \begin{minipage}[t]{0.28\linewidth}\vspace{0pt}
    \caption{\label{fig:aggregate_full}\textbf{Aggregate statistics of the normalized AUC} for coverage (top) and entropy (bottom), over all runs of all environments but \stt{LunarLander(Full)}. Bars are 95\% percentile intervals from a stratified bootstrap (50{,}000 resamples of the seeds within each environment). For the optimality gap, lower is better. \textit{Adaptive multiplicative SUN ranks first under all four aggregates on both metrics.}}
    \end{minipage}
\end{figure}

\begin{table}[!b]
    \centering
    \caption{\label{tab:per_env}\textbf{Per-environment normalized AUC}, mean $\pm$ standard deviation over the seeds. These are the scores the aggregates of Figure~\ref{fig:aggregate_full} and Table~\ref{tab:iqm_poi} are computed from.}
    \resizebox{\linewidth}{!}{%
    \begin{tabular}{lccccccccccccc}
    \toprule
     & \rotatebox{60}{ThreeRoom} & \rotatebox{60}{FourRoomStuck} & \rotatebox{60}{GridMaze} & \rotatebox{60}{MountainCar} & \rotatebox{60}{Pendulum} & \rotatebox{60}{LunarLander} & \rotatebox{60}{Acrobot} & \rotatebox{60}{CartPole} & \rotatebox{60}{PointMaze-S} & \rotatebox{60}{PointMaze-H} & \rotatebox{60}{AntMaze-S} & \rotatebox{60}{AntMaze-H} & \rotatebox{60}{ArmPush-H} \\
    \midrule
    \multicolumn{14}{l}{\textbf{Coverage}} \\
    SUN (Adaptive, Multipl.) & 0.996 $\pm$ 0.002 & 0.997 $\pm$ 0.001 & 0.998 $\pm$ 0.001 & 0.645 $\pm$ 0.007 & 0.827 $\pm$ 0.003 & 0.624 $\pm$ 0.021 & 0.737 $\pm$ 0.004 & 0.870 $\pm$ 0.004 & 0.975 $\pm$ 0.019 & 0.939 $\pm$ 0.006 & 0.570 $\pm$ 0.047 & 0.323 $\pm$ 0.055 & 0.924 $\pm$ 0.043 \\
    SUN (Adaptive, Additive) & 0.996 $\pm$ 0.001 & 0.997 $\pm$ 0.001 & 0.998 $\pm$ 0.001 & 0.646 $\pm$ 0.007 & 0.825 $\pm$ 0.004 & 0.646 $\pm$ 0.017 & 0.737 $\pm$ 0.004 & 0.869 $\pm$ 0.005 & 0.906 $\pm$ 0.070 & 0.774 $\pm$ 0.208 & 0.571 $\pm$ 0.103 & 0.091 $\pm$ 0.249 & 0.905 $\pm$ 0.089 \\
    SUN (Per-Step, Multipl.) & 0.997 $\pm$ 0.002 & 0.996 $\pm$ 0.002 & 0.998 $\pm$ 0.001 & 0.619 $\pm$ 0.018 & 0.824 $\pm$ 0.003 & 0.564 $\pm$ 0.027 & 0.744 $\pm$ 0.003 & 0.826 $\pm$ 0.011 & 0.860 $\pm$ 0.127 & 0.906 $\pm$ 0.090 & 0.566 $\pm$ 0.051 & 0.304 $\pm$ 0.151 & 0.886 $\pm$ 0.101 \\
    SUN (Episodic, Multipl.) & 0.996 $\pm$ 0.001 & 0.995 $\pm$ 0.002 & 0.998 $\pm$ 0.001 & 0.547 $\pm$ 0.021 & 0.819 $\pm$ 0.004 & 0.581 $\pm$ 0.028 & 0.667 $\pm$ 0.011 & 0.849 $\pm$ 0.011 & 0.986 $\pm$ 0.009 & 0.879 $\pm$ 0.043 & 0.326 $\pm$ 0.236 & 0.203 $\pm$ 0.104 & 0.852 $\pm$ 0.039 \\
    AdaGoal & 0.994 $\pm$ 0.002 & 0.994 $\pm$ 0.002 & 0.996 $\pm$ 0.002 & 0.481 $\pm$ 0.027 & 0.770 $\pm$ 0.006 & 0.413 $\pm$ 0.027 & 0.561 $\pm$ 0.017 & 0.858 $\pm$ 0.015 & 0.923 $\pm$ 0.035 & 0.935 $\pm$ 0.012 & 0.290 $\pm$ 0.051 & 0.186 $\pm$ 0.010 & 0.879 $\pm$ 0.027 \\
    DISCOVER & 0.994 $\pm$ 0.002 & 0.995 $\pm$ 0.001 & 0.996 $\pm$ 0.001 & 0.366 $\pm$ 0.020 & 0.763 $\pm$ 0.005 & 0.389 $\pm$ 0.025 & 0.438 $\pm$ 0.014 & 0.784 $\pm$ 0.046 & 0.831 $\pm$ 0.190 & 0.934 $\pm$ 0.028 & 0.444 $\pm$ 0.110 & 0.247 $\pm$ 0.033 & 0.932 $\pm$ 0.027 \\
    Random & 0.996 $\pm$ 0.001 & 0.992 $\pm$ 0.003 & 0.997 $\pm$ 0.001 & 0.129 $\pm$ 0.006 & 0.739 $\pm$ 0.004 & 0.186 $\pm$ 0.001 & 0.408 $\pm$ 0.010 & 0.271 $\pm$ 0.010 & 0.282 $\pm$ 0.070 & 0.267 $\pm$ 0.044 & 0.080 $\pm$ 0.003 & 0.047 $\pm$ 0.004 & 0.541 $\pm$ 0.254 \\
    \midrule
    \multicolumn{14}{l}{\textbf{Entropy}} \\
    SUN (Adaptive, Multipl.) & 0.944 $\pm$ 0.003 & 0.977 $\pm$ 0.001 & 0.988 $\pm$ 0.001 & 0.856 $\pm$ 0.003 & 0.934 $\pm$ 0.001 & 0.795 $\pm$ 0.005 & 0.885 $\pm$ 0.002 & 0.821 $\pm$ 0.004 & 0.954 $\pm$ 0.006 & 0.887 $\pm$ 0.002 & 0.671 $\pm$ 0.036 & 0.652 $\pm$ 0.021 & 0.758 $\pm$ 0.036 \\
    SUN (Adaptive, Additive) & 0.943 $\pm$ 0.002 & 0.974 $\pm$ 0.001 & 0.985 $\pm$ 0.001 & 0.857 $\pm$ 0.003 & 0.938 $\pm$ 0.003 & 0.798 $\pm$ 0.006 & 0.886 $\pm$ 0.003 & 0.823 $\pm$ 0.003 & 0.830 $\pm$ 0.269 & 0.795 $\pm$ 0.158 & 0.627 $\pm$ 0.061 & 0.339 $\pm$ 0.500 & 0.753 $\pm$ 0.059 \\
    SUN (Per-Step, Multipl.) & 0.945 $\pm$ 0.003 & 0.975 $\pm$ 0.002 & 0.987 $\pm$ 0.002 & 0.843 $\pm$ 0.009 & 0.933 $\pm$ 0.001 & 0.783 $\pm$ 0.007 & 0.885 $\pm$ 0.002 & 0.796 $\pm$ 0.004 & 0.902 $\pm$ 0.038 & 0.886 $\pm$ 0.033 & 0.647 $\pm$ 0.032 & 0.641 $\pm$ 0.036 & 0.726 $\pm$ 0.058 \\
    SUN (Episodic, Multipl.) & 0.940 $\pm$ 0.002 & 0.969 $\pm$ 0.002 & 0.984 $\pm$ 0.002 & 0.799 $\pm$ 0.004 & 0.944 $\pm$ 0.003 & 0.787 $\pm$ 0.006 & 0.865 $\pm$ 0.006 & 0.815 $\pm$ 0.006 & 0.970 $\pm$ 0.011 & 0.881 $\pm$ 0.015 & 0.475 $\pm$ 0.101 & 0.565 $\pm$ 0.097 & 0.739 $\pm$ 0.025 \\
    AdaGoal & 0.901 $\pm$ 0.002 & 0.964 $\pm$ 0.002 & 0.973 $\pm$ 0.002 & 0.740 $\pm$ 0.009 & 0.942 $\pm$ 0.003 & 0.731 $\pm$ 0.008 & 0.815 $\pm$ 0.009 & 0.829 $\pm$ 0.007 & 0.910 $\pm$ 0.022 & 0.883 $\pm$ 0.017 & 0.545 $\pm$ 0.023 & 0.571 $\pm$ 0.009 & 0.672 $\pm$ 0.009 \\
    DISCOVER & 0.920 $\pm$ 0.003 & 0.952 $\pm$ 0.001 & 0.954 $\pm$ 0.002 & 0.672 $\pm$ 0.005 & 0.941 $\pm$ 0.002 & 0.725 $\pm$ 0.007 & 0.770 $\pm$ 0.004 & 0.771 $\pm$ 0.018 & 0.841 $\pm$ 0.075 & 0.888 $\pm$ 0.040 & 0.564 $\pm$ 0.038 & 0.607 $\pm$ 0.008 & 0.736 $\pm$ 0.035 \\
    Random & 0.942 $\pm$ 0.001 & 0.915 $\pm$ 0.002 & 0.910 $\pm$ 0.002 & 0.614 $\pm$ 0.004 & 0.947 $\pm$ 0.002 & 0.705 $\pm$ 0.001 & 0.760 $\pm$ 0.005 & 0.633 $\pm$ 0.001 & 0.471 $\pm$ 0.042 & 0.546 $\pm$ 0.013 & 0.382 $\pm$ 0.005 & 0.238 $\pm$ 0.006 & 0.624 $\pm$ 0.063 \\
    \bottomrule
    \end{tabular}}
\end{table}

\clearpage

\begin{figure}[t]
    \centering
    % ---- Row 1 ----
    \begin{minipage}[t]{0.48\textwidth}
        \begin{tikzpicture}[
            spy using outlines={
                magnification=5,
                size=3.2cm,
                connect spies,
                red,
                every spy on node/.append style={
                    draw=red,
                    ultra thick
                },
                every spy in node/.append style={
                    draw=red,
                    ultra thick
                },
                spy connection path={
                    \draw[red, ultra thick] (tikzspyonnode) -- (tikzspyinnode);
                }
            },
        ]
            \node[inner sep=0] (mcimg) {
                \includegraphics[width=0.48\linewidth]{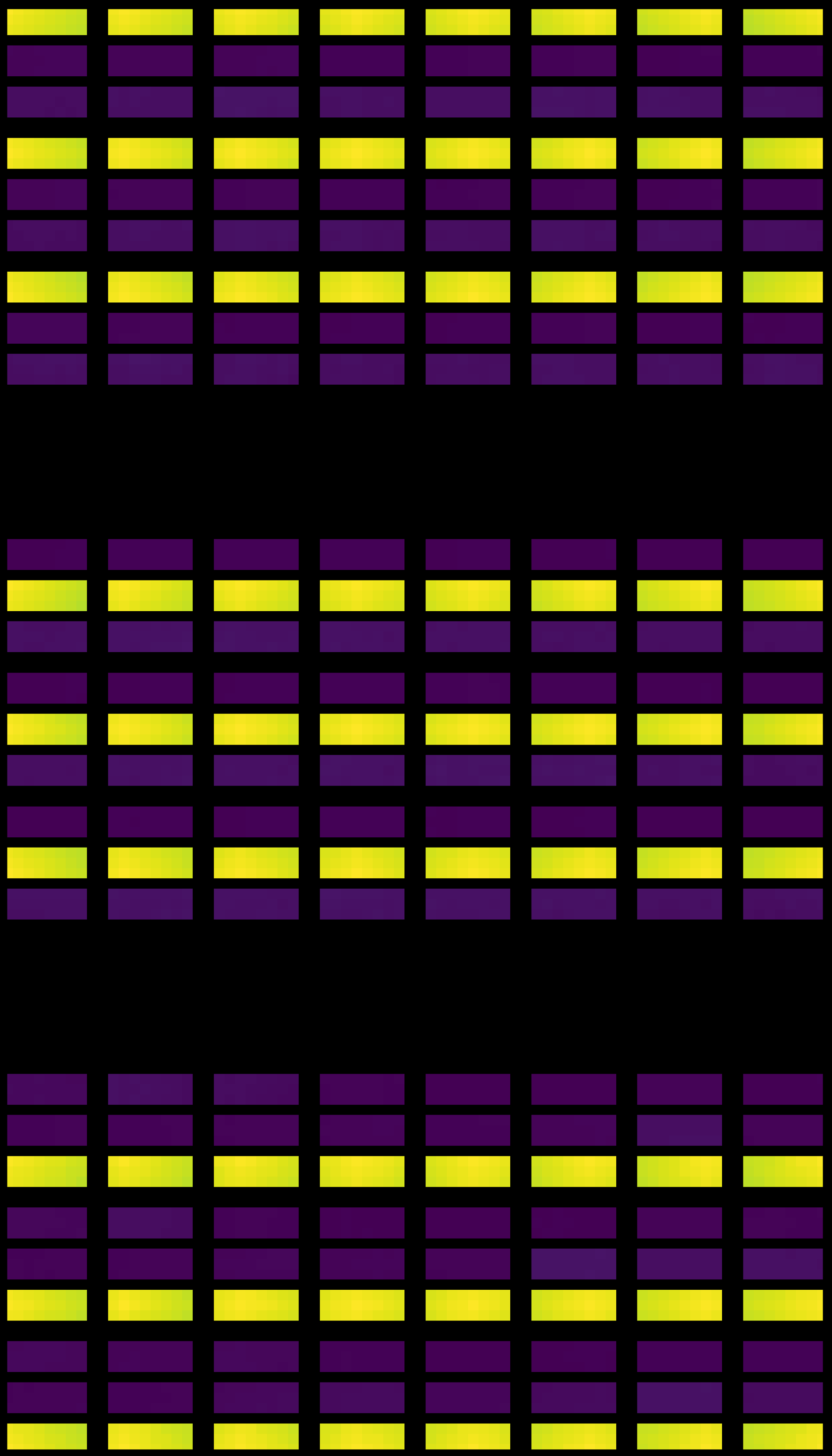}
            };
            \coordinate (mcx) at ($(mcimg.south west)!0.15!(mcimg.south east)$);
            \coordinate (mcy) at ($(mcimg.south west)!0.125!(mcimg.north west)$);
            \coordinate (mczoom) at (mcx |- mcy);
            \spy[overlay] on (mczoom) in node at ($(mcimg.south)+(3.5cm,1.6cm)$);
        \end{tikzpicture}
        \subcaption{\stt{ThreeRoom}}
    \end{minipage}
    \hfill
    \begin{minipage}[t]{0.48\textwidth}
        \centering
        \begin{tikzpicture}[
            spy using outlines={
                magnification=5,
                size=3.2cm,
                connect spies,
                red,
                every spy on node/.append style={
                    draw=red,
                    ultra thick
                },
                every spy in node/.append style={
                    draw=red,
                    ultra thick
                },
                spy connection path={
                    \draw[red, ultra thick] (tikzspyonnode) -- (tikzspyinnode);
                }
            },
        ]
            \node[inner sep=0] (mcimg) {
                \includegraphics[width=\linewidth]{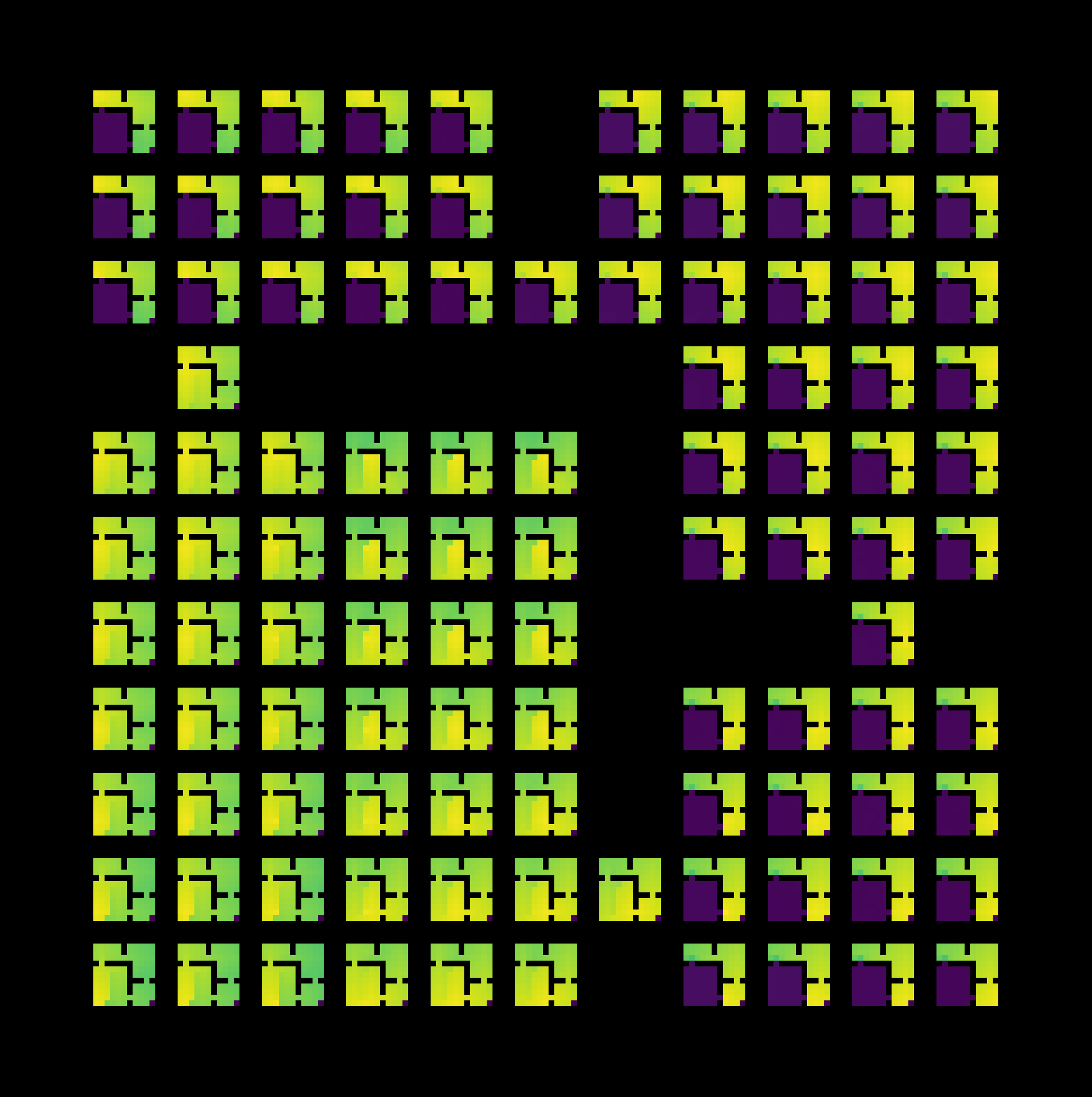}
            };
            \coordinate (mcx) at ($(mcimg.south west)!0.50!(mcimg.south east)$);
            \coordinate (mcy) at ($(mcimg.south west)!0.50!(mcimg.north west)$);
            \coordinate (mczoom) at (mcx |- mcy);
            \spy[overlay] on (mczoom) in node at ($(mcimg.south)+(-5.2cm,5.05cm)$);
        \end{tikzpicture}
        \subcaption{\stt{FourRoomStuck}}
    \end{minipage}
    \hfill
    \\[2pt]
    % ---- Row 2 ----
    \hfill
    \begin{minipage}[t]{0.48\textwidth}
        \centering
        \begin{tikzpicture}[
            spy using outlines={
                magnification=5,
                size=3.2cm,
                connect spies,
                red,
                every spy on node/.append style={
                    draw=red,
                    ultra thick
                },
                every spy in node/.append style={
                    draw=red,
                    ultra thick
                },
                spy connection path={
                    \draw[red, ultra thick] (tikzspyonnode) -- (tikzspyinnode);
                }
            },
            ]
            \node[inner sep=0] (penimg) {\includegraphics[width=\linewidth]{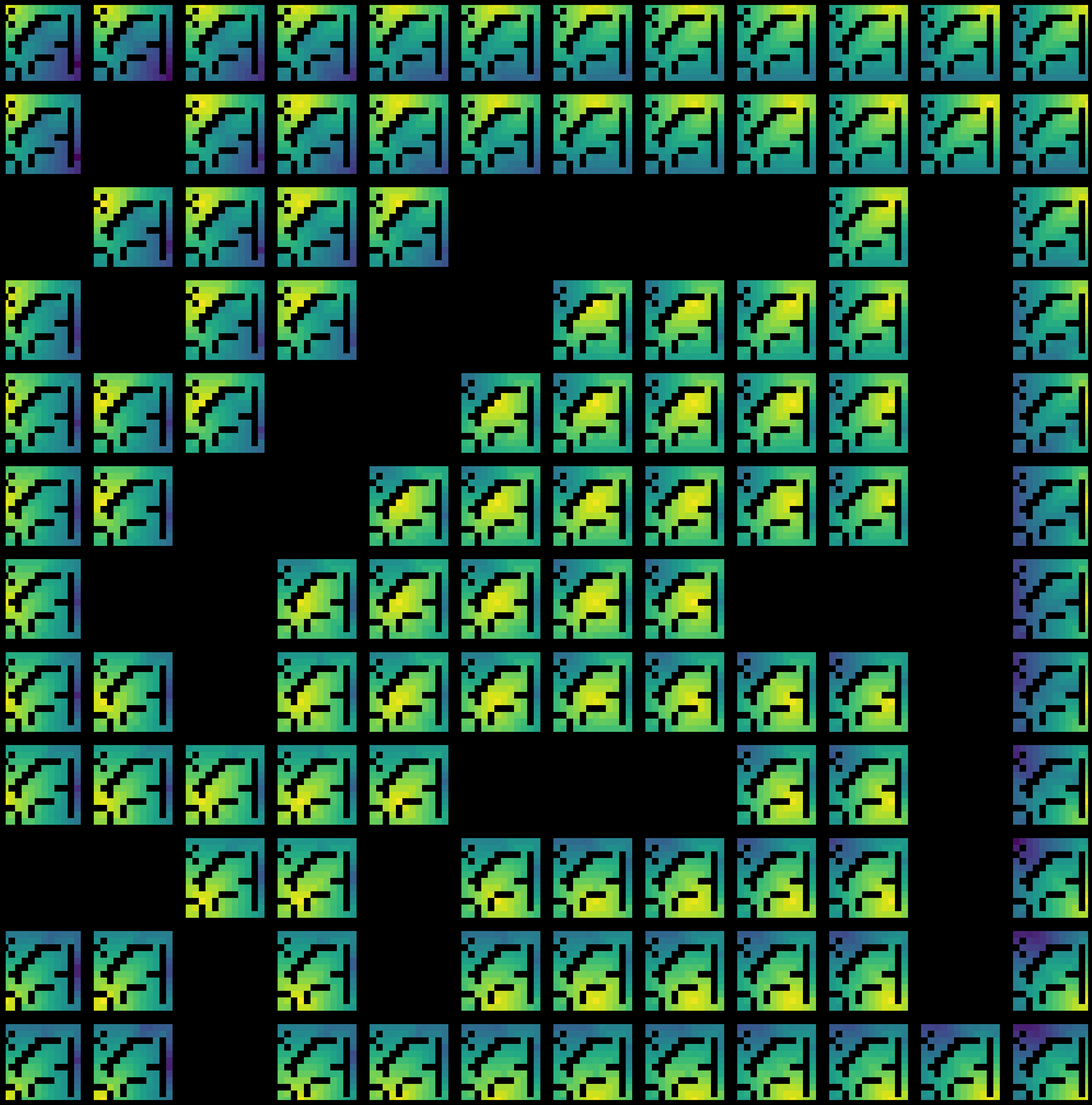}};
            \coordinate (penx) at ($(penimg.south west)!0.95!(penimg.south east)$);  % middle x
            \coordinate (peny) at ($(penimg.south west)!0.05!(penimg.north west)$);  % near bottom
            \coordinate (penzoom) at (penx |- peny);
            \spy[overlay] on (penzoom) in node at ($(penimg.east)+(1.8cm,-5.7cm)$);
        \end{tikzpicture}
    \end{minipage}
    \hfill
    \begin{minipage}[t]{0.48\textwidth}
        \centering
        \begin{tikzpicture}[
            spy using outlines={
                magnification=5,
                size=3.2cm,
                connect spies,
                red,
                every spy on node/.append style={
                    draw=red,
                    ultra thick
                },
                every spy in node/.append style={
                    draw=red,
                    ultra thick
                },
                spy connection path={
                    \draw[red, ultra thick] (tikzspyonnode) -- (tikzspyinnode);
                }
            },
            ]
            \node[inner sep=0] (mcimg) {\includegraphics[width=\linewidth]{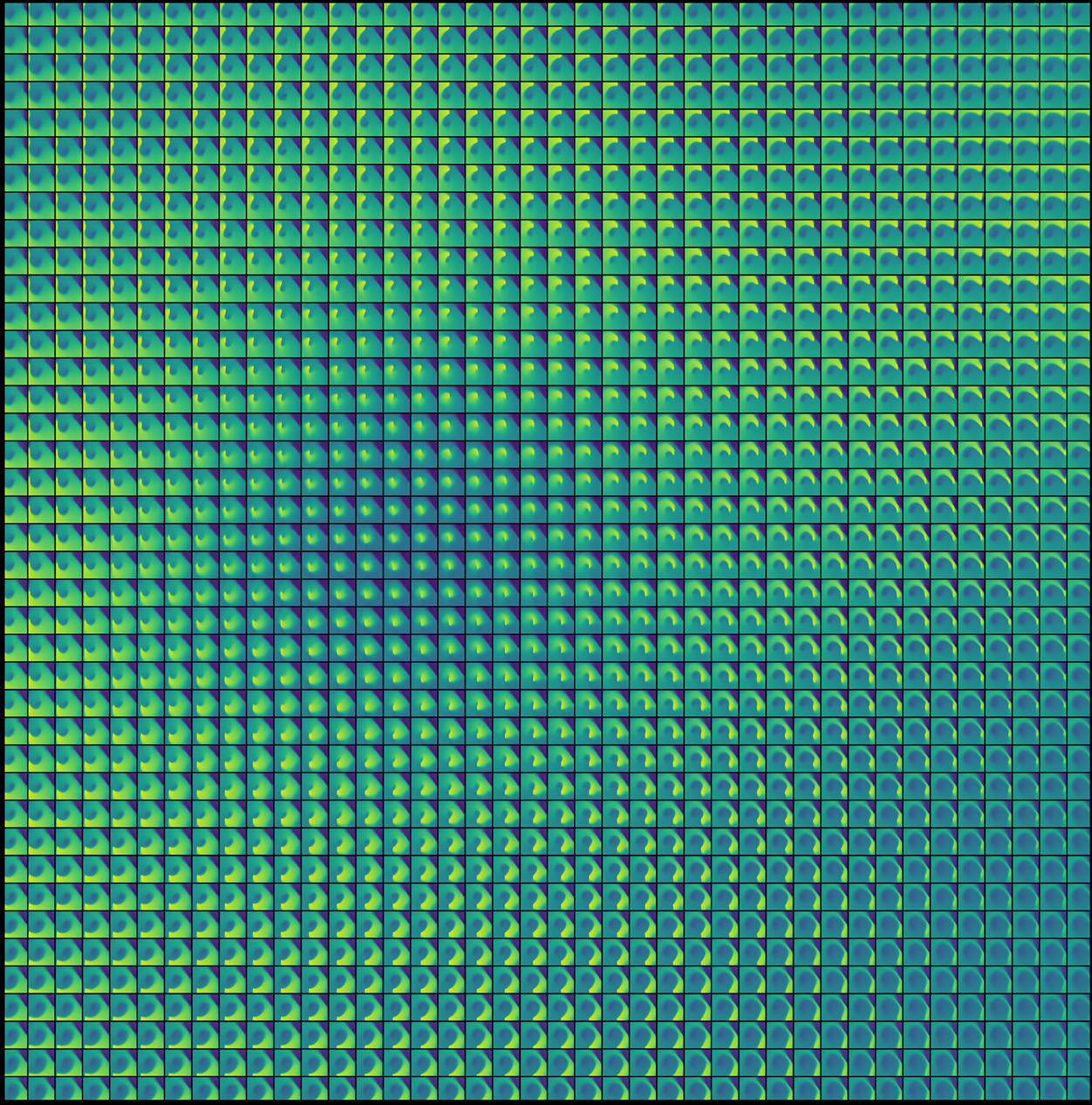}};
            \coordinate (mcx) at ($(mcimg.south west)!0.50!(mcimg.south east)$);  % middle x
            \coordinate (mcy) at ($(mcimg.south west)!0.50!(mcimg.north west)$);  % middle y
            \coordinate (mczoom) at (mcx |- mcy);
            \spy[overlay] on (mczoom) in node at ($(mcimg.south)+(1.6cm,-2.3cm)$);
        \end{tikzpicture}
        \subcaption{\stt{MountainCar}}
    \end{minipage}
        \hfill
    \\[2pt]
    % ---- Row 3 ----
    \hfill
    \begin{minipage}[t]{0.48\textwidth}
        \centering
        \begin{tikzpicture}[
            spy using outlines={
                magnification=5,
                size=3.2cm,
                connect spies,
                red,
                every spy on node/.append style={
                    draw=red,
                    ultra thick
                },
                every spy in node/.append style={
                    draw=red,
                    ultra thick
                },
                spy connection path={
                    \draw[red, ultra thick] (tikzspyonnode) -- (tikzspyinnode);
                }
            },
            ]
            \node[inner sep=0] (penimg) {\includegraphics[width=\linewidth]{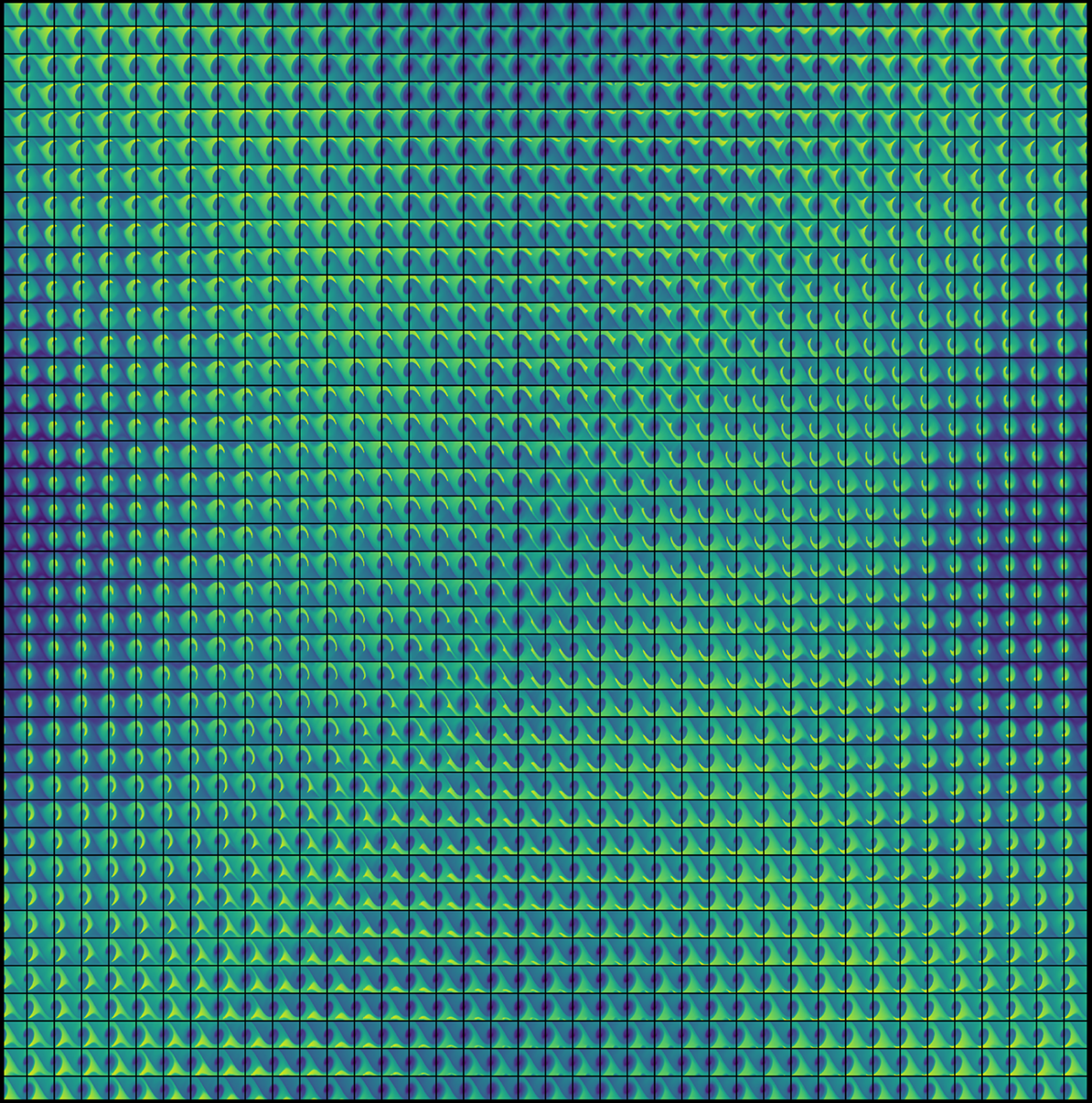}};
            \coordinate (penx) at ($(penimg.south west)!0.50!(penimg.south east)$);  % middle x
            \coordinate (peny) at ($(penimg.south west)!0.155!(penimg.north west)$);  % near bottom
            \coordinate (penzoom) at (penx |- peny);
            \spy[overlay] on (penzoom) in node at ($(penimg.east)+(1.9cm,-1.8cm)$);
        \end{tikzpicture}
        \subcaption{\stt{Pendulum}}
    \end{minipage}
    \hfill
    \begin{minipage}[t]{0.48\textwidth}
        \phantom{\rule{0pt}{4.5cm}}
    \end{minipage}
    \hfill
    \caption{\label{fig:value_function_heatmaps}\textbf{SVFs $\Vtheta$} learned by SUN. Each figure has many sub-heatmaps, one per goal state (see magnified regions). In \stt{MountainCar} and \stt{Pendulum}, axes are position (x) and velocity (y).}
\end{figure}

\clearpage

\section{Successor Value Function Visualization}
\label{app:svf_plots}
In Figure \ref{fig:value_function_heatmaps}, we visualize $\Vtheta(s, g_s) = \max_{a, g_a} \Qtheta(s, a, g_s, g_a)$ learned by SUN, showing its accuracy. Visualization is possible only for environments with two-dimensional state and goal spaces, and for \stt{Pendulum}, whose sine/cosine state can be transformed into an angle.
\\[2pt]
Each heatmap is composed of many sub-heatmaps, one per goal state. For example, the magnified region of \stt{FourRoomStuck} shows $\Vtheta(s, g_\textrm{middle})$, the value of reaching the middle tile from every other tile. The pattern is clear: states near the goal have higher value. Zooming in other sub-heatmaps, one can see that tiles inside the bottom-left room have almost-zero value for goals outside it, reflecting the room's irreversible transitions. The bottom-right tile also has almost-zero value in all heatmaps, as it is a terminal state. 
Similar patterns appear across all heatmaps.

\section{SUN-UCB: A Conceptual Tabular Analysis of Unreachability Filtering and the Scale of the Additive Bonus}
\label{app:pac}

\textbf{Scope.} This appendix is a conceptual analysis and supplies no learning guarantee for the SUN algorithm of Section~\ref{sec:method}. It studies a UCB-like variant of SUN, with additive reachability and novelty, in the tabular setting of \citet{tarbouriech2022adaptive}. The variant differs from the implemented method in three respects: it is tabular; its score is additive, with the uncertainty bonus of Eq.~\eqref{eq:sun-ucb-bonus} rather than the count-based bonus evaluated in Section~\ref{subsec:ablation_results} (see Appendix~\ref{app:sun_ucb_exp}); and, as stated, it takes no environment step before an initial expansion step is added (see \textit{Discovery} below). The structural results in Appendix~\ref{app:theory} characterize SUN's selection rule under an oracle but do not provide sample-complexity guarantees, and neither does this appendix. Its purpose is narrower: we state the variant, prove its unreachability filter sound, and show that its selection rule requires the two terms to be rescaled against each other (Remark~\ref{rem:beta}) --- the tabular counterpart of the design argument in Section~\ref{subsec:theory}. The two components a sample-complexity result would additionally need (a discovery guarantee for the filter and a policy-selection rule) are identified at the end of the analysis.

\subsection{Setup}

We consider tabular MDPs $\langle \statespace, \actionspace, P, s_0 \rangle$ with finite spaces $|\statespace| = S$ and $|\actionspace| = A$, and stochastic transitions $P(s' \mid s, a)$. Of Assumption~\ref{ass:ideal} we retain only the finite state space of (C1) and the first-hit termination convention of (C4): transitions are stochastic rather than deterministic, so (C2) is dropped as in Section~\ref{subsec:thm-stochastic}, and the counts, value estimates and policies all change during learning, so (C3) is dropped as well. The policy optimality that (C4) assumes in the oracle analysis is likewise not available here, since the goal-conditioned policies are what the algorithm is learning. Under first-hit termination, the SVF $V^\pi(s, g) = \EV_\pi[\gamma^{\tau_g} \ind{\scriptstyle\{\tau_g < \infty\}}] \in [0, 1]$.

As in AdaGoal, two horizons are distinguished. The \textit{reachability radius} $L \geq 1$ defines the target set
\begin{equation}
\mathcal{G}_L \,\triangleq\, \{g \in \statespace : \exists \pi \text{ s.t. } \EV_\pi[\tau_g] \leq L\},
\qquad |\mathcal{G}_L| \leq S,
\end{equation}
and the \textit{episode horizon} $H$ is the number of steps the agent runs per episode. As in AdaGoal, we assume a reset: the agent returns to $s_0$ at the end of each episode. \citet{tarbouriech2022adaptive} show this is not a convenience --- their Lemma~5 establishes an exponential separation between multi-goal exploration with and without it. We write $T$ for the total number of exploration steps; $\delta \in (0, 1)$ and $\varepsilon \in (0, 1)$ are the confidence and accuracy parameters that enter the uncertainty bonus and the horizon below.

\subsection{SUN-UCB Algorithm}

\textbf{Estimators.} For each goal $g$, let $n_t(g)$ be the number of episodes up to step $t$ in which $g$ was the selected goal, and let $\hat V_t(s_0, g)$ be the empirical mean of the returns collected in those episodes, with $\hat V_t(s_0, g) \triangleq 0$ when $n_t(g) = 0$.\footnote{In the deep-RL algorithm of Section~\ref{subsec:pseudo}, $n_g$ counts how often $g$ has been \textit{visited}, since the novelty signal there is a density estimate. Here it counts how often $g$ has been \textit{targeted}, which is what the number of return samples --- and hence the accuracy of the estimate --- depends on. The two are different quantities: a goal can be traversed while another is being pursued, and a targeted goal need not be reached.} Define the uncertainty
\begin{equation}
\label{eq:sun-ucb-bonus}
U_t(g) \,\triangleq\, \min\!\left\{1, \; \sqrt{\frac{\log(SAT/\delta)}{n_t(g)}}\right\},
\qquad U_t(g) \triangleq 1 \text{ when } n_t(g) = 0.
\end{equation}
The cap and the convention at $n_t(g) = 0$ are justified because $V^\pi(s_0, g) \in [0, 1]$ under first-hit termination, so no uncertainty larger than the value range is informative.

\textbf{Truncation.} An episode targeting $g$ under policy $\pi$ contributes the sample $\gamma^{\tau_g}$ if $g$ is hit at some step $\tau_g \leq H$, and $0$ otherwise. The return is therefore truncated at the episode horizon, and each sample estimates $\EV_\pi[\gamma^{\tau_g} \ind{\scriptstyle\{\tau_g \leq H\}}]$ rather than $V^\pi(s_0, g)$, a downward bias of at most $\gamma^{H+1}$. We take
\begin{equation}
\label{eq:horizon}
H \,=\, \max\!\left\{2L, \; \left\lceil \frac{\log(1/\varepsilon)}{\log(1/\gamma)} \right\rceil \right\},
\end{equation}
which bounds the bias by $\varepsilon$. The first term is what makes the horizon useful rather than merely unbiased: for $g \in \mathcal{G}_L$ and a policy attaining $\EV_\pi[\tau_g] \leq L$, Markov's inequality gives $\Pr_\pi[\tau_g > H] \leq L/H \leq 1/2$, so such a goal is hit within an episode with probability at least one half. As in AdaGoal, the horizon depends on both the radius and the accuracy, while the target set $\mathcal{G}_L$ depends on the radius alone.

\textbf{Empirical reachability filter.} Let $E_t \triangleq \{(s, s') : \exists a \text{ s.t.\ } N_t(s, a, s') > 0\}$ be the set of transitions observed up to step $t$, and let
\begin{equation}
\label{eq:reach-filter}
\mathcal{G}_t^{\mathrm{reach}} \,\triangleq\, \{g \in \statespace : g \text{ is reachable from } s_0 \text{ along edges of } E_t\},
\end{equation}
where the empty path is admitted, so that $s_0 \in \mathcal{G}_t^{\mathrm{reach}}$ for every $t$ and the filter is never empty. The filter is computed from observed transitions, not from value estimates. This matters: a filter of the form $\{g : \hat V_t(s_0, g) > 0\}$ would exclude every goal that has never been \emph{hit}, not merely every goal that is \emph{unreachable}, and would therefore block exploration toward the frontier. A goal that has been observed but never targeted lies in $\mathcal{G}_t^{\mathrm{reach}}$ and remains selectable.

The algorithm takes $L$, $\varepsilon$ and $\delta$ as inputs, as AdaGoal does. At the start of each episode, at step $t$, given $n_t(\cdot)$:
\begin{enumerate}[leftmargin=*, itemsep=0pt, topsep=0pt]
\item Update $\hat V_t$, $n_t$, and $E_t$ from the last episode.
\item Compute the filter $\mathcal{G}_t^{\mathrm{reach}}$ of Eq.~\eqref{eq:reach-filter}.
\item Select
\begin{equation}
\label{eq:sun-ucb}
g_t \,=\, \argmax_{g \in \mathcal{G}_t^{\mathrm{reach}}} \,\hat V_t(s_0, g) \,+\, \beta \, U_t(g),
\qquad \beta = 2/\varepsilon,
\end{equation}
with ties broken arbitrarily.
\item Run one episode of length $H$ targeting $g_t$ with a goal-specific instance of UCBVI~\citep{azar2017minimax}, one instance per goal, each updated only on the episodes in which its goal was selected; record whether $g_t$ was hit and at which step.
\end{enumerate}

The filter in step 2 is the algorithmic counterpart of unreachability rejection (Theorem~\ref{thm:unreachable}). The coefficient $\beta$ in step 3 is not cosmetic: Lemma~\ref{lem:select-uncert} fails without it, for reasons discussed in Remark~\ref{rem:beta}. Note also that, as stated, the algorithm has no step that leaves $s_0$ before an edge has been observed; we return to this under \textit{Discovery} below.

\subsection{Analysis}

We prove two components --- soundness of the reachability filter and near-maximal uncertainty of the selected goal --- and then identify what a complete sample-complexity argument would additionally require.
Both components hold \textit{deterministically}, with no appeal to a concentration event: the filter is sound because every edge it uses was traversed, and the selection rule is near-maximal in uncertainty because $\hat V_t$ is bounded.

\textbf{Soundness of the reachability filter.}
\begin{lemma}[Filter soundness and monotonicity]
\label{lem:reject}
Deterministically, for every $t$:
\begin{enumerate}[leftmargin=*, noitemsep, topsep=2pt, label=\textup{(\roman*)}]
\item $\mathcal{G}_t^{\mathrm{reach}} \subseteq \{g \in \statespace : g \text{ is reachable from } s_0\}$;
\item $\mathcal{G}_t^{\mathrm{reach}} \subseteq \mathcal{G}_{t+1}^{\mathrm{reach}}$.
\end{enumerate}
\end{lemma}

\begin{proof}
(i) Every edge $(s, s') \in E_t$ was traversed by the agent, so $P(s' \mid s, a) > 0$ for the action $a$ that produced it. A path from $s_0$ to $g$ using only edges of $E_t$ is therefore a path of positive probability in the true MDP, and $g$ is reachable; $s_0$ itself is reachable by the empty path. (ii) $N_t(s, a, s')$ is non-decreasing in $t$, hence $E_t \subseteq E_{t+1}$ and reachability along $E_t$ implies reachability along $E_{t+1}$.
\end{proof}

SUN-UCB therefore \emph{never selects a truly unreachable goal}, regardless of how large the uncertainty $U_t(g)$ may be for it --- and unlike a value-based filter, this holds deterministically rather than on a high-probability event. This is the formal counterpart of Theorem~\ref{thm:unreachable} in the learning setting. By (ii), no reachable goal is permanently excluded once a path to it has been observed. Note, however, that (ii) only \textit{preserves} what has been discovered; it does not guarantee discovery, which we return to below.

\textbf{The selected goal has near-maximal uncertainty.}
\begin{lemma}[High-uncertainty selection]
\label{lem:select-uncert}
Deterministically, at any step $t$ at which some $g \in \mathcal{G}_L \cap \mathcal{G}_t^{\mathrm{reach}}$ satisfies $U_t(g) > \varepsilon$, the selected goal satisfies
\begin{equation}
U_t(g_t) \,\geq\, \tfrac{1}{2} \max_{g \in \mathcal{G}_L \cap \mathcal{G}_t^{\mathrm{reach}}} U_t(g).
\end{equation}
\end{lemma}

\begin{proof}
Let $g^{\max}_t \triangleq \argmax_{g \in \mathcal{G}_L \cap \mathcal{G}_t^{\mathrm{reach}}} U_t(g)$, so that $U_t(g^{\max}_t) > \varepsilon$ by hypothesis. Both $g_t$ and $g^{\max}_t$ lie in $\mathcal{G}_t^{\mathrm{reach}}$, so by optimality of $g_t$ in Eq.~\eqref{eq:sun-ucb},
\begin{equation}
\hat V_t(s_0, g_t) + \beta \, U_t(g_t) \,\geq\, \hat V_t(s_0, g^{\max}_t) + \beta \, U_t(g^{\max}_t).
\end{equation}
Bounding $\hat V_t(s_0, g_t) \leq 1$ and $\hat V_t(s_0, g^{\max}_t) \geq 0$ and rearranging,
\begin{equation}
U_t(g_t) \,\geq\, U_t(g^{\max}_t) - \frac{1}{\beta} \,=\, U_t(g^{\max}_t) - \frac{\varepsilon}{2} \,\geq\, U_t(g^{\max}_t) - \frac{U_t(g^{\max}_t)}{2} \,=\, \frac{U_t(g^{\max}_t)}{2},
\end{equation}
where the second inequality uses $\varepsilon < U_t(g^{\max}_t)$.
\end{proof}

\begin{remark}[The additive form requires a scale coefficient]
\label{rem:beta}
With $\beta = 1$ the same argument yields only $U_t(g_t) \geq U_t(g^{\max}_t) - 1$, which is vacuous because $U_t \leq 1$ by construction. The additive rule therefore tracks uncertainty only once its bonus is scaled to dominate that range, with $\beta$ tied to the target accuracy $\varepsilon$.
\\[1pt]
That scaling has a second consequence, which is instructive for the design argument of Section~\ref{subsec:theory}. Since $\hat V_t \in [0, 1]$ and $\beta = 2/\varepsilon$, two goals whose uncertainties differ by more than $\varepsilon/2$ are ordered by uncertainty alone, whatever their values; the value term decides only among goals of near-equal uncertainty. The additive rule is therefore lexicographic rather than a balance of the two signals: novelty orders the candidates, and reachability only separates those of near-equal novelty. What rejects unreachable goals in the tabular setting is consequently the filter, not the value term. The multiplicative form of Eq.~\eqref{eq:sun-score} needs no coefficient and is not lexicographic: a change in either factor moves the score by the same relative amount, because the two share a zero and a scale.
\end{remark}

\textbf{What a complete argument would need.}

\textit{Discovery.} The filter of Eq.~\eqref{eq:reach-filter} is sound but not complete. Nothing in the algorithm guarantees that every $g \in \mathcal{G}_L$ eventually enters $\mathcal{G}_t^{\mathrm{reach}}$, and Lemma~\ref{lem:reject}(ii) only preserves goals already discovered. The gap is not a corner case. Before any edge is observed, $\mathcal{G}_0^{\mathrm{reach}} = \{s_0\}$ in every MDP, so the first selection is $s_0$, which triggers first-hit termination at time zero; no edge is observed, $E_1 = E_0$, and the same happens at every subsequent episode. As stated, SUN-UCB therefore never takes an environment step, whatever the MDP, and some initial expansion step is required before the analysis above describes any trajectory at all. A finite warm-up of random actions supplies such a step but no discovery guarantee: consider a two-state deterministic MDP with two actions, one moving from $s_0$ to a distinct goal $g$ and one staying at $s_0$, with the identity of each unknown. A warm-up that happens to try only the staying action observes only the self-loop at $s_0$, the filter remains $\{s_0\}$, and the selection rule never leaves it; the two MDPs obtained by swapping the actions remain indistinguishable. A complete argument therefore requires an explicit frontier-expansion mechanism with its own discovery guarantee, as in \citet{lim2012autonomous} and \citet{tarbouriech2020improved}.
% \\[1pt]

\textit{Policy optimality.} The estimates $\hat V_t$ are averages of returns under the policies actually executed, not values of any returned policy. A sample-complexity statement would additionally require a stopping rule, a specification of which policy is returned for each goal, and a bound on its suboptimality; invoking a finite-horizon regret analysis such as \citet{azar2017minimax} does not supply these, since it addresses a single fixed objective rather than the all-goal scheduling problem SUN-UCB poses.

\subsection{Discussion}
\label{subsec:pac-discussion}

\textbf{Comparison with AdaGoal.} The most instructive difference between SUN-UCB and AdaGoal is in how the two methods treat goals they have not yet reached, as the two filters are duals of one another.
AdaGoal constrains selection to $X_k = \{g : D_k(g) \leq L\}$, where $D_k$ is an \textit{optimistic} distance estimate initialized at $D_1(g) = \ind{\scriptstyle\{g \neq s_0\}}$.
Every goal therefore passes the constraint from the first episode, and in the infinite-horizon limit their Eq.~(3) gives $\mathcal{G}_L \subseteq X_k \subseteq \mathcal{G}_{L + \varepsilon_k}$: the constraint never removes an $L$-reachable goal, but it may admit goals outside $\mathcal{G}_L$, which \citet{tarbouriech2022adaptive} state the strategy ``gradually discards''.
SUN-UCB's filter is the opposite. It admits a goal only once a path to it has been observed, so it never selects an unreachable goal (Lemma~\ref{lem:reject}) but may never discover a reachable one.
This is why AdaGoal has no discovery gap and SUN-UCB does, and it is also why their constraint is dispensable --- \citet{tarbouriech2022adaptive} note that it ``is not required for the final sample complexity result'' --- whereas ours is load-bearing.

% \\[1pt]
\textbf{Role of unreachability rejection.} The filter in step 2 of SUN-UCB plays for the learner the role that Theorem~\ref{thm:unreachable} (proved in the oracle setting in Section~\ref{subsec:thm-unreachable}) plays for the indicator, and Lemma~\ref{lem:reject} is its learning-setting counterpart. An unreachable goal returns no reward, so its value estimate stays at zero while its uncertainty decays only as it is selected; without a filter it would therefore be selected on the order of $\log(SAT/\delta)/\varepsilon^2$ times before its bonus falls to the target accuracy. The filter does not change the worst case, in which every state is reachable, but it removes the cost of the unreachable ones whenever the reachable set is a small fraction of the state space --- which is the regime the environments of Section~\ref{sec:experiments} are built to stress. What this analysis adds to the oracle statement is that the filter must be built from observed \emph{transitions} rather than from value estimates, since the latter cannot distinguish an unreachable goal from a reachable one that has not yet been targeted.

\subsection{Additive vs.\ Multiplicative SUN}
\label{app:sun_ucb_exp}
The analysis above concerns an additive score, whereas our experiments (Section~\ref{sec:experiments}) use the multiplicative form $\smash{\hat{V}}/n_g$. Remark~\ref{rem:beta} makes the difference concrete: the additive rule tracks uncertainty only once $\beta$ is scaled to the target accuracy, since with $\beta = 1$ the novelty term is dominated by the range of $\smash{\hat V}$. In the tabular setting this is a mild requirement, as $\varepsilon$ is given and the range $[0, 1]$ of $V^\pi$ is known exactly. In deep RL neither holds: $\Vtheta$ is approximate, its effective range varies across environments and over training, and there is no target accuracy from which to derive $\beta$. The coefficient must therefore be tuned --- which is precisely the failure mode we observe for DISCOVER in Section~\ref{subsec:sun_better}.
SUN's multiplicative form carries no such coefficient. An unreachable goal has zero value, so its score is zero however novel the goal is; conversely, since the value is at most $1$, the score of a goal is at most $1/n_g$ and vanishes as $n_g$ grows, however reachable the goal is. Rescaling $\Vtheta$ by a positive constant also leaves the selected goal unchanged.
Section~\ref{subsec:ablation_results} (Figure~\ref{fig:ablation_score}) and Appendix~\ref{app:ablation_full} compare against the additive variant with unit coefficient.
Note that for the sake of simplicity, we used the count-based form $\smash{\sqrt{\log N_{\mathrm{tot}} / n_g}}$, with $\smash{N_{\mathrm{tot}} = \sum_g n_g}$, rather than the scaled bonus $\beta\,U_t(g)$ of Eq.~\eqref{eq:sun-ucb}.\footnote{The two bonuses differ in more than constants. Both decay as $1/\sqrt{n_g}$, but the logarithm in ours grows with the total count $N_{\mathrm{tot}}$, whereas the one in $U_t(g)$ is fixed by $S$, $A$, $T$ and $\delta$, and $U_t(g)$ is capped at $1$. Nor does the count-based form carry the $\varepsilon$-dependent scaling $\beta = 2/\varepsilon$ that Lemma~\ref{lem:select-uncert} requires. The additive variant we evaluate is therefore the unscaled, count-based one --- the regime in which Remark~\ref{rem:beta} shows the selection guarantee is lost --- and the tabular analysis above should be read as motivating the \textit{form} of the comparison rather than as describing the implemented bonus.}

%\clearpage

%\input{appendix_additional}
\end{appendix}

% \clearpage
% \input{checklist.tex}

\end{document}